\documentclass{article}

\usepackage{microtype}
\usepackage{graphicx}
\usepackage{subfigure}
\usepackage{booktabs} 

\usepackage[preprint,nonatbib]{neurips_2019}

\RequirePackage{times}

\usepackage{amssymb,amsmath,amscd,amsfonts,amstext,amsthm,bbm, enumerate, dsfont, mathtools}
\usepackage{array}
\usepackage{multicol}
\usepackage{nicefrac}
\usepackage{graphicx}
\usepackage{varwidth}
\usepackage{hyperref}
\usepackage{import}
\usepackage[none]{hyphenat}
\usepackage{footnote}
\usepackage{threeparttable}
\usepackage{tcolorbox}
\RequirePackage{algorithm}
\RequirePackage{algorithmic}
\graphicspath{ {plots_and_logs/} }
\usepackage{float}

\usepackage{enumitem}

\usepackage{pifont}
\definecolor{mydarkgreen}{RGB}{39,130,67}
\newcommand{\green}{\color{mydarkgreen}}
\newcommand{\cmark}{{\green\ding{51}}}%
\definecolor{mydarkred}{RGB}{192,47,25}
\newcommand{\red}{\color{mydarkred}}
\newcommand{\xmark}{{\red\ding{55}}}%

\newtheorem{theorem}{Theorem}
\newtheorem{proposition}{Proposition}
\newtheorem{lemma}{Lemma}
\newtheorem{corollary}{Corollary}
\newtheorem{assumption}{Assumption}
\newtheorem{remark}{Remark}
\newtheorem{definition}{Definition}

\renewcommand{\algorithmicrequire}{\textbf{Input:}}
\renewcommand{\algorithmicensure}{\textbf{Output:}}

\DeclareMathOperator*{\argmin}{argmin}

\newcommand{\prox}{\mathop{\mathrm{prox}}\nolimits}

\newcommand{\sumin}{\sum_{i=1}^n}

\newcommand{\avein}{\frac{1}{n}\sum_{i=1}^n}
\newcommand{\avejn}{\frac{1}{n}\sum_{j=1}^n}
\newcommand{\proxR}{\prox_{\gamma R}}

\newcommand{\EE}{\mathbb{E}}

\newcommand{\RR}{\mathbb{R}}

\newcommand{\cL}{{\cal L}}

\newcommand{\Var}{\mathrm{Var}}
\newcommand{\eqdef}{\triangleq}

\newcommand{\Lgen}{{\Phi}}
\newcommand{\smo}{{\psi}}
\newcommand{\seganu}{{\omega}}
 \newcommand{\quadb}{{o}}
 \newcommand{\cF}{{\cal F}}

\usepackage[colorinlistoftodos,bordercolor=orange,backgroundcolor=orange!20,linecolor=orange,textsize=scriptsize]{todonotes}

\def\<#1,#2>{\left\langle #1,#2\right\rangle}

\usepackage{hyperref}

\author{
 Konstantin Mishchenko \\
  KAUST  \\
  Thuwal, Saudi Arabia \\
 \texttt{konstantin.mishchenko@kaust.edu.sa} \\
  \And
 Filip Hanzely \\
  KAUST  \\
  Thuwal, Saudi Arabia \\
 \texttt{filip.hanzely@kaust.edu.sa} \\
  \And
  Peter Richt\'{a}rik\thanks{Also affiliated with the Moscow Institute of Physics and Technology, Dolgoprudny, Russia.} \\
  KAUST  \\
  Thuwal, Saudi Arabia \\
  \texttt{peter.richtarik@kaust.edu.sa } \\
}

\begin{document}

\title{99\% of Parallel Optimization is Inevitably a Waste of Time}

\title{99\% of Distributed Optimization is  a Waste of Time: \\ The Issue and How to Fix it}

\maketitle







\begin{abstract}
Many popular distributed optimization methods for training machine learning models fit the following template: a local gradient estimate is computed independently by each worker, then communicated to a master, which subsequently performs averaging. The average is broadcast back to the workers, which use it to perform a gradient-type step to update the local version of the model. It is also well known that many such  methods, including SGD, SAGA, and accelerated SGD for over-parameterized models, do not scale well with the number of parallel workers.  In this paper we observe that the above template is fundamentally inefficient in that too much data is unnecessarily communicated by the workers, which slows down the overall system.  We propose a fix based on a new update-sparsification method we develop in this work, which we suggest be used on top of existing methods. Namely, we develop a new variant of parallel block coordinate descent based on independent sparsification of the local gradient estimates before communication. We demonstrate that with only $m/n$ blocks sent by each of $n$ workers, where $m$ is the total number of parameter blocks, the theoretical iteration complexity of the underlying distributed methods is essentially unaffected. As an illustration, this means that when $n=100$ parallel workers are used, the communication of  $99\%$  blocks is redundant, and hence a waste of time. Our theoretical claims are supported through extensive numerical experiments which demonstrate an almost  perfect match with our theory on a number of synthetic and real datasets.
\end{abstract}

\section{Introduction}

In this work we are concerned with parallel/distributed  algorithms for solving finite sum minimization problems of the form 
\begin{align}
\textstyle   \min_{x\in \RR^d} f(x) \eqdef  \frac{1}{n}\sum \limits_{i=1}^n f_i(x),\label{eq:problem}
\end{align}
where each $f_i$ is convex and smooth. In particular, we are interested in methods which employ $n$ parallel units/workers/nodes/processors, each of which has access to a single function $f_i$ and its gradients (or unbiased estimators thereof). Let $x^*$ be an optimal solution of~\eqref{eq:problem}. In practical parallel or distributed scenarios, $f_i$ is often of the form
\begin{equation}\label{eq:stoch-f_i}  \textstyle  f_i(x) = \EE_{\xi} \phi_i(x; \xi),\end{equation}
where the expectation is with respect to a distribution of training examples stored locally at machine $i$. More typically, however, each machine contains a very large but finite number of examples (for simplicity, say there are $l$ examples on each machine), and $f_i$ is of the form \begin{equation}\label{eq:problem_saga_dist}
 \textstyle f_i(x) = \frac{1}{l}\sum \limits_{j=1}^l f_{ij}(x).
\end{equation}

 In the rest of this section we provide some basic motivation and intuitions in support of our approach. To this purpose, assume, for simplicity of exposition, that $f_i$ is of the finite-sum form \eqref{eq:problem_saga_dist}. In typical modern machine learning workloads, the number of machines $n$  is much smaller than the number of data points on each machine $l$. In a large scale regime (i.e., when the model size $d$,  the number of data points $nl$, or both are large),  problem \eqref{eq:problem} needs to be solved by a combination of efficient methods and modern hardware. In recent years there has been a lot of progress in designing new algorithms for solving this problem using techniques such as stochastic approximation~\cite{sgd}, variance reduction~\cite{sag, svrg, saga}, coordinate descent~\cite{nesterov_cd,richtarik2014iteration, wright2015coordinate} and acceleration~\cite{nesterov}, resulting in excellent theoretical and practical performance.

The computational power of the hardware is increasing as well. In recent years, a very significant amount of such increase is due to parallelism. Since many methods, such as minibatch Stochastic Gradient Descent (SGD), are embarrassingly parallel, it is very simple to use them in big data applications. However, it has been observed in practice that adding more resources beyond a certain limit does not improve iteration complexity significantly. Moreover, having more parallel units makes their synchronization harder due to so-called communication bottleneck.  Minibatch versions of most variance reduced methods\footnote{We shall mention that there are already a few variance reduced methods that scale, up to some level, linearly in a parallel setup: Quartz for sparse data~\cite{QUARTZ}, Katyusha~\cite{allen2017katyusha},  or SAGA/SVRG/SARAH with importance sampling for non-convex problems~\cite{horvath2018nonconvex}. } such as SAGA~\cite{saga} or SVRG~\cite{svrg} scale even worse in parallel setting -- they do not guarantee, in the worst case,  any speedup from using more than one function at a time. Unfortunately, numerical experiments show that this is not a proof flaw, but rather a real property of these methods~\cite{gower2018stochastic}. A similar observation was made for SVRG in~\cite{zhao2014accelerated}, where it was shown that only a small number of partial derivatives are needed at each iteration.

Since there are too many possible situations, we choose to focus on black-box optimization, although we admit that much can be achieved by assuming the sparsity structure. In fact, for any method there exists a toy situation where the method would scale perfectly -- one simply needs to assume that each function $f_i$ depends on its own subset of coordinates and minimize each $f_i$ independently. This can be generalized assuming sparsity patterns~\cite{leblond2017asaga, leblond2018improved} to get almost linear scaling if any coordinate appears in a small number of functions. Our interest, however, is in explaining situations as in~\cite{gower2018stochastic} where the models almost do not scale.

In this paper, we demonstrate that a simple trick of {\em independent} block sampling can remedy the problem of scaling,  to a substantial but limited extent. To illustrate one of the key insights of our paper on a simple example, in what follows consider a thought experiment in which GD is a baseline method we would want to improve on.

\subsection{From gradient descent to block coordinate descent and back}

A simple benchmark in the distributed setting is a parallel implementation of gradient descent (GD). GD arises as a special case of the more general class of block coordinate descent methods (BCD)~\cite{PCDM}.  The conventional way to run  BCD for problem \eqref{eq:problem} is to update a single or several blocks\footnote{Assume the entries of $x$ are partitioned into several non-overlapping blocks.} of $x$, chosen at random, on all  $n$ machines~\cite{PCDM, fercoq2015accelerated}, followed by an update aggregation step. Such updates on each worker typically involve a gradient step on a subspace corresponding to the selected blocks.  Importantly, and this is a key structural property of BCD methods, {\em the same set of blocks is updated on each machine}. If  communication is expensive, it often makes sense to do more work on each machine,  which in the context of BCD means updating more blocks.
A particular special case is to update {\em all} blocks, which leads to parallel implementation of GD for problem \eqref{eq:problem}, as mentioned above. Moreover, it is known that the theoretical iteration complexity of BCD improves as the number of blocks updated increases~\cite{PCDM, ALPHA, ESO}. For these and similar reasons, GD (or one of its variants, such as GD with momentum),  is often the preferable method to BCD. Having said that, we did not choose to describe BCD only to discard it at this point; we shall soon return to it, albeit with a twist.

\subsection{From gradient descent to independent block coordinate descent}

Because of what we just said, iteration complexity of GD will not improve by any variant running BCD; it can only get worse.  Despite this, {\em we propose to run  BCD, but a new variant which allows each worker to sample an independent subset of blocks} instead. This variant of BCD for \eqref{eq:problem} was not considered before. As we shall show,  our {\em independent sampling} approach leads to a better-behaved aggregated gradient estimator when compared to that of BCD, which in turn leads to better overall iteration complexity. We call our method {\em independent block coordinate descent (IBCD)}. We provide a unified analysis of our method, allowing for a random subset of $\tau m$ out of a total of $m$ blocks to be sampled on each machine, independently from other machines. GD  arises as a special case of this method by setting $\tau=1$.  However, as we show (see Corollary~\ref{cor:0893y83}), {\em the same iteration complexity guarantee can be obtained by choosing $\tau$ as low as $\tau=\nicefrac{1}{n}$.} The immediate consequence of this result is that {\em it is suboptimal to run GD in terms of communication complexity.} Indeed, GD needs to communicate all $m$ blocks per machine, while IBCD achieves the same rate with $\nicefrac{m}{n}$ blocks per machine only. Coming back to the abstract, consider an example with $n=100$ machines. In this case, when compared to GD, IBCD only communicates $1\%$ of the data. Because the iteration complexities of the two methods are the same, and if communication cost is dominant, this means that the problem can be solved in just $1\%$ of the time. In contrast, and when compared to the potential of IBCD, parallel implementation of GD inevitably wastes 99\% of the time.

The intuition behind why our approach works lies in the law of large numbers. By averaging independent noise we reduce the total variance of the resulting estimator by the factor of $n$. If, however, the noise is already tiny, as, in non-accelerated variance reduced methods, there is no improvement. On the other hand, (uniform)  block coordinate descent (CD) has variance proportional to $\nicefrac{1}{\tau}$~\cite{wangni2018gradient}, where $\tau < 1$ is the ratio of used blocks. Therefore, after the averaging step the variance is $\nicefrac{1}{\tau n}$, which illustrates why setting any $\tau > \nicefrac{1}{n}$ should not yield a significant speedup when compared to the choice $\tau = \nicefrac1n$. It also indicates that it should be possible to throw away a $(1-\nicefrac1n)$ fraction of blocks while keeping the same convergence rate.

\subsection{Beyond gradient descent and further contributions}

The goal of the above discussion was to introduce one of the ideas of this paper in a gentle way. However, our independent sampling idea has immense consequences beyond the realm of GD, as we show in the rest of the paper. Let us summarize the contributions here:

\begin{itemize}
\item We show that the independent sampling idea can be coupled with variance reduction/SAGA (see Sec~\ref{sec:saga}),  SGD for problem  \eqref{eq:problem}+\eqref{eq:stoch-f_i} (see Sec~\ref{sec:sgd}),  acceleration (under mild assumption on stochastic gradients; see Sec~\ref{sec:ABCDE}) and regularization/SEGA (see Sec~\ref{sec:sega}). We call the new methods ISAGA, ISGD, IASGD and ISEGA, respectively. We also develop ISGD variant for asynchronous distributed optimization -- IASGD (Sec~\ref{sec:asynch}). 

\item We present two versions of the  SAGA algorithm coupled with IBCD. The first one is for a distributed setting, where each machine owns a subset of data and runs a SAGA iteration with block sampling locally, followed by aggregation. The second version is in a shared data setting, where each machine has access to all functions. This allows for linear convergence even if $\nabla f_i(x^*)\neq 0$. 
 
 \item We show that when combined with IBCD, the  SEGA trick~\cite{SEGA} leads to a method that enjoys a   linear rate for problems where $\nabla f_i(x^*)\neq 0$ and allows for more general objectives which may include a non-separable non-smooth regularizer. 
\end{itemize}

A comprehensive summary of all proposed algorithms is given in Table~\ref{tbl:algs}.

 \begin{table}[!t]
\begin{center}
\small
\begin{tabular}{|c|c|c|c|c|c|c|c|}
\hline
{\bf \#}& {\bf Name}  & {\bf Origin}& { \bf $\nabla f_i(x^*)\neq0 $}&  { \bf  Lin. rate }  & {  \bf \begin{tabular}{c} In-machine \\ randomization \end{tabular}} &{\bf Note} \\
 \hline
  \hline
\ref{alg:cd}   & IBCD &  I+ CD~\cite{nesterov_cd} & \xmark & \cmark &\xmark & Simplest \\
  \hline
\ref{alg:sega}   & ISEGA   & I + SEGA~\cite{SEGA}& \cmark  & \cmark&\xmark & Allows prox \\
  \hline
\ref{alg:ibd}   & IBGD &  I+ GD  & \xmark & \cmark &\xmark & Bernoulli, no CD \\
  \hline
\ref{alg:saga}   & ISAGA  & I+ SAGA~\cite{saga}& \cmark &  \cmark& \cmark &Shared memory \\
  \hline
\ref{alg:saga_dist}   & ISAGA    &I+ SAGA~\cite{saga}& \xmark & \cmark&\cmark & \\
  \hline
\ref{alg:sgd}   & ISGD  & I + SGD~\cite{sgd}&\cmark  &  \xmark &\cmark & + Non-convex\\
  \hline
\ref{alg:acc}   & IASGD  &I + ASGD~\cite{vaswani2018fast}& \cmark  &  \xmark&\cmark & \\
 \hline
\ref{alg:asynch_sgd}   & IASGD   & I + ASGD~\cite{NIPS2011_4390}& \cmark  & \xmark&\cmark & Asynchronous \\
\hline
\end{tabular}
\end{center}
\caption{Summary of all algorithms proposed in the paper.}
\label{tbl:algs}
\end{table}

\section{Practical Implications and Limitations \label{sec:practical}}
In this section, we outline some further limitations and practical implications of our framework.
\subsection{Main limitation}

The main limitation of this work is that  independent sampling does not generally result in a sparse aggregated update. Indeed, since each machine might sample a different subset of blocks, all these updates add up to a dense one, and this problem gets worse as $n$ increases, other things equal. For instance, if every parallel unit updates a single unique block\footnote{Assume $x$ is partitioned into several ``blocks'' of variables.}, the total number of updated blocks is equal $n$. In contrast, standard BCD, one that samples the {\em same} block on each worker, would update a single block only. For simple linear problems, such as logistic regression, sparse updates allow for a fast implementation of BCD via memorization of the residuals. However, this limitation is not crucial in common settings where broadcast is much faster than reduce.

\subsection{Practical implications}

The main body of this work focuses on theoretical analysis and on verifying our claims via experiments. However, there are several straightforward and important applications of our technique.

\textbf{Distributed synchronous learning.} A common way to run a distributed optimization method is to perform a local update, communicate the result to a parameter server using a 'reduce' operation, and inform all workers using 'broadcast'. Typically, if the number of workers is significantly large, the bottleneck of such a system is communication. In particular, the 'reduce' operation takes much more time than 'broadcast' as it requires to add up different vectors computed locally, while 'broadcast' informs the workers about \textit{the same} data (see~\cite{mishchenko2019distributed} for a numerical validation that 'broadcast' is 10-20 times faster across a wide range of dimensions).
 Nevertheless, if every worker can instead send to the parameter server only $\tau = \nicefrac{1}{n}$ fraction of the $d$-dimensional update, essentially the server node will receive just one full $d$-dimensional vector, and thus our approach can compete against methods like QSGD~\cite{alistarh2017qsgd}, signSGD~\cite{bernstein2018signsgd}, TernGrad~\cite{wen2017terngrad}, DGC~\cite{lin2017deep} or ATOMO~\cite{wang2018atomo}. In fact, our approach may completely remove the communication bottleneck. 

\textbf{Distributed asynchronous learning.} The main difference with the synchronous case is that only one-to-one communications will be used instead of highly efficient 'reduce' and 'broadcast'. Clearly, the communication to the server will be much faster with $\tau=\nicefrac{1}{n}$, so the main question is how to make the communication back fast as well. Hopefully, the parameter server can copy the current vector and send it using non-blocking communication, such as \textit{isend()} in MPI4PY~\cite{dalcin2011parallel}. Then, the communication back will not prevent the server from receiving the new updates. We combine the IBCD approach with asynchronous updates, which leads to a new method:  IASGD (Algorithm~\ref{alg:asynch_sgd}).

\textbf{Distributed sparse learning.} Large datasets, such as binary classification data from LibSVM, often have sparse gradients. In this case, the 'reduce' operation is not efficient and one needs to communicate data by sending positions of nonzeros and their values. Moreover, as we prove later, one can use independent sampling with $\ell_1$-penalty, which makes the problem solution sparse. In that case, only communication from a worker to the parameter server is slow, so both synchronous and asynchronous methods gain in performance.

\textbf{Methods with local subproblems.} One can also try to extend our analysis to methods with exact block-coordinate minimization or primal-dual and proximal methods such as Point-SAGA~\cite{defazio2016simple}, PDHG~\cite{chambolle2011first}, DANE~\cite{shamir2014communication}, etc. There, by restricting ourselves to a subset of coordinates, we may obtain a subproblem that is easier to solve by orders of magnitude.
    
\textbf{Block-separable problems within machines.} 
Given that the local problem on each machine is block coordinate-wise separable, partial derivative blocks can be evaluated $\nicefrac{1}{\tau}$ times cheaper than the gradients. Thus, independent sampling improves scalability at no cost. Such problems can be obtained considering the dual problem, as is done in~\cite{COCOA+}, for example.

For a comprehensive list of frequently used notation, see Table~\ref{tbl:notation} in the supplementary material.



\section{Independent Block Coordinate Descent \label{sec:basic}}

\subsection{Technical assumptions}
We present the most common technical assumptions required in order to derive convergence rates. 
\begin{definition}
Function $F$ is $L$ smooth if for all $x,y\in \RR^d$ we have: 
\begin{equation}\label{eq:smooth}
 F(x)\leq F(y)+ \< \nabla F(y), x-y> + \tfrac{L}{2}\|x-y \|_2^2.
\end{equation} 
Similarly, $F$ is $\mu$ strongly convex if for all $x,y\in \RR^d$: 
\begin{equation}\label{eq:strong_convex}
 F(x)\geq F(y)+ \< \nabla F(y), x-y> + \tfrac{\mu}{2}\|x-y \|_2^2.
\end{equation} 
\end{definition}
In most results we present, functions $f_i$ are required to be smooth and convex, and $f$ strongly convex. 

\begin{assumption}\label{as:smooth_sc}
For every $i$, function $f_i$ is convex, $L$ smooth and function $f$ is $\mu$ strongly convex. 
\end{assumption}

As mentioned, since independent sampling does not preserve the variance reduction property, in some of our results we shall consider $\nabla f_i(x^*)=0$ for all $i$. 
\begin{assumption}\label{as:zero_grads}
For all $1\leq i \leq n$ we have $\nabla f_i(x^*)=0$.
\end{assumption}

In Sec~\ref{sec:saga} we show that Assumption~\ref{as:zero_grads} can be dropped once the memory is shared among the machines.
Further, in Sec~\ref{sec:sega} we show that Assumption~\ref{as:zero_grads} can be dropped even in the fully distributed setup using the SEGA trick. Lastly, Assumption~\ref{as:zero_grads} is naturally satisfied in many applications. For example, in least squares setting $\min \|Ax-b\|_2^2$, it is equivalent to existence of $x^*$ such that $Ax^*=b$. On the other hand, current state-of-the-art deep learning models are often overparameterized so that they allow zero training loss, which is again equivalent to $\nabla f_i(x^*)=0$ for all $i$  (however, such problems are typically non-convex).

\subsection{Block structure of $\RR^d$~\label{sec:notation}}

 Let $\RR^d$  be partitioned into $m$ blocks $u_1, \dotsc, u_m$ of arbitrary sizes, so that the parameter space is $\RR^{|u_1|}\times\dotsb \RR^{|u_m|}$. For any vector $x\in \RR^d$ and a set of blocks $U$ we denote by $x_U$ the vector that has the same coordinate as $x$ in the set of blocks $U$ and zeros elsewhere.

\subsection{IBCD}

In order to provide a quick taste of our results, we first present the IBCD method described in the introduction and formalized as Algorithm~\ref{alg:cd}. 

\begin{algorithm}[h]
  \caption{Independent Block Coordinate Descent (IBCD)}
  \label{alg:cd}
\begin{algorithmic}[1]
\STATE{\bfseries Input: } {$x^0\in\RR^d$, partition of $\RR^d$ into $m$ blocks $u_1,\dotsc, u_m$, ratio of blocks to be sampled $\tau$, stepsize $\gamma$, \# of parallel units $n$}
  \FOR{$t=0,1,\dotsc$}
    \FOR{$i=1,\dotsc,n$ in parallel}
        \STATE Sample independently and uniformly a subset of $\tau m$ blocks $U_i^t \subseteq \{u_1, \dotsc, u_m\}$
        \STATE $x_i^{t+1} = x^t - \gamma  (\nabla f_i(x^t))_{U_i^t}$
    \ENDFOR
    \STATE $x^{t+1} = \frac{1}{n}\sumin x_i^{t+1}$
  \ENDFOR
\end{algorithmic}
\end{algorithm}

A key parameter of the method is $\nicefrac{1}{m} \leq \tau \leq 1$ (chosen so that $\tau m$ is an integer), representing a fraction of blocks to be sampled by each worker. At iteration $t$, each machine independently samples a subset of $\tau m$ blocks $U_i^t \subseteq \{u_1,\dots,u_m\}$, uniformly at random. The $i$th worker then performs a subspace gradient step of the form $x_i^{t+1} = x^t - \gamma (\nabla f_i(x^t))_{U_i^t}$, where $\gamma>0$ is a stepsize.  Note that only coordinates of $x^t$ belonging to $U_i^t$ get updated. This is then followed by aggregating all $n$ gradient updates: $x^{t+1} = \tfrac{1}{n}\sum_i x_i^{t+1}$.

\subsection{Convergence of IBCD}

Theorem~\ref{th:cd} provides a convergence rate for Algorithm~\ref{alg:cd}. Admittedly, the assumptions of Theorem~\ref{th:cd} are somewhat restrictive; in particular, we  require $\nabla f_i(x^*)=0$ for all $i$. However, this is necessary. Indeed, in general one can not expect to have $\sum_{i=1}^n(\nabla f_i(x^*))_{U_i}=0$ (which would be required for the method to converge to $x^*$) for independently sampled sets of blocks $U_i$ unless $\nabla f_i(x^*)=0$ for all $i$. As mentioned, the issue is resolved in Sec~\ref{sec:sega} using the SEGA trick~\cite{SEGA}. 
\begin{theorem}\label{th:cd} 
Suppose that Assumptions~\ref{as:smooth_sc},~\ref{as:zero_grads} hold. For Algorithm~\ref{alg:cd} with $\gamma = \frac{n}{\tau n + 2(1 - \tau)}\frac{1}{2L}$ we have
\[
\EE\left[ \|x^{t} -x^*\|_2^2 \right] \leq \left(1-\tfrac{\mu}{2L} \tfrac{\tau n}{\tau n + 2(1-\tau)}\right)^t\|x^{0} -x^*\|_2^2.
\]
\end{theorem}
As a consequence of Theorem~\ref{th:cd}, we can choose $\tau$ as small as $\nicefrac{1}{n}$ and get, up to a constant factor, the same convergence rate as gradient descent, as described next.
\begin{corollary}\label{cor:0893y83}
If $\tau=\nicefrac1n$, the iteration complexity\footnote{Number of iterations to reach $\epsilon$ accurate solution.} of Algorithm~\ref{alg:cd} is ${\cal O} (\nicefrac{L}{\mu} \log\nicefrac{1}{\epsilon})$. 
\end{corollary}
\subsection{Optimal block sizes}
If we naively use coordinates as blocks, i.e.\ all blocks have size equal 1, the update will be very sparse and the efficient way to send it is by providing positions of nonzeros and the corresponding values. If, however, we partition $\RR^d$ into blocks of size approximately equal $d/n$, then on average only one block will be updated by each worker. This means that it will be just enough for each worker to communicate the block number and its entries, which is twice less data sent than when using coordinates as blocks.

\section{Variance Reduction \label{sec:saga}}
As the first extension of IBCD, we inject independent coordinate sampling into SAGA\footnote{Independent coordinate sampling is not limited to SAGA and can be similarly applied to other variance reduction techniques.}~\cite{saga}, resulting in a new method we call ISAGA. We consider two different settings for ISAGA. The first one is standard distributed setup~\eqref{eq:problem}, where each $f_i$ is of the fine-sum form \eqref{eq:problem_saga_dist}. The idea is to run SAGA with independent coordinate sampling locally on each worker, followed by aggregating the updates. However, as for IBCD, we require $\nabla f_i(x^*) = 0$ for all $i$. The second setting is a {\em shared data/memory} setup; i.e., we assume that all workers have access to all functions from the finite sum. This allows us to drop Assumption~\ref{as:zero_grads}. Due to space limitations, we present distributed ISAGA in Sec~\ref{sec:saga_dist} of the supplementary.

\subsection{Shared data ISAGA}
We now present a different setup for ISAGA in which the requirement $\nabla f_i(x^*)=0$ is not needed. Instead of \eqref{eq:problem}, we rather solve the problem
\begin{align} \label{eq:problem_saga_sm}
\textstyle    \min_{x\in \RR^d} f(x) \eqdef \frac{1}{N}\sum_{j=1}^N \smo_{j}(x)
\end{align}
with $n$ workers all of which have  access to all  data describing $f$. Therefore, all workers can evaluate $\nabla \smo_{j}(x)$ for any $1\leq j\leq N$. Similarly to plain SAGA, we remember the freshest gradient information in vectors $\alpha_{j}$, and update them as 
\begin{equation}
\alpha_{j_i^t}^{t+1} = \alpha_{j_i^t}^t+ (\nabla \smo_{j_i^t}(x^t)- \alpha_{j_i^t}^t )_{U_i^t}, \quad \alpha_{j'}^{t+1} = \alpha_{j'}^t, \label{eq:saga_alpha_sm}
\end{equation}
where  $j_i^t$ is the  index  sampled at iteration $t$ by machine $i$, and $j'$ refers to all indices that were not sampled at iteration $t$ by any machine. The iterate updates within each machine are taken only on a sampled set of coordinates, i.e.,  
$
x_i^{t+1} = x^t - \gamma  (\nabla \smo_{j_i^t}(x^t) - \alpha_{j_i^t}^t + \overline \alpha^t)_{U_i^t}.
$
where $\overline \alpha^t$ stands for the average of all $\alpha$, and thus it is a delayed estimate of $\nabla f(x^t)$. Lastly, we set the next iterate as the average of proposed iterates by each machine  $x^{t+1} = \frac1n \sumin x_i^{t+1} $. The formal statement of the algorithm is given in the supplementary as Algorithm~\ref{alg:saga}.

\begin{theorem}\label{th:saga_shared}
 Suppose that function $f$ is $\mu$ strongly convex and each $\smo_i$ is $L$ smooth and convex. If $\gamma\le \frac{1}{L\left(\frac{3}{n} + \tau\right)}$, then for iterates of Algorithm~\ref{alg:saga} we have
    \begin{align*}
        \EE \|x^t - x^*\|_2^2 \le (1 - \vartheta)^t\left(\|x^0 - x^*\|_2^2 + c \gamma^2 \Psi^0\right),
    \end{align*}
    where $\Psi^0\eqdef \sum_j \|\alpha_{j}^0 - \nabla \smo_{j}(x^*)\|_2^2$, $\vartheta\eqdef \tau\min\left\{\gamma\mu, \frac{ n}{N} - \frac{2}{nNc} \right\}\ge 0$ and $c\eqdef \frac{1}{n}( \frac{1}{\gamma L} - \frac{1}{n} - \tau) > 0$.
\end{theorem}

As in Sec~\ref{sec:saga_dist}, the choice $\tau = \nicefrac{1}{n}$ yields a convergence rate which is, up to a constant factor, the same as the convergence rate of SAGA. Therefore, Algorithm~\ref{alg:saga} enjoys the desired parallel linear scaling, without the additional requirement of Assumption~\ref{as:zero_grads}. Corollary~\ref{cor:saga_sm} formalizes the claim.

\begin{corollary} \label{cor:saga_sm}
  Consider the setting from Theorem~\ref{th:saga_shared}.  Set $\tau = \nicefrac{1}{n}$ and $\gamma = \nicefrac{n}{5L}$. Then $c=\nicefrac{3}{n^2}$, $\rho = \min \left\{ \nicefrac{\mu}{5L}, \nicefrac{1}{3N}\right\}$ and the complexity of Algorithm~\ref{alg:saga} is $O\left(\max\{\nicefrac{L}{\mu}, N\}\log\nicefrac{1}{\varepsilon} \right)$.
\end{corollary}

\section{Beyond Assumption~\ref{as:zero_grads} and Regularization \label{sec:sega}}
For this section only, let us consider a regularized objective of the form
\begin{align} \label{eq:problem_sega}
\textstyle    \min_{x\in \RR^d} f(x) \eqdef \frac{1}{n}\sum_{i=1}^n f_i(x) + R(x), 
\end{align}
where $R$ is a closed convex regularizer such that its proximal operator is computable:
$   \prox_{\gamma R}(x) \eqdef \argmin_{y} \left\{R(y) + \frac{1}{2\gamma}\|y - x\|^2_2  \right\}.
$ In this section we propose ISEGA:  an independent sampling variant of SEGA~\cite{SEGA}. We do this in order to both i) avoid Assumption~\ref{as:zero_grads} (while keeping linear convergence) and ii) allow for $R$. Original SEGA learns gradients $\nabla f(x^t)$ from sketched gradient information via the so called sketch-and-project process~\cite{gower2015randomized}, constructing a vector sequence $h^t$. In ISEGA on each machine $i$ we iteratively construct a sequence of vectors $h^t_i$ which play the role of estimates of $\nabla f_i(x^t)$. This is done via the following rule: 
\begin{equation} \label{eq:sega_h}
h_i^{t+1} = h_i^t + (\nabla f_i(x^t)- h_i^t)_{U_i^t}.
\end{equation}
The key idea is again that these vectors are created from random blocks independently sampled on each machine. Next, using $h^t$, SEGA builds an unbiased gradient estimator $g_i^t$ of $\nabla f_i(x^t)$ as follows:
\begin{equation} \label{eq:sega_g}
\textstyle g_i^t = h_i^t + \frac{1}{\tau} (\nabla f_i(x^t) - h_i^t)_{U_i^t}. 
\end{equation}
Then, we average the vectors $g_i^t$ and take a proximal step.

Unlike coordinate descent, SEGA (or ISEGA) is not limited to separable proximal operators since, as follows from our analysis 
, $h_i^t\to \nabla f_i(x^*)$. Therefore, ISEGA can be seen as a variance reduced version of IBCD for problems with non-separable regularizers. The price to be paid for dropping Assumption~\ref{as:zero_grads} and having more general objective~\eqref{eq:problem_sega} is that updates from each worker are dense, in contrast to those in Algorithm~\ref{alg:cd}.

In order to be consistent with the rest of the paper, we only develop a simple variant of  ISEGA (Algorithm~\ref{alg:sega}) in which we consider block coordinate sketches with uniform probabilities and non-weighted Euclidean metric (i.e.\ $B=I$ in notation of~\cite{SEGA}). It is possible to develop the theory in full generality as in~\cite{SEGA}. However, we do not do this for the sake of simplicity. 

\begin{algorithm}[h]
  \caption{ISEGA}\label{alg:sega}
  \label{alg:sega}
\begin{algorithmic}[1]
\STATE{\bfseries Input: }{$x^0\in \RR^d$, initial gradient estimates $h_1^0, \dotsc, h_n^0\in \RR^d$, partition of $\RR^d$ into $m$ blocks $u_1,\dotsc, i_m$, ratio of blocks to be sampled $\tau$, stepsize $\gamma$,  \# parallel units $n$}
  \FOR{$t=0,1,\dotsc$}
    \FOR{$i=1,\dotsc,n$ in parallel}
        \STATE Sample independently and uniformly a subset of $\tau m$ blocks $U_i^t$
        \STATE $g_i^t = h_i^t + \frac{1}{\tau} (\nabla f_i(x^t) - h_i^t)_{U_i^t}$ 
        \STATE  $h_i^{t+1} = h_i^t + \tau (g_i^t - h^t)$  
    \ENDFOR
    \STATE $x^{t+1} = \proxR\left( x^t - \gamma \frac{1}{n}\sumin g_i^{t} \right)$  
  \ENDFOR
\end{algorithmic}
\end{algorithm}

We next present the convergence rate of ISEGA (Algorithm~\ref{alg:sega}). 

\begin{theorem}\label{thm:sega}
Suppose Assumption~\ref{as:smooth_sc} holds. Algorithm~\ref{alg:sega} with 
$ \gamma = \min \{ \frac{1}{4L\left( 1+\frac{1}{n\tau }\right)}, \frac{1}{\frac{\mu}{\tau}+ \frac{4L}{n\tau}} \}
$
satisfies
$
 \EE[\|x^t - x^*\|^2_2] \leq (1-\gamma\mu)^t\Lgen^0,
$
where $\Lgen^0 = \|x^0 - x^*\|^2_2 + \frac{ \gamma}{2L\tau n} \sum_{i=1}^n\|h^0 - \nabla f(x^*)\|^2_2$.
\end{theorem}
Note that if the condition number of the problem is not too small so that $n= {\cal O}\left(L/\mu\right)$ (which is usually the case in practice), ISEGA scales linearly in the parallel setting. In particular, when doubling the number of workers, each worker can afford to evaluate only half of  the block partial derivatives while keeping the same convergence speed. Moreover, setting $\tau = \nicefrac1n$, the rate corresponds, up to a constant factor, to the rate of gradient descent. Corollary~\ref{cor:sega} states the result.

\begin{corollary}\label{cor:sega}
Consider the setting from Theorem~\ref{thm:sega}. Suppose that $\nicefrac{L}{\mu}\geq n$ and choose $\tau = \nicefrac1n$. Then, complexity of Algorithm~\ref{alg:sega} is ${\cal O} (\nicefrac{L}{\mu}\log\nicefrac1\epsilon)$.
\end{corollary}

\begin{remark}
Parallel implementation Algorithm~\ref{alg:sega} would be to always send $(\nabla f_i(x^k))_{U_i^t}$ to the server; which keeps updating vector $h^t$ and takes the prox step. 
\end{remark}

\section{Experiments}
In this section, we numerically verify our theoretical claims. Recall that there are various settings where it is possible to make practical experiments (see Sec~\ref{sec:practical}), however, we do not restrain ourselves to any of them in order to deliver as clear a message as possible.

Due to space limitations, we only present a small fraction of the experiments here. A full and exhaustive comparison, together with the complete experiment setup description, is presented in Sec~\ref{sec:app:experiments} of the supplementary material.

In the first experiment presented here, we compare SAGA against ISAGA in a shared data setup (Algorithm~\ref{alg:saga}) for various values of $n$ with $\tau = \nicefrac{1}{n}$ in order to demonstrate linear scaling. We consider logistic regression problem on LibSVM data~\cite{chang2011libsvm}. The results (Figure~\ref{fig:saga_main}) corroborate our theory: indeed, setting $n\tau=1$ does not lead to a decrease in the convergence rate when compared to the original SAGA.

The next experiment (Figure~\ref{fig:sega_main}) supports an analogous claim for ISEGA (Algorithm~\ref{alg:sega}). We run the method for several  $(n,\tau)$ pairs for which $n\tau=1$; on logistic regression problems and LibSVM data. We also plot convergence of gradient descent with the analogous stepsize. As our theory predicts, all the methods exhibit almost the same convergence rate.\footnote{We have chosen the stepsize $\gamma = \nicefrac{1}{2L}$ for GD, as this is the baseline to Algorithm~\ref{alg:sega} with zero variance. One can in fact set $\gamma = \nicefrac{1}{L}$ for GD and get $2\times$ faster convergence. However, this is only a constant factor. } Note that for $n=100$, Algorithm~\ref{alg:sega} throws away $99\%$ of partial derivatives while keeping the same convergence speed as GD, which justifies the title of the paper.

\begin{figure}[t]
\centering
\begin{minipage}{0.25\textwidth}
  \centering
\includegraphics[width =  \textwidth ]{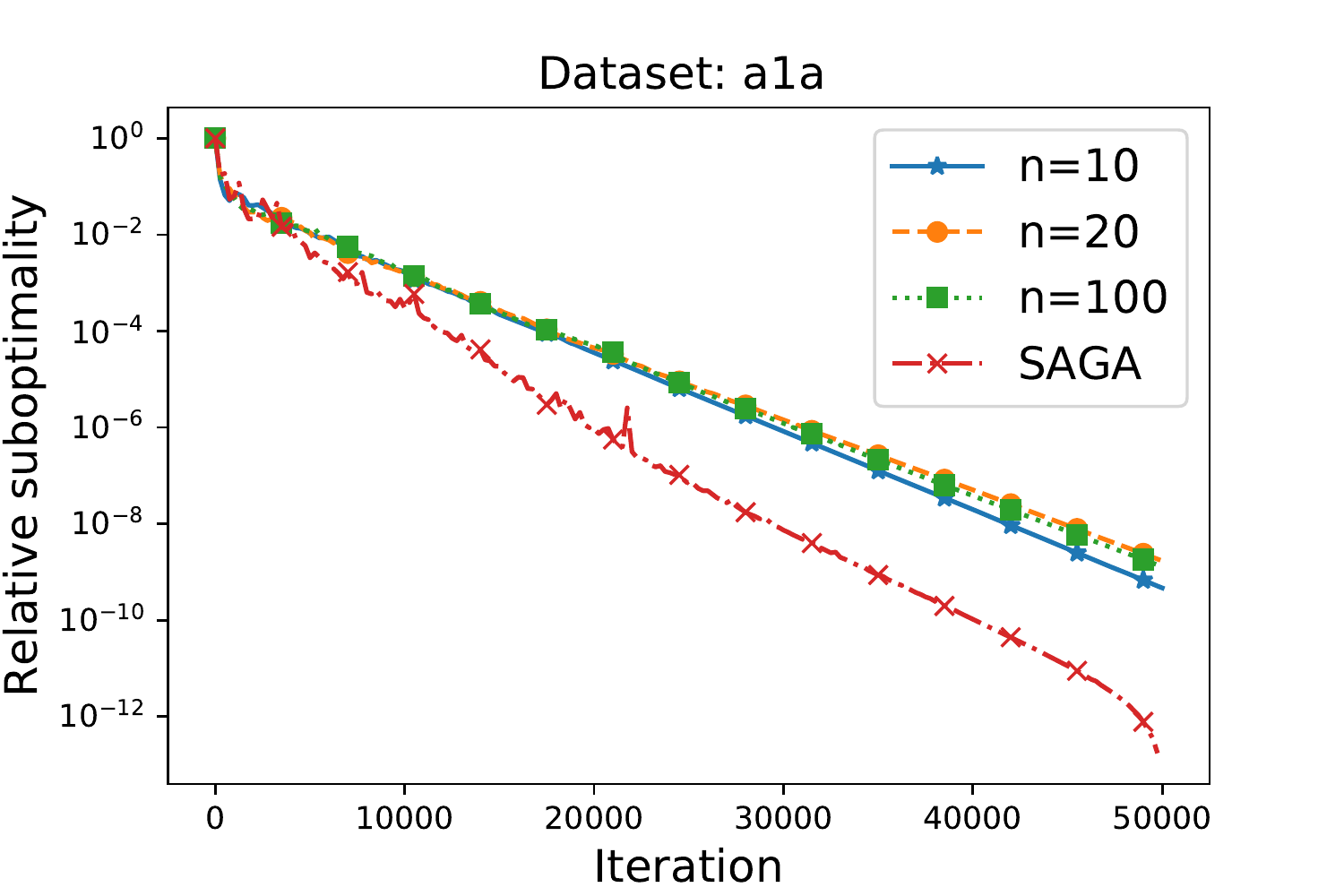}
\end{minipage}%
\begin{minipage}{0.25\textwidth}
  \centering
\includegraphics[width =  \textwidth ]{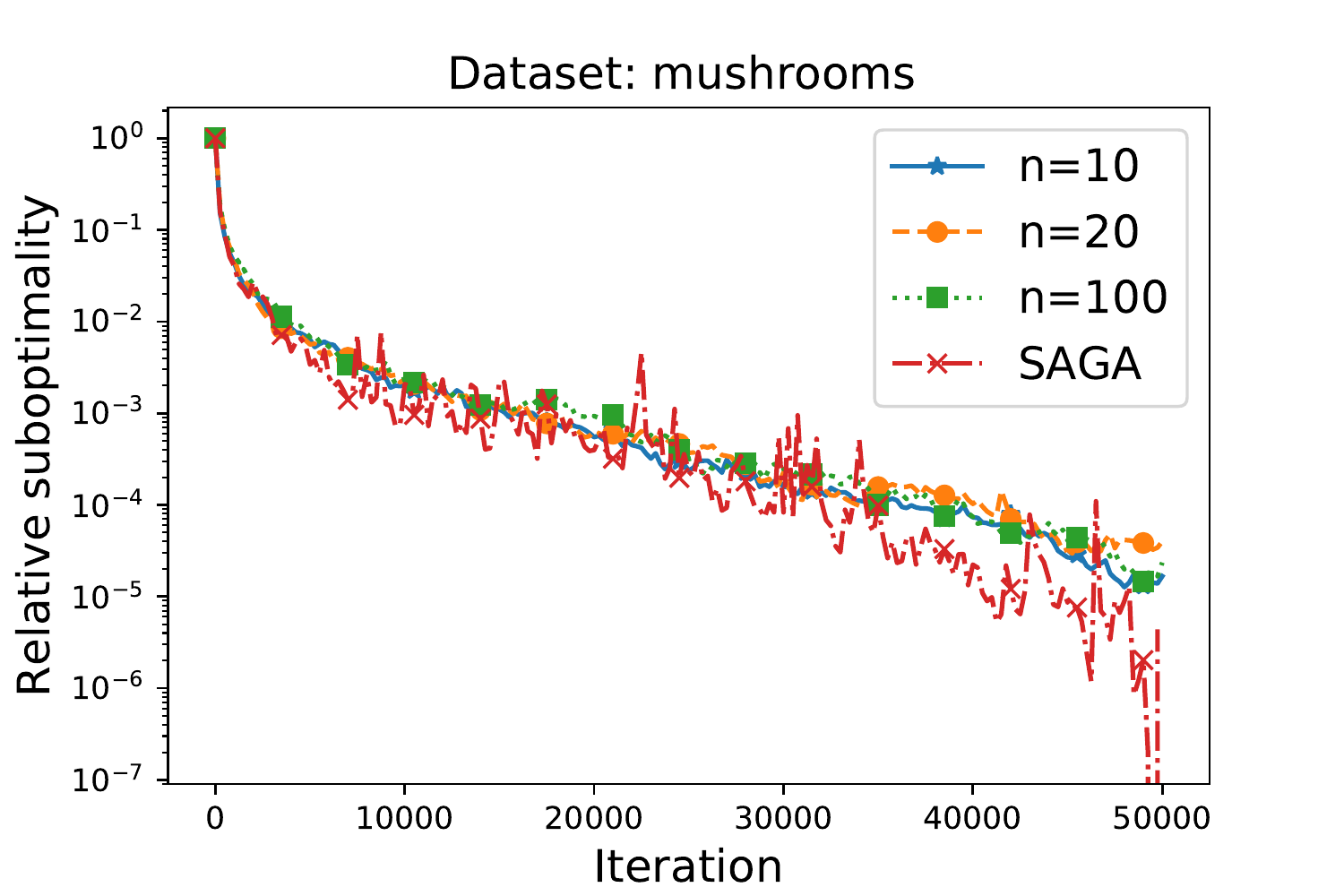}
\end{minipage}%
\begin{minipage}{0.25\textwidth}
  \centering
\includegraphics[width =  \textwidth ]{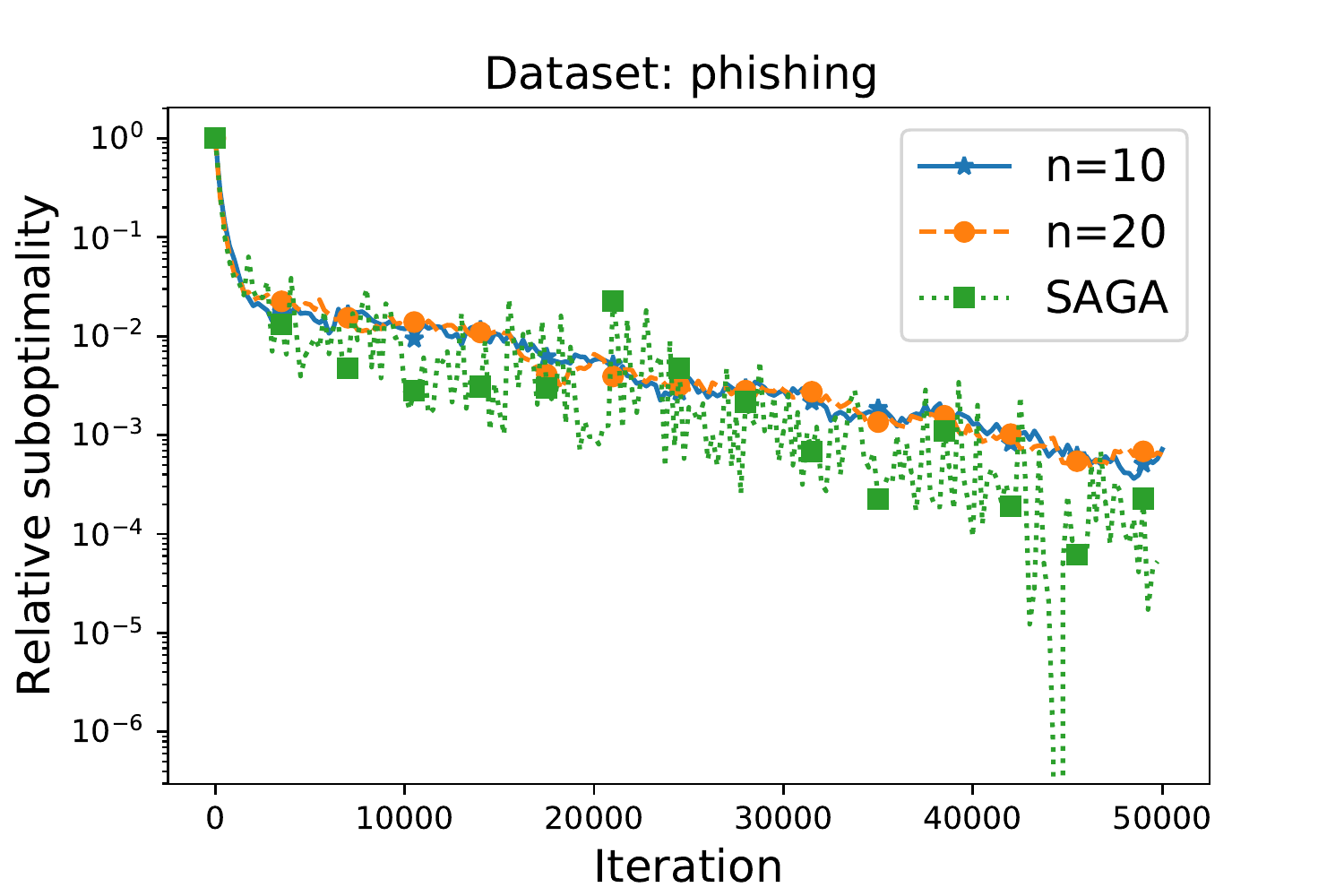}
\end{minipage}%
\begin{minipage}{0.25\textwidth}
  \centering
\includegraphics[width =  \textwidth ]{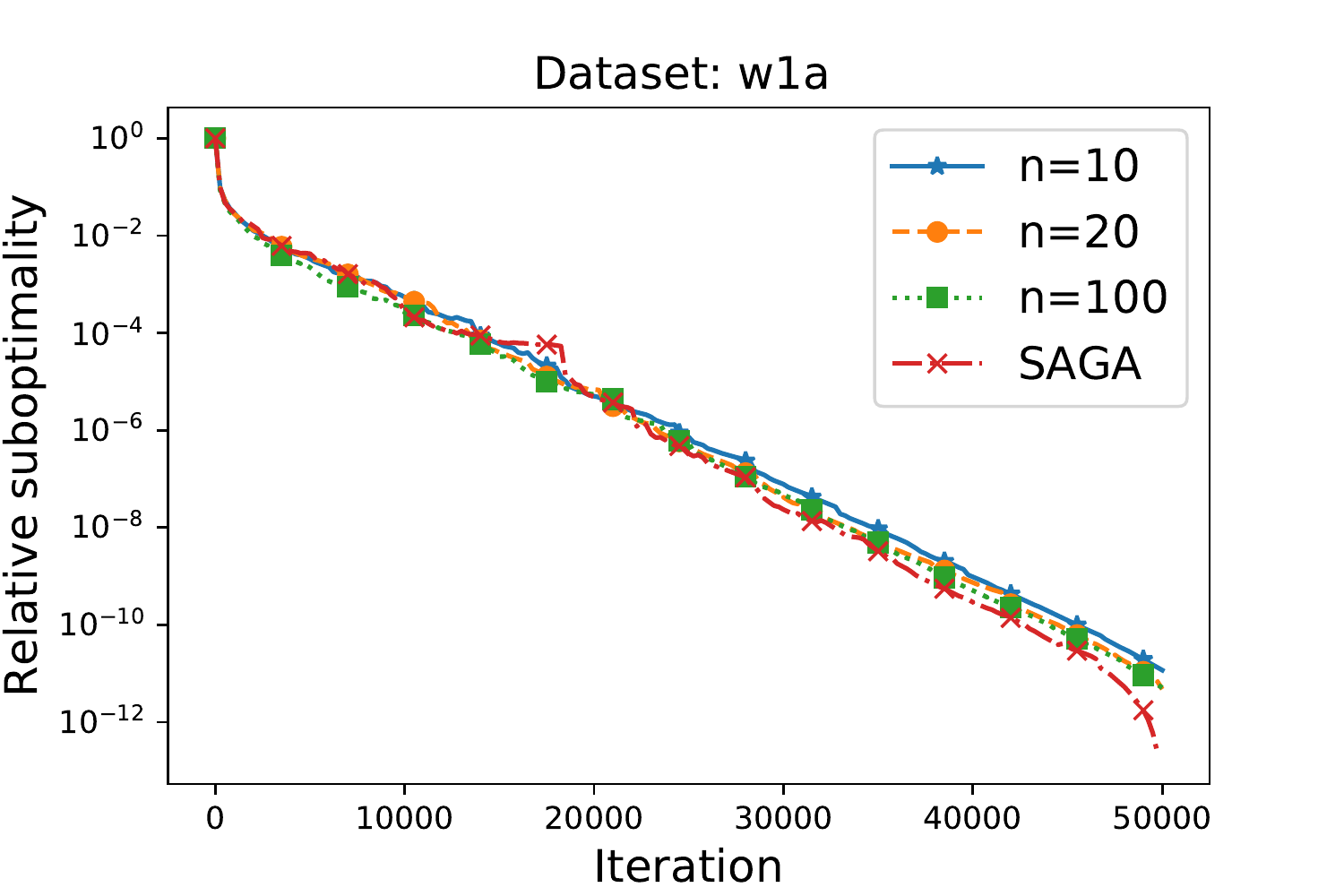}
\end{minipage}%
\caption{Comparison of SAGA and Algorithm~\ref{alg:saga} for various values of $n$ (number of workers) and $\tau=n^{-1}$ on LibSVM datasets. Stepsize $\gamma = \frac{1}{L(3n^{-1}+\tau)}$ is chosen in each case. } \label{fig:saga_main}
\end{figure}

\begin{figure}[ht]
\centering
\begin{minipage}{0.25\textwidth}
  \centering
\includegraphics[width =  \textwidth ]{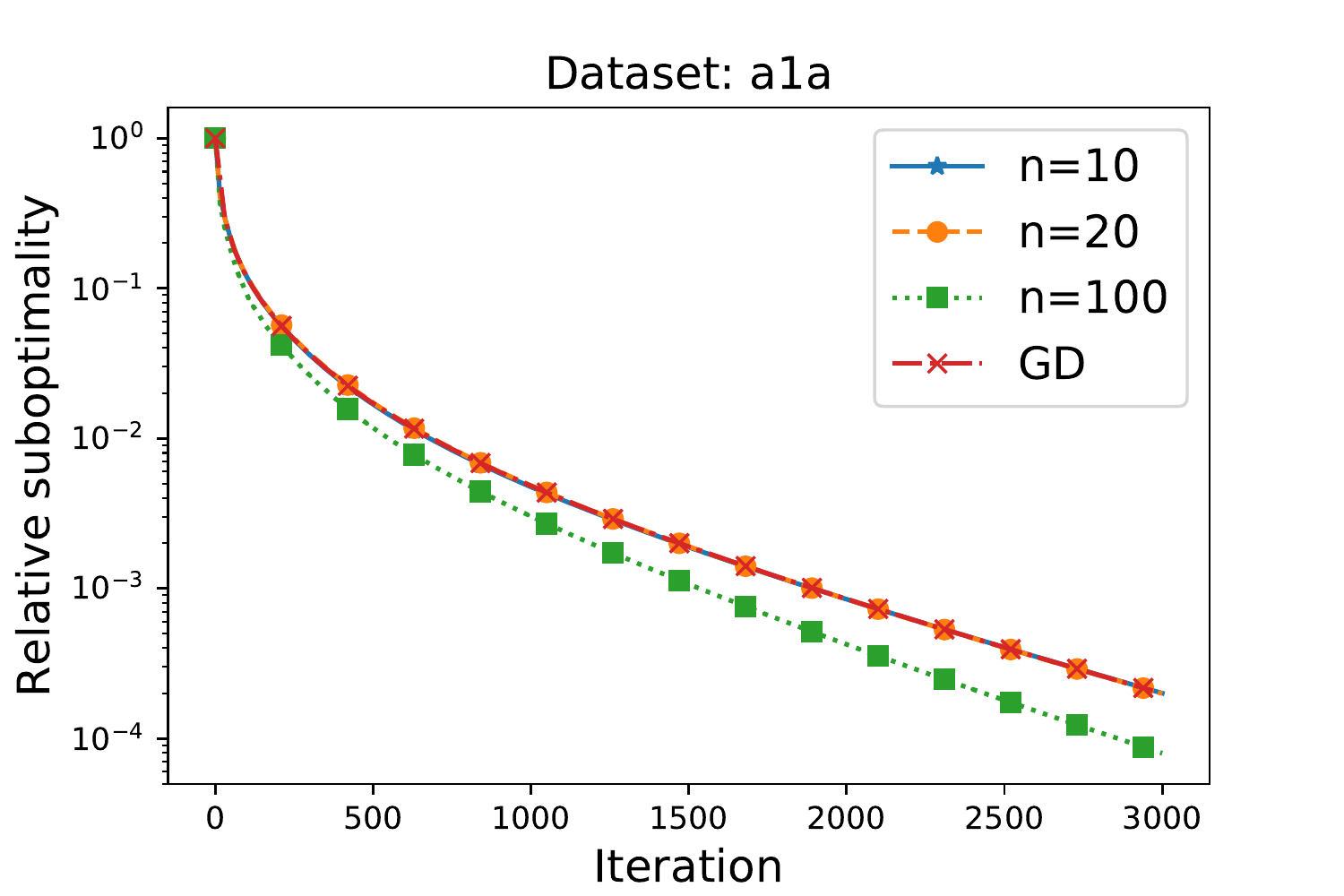}
\end{minipage}%
\begin{minipage}{0.25\textwidth}
  \centering
\includegraphics[width =  \textwidth ]{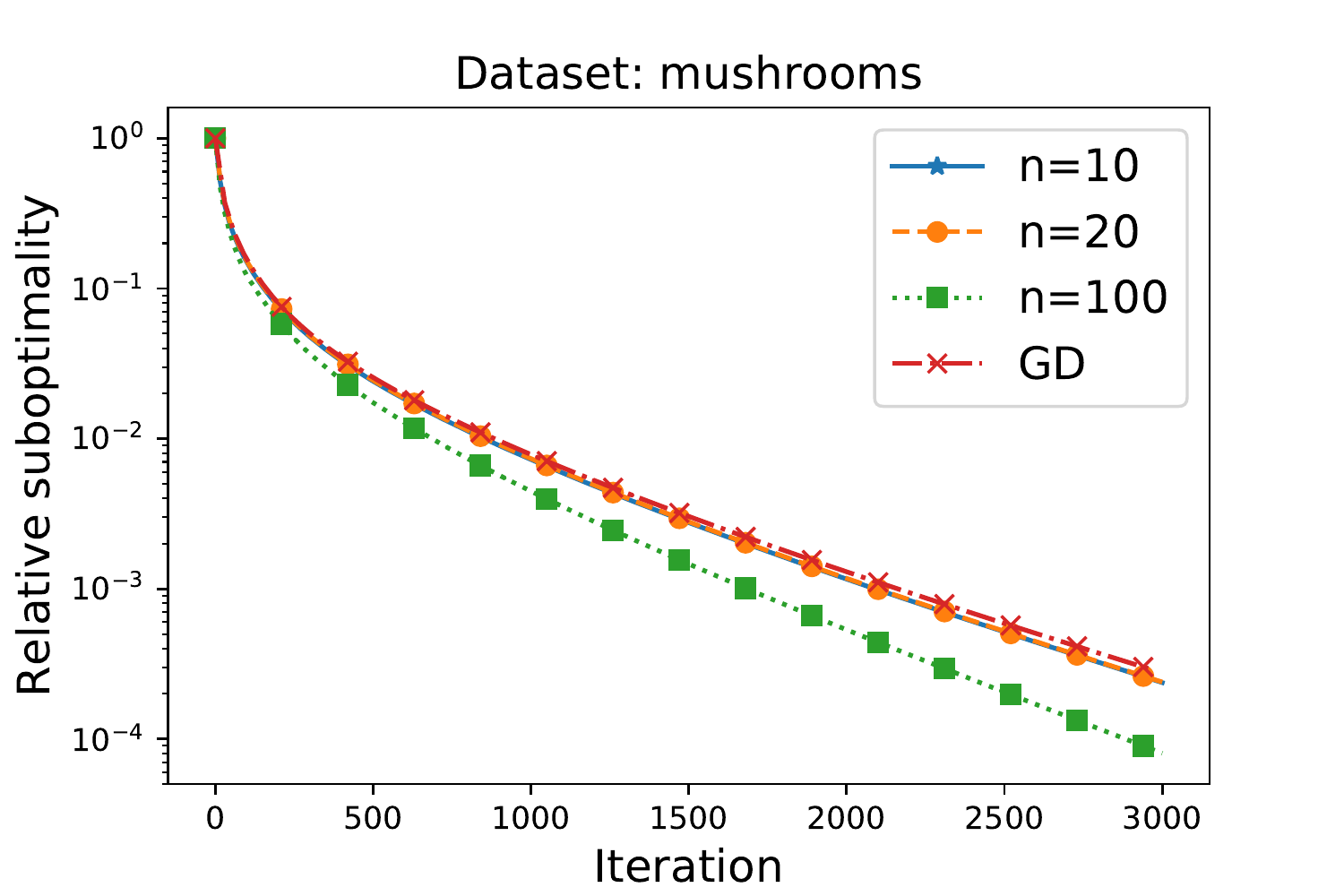}
\end{minipage}%
\begin{minipage}{0.25\textwidth}
  \centering
\includegraphics[width =  \textwidth ]{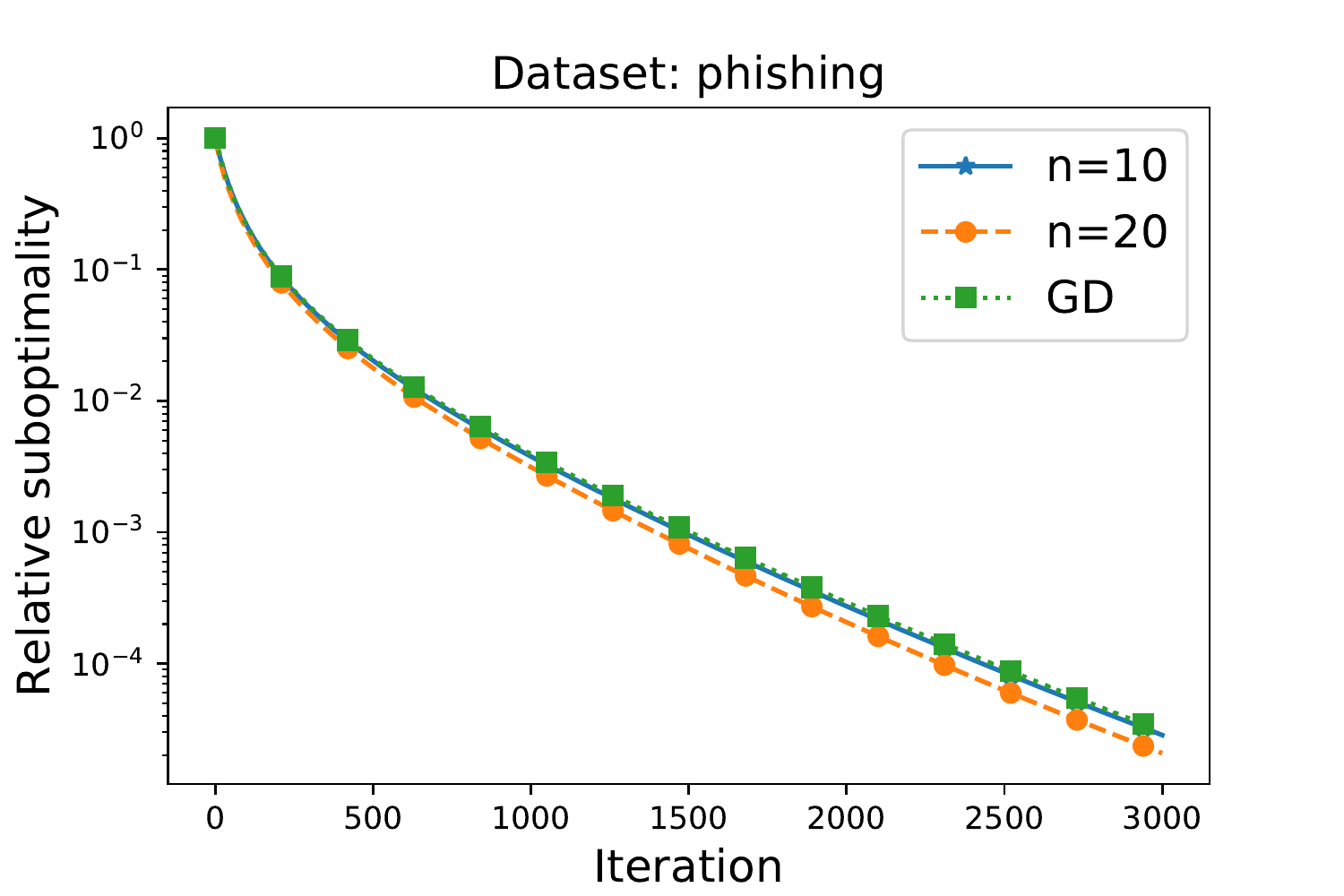}
\end{minipage}%
\begin{minipage}{0.25\textwidth}
  \centering
\includegraphics[width =  \textwidth ]{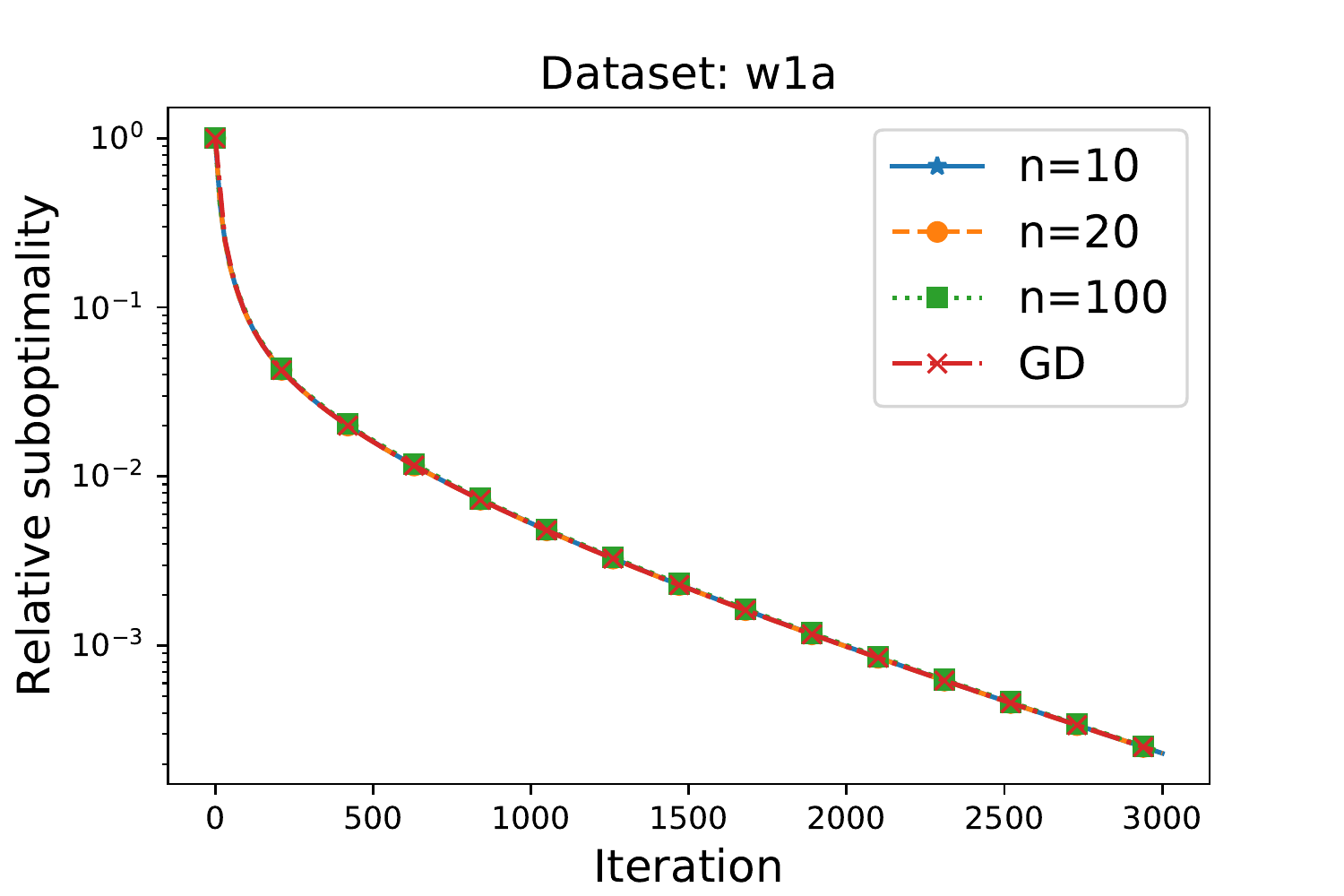}
\end{minipage}%
\caption{Comparison of Algorithm~\ref{alg:sega} for various $(n,\tau)$ such that $n\tau=1$ and GD on LibSVM datasets. Stepsize $1/(L\left(1+\frac{1}{n\tau}\right)) $ was chosen for Algorithm~\ref{alg:sega} and $\frac1{2L}$ for GD.}\label{fig:sega_main}
\end{figure}

\clearpage

 \bibliographystyle{plain} 
 \bibliography{99}
 
\newpage
\onecolumn
\appendix

\section{Table of Frequently Used Notation~\label{sec:table}}
\begin{table}[!h]
\begin{center}

{
\begin{tabular}{|c|l|c|}
\hline
\multicolumn{3}{|c|}{{\bf General} }\\
\hline
  $x^*$ & Optimal solution of the optimization problem& \\
  $n$ & Number of parallel workers/machines & Sec.~\ref{sec:notation} \\
  $\tau$ & Ratio of coordinate blocks to be sampled by each machine & Sec.~\ref{sec:notation}\\
  $d$ & Dimensionality of space $x\in \RR^d$ & Sec.~\ref{sec:notation} \\
   $m$ & Number of coordinate blocks  & Sec.~\ref{sec:notation} \\
   $f_i$ & Part of the objective owned by machine $i$ & \eqref{eq:problem}\\
   $L$ & Each $f_i$ is $L$ smooth & As.~\ref{as:smooth_sc} and \eqref{eq:smooth}\\
   $\mu$ & $f$ is $\mu$ strongly convex & As.~\ref{as:smooth_sc} and \eqref{eq:strong_convex}\\
   $U_i^t$ & Subset of blocks sampled at iteration $t$ and worker $i$ & \\
    $\gamma$ & Stepsize & \\
    $g$ & Unbiased gradient estimator & \\
     \hline
     \multicolumn{3}{|c|}{{\bf SAGA} }\\
 \hline
  $\alpha_j$ & Delayed estimate of $j$-th objective & \eqref{eq:saga_alpha_dist},  \eqref{eq:saga_alpha_sm}  \\
  $N$ & Finite sum size for shared data problem & \eqref{eq:problem_saga_sm} \\
  $l$ & Number of datapoints per machine in distributed setup  & \eqref{eq:problem_saga_dist} \\
  $\cL^t $ & Lyapunov function & \eqref{eq:saga_lyapunov}\\
 \hline
\multicolumn{3}{|c|}{{\bf SGD } }\\
 \hline
  $g_i^t$ & Unbiased stochastic gradient; $\EE g_i^t = \nabla f_i(x^t)$   &  \\
  $\sigma^2$ & An upper bound on the variance of stochastic gradients   & As. \ref{as:bounded_noise} \\
   \hline
\multicolumn{3}{|c|}{{\bf SEGA} }\\
 \hline
  $R$ & Regularizer & \eqref{eq:problem_sega} \\
  $h_i^t$ & Sequence  of biased estimators for $\nabla f_i(x^t)$  & \eqref{eq:sega_h} \\
  $g_i^t$ & Sequence of unbiased estimators for $\nabla f_i(x^t)$  & \eqref{eq:sega_g} \\
  $\Lgen^t$ & Lyapunov function  & Thm.~\ref{thm:sega} \\
\hline
\end{tabular}
}
\end{center}
\caption{Summary of frequently used notation.}
\label{tbl:notation}
\end{table}

\section{Future Work}
We sketch several possible extensions of this work. 
\begin{itemize}
\item Combining the tricks from the paper. Distributed ISAGA requires $\nabla f_i(x)=0$. We believe it would be possible to develop SEGA approach on top of it (such as it is developed on top of coordinate descent) and drop the mentioned requirement. We also believe it should be possible to accelerate the combination of SEGA and ISAGA.
\item Convergence in the asynchronous setup. We have provided theoretical results for parallel algorithms in the synchronous setting and the asynchronous theory is a very natural future step. Moreover, as we mentioned before, it has very direct practical implications. We believe that it possible to extend the works~\cite{mishchenko2018delay, mishchenko2018distributed} to design a method with proximable regularizer (for $\ell_1$ penalty) that would communicate little both sides.
\item Importance sampling. Standard coordinate descent exploits a smoothness structure of objective (either via coordinate-wise smoothness constants or more generally using a smoothness matrix) in order to sample coordinates non-uniformly~\cite{ESO, csiba2018importance, hanzely2018accelerated}. It would be interesting to derive an importance sampling in our setting in order to converge even faster. 
\end{itemize}

\clearpage

\section{IBGD: Bernoulli alternative to IBCD}
As an alternative to computing a random block of partial derivatives of size $\tau m$, it is possible to compute the whole gradient with probability $\tau$, and attain the same complexity result. While this can be inserted in all algorithms we propose, we only present an alternative to IBCD, which we call IBGD.

\begin{algorithm}[h]
  \caption{Independent Bernoulli Gradient Descent (IBGD)}
  \label{alg:ibd}
\begin{algorithmic}[1]
\STATE{\bfseries Input: } {$x^0\in\RR^d$, probability of computing the whole gradient $\tau$, stepsize $\gamma$, \# of parallel units $n$}
  \FOR{$t=0,1,\dotsc$}
    \FOR{$i=1,\dotsc,n$ in parallel}
        \STATE Set $g^t_i = \begin{cases} \nabla f_i(x^t) & \text{with probability} \quad \tau \\  0 & \text{with probability} \quad 1-\tau  \end{cases}\quad $ independently 
        \STATE $x_i^{t+1} = x^t - \gamma  g^t_i$
    \ENDFOR
    \STATE $x^{t+1} = \frac{1}{n}\sumin x_i^{t+1}$
  \ENDFOR
\end{algorithmic}
\end{algorithm}

\begin{theorem}\label{th:ibd} 
Suppose that Assumptions~\ref{as:smooth_sc},~\ref{as:zero_grads} hold. For Algorithm~\ref{alg:ibd} with $\gamma = \frac{n}{\tau n + 2(1 - \tau)}\frac{1}{2L}$ we have
\[
\EE\left[ \|x^{t} -x^*\|_2^2 \right] \leq \left(1-\frac{\mu}{2L} \frac{\tau n}{\tau n + 2(1-\tau)}\right)^t\|x^{0} -x^*\|_2^2.
\]
\end{theorem}

Note that IBGD does not perform sparse updates to the server; it is either full (dense), or none. This resembles the most naive asynchronous setup -- where each iteration, a random subset of machines communicates with the server\footnote{In reality, there subset is not drawn from a fixed distribution}. Our findings thus show that we can expect perfect linear scaling for such unreal asynchronous setup. In the honest asynchronous setup, we shall still expect good parallel scaling once the sequence of machines that communicate with server somewhat resembles a fixed uniform distribution. 

\section{Shared data ISAGA -- algorithm}
\begin{algorithm}[h]
  \caption{ISAGA with shared data}
  \label{alg:saga}
\begin{algorithmic}[1]
\STATE{\bfseries Input: }{$x^0\in\RR^d$, $\alpha_1^0,\dotsc, \alpha_N^0$ partition of $\RR^d$ into $m$ blocks $u_1,\dotsc, u_m$, ratio of blocks to be sampled $\tau$, stepsize $\gamma$, \# parallel units $n$}
  \STATE Set $\overline \alpha^0 \eqdef \frac{1}{N}\sumin \alpha_i^0$
  \FOR{$t=0,1,\dotsc$}
          \STATE Sample uniformly set of indices $\{j_{1}^t, \dots, j_n^t\} \subseteq  \{1,\dots, N\}$ without replacement
    \FOR{$i=1,\dotsc,n$ in parallel}
        \STATE Sample independently and uniformly a subset of $\tau m$ blocks $U_t^i$
        \STATE $x_i^{t+1} = x^t - \gamma  (\nabla \smo_{j_i^t}(x^t) - \alpha_{j_i^t}^t + \overline \alpha^t)_{U_i^t}$
        \STATE $(\alpha_{j_i^t}^{t+1})_{U_i^t} = \alpha_{j_i^t}^t+ (\nabla \smo_{j_i^t}(x^t)- \alpha_{j_i^t}^t )_{U_i^t}$
    \ENDFOR
    \STATE For $j\not\in \{j_{1}^t, \dots, j_n^t\} $ set $(\alpha_{j}^{t+1})= \alpha_{i}^t$
    \STATE $x^{t+1} = \frac{1}{n}\sumin x_i^{t+1}$
    \STATE $\overline \alpha^{t+1} = \frac{1}{n}\sum_{j=1}^N \alpha_j^{t+1}$
  \ENDFOR
\end{algorithmic}
\end{algorithm}

\section{Distributed ISAGA \label{sec:saga_dist}}
In this section we consider problem \eqref{eq:problem} with $f_i$ of the finite-sum structure \eqref{eq:problem_saga_dist}.  Just like SAGA, every machine remembers the freshest gradient information of all local functions (stored in arrays $\alpha_{ij}$), and updates them once a new gradient information is observed. Given that index $j_i^t$ is sampled on $i$-th machine at iteration $t$, the iterate update step within each machine is taken only on a sampled set of coordinates: 
\[
x_i^{t+1} = x^t - \gamma  (\nabla f_{ij_i^t}(x^t) - \alpha_{ij_i^t}^t + \overline \alpha_i^t)_{U_i^t} .
\]
Above, $\overline \alpha_i^t$ stands for the average of $\alpha$ variables on $i$-th machine, i.e.\ it is a delayed estimate of $\nabla f_i(x^t)$. 

Since the new gradient information is a set of partial derivatives of $\nabla f_{ij_i^t}(x^t)$, we shall update
\begin{equation}
\alpha_{ij}^{t+1} = 
\left\{
                \begin{array}{ll}
                    \alpha_{ij}^t+ (\nabla f_{ij}(x^t)- \alpha_{ij}^t )_{U_i^t}   & j=j_i^t\\
                       \alpha_{ij}^t & j\neq j_i^t
                \end{array}
     \right.\label{eq:saga_alpha_dist}
\end{equation}
Lastly, the local results are aggregated. See Algorithm~\ref{alg:saga_dist} for details. 
\begin{algorithm}[h]
  \caption{Distributed ISAGA}
  \label{alg:saga_dist}
\begin{algorithmic}[1]
\STATE{\bfseries Input: }{$x^0\in\RR^d$, \# parallel units $n$, $i$-th unit owns $l$ functions $f_{I1} ,\dots, f_{il}$, partition of $\RR^d$ into $m$ blocks $u_1,\dotsc, u_m$, ratio of blocks to be sampled $\tau$, stepsize $\gamma$, initial vectors $\alpha_{ij}^0 \in \RR^d$ for $1\leq i \leq n, 1\leq j\leq l$  }
  \STATE Set $\overline \alpha^0 \eqdef \frac{1}{N}\sumin \alpha_i^0$
  \FOR{$t=0,1,\dotsc$}
    \FOR{$i=1,\dotsc,n$ in parallel}
         \STATE Sample independently \& uniformly $ j_i^t\in [l]$ 
        \STATE Sample independently \& uniformly a subset of $\tau m$ blocks $U_t^i$
        \STATE $x_i^{t+1} = x^t - \gamma  (\nabla f_{ij_i^t}(x^t) - \alpha_{ij_i^t}^t + \overline \alpha_i^t)_{U_i^t}$
        \STATE $\alpha_{ij^t}^{t+1} = \alpha_{ij_i^t}^t+ (\nabla f_{ij_i^t}(x^t)- \alpha_{ij_i^t}^t )_{U_i^t}$
        \STATE For any $j\neq j_i^t$ set $\alpha_{ij}^{t+1} = \alpha_{ij}^t$
        \STATE $\overline \alpha^{t+1} = \frac{1}{l}\sum_{j=1}^l \alpha_{ij}^{t+1}$
    \ENDFOR
    \STATE $x^{t+1} = \frac{1}{n}\sumin x_i^{t+1}$
  \ENDFOR
\end{algorithmic}
\end{algorithm}

The next result provides a convergence rate of distributed ISAGA.
\begin{theorem}\label{th:saga_dist}
    Suppose that Assumptions~\ref{as:smooth_sc},~\ref{as:zero_grads} hold. If $\gamma\le \frac{1}{L\left(\frac{3}{n} + \tau\right)}$, for iterates of distributed ISAGA we have
    \begin{align*}
        \EE \|x^t - x^*\|_2^2 \le (1 - \vartheta )^t\left(\|x^0 - x^*\|_2^2 + c \gamma^2 \Psi^0\right),
    \end{align*}
    where $\Psi^0\eqdef\sum_{i=1}^n\sum_{j=1}^l \|\alpha_{ij}^t - \nabla f_{ij}(x^*)\|_2^2$, $\vartheta\eqdef \tau\min\left\{\gamma\mu, \frac{ 1}{l} - \frac{2}{n^2lc} \right\}\ge 0$ and $c\eqdef \frac{1}{n} ( \frac{1}{\gamma L} - \frac{1}{n} - \tau) > 0$.
\end{theorem}

The choice $\tau = n^{-1}$ yields a convergence rate which is, up to a constant factor, the same as convergence rate of original SAGA. Thus, distributed ISAGA enjoys the desired parallel linear scaling. Corollary~\ref{cor:saga_dist} formalizes this claim. 

\begin{corollary}\label{cor:saga_dist}
 Consider the setting from Theorem~\ref{th:saga_dist}. Set $\tau = \frac{1}{n}$ and $\gamma = \frac{n}{5L}$. Then $c=\frac{3}{n^2}$, $\rho = \min \left\{ \frac{\mu}{5L}, \frac{1}{3nl}\right\}$ and the complexity of distributed ISAGA is \[{\cal O}\left(\max \left\{\frac{L}{\mu}, nl \right\}\log\frac{1}{\varepsilon} \right).\]
\end{corollary}

\section{SGD \label{sec:sgd}}

In this section, we apply independent sampling in a setup with a stochastic objective. In particular, we consider problem~\eqref{eq:problem} where $f_i$ is given as an expectation; see \eqref{eq:stoch-f_i}.
We assume we have access to a stochastic gradient oracle which, when queried at $x^t$,  outputs a random vector $g_i^t$ whose mean is $\nabla f_i(x^t)$:  $\EE g_i^t = \nabla f_i(x^t)$. 

Our proposed algorithm---ISGD---evaluates a subset of stochastic partial derivatives for the local objective and takes a step in the given direction for each machine. Next, the results are averaged and followed by the next iteration. We stress that the coordinate blocks have to be sampled independently within each machine.

\begin{algorithm}[h]
  \caption{ISGD}
  \label{alg:sgd}
\begin{algorithmic}[1]
\STATE{\bfseries Input: }{$x^0\in\RR^d$, partition of $\RR^d$ into $m$ blocks $u_1,\dotsc, u_m$, ratio of blocks to be sampled $\tau$, stepsize sequence $\{\gamma^t\}_{t=1}^\infty$, \# parallel units $n$}
  \FOR{$t=0,1,\dotsc$}
    \FOR{$i=1,\dotsc,n$ in parallel}
        \STATE Sample independently and uniformly a subset of $\tau m$ blocks $U_i^t \subseteq \{u_1, \dotsc, u_m\}$
        \STATE Sample blocks of stochastic gradient $(g_i^t)_{U_i^t}$  such that $\EE [g_i^t \, |\, x^t] = \nabla f_i(x^t)$
        \STATE $x_i^{t+1} = x^t - \gamma^t  (g_i^t)_{U_i^t}$
    \ENDFOR
    \STATE $x^{t+1} = \frac{1}{n}\sumin x_i^{t+1}$
  \ENDFOR
\end{algorithmic}
\end{algorithm}

In order to establish a convergence rate of ISGD, we shall assume boundedness of stochastic gradients for each worker.

\begin{assumption}\label{as:bounded_noise}
   Consider a sequence of iterates $\{x^t \}_{t=0}^\infty$ of Algorithm~\ref{alg:sgd}. Assume that $g_i^t$ is an unbiased estimator of $\nabla f_i(x^t)$ satisfying
   $
        \EE \|g_i^t - \nabla f_i(x^t)\|_2^2 \le \sigma^2.
$
    
\end{assumption}
\begin{assumption}\label{as:bounded_noise_at_opt}
   Stochastic gradients of function $f_i$ have bounded variance at the optimum of $f$:
   $
        \EE \|g_i - \nabla f_i(x^*)\|_2^2 \le \sigma^2,
$
    where $g_i$ is a random vector such that $\EE g_i = \nabla f_i(x^*)$.
\end{assumption}

Next, we present the convergence rate of Algorithm~\ref{alg:sgd}. Since SGD is not a variance reduced algorithm, it does not enjoy a linear convergence rate and one shall use decreasing step sizes. As a consequence, Assumption~\ref{as:zero_grads} is not required anymore since there is no variance reduction property to be broken. 

\begin{theorem} \label{th:sgd}
Let Assumptions~\ref{as:smooth_sc} and~\ref{as:bounded_noise} hold. If $\gamma^t = \frac{1}{a + ct}$, where $a= 2\left(\tau + \tfrac{2(1 - \tau)}{n} \right)L$, $c= \frac14 \mu\tau$, then for Algorithm~\ref{alg:sgd} we can upper bound $ \EE [f(\hat x^t) - f(x^*) ]$ by
    \begin{align*}
 \frac{a^2\left(1 - \frac{\tau\mu}{a}\right)\|x^0 - x^*\|_2^2}{\tau(t+1)a+ \tfrac{c\tau}{2}t(t+1)}  +  \frac{\sigma^2 + (1 - \tau)\frac{2}{n}\sumin\|\nabla f_i(x^*)\|_2^2 }{n\left(1+\tfrac1t\right)a+ \frac{nc}{2}(t+1)},
    \end{align*}
where $\hat x^t\eqdef \frac{1}{(t+1)a+ \frac{c}{2}t(t+1)}\sum_{k=0}^t (\gamma^k)^{-1} x^k$.
\end{theorem}

Note that the residuals decrease as ${\cal O}(t^{-1})$, which is a behavior one expects from standard SGD. Moreover, the leading complexity term scales linearly: if the number of workers $n$ is doubled, one can afford to halve $\tau$ to keep the same complexity. 
\begin{corollary}\label{cor:sgd}
Consider the setting from Theorem~\ref{th:sgd}. Then, iteration complexity of Algorithm~\ref{alg:sgd} is 
\[
{\cal O} \left( 
\frac{\sigma^2 + \frac{1}{n}\sumin\|\nabla f_i(x^*)\|_2^2 }{n\tau \mu \epsilon}
\right).
\]
\end{corollary}

Although problem~\eqref{eq:problem} explicitly assumes convex $f_i$, we also consider a non-convex extension, where smoothness of each individual $f_i$ is not required either. Theorem~\ref{th:sgd_ncvx} provides the result. 

\begin{theorem}[Non-convex rate]\label{th:sgd_ncvx}
    Assume $f$ is $L$ smooth, Assumption~\ref{as:bounded_noise} holds and for all $x\in\RR^d$ the difference between gradients of $f$ and $f_i$'s is bounded: $\frac{1}{n}\sumin \|\nabla f(x) - \nabla f_i(x)\|_2^2\le \nu^2$ for some constant $\nu\ge 0$. If $\hat x^t$ is sampled uniformly from $\{x^0, \dotsc, x^t\}$, then for Algorithm~\ref{alg:sgd} we have
    \begin{align*}
        \EE \|\nabla f(\hat x^t)\|_2^2
        \le \frac{\tfrac{f(x^0) - f^*}{t\tau\gamma} + \gamma  L\frac{\left(1 - \tau\right)\nu^2+\frac12\sigma^2}{n}}{1 - \frac{\gamma \tau L}{2} - \gamma L\left(1 - \tau\right)\frac{1}{n}}.
    \end{align*}
    \end{theorem}

Again, the convergence rate from Theorem~\ref{th:sgd_ncvx} scales almost linearly with $\tau$: with doubling the number of workers one can afford to halve $\tau$ to keep essentially the same guarantees. Note that if $n$ is sufficiently large, increasing $\tau$ beyond a certain threshold does not improve convergence. This is a slightly weaker conclusion to the rest of our results where increasing $\tau$ beyond $n^{-1}$ might still offer speedup. The main reason behind this is  the fact that SGD may be  noisy enough on its own  to still benefit from the averaging step.

    \begin{corollary}\label{cor:ncvx}
 Consider the setting from Theorem~\ref{th:sgd_ncvx}. i) Choose $\tau\ge \frac{1}{n}$ and $\gamma = \frac{\sqrt{n}}{L\sqrt{\tau t}} \le \frac{1}{2L\left(\tau/2 + (1 - \tau)/n \right)}$. Then \[\EE \|\nabla f(\hat x^t)\|_2^2 \le \frac{2}{\sqrt{t\tau n}}\left(\frac{f(x^0) - f^*}{L} + (1 - \tau)\nu^2\right) = O\left(\frac{1}{\sqrt{t}}\right).\] ii) For any $\tau$ there is sufficiently large $n$ such that choosing $\gamma = {\cal O}\left( \frac{\epsilon}{\tau L^2}\right)$ yields complexity ${\cal O} \left( \frac{L^2}{\epsilon^2}\right)$. The complexity does not improve significantly when $\tau$ is increased. 

    \end{corollary}

  \section{Acceleration~\label{sec:ABCDE}}
Here we describe an accelerated variant of IBCD in the sense of~\cite{nesterov}. In fact, we will do something more general and accelerate ISGD, obtaining the IASGD algorithm. We again assume that machine $i$ owns $f_i$, which is itself a stochastic objective as in~\eqref{eq:stoch-f_i} with an access to an unbiased stochastic gradient $g^t$ every iteration: $\EE g_i^t = \nabla f_i(x^t)$. A key assumption for the accelerated SGD used to derive the best known rates~\cite{vaswani2018fast} is so the called strong growth of the unbiased gradient estimator.
\begin{definition}
Function $\phi(x)=\EE_\zeta \phi(x,\zeta)$ satisfies the strong growth condition with parameters $\rho, \sigma^2$, if for all $x$ we have
\[
\EE_{\zeta} \| \nabla \phi(x,\zeta)\|^2_2\leq \rho \| \nabla \phi(x)\|^2_2 +\sigma^2.
\]  
\end{definition} 

In order to derive a strong growth property of the gradient estimator coming from the independent block coordinate sampling, we require a strong growth condition on $f$ with respect to $f_1, \dots, f_n$ and also a variance bound on stochastic gradients of each individual $f_i$.

\begin{assumption} \label{as:strong_growth}
Function $f$ satisfies the strong growth condition with respect to $f_1, \dots, f_n$ : 
\begin{equation}\label{eq:acc_sg_f}
\frac{1}{n}\sum_{i=1}^n \|\nabla f_i(x)\|^2_2 \leq  \tilde{\rho} \|\nabla f(x) \|^2_2+  \tilde{\sigma}^2.
\end{equation}
Similarly, given that $g_i = g_i(x)$ provides an unbiased estimator of $\nabla f_i(x)$, i.e.\ $\EE g_i =\nabla f_i(x)$, variance of $g_i$ is bounded as follows for all $i$:
\begin{equation}\label{eq:acc_sg_fi}
 \Var\left[  g_i\right] \leq  \bar{\rho} \|\nabla f_i(x) \|^2_2+  \bar{\sigma}^2.
\end{equation}
\end{assumption}

Note that the variance bound \eqref{eq:acc_sg_fi} is weaker than the strong growth property as we always have $ \Var\left[  g_i\right]  \leq  \EE\left[ \| g_i \|^2_2\right] $.

Given that Assumption~\ref{as:strong_growth} is satisfied, we derive a strong growth property for the unbiased gradient estimator $q \eqdef \frac{1}{n\tau}\sum_{i=1}^n(\nabla g_i)_{U_i}$ in Lemma~\ref{lem:stronggrowth}. Next, IASGD is nothing but the scheme from~\cite{vaswani2018fast} applied to stochastic gradients $q$. For completeness, we state IASGD as Algorithm~\ref{alg:acc}.

\begin{algorithm}[h]
  \caption{IASGD.}
  \label{alg:acc}
\begin{algorithmic}[1]
\STATE{\bfseries Input: } {Starting point $y^0=v^0\in\RR^d$, partition of $\RR^d$ into $m$ blocks $u_1,\dotsc, u_m$, ratio of blocks to be sampled $\tau$, stepsize $\gamma$, number of parallel units $n$, acceleration parameter sequences $\{a,b,\eta\}_{t=0}^\infty$}
  \FOR{$t=0,1,\dotsc$}
         \STATE $x^t  = a^t v^t + (1 - a^t)y^t $ \\
    \FOR{$i=1,\dotsc,n$ in parallel}
        \STATE Sample independently and uniformly a subset of $\tau m$ blocks $U_i^t \subset \{u_1, \dotsc, u_m\}$
        \STATE Sample blocks of stochastic gradient $(g_i^t)_{U_i^t}$  such that $\EE [g_i^t \, |\, x^t] = \nabla f_i(x^t)$
       \ENDFOR
        \STATE  $q^t  =\frac{1}{n\tau}\sumin (g_i^t)_{U_i^t}$ \\
    	   \STATE  $y^{t+1}  = x^t -\gamma  q^t$ \\
		\STATE $v^{t+1} = b^t v^t + (1 - b^t)x^t  - \eta^t \gamma q^t$.
  \ENDFOR
\end{algorithmic}
\end{algorithm}

\begin{lemma}\label{lem:stronggrowth}
Suppose that Assumption~\ref{as:strong_growth} is satisfied. Then, we have
$
\EE\left[\|q\|^2_2\right] \leq \hat{\rho} \|\nabla f(x) \|^2_2 + \hat{\sigma}^2
$
for
\begin{eqnarray}
\label{eq:acc_rho}
\hat{\rho} &\eqdef& \left(1+ \tfrac{\tilde{\rho}}{n}   \left(\tfrac1\tau-1+\tfrac{\bar{\rho}}{\tau} \right) \right), 
\\
\hat{\sigma}^2 &\eqdef& \tfrac{\bar{\sigma}^2}{n\tau} + \tfrac{\tilde{\sigma}^2}{n}\left(\tfrac1\tau-1+\tfrac{\bar{\rho}}{\tau} \right). 
\label{eq:acc_sigma}
\end{eqnarray}
\end{lemma}

It remains to use the stochastic gradient $q$ (with the strong growth bound from Lemma~\ref{lem:stronggrowth}) as a gradient estimate in~\cite{vaswani2018fast}[Theorem 6], which we restate as Theorem~\ref{th:accelerated} for completeness. 

\begin{theorem}\label{th:accelerated}
Suppose that $f$ is $L$ smooth, $\mu$ strongly convex and Assumption~\ref{as:strong_growth} holds. Then, for a specific choice of parameter sequences $\{a,b, \eta\}_{t=0}^\infty$ (See~\cite{vaswani2018fast}[Theorem 6] for details), iterates of IASGD admit an upper bound on $ \EE \left[f(x^{t+1}) \right]-f(x^*)$ of the form
\begin{eqnarray*}
 \left(1-\sqrt{\frac{\mu}{L\hat{\rho}^2}} \right)^t \left(f(x^0)-f(x^*) + \frac{\mu}{2}\|x^0-x^* \|^2_2\right) + \frac{\hat{\sigma}^2}{\hat{\rho} \sqrt{L\mu}}.
\end{eqnarray*}

\end{theorem}

The next corollary provides a complexity of Algorithm~\ref{alg:acc} in a simplified setting where $\bar{\sigma}^2=\tilde{\sigma}^2=0$. Note that $\tilde{\sigma}^2=0$ implies $\nabla f_i(x^*) = 0$ for all $i$. It again shows a desired linear scaling: given that we double the number of workers, we can halve the number of  blocks to be evaluated on each machine and still keep the same convergence guarantees. It also shows that increasing $\tau$ beyond $\frac{\tilde{\rho}\bar{\rho}}{n}$ does not improve the convergence significantly. 
\begin{corollary}
Suppose that $\bar{\sigma}^2=\tilde{\sigma}^2=0$. Then, complexity of IASGD is \[{\cal O}\left(\frac{1}{\hat{\rho}}\sqrt{\frac{\mu}{L}}\log\frac1\epsilon \right) ={\cal O}\left(\frac{1}{1+\frac{\tilde{\rho}}{\tau n}(1+\bar{\rho})}\sqrt{\frac{\mu}{L}}\log\frac1\epsilon \right)  .\]
\end{corollary}

Theorem~\ref{th:accelerated} shows an accelerated rate for strongly convex functions applying~\cite{vaswani2018fast}[Thm 6] to the bound. A non-strongly convex rate can be obtained analogously from~\cite{vaswani2018fast}[Thm 7].

\section{Asynchronous ISGD}\label{sec:asynch}
In this section we extend ISGD algorithm to the asynchronous setup.  In particular, we revisit the method that was considered in~\cite{grishchenko2018asynchronous}, extend its convergence to stochastic oracle and show better dependency on quantization noise.

\begin{algorithm}[h]
  \caption{Asynchronous ISGD}
  \label{alg:asynch_sgd}
\begin{algorithmic}[1]
\STATE{\bfseries Input: }{$x^0\in\RR^d$, partition of $\RR^d$ into $m$ blocks $u_1,\dotsc, u_m$, ratio of blocks to be sampled $\tau$, stepsize $\gamma$, \# parallel units $n$}
  \FOR{$t=0,1,\dotsc$}
  	\STATE Worker $i=i_t$ is making update
       \STATE $w^{t - d_i^t} = \avejn x_j^{t-d_i^t}$
    	   \STATE $x^{t - d_i^t} = \proxR(w^{t-d_i^t})$
        \STATE Sample independently and uniformly a subset of $\tau m$ blocks $U_i^t \subseteq \{u_1, \dotsc, u_m\}$
        \STATE Sample blocks of stochastic gradient $(g_i^t)_{U_i^t}$  such that $\EE [g_i^t \, |\, x^t] = \nabla f_i(x^{t-d_i^t})$
        \STATE $x_i^{t} = x^{t - d_i^t} - \gamma  (g_i^t)_{U_i^t}$
        \STATE Send $(g_i^t)_{U_i^t}$ and receive $w^{t+1} = \avejn x_j^{t+1}$
  \ENDFOR
  \STATE {\bfseries Output:} $x^t = \proxR(w^t)$
\end{algorithmic}
\end{algorithm}

Let us denote the delay of worker $i$ at moment $t$ by $d_i^t$.
\begin{theorem}\label{th:asynch}
	Assume $f_1,\dotsc, f_n$ are $L$-smooth and $\mu$-strongly convex and let Assumption~\ref{as:bounded_noise_at_opt} be satisfied. Let us run Algorithm~\ref{alg:asynch_sgd} for $t$ iterations and assume that delays are bounded: $d_i^t\le M$ for any $i$ and $t$. If $\gamma\le \frac{1}{2L(\tau + \frac{2}{n})}$, then
	\begin{align*}
		\EE \|x^t - x^*\|_2^2
		&\le \left(1 - \tau\gamma\mu\right)^{\lfloor t/M\rfloor}C + 4\gamma\frac{\sigma^2}{n},
	\end{align*}
	where $C\eqdef \max_{i=1,\dotsc, n}\|x^0 - x_i^*\|_2^2$, $x_i^*\eqdef x^* - \tau\gamma \nabla f_i(x^*)$ and $\lfloor \cdot \rfloor$ is the floor operator.
\end{theorem}
Plugging $\gamma=\frac{1}{2L(\tau + \frac{2}{n})}$ gives complexity that will be significantly improving from increasing $\tau$ until $\tau = \frac{1}{n}$, and then only if $\tau$ jumps from $\frac{1}{n}$ to 1. In contrast, doubling $\tau$ from $\frac{2}{n}$ to $\frac{4}{n}$ would make little difference.

We note that if $\ell_1$ penalty is used, in practice $z_i^t$ should be rather computed on the parameter server side because it will sparsify the vector for communication back.

\clearpage
\section{Extra Experiments\label{sec:app:experiments}}
We present exhaustive numerical experiments to verify the theoretical claims of the paper. The experiments are performed in a simulated environment instead of the honestly distributed setup, as we only aim to verify the iteration complexity of proposed methods.  

First, in Sec~\ref{sec:exp_quad} provides the simplest setting in order to gain the best possible insight -- Algorithm~\ref{alg:cd} is tested on the artificial quadratic minimization problem. We compare Algorithm~\ref{alg:cd} against both gradient descent (GD) and standard CD (in our setting: when each machine samples the same subset of coordinates). We also study the effect of changing $\tau$ on the convergence speed. 

In the remaining parts, we consider a logistic regression problem on LibSVM data~\cite{chang2011libsvm}. Recall that logistic regression problem is given as

\begin{equation}
\label{eq:logreg}
f(x)= \frac1N \sum_{j=1}^N \left( \log \left(1+\exp\left(A_{j,:}x\cdot  b_j\right) \right)+\frac{\ell_2}{2} \| x\|_2^2\right),
\end{equation}
where $A$ is data matrix and $b$ is vector of data labels: $b_j\in \{-1,1 \}$\footnote{
The datapoints (rows of $A$) have been normalized so that each is of norm $1$. Therefore, each $f_i$ is $\frac14$ smooth in all cases. We set regularization parameter as $\ell_2 = 0.00025$ in all cases. 
}. In the distributed scenario (everything except of Algorithm~\ref{alg:saga}), we imitate that the data is evenly distributed to $n$ workers (i.e.\ each worker owns a subset of rows of $A$ and corresponding labels, all subsets have almost the same size). 

As our experiments are not aimed to be practical at this point (we aim to properly prove the conceptual idea), we consider multiple of rather smaller datasets: \texttt{a1a} ($d=123, n =1605$), \texttt{mushrooms} ($d=112, n =8124$), \texttt{phishing} ($d=68, n = 11055$), \texttt{w1a} ($d=300, n =2477$). The experiments are essentially of 2 types: one shows that setting $n\tau=1$ does not significantly violate the convergence of the original method. In the second type of experiments we study the behavior for varying $\tau$, and show that beyond certain threshold, increasing $\tau$ does not significantly improve the convergence. The threshold is smaller as $n$ increases, as predicted by theory.

\subsection{Simple, well understood experiment \label{sec:exp_quad}}

In this section we study the simplest possible setting -- we test the behavior of Algorithm~\ref{alg:cd} on the artificial quadratic minimization problem. The considered quadratic objective is set as 

\begin{equation} f_i(x) \eqdef \frac12 x^\top M_ix,   \qquad  M_i \eqdef vv^\top + \left(I- vv^\top\right) \frac{A_iA_i^\top}{\lambda_{\max}\left( A_iA_i^\top \right)}   \left(I- vv^\top\right), \qquad  v=\frac{v'}{\|v'\|}, \label{eq:quadratic}
\end{equation}
where entries of $v'\in \RR^d$ and $A_i\in \RR^{d\times \quadb}$ are sampled independently from standard normal distribution.

In the first experiment (Figure~\ref{fig:artif_1}), we compare Algorithm~\ref{alg:cd} with $n\tau=1$ against gradient descent (GD) and two versions of coordinate descent - a default version with stepsize $\frac1L$, and a coordinate descent with importance sampling (sample proportionally to coordinate-wise smoothness constants) and optimal step sizes (inverse of coordinate-wise smoothness constants). In all experiments, gradient descent enjoys twice better iteration complexity than Algorithm~\ref{alg:cd} which is caused by twice larger stepsize. However, in each case, Algorithm~\ref{alg:cd} requires fewer iterations to CD with importance sampling, which is itself significantly faster to plain CD.

\begin{figure}[H]
\centering
\begin{minipage}{0.24\textwidth}
  \centering
\includegraphics[width =  \textwidth ]{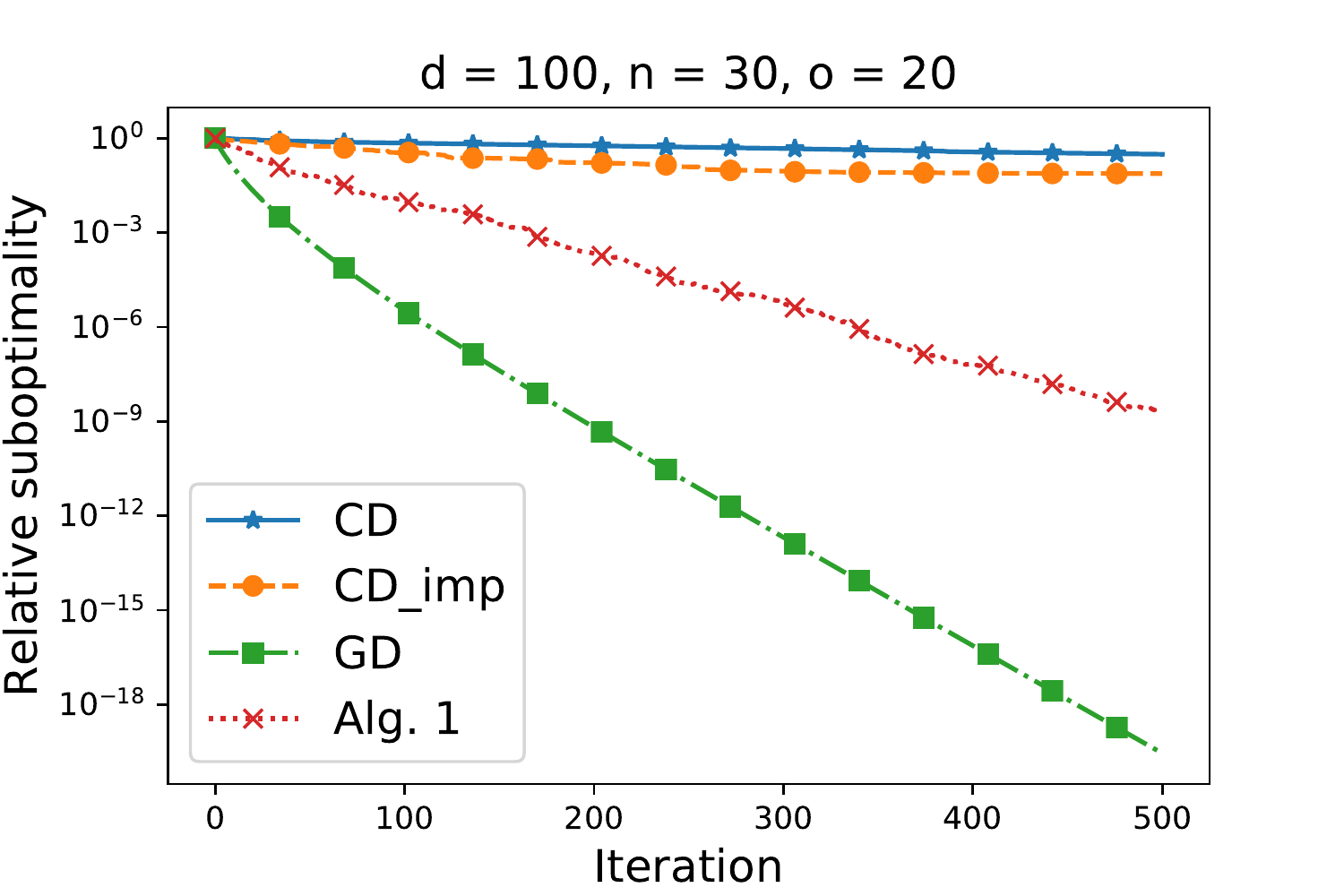}
\end{minipage}%
\begin{minipage}{0.24\textwidth}
  \centering
\includegraphics[width =  \textwidth ]{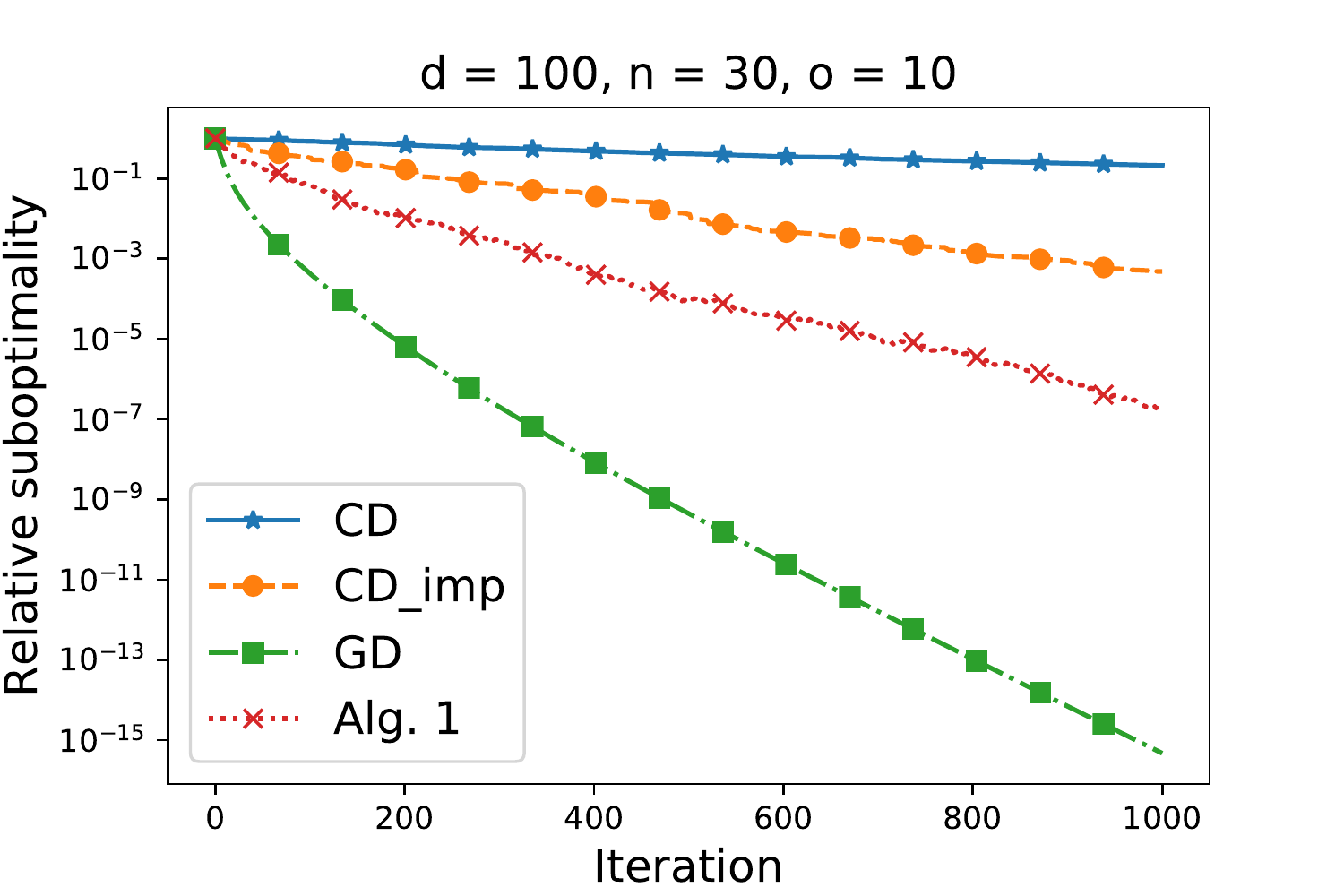}
\end{minipage}%
\begin{minipage}{0.24\textwidth}
  \centering
\includegraphics[width =  \textwidth ]{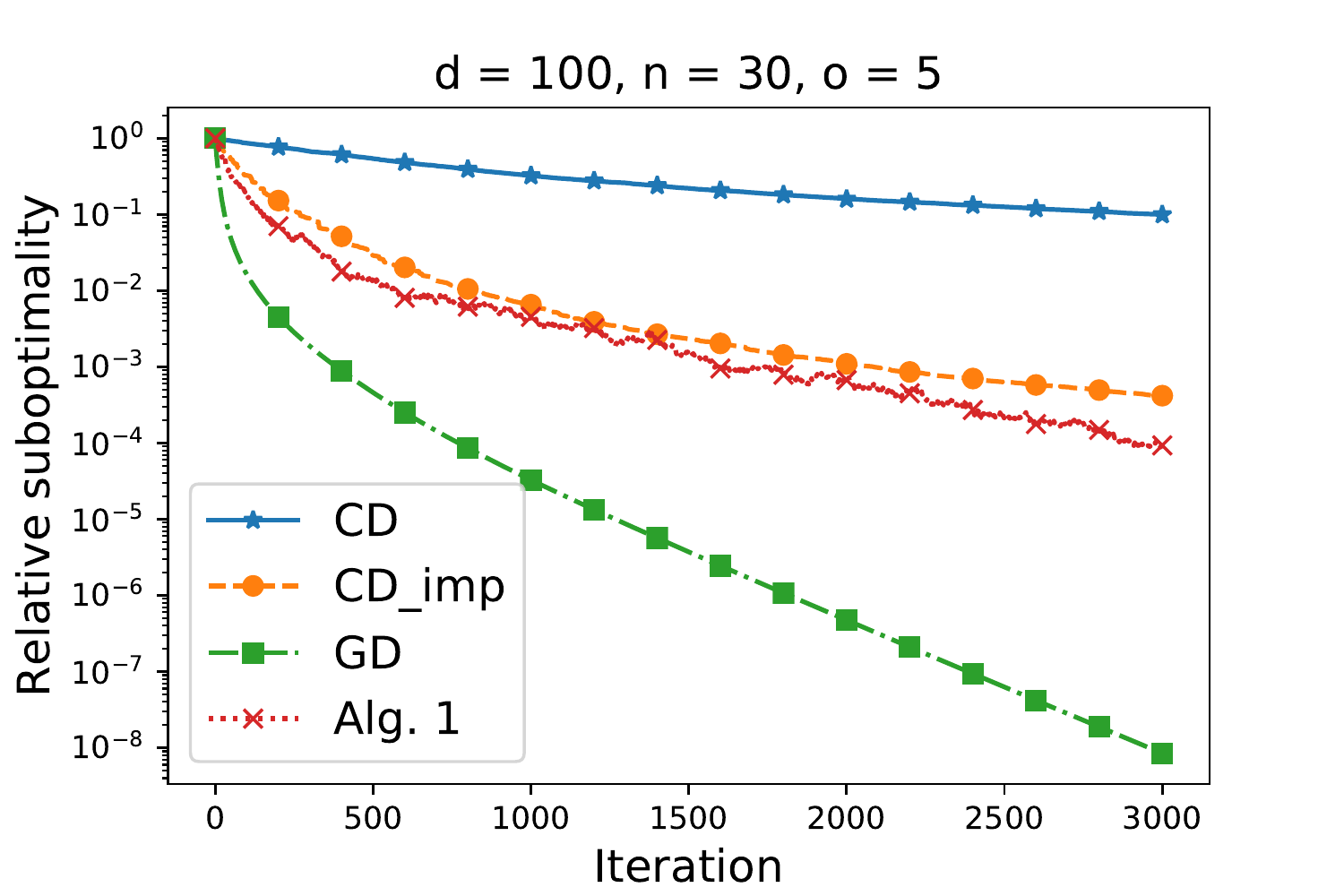}
\end{minipage}%
\begin{minipage}{0.24\textwidth}
  \centering
\includegraphics[width =  \textwidth ]{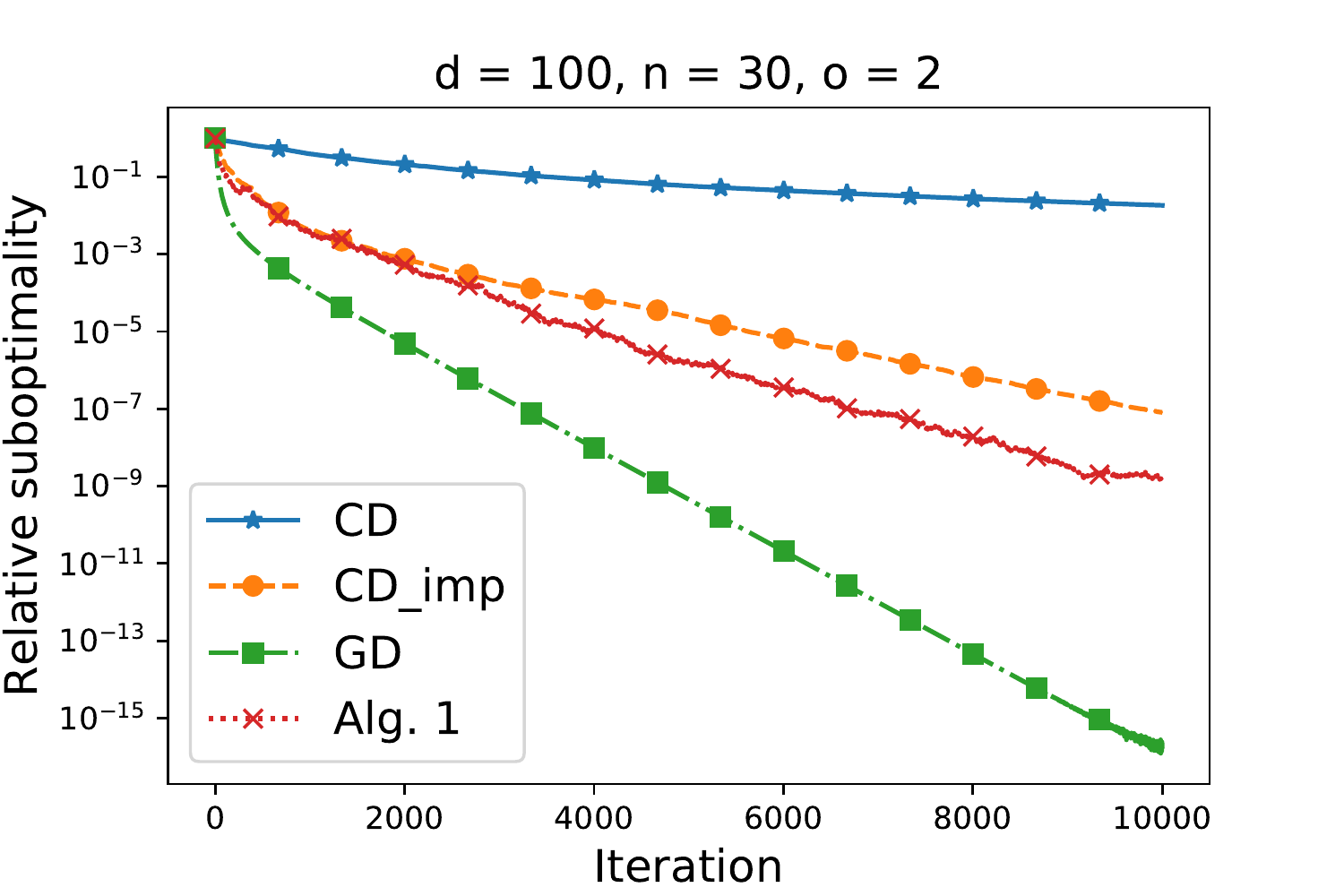}
\end{minipage}%
\\
\begin{minipage}{0.24\textwidth}
  \centering
\includegraphics[width =  \textwidth ]{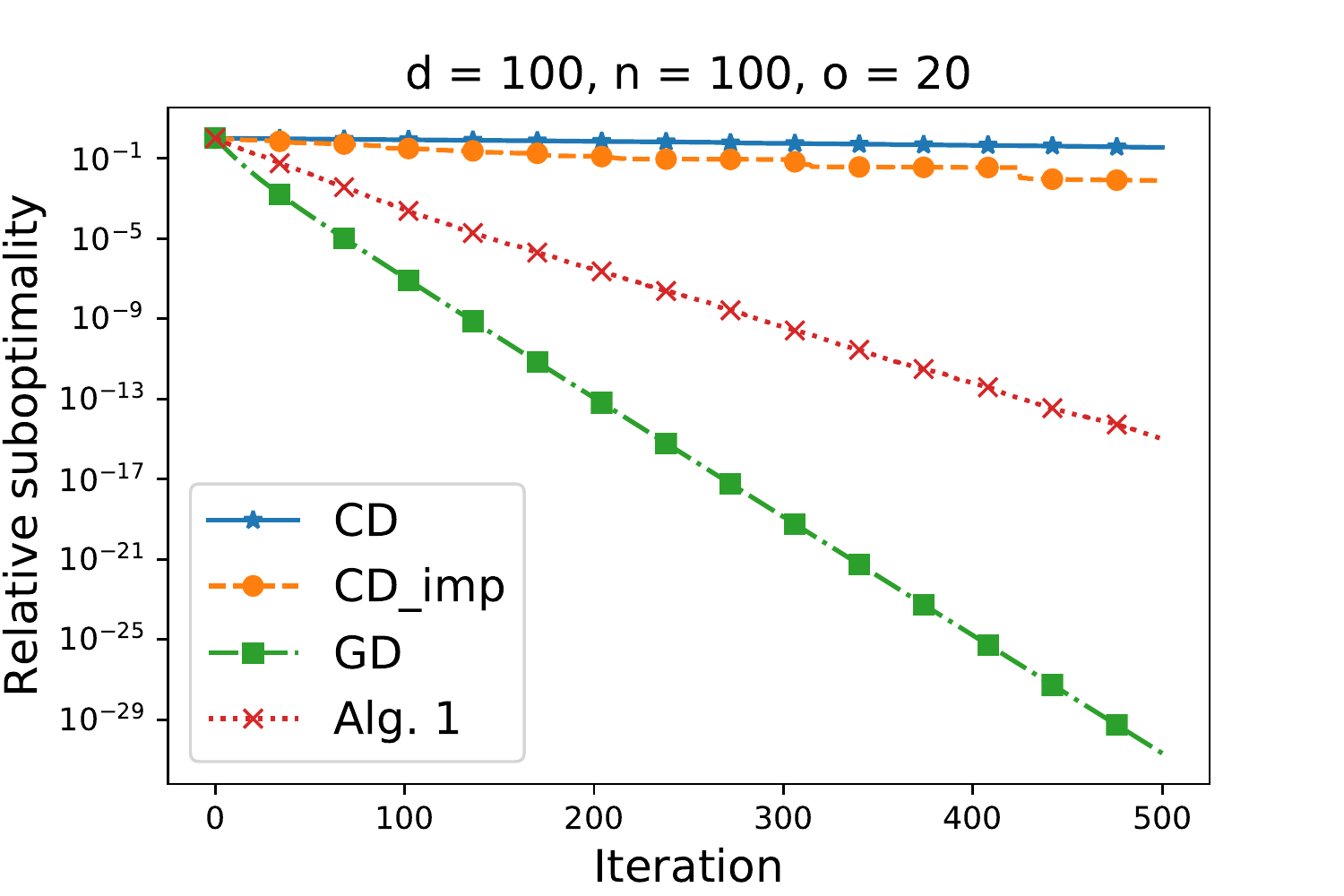}
\end{minipage}%
\begin{minipage}{0.24\textwidth}
  \centering
\includegraphics[width =  \textwidth ]{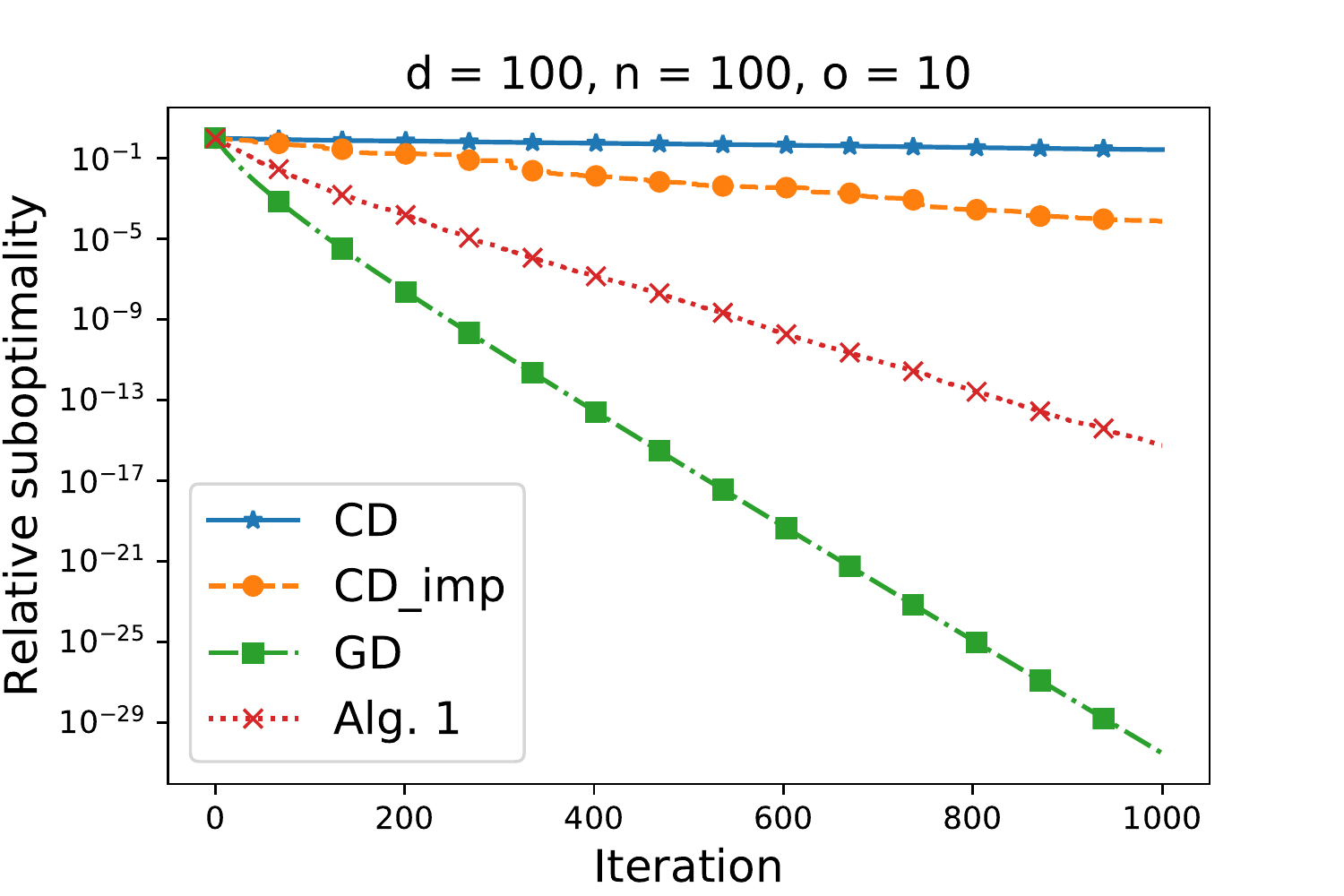}
\end{minipage}%
\begin{minipage}{0.24\textwidth}
  \centering
\includegraphics[width =  \textwidth ]{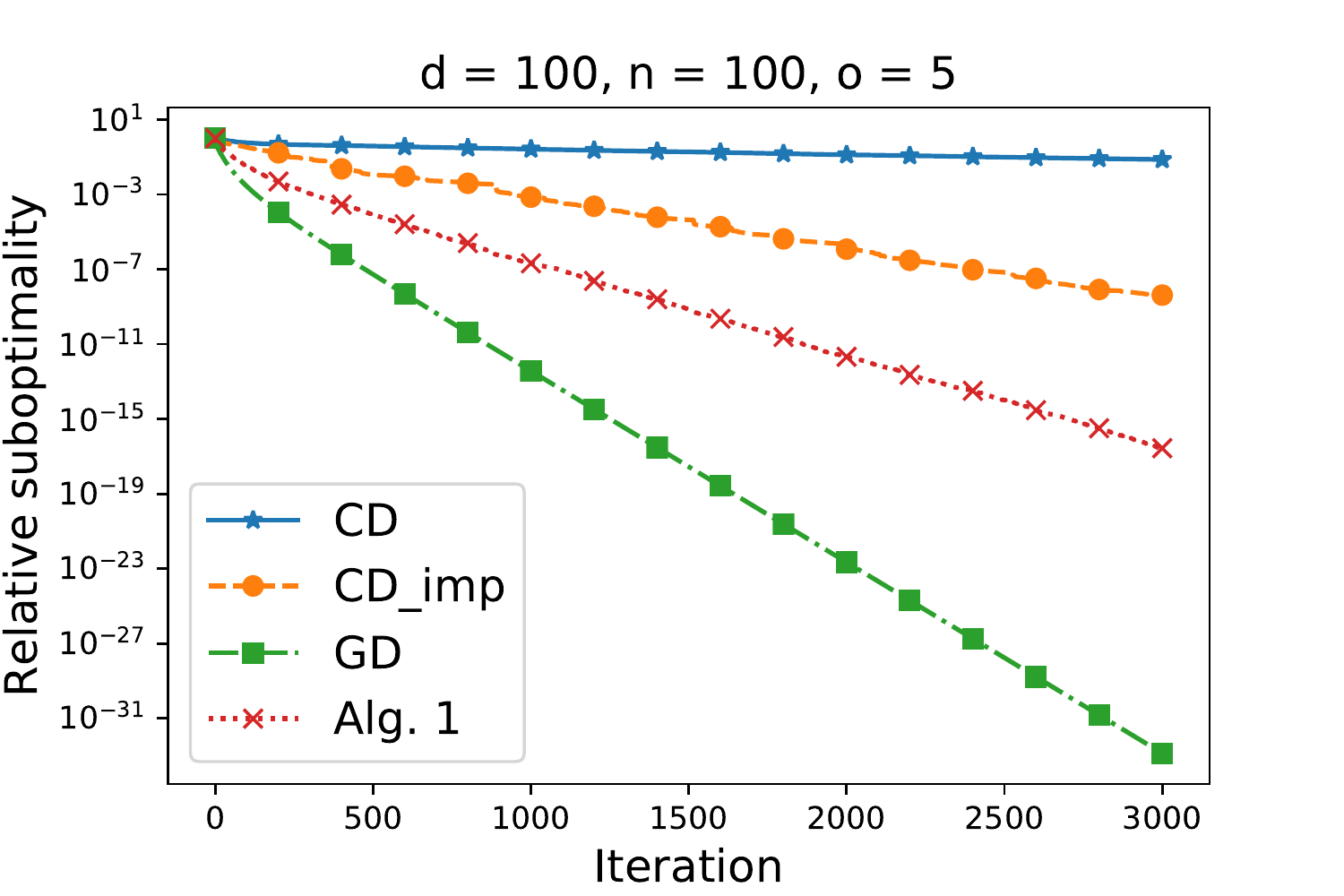}
\end{minipage}%
\begin{minipage}{0.24\textwidth}
  \centering
\includegraphics[width =  \textwidth ]{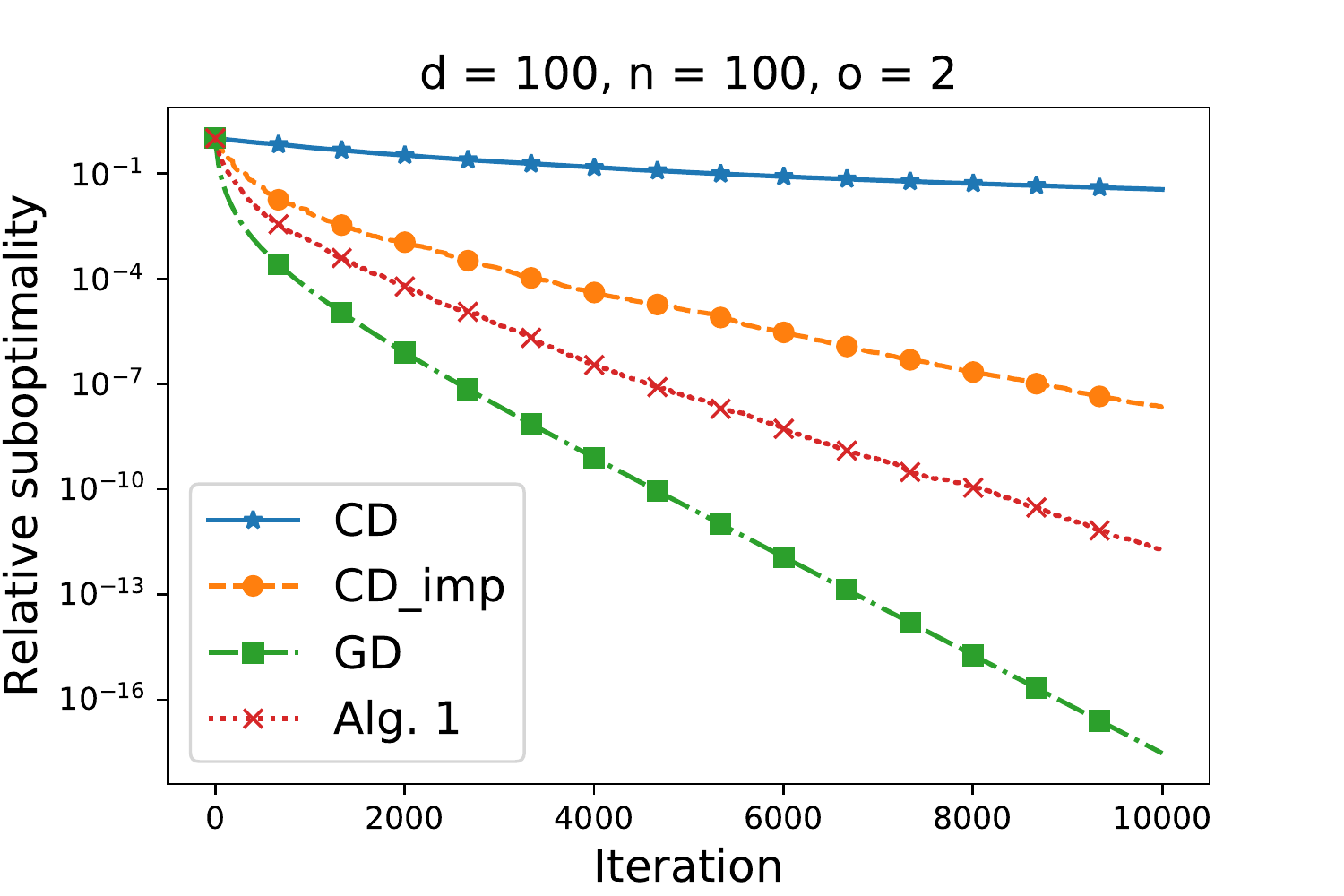}
\end{minipage}%
\\
\begin{minipage}{0.24\textwidth}
  \centering
\includegraphics[width =  \textwidth ]{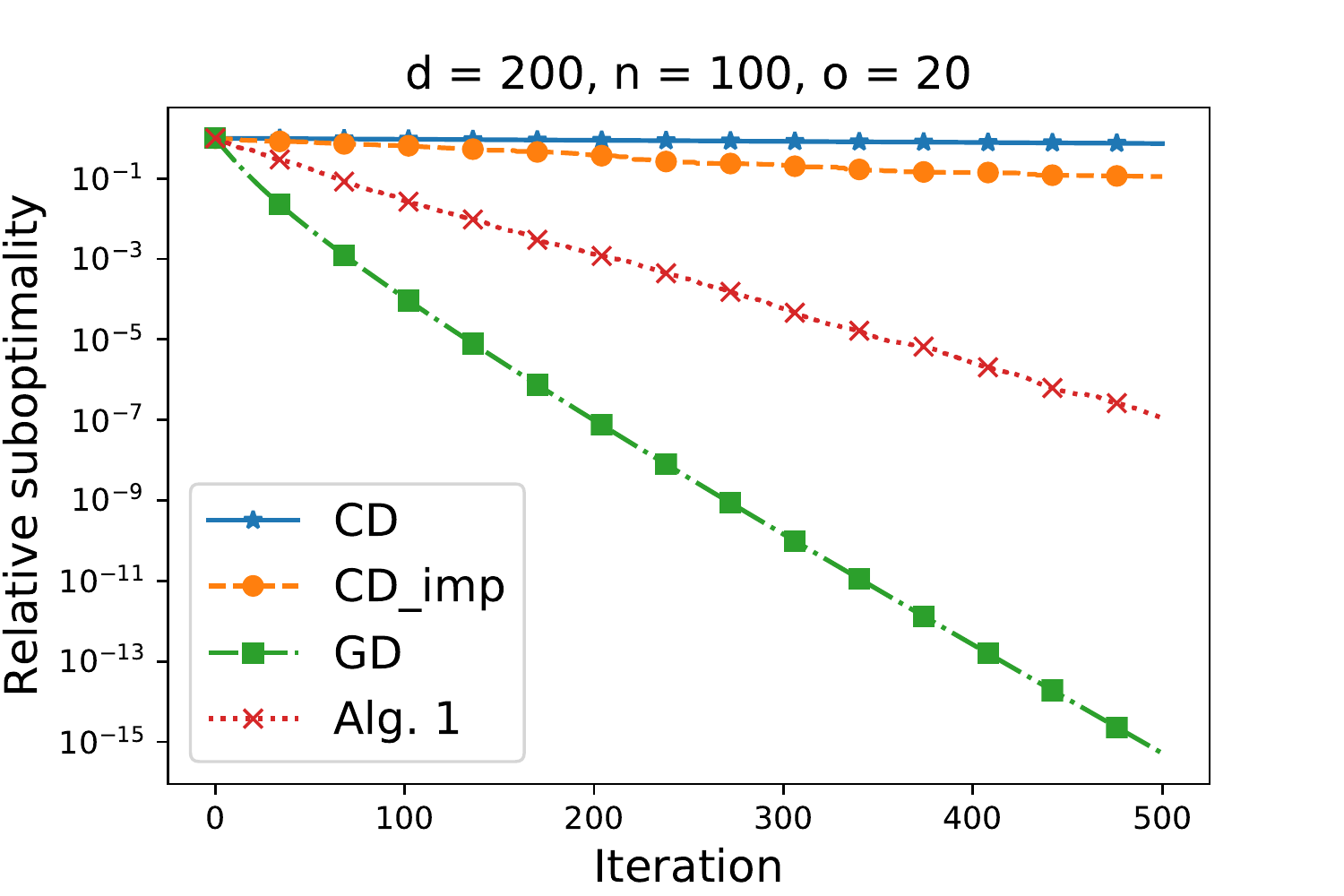}
\end{minipage}%
\begin{minipage}{0.24\textwidth}
  \centering
\includegraphics[width =  \textwidth ]{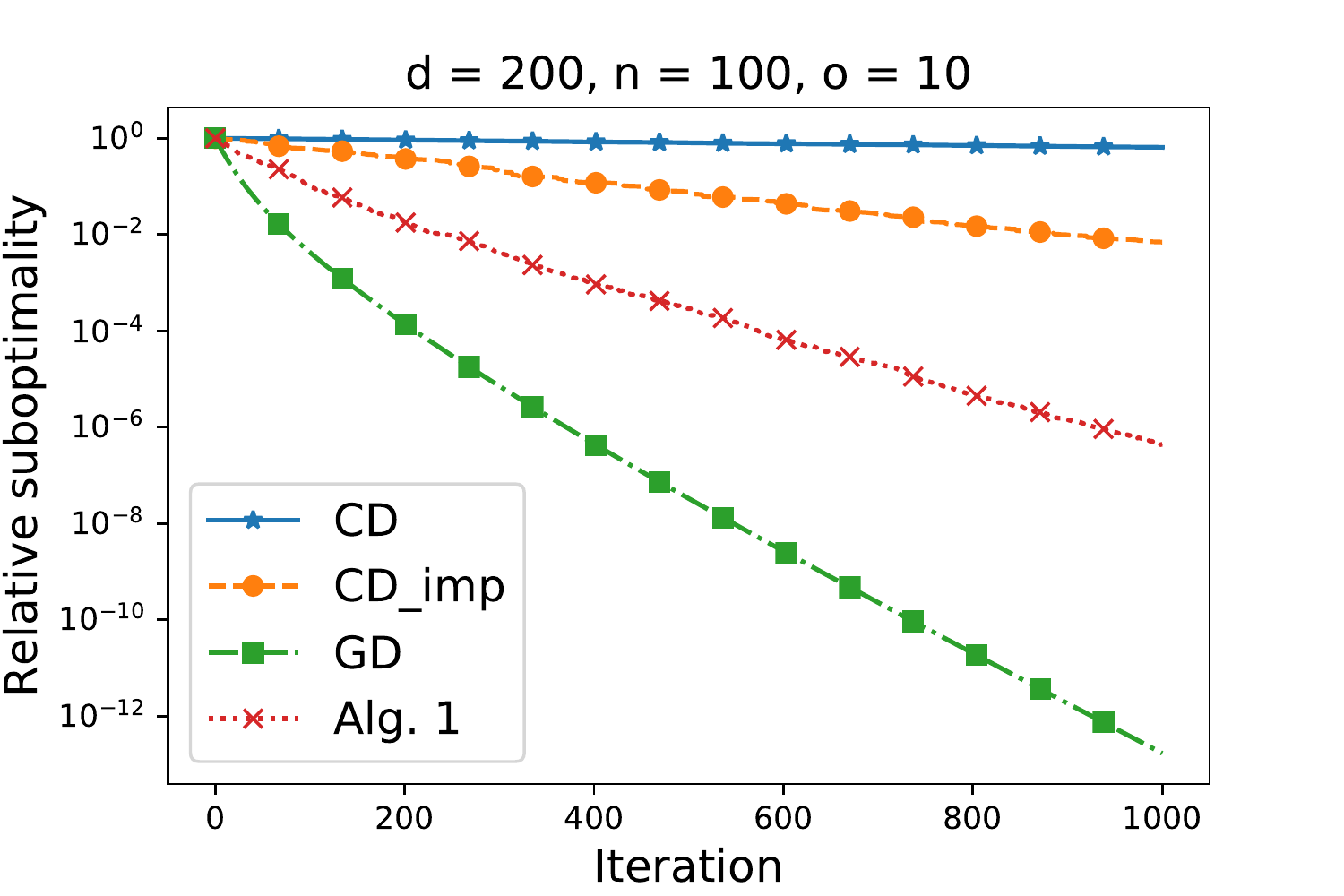}
\end{minipage}%
\begin{minipage}{0.24\textwidth}
  \centering
\includegraphics[width =  \textwidth ]{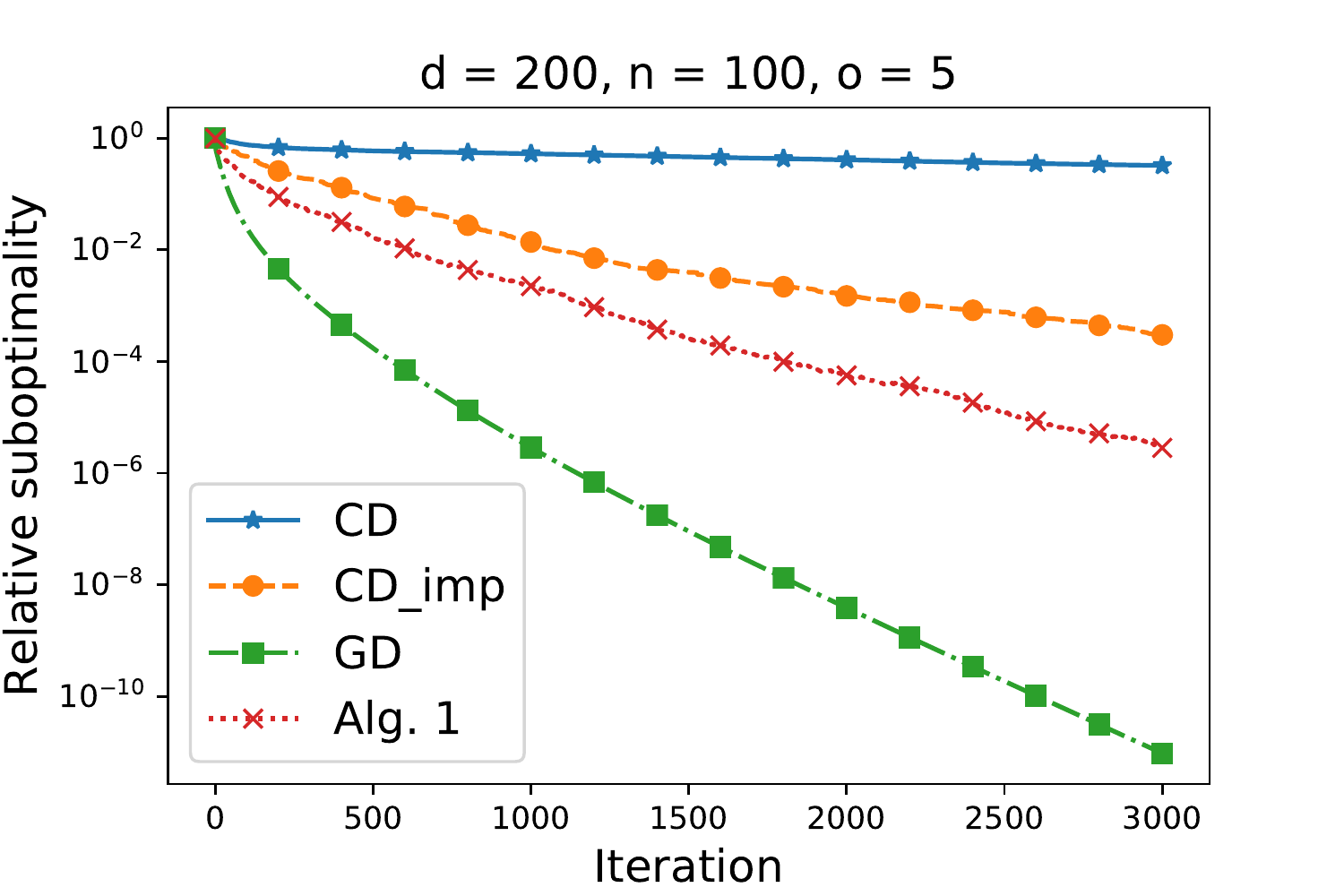}
\end{minipage}%
\begin{minipage}{0.24\textwidth}
  \centering
\includegraphics[width =  \textwidth ]{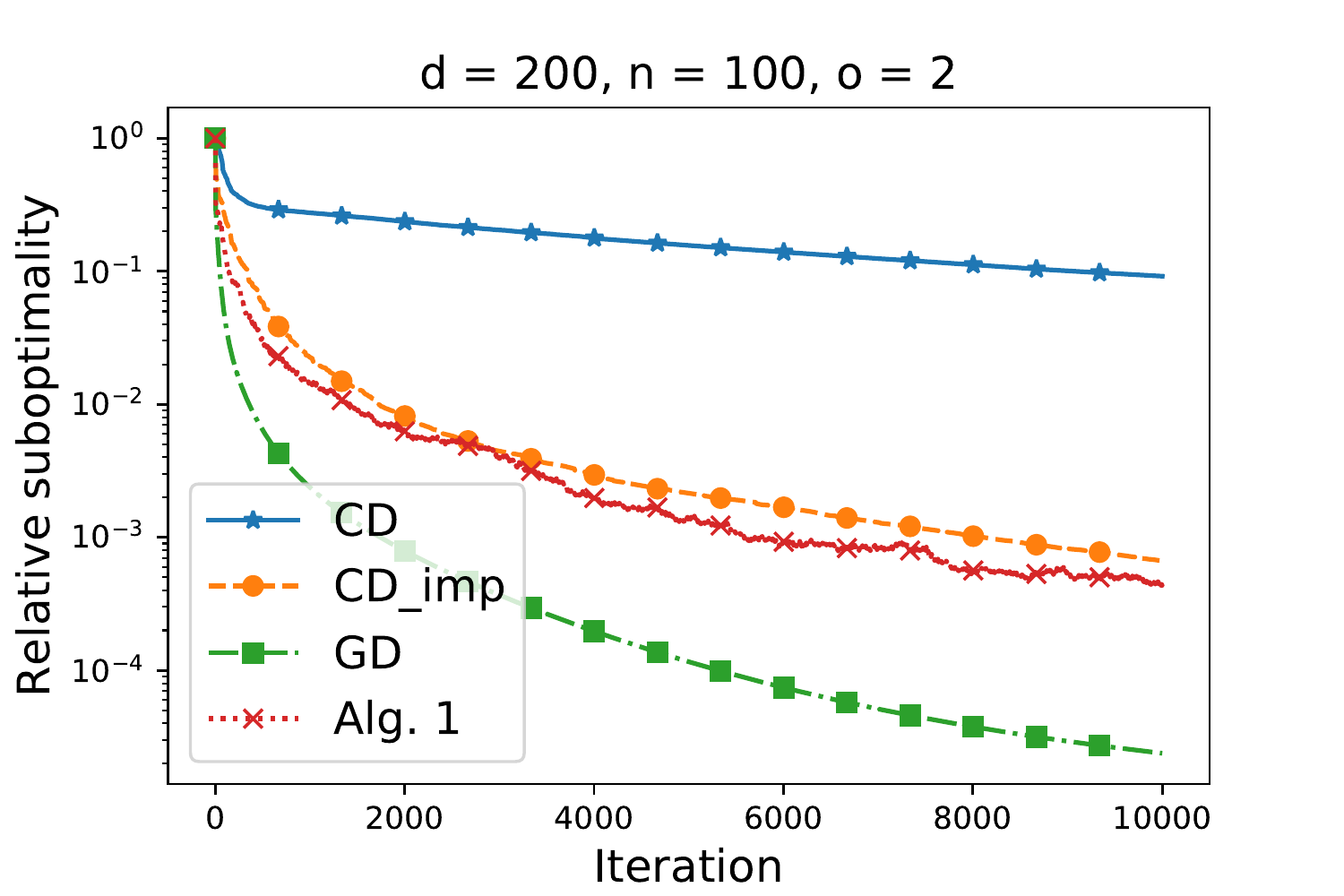}
\end{minipage}%
\\
\caption{Comparison of gradient descent, (standard) coordinate descent, (standard) coordinate descent with importance sampling  and Algorithm~\ref{alg:cd} on artificial quadratic problem~\eqref{eq:quadratic}.}\label{fig:artif_1}
\end{figure}

Next, we study the effect of changing $\tau$ on the iteration complexity of Algorithm~\ref{alg:cd}. Figure~\ref{fig:artif_2} provides the result. The behavior predicted from theory is observed --  increasing  $\tau$ over $n^{-1}$ does not significantly improve the convergence speed, while decreasing it below $n^{-1}$ slows the algorithm notably.

\begin{figure}[H]
\centering
\begin{minipage}{0.24\textwidth}
  \centering
\includegraphics[width =  \textwidth ]{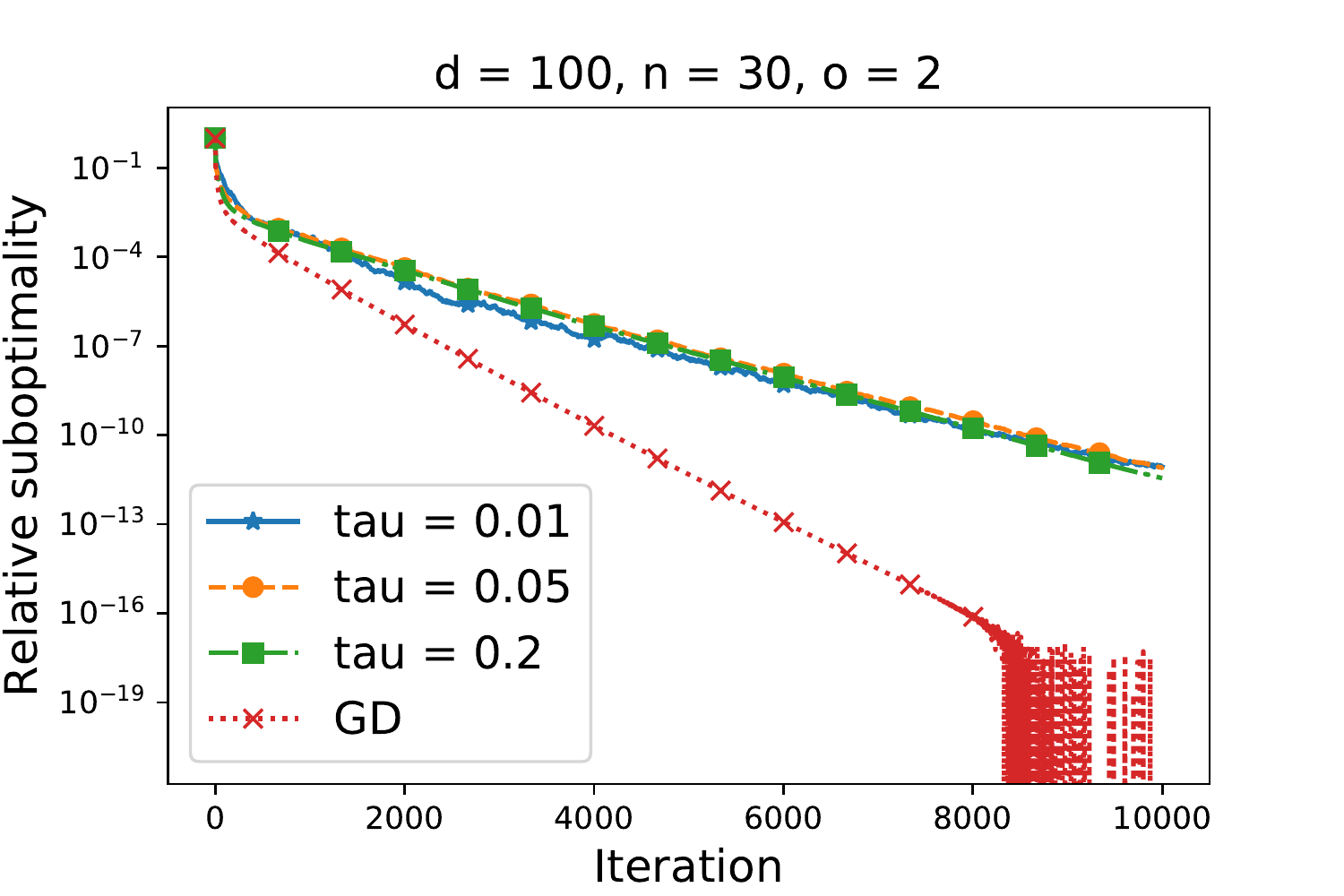}
\end{minipage}%
\begin{minipage}{0.24\textwidth}
  \centering
\includegraphics[width =  \textwidth ]{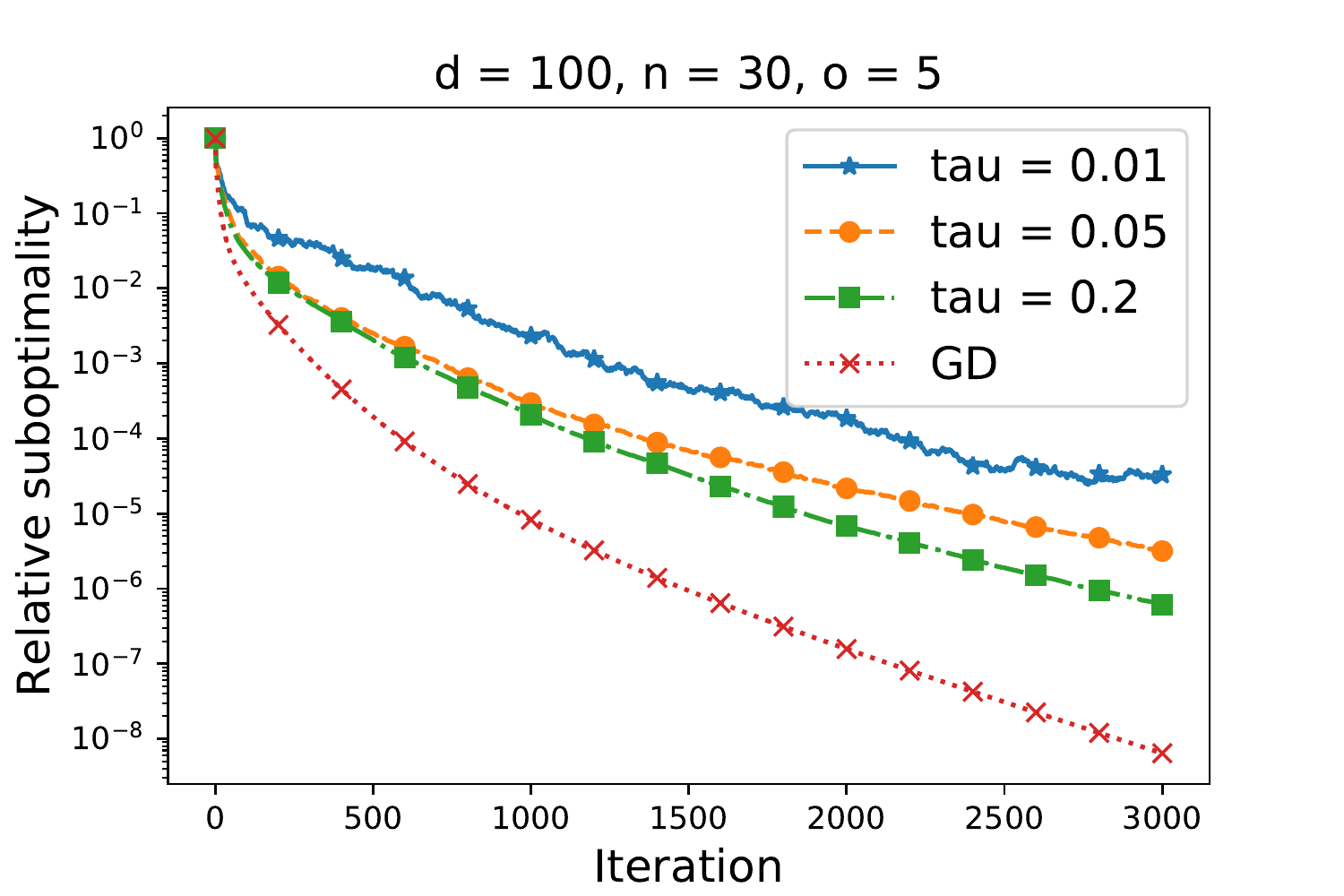}
\end{minipage}%
\begin{minipage}{0.24\textwidth}
  \centering
\includegraphics[width =  \textwidth ]{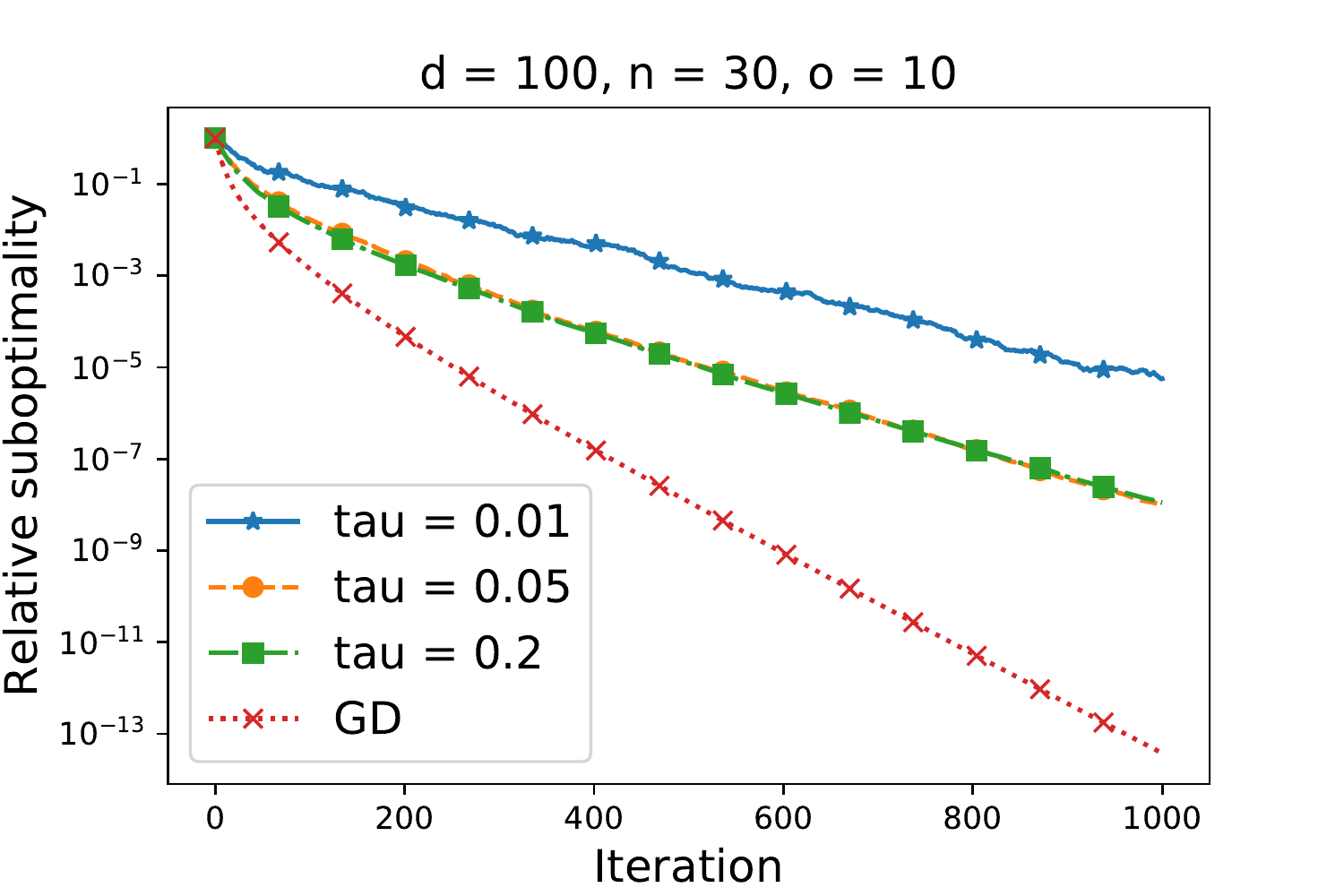}
\end{minipage}%
\begin{minipage}{0.24\textwidth}
  \centering
\includegraphics[width =  \textwidth ]{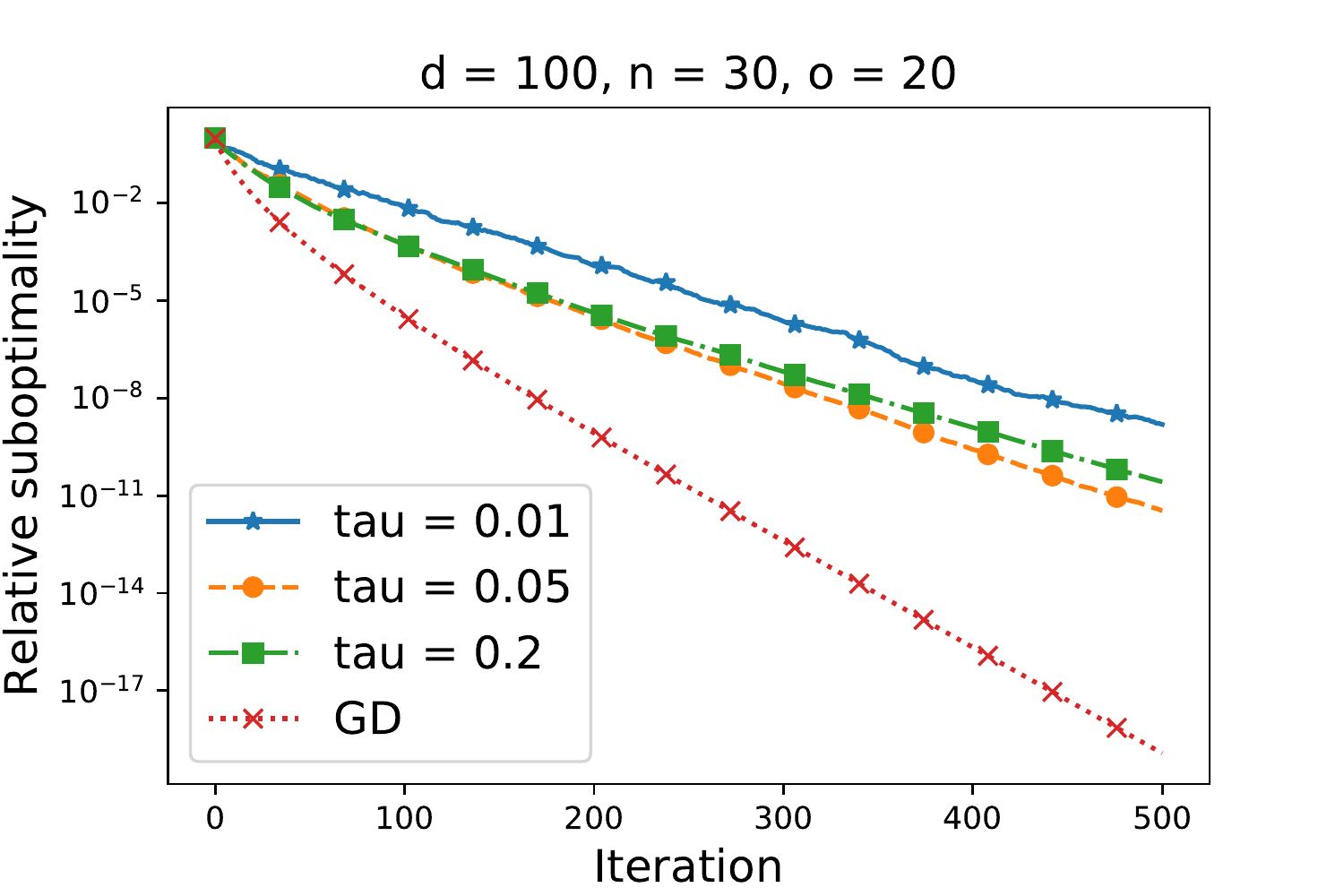}
\end{minipage}%
\\
\begin{minipage}{0.24\textwidth}
  \centering
\includegraphics[width =  \textwidth ]{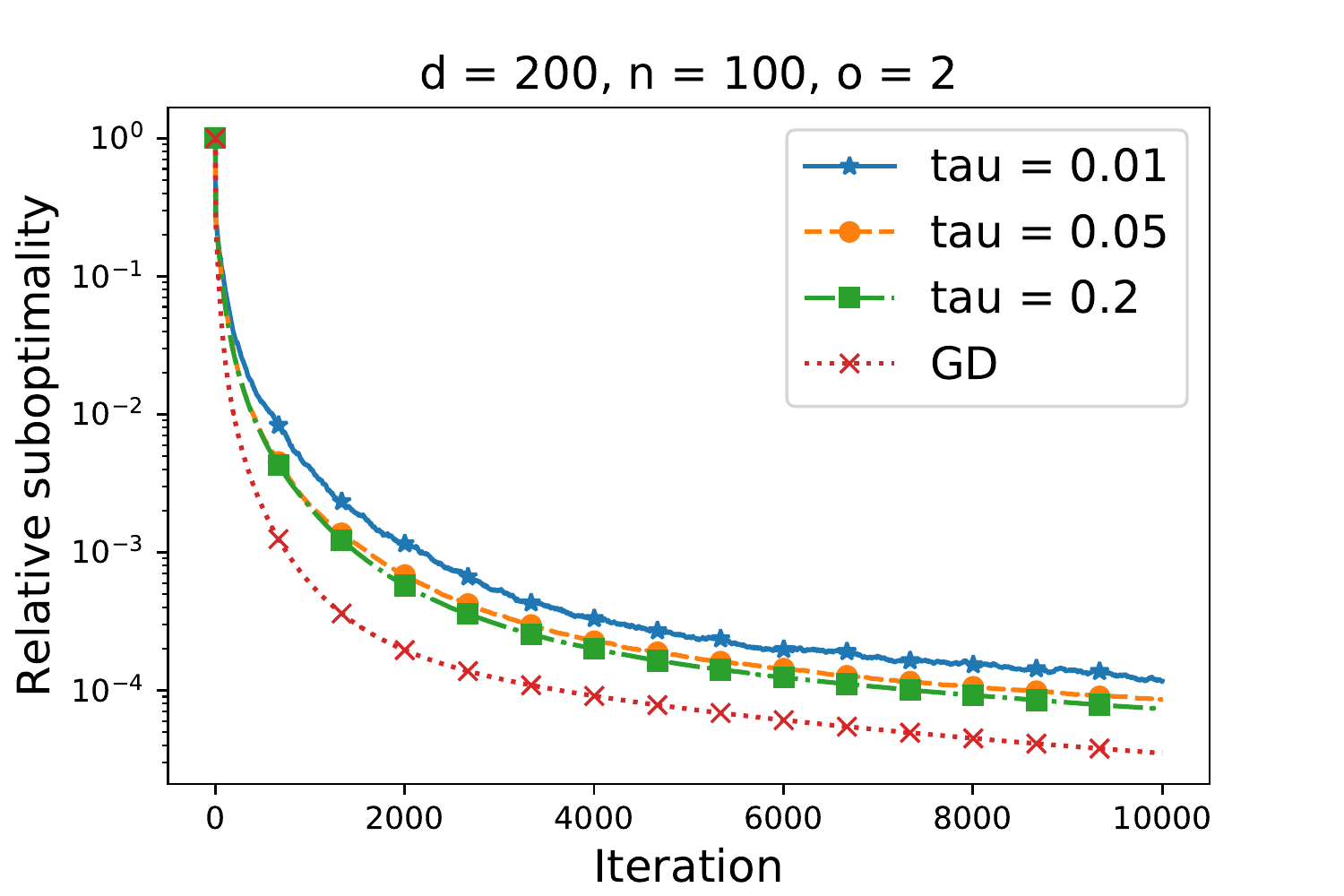}
\end{minipage}%
\begin{minipage}{0.24\textwidth}
  \centering
\includegraphics[width =  \textwidth ]{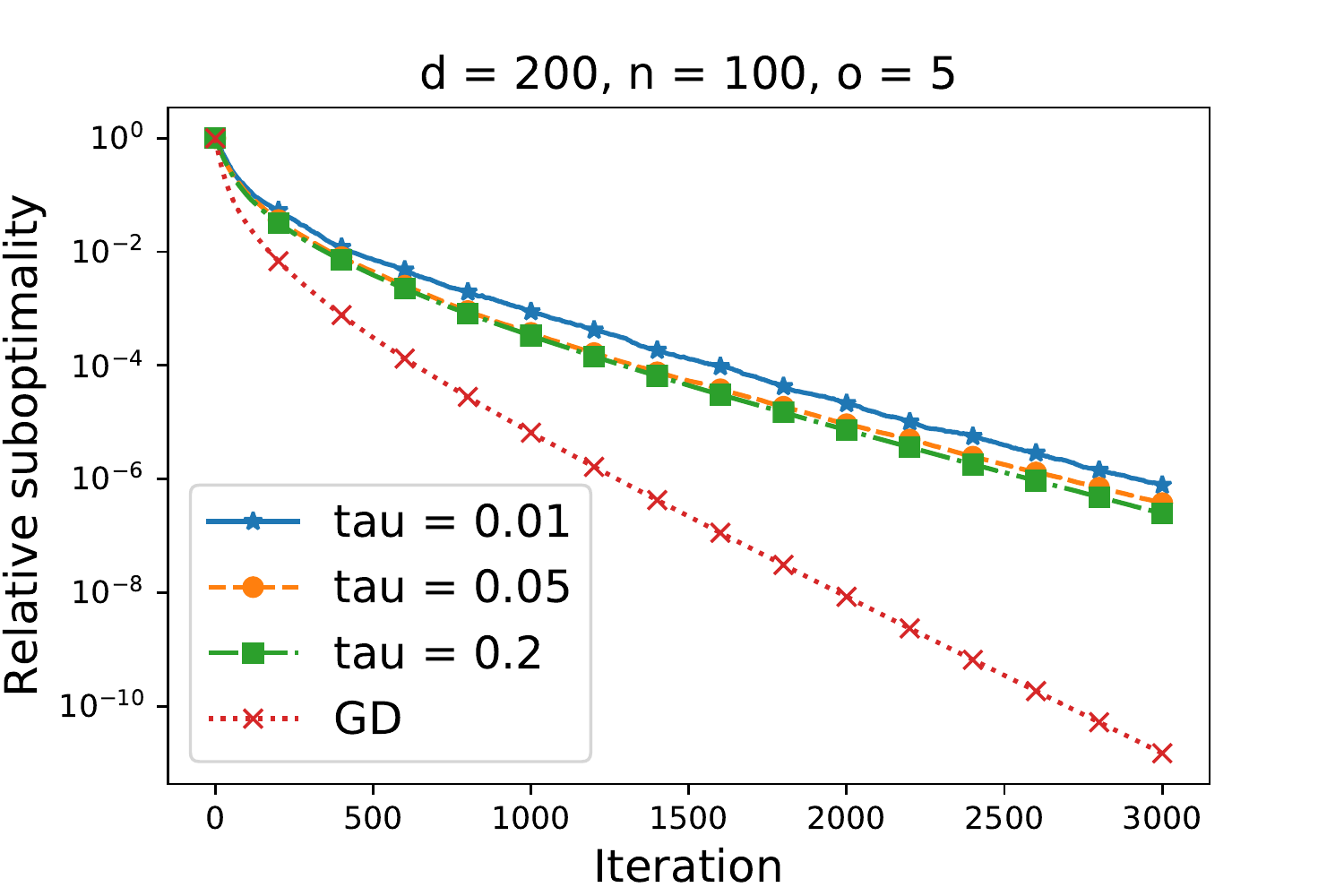}
\end{minipage}%
\begin{minipage}{0.24\textwidth}
  \centering
\includegraphics[width =  \textwidth ]{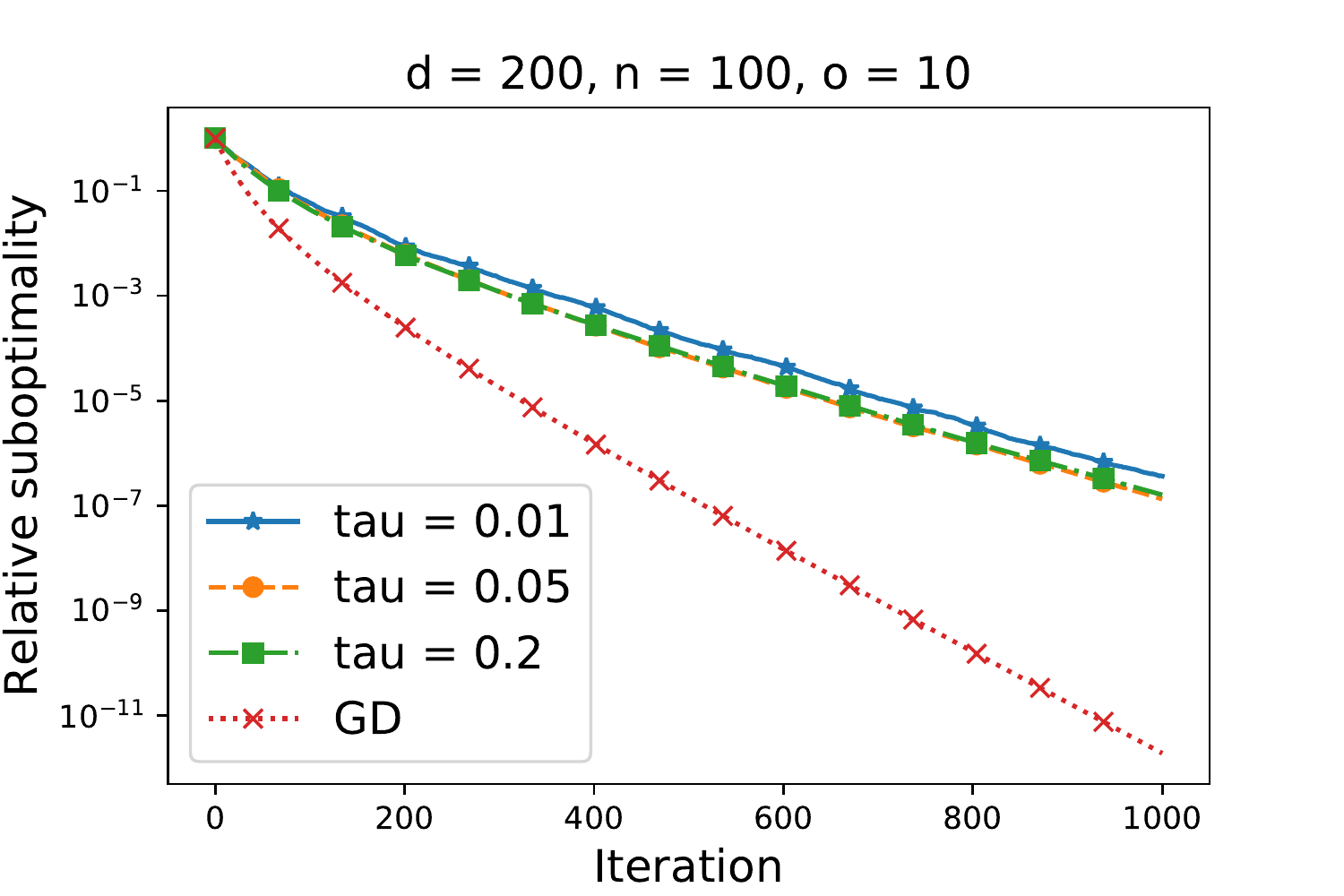}
\end{minipage}%
\begin{minipage}{0.24\textwidth}
  \centering
\includegraphics[width =  \textwidth ]{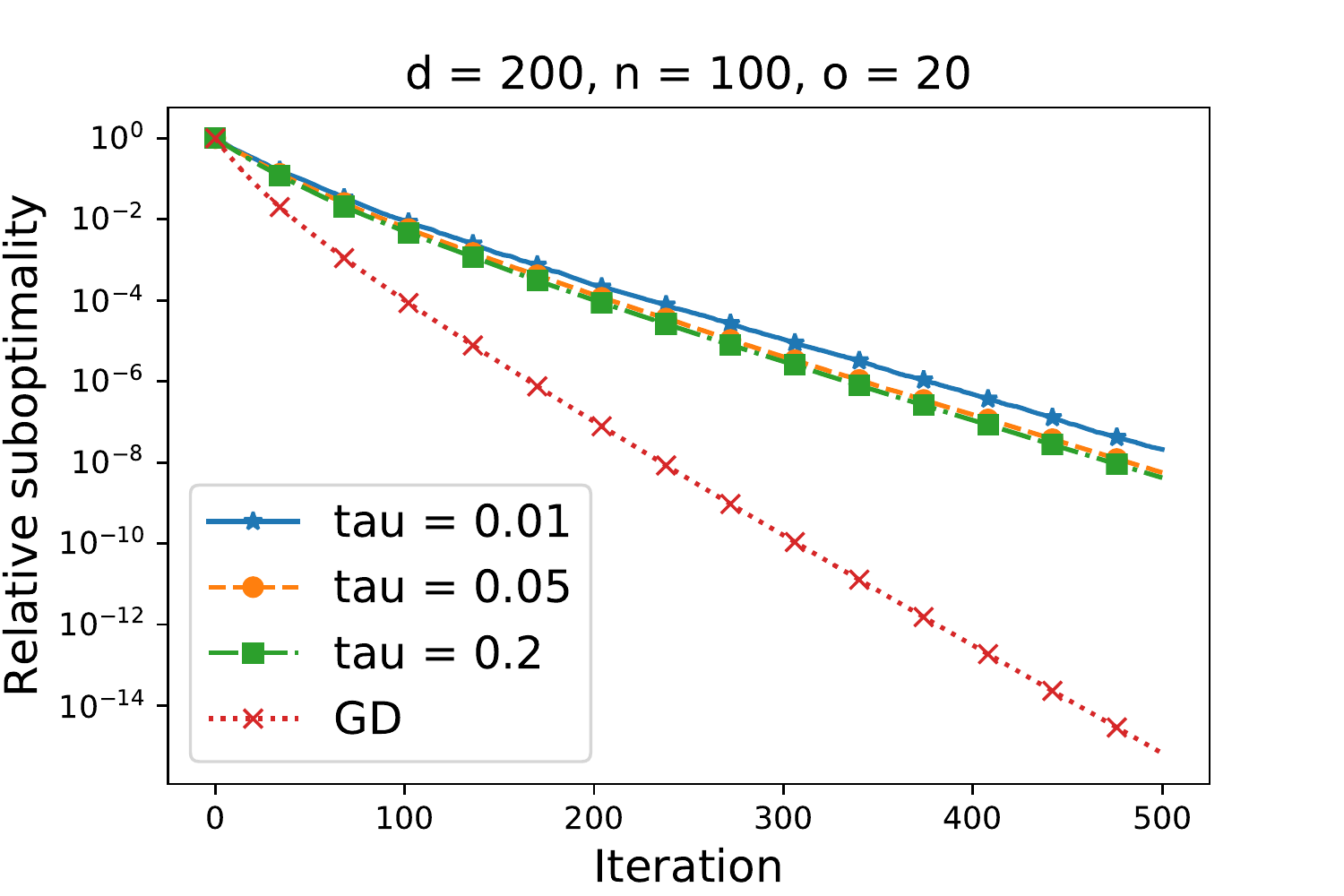}
\end{minipage}%
\\
\begin{minipage}{0.24\textwidth}
  \centering
\includegraphics[width =  \textwidth ]{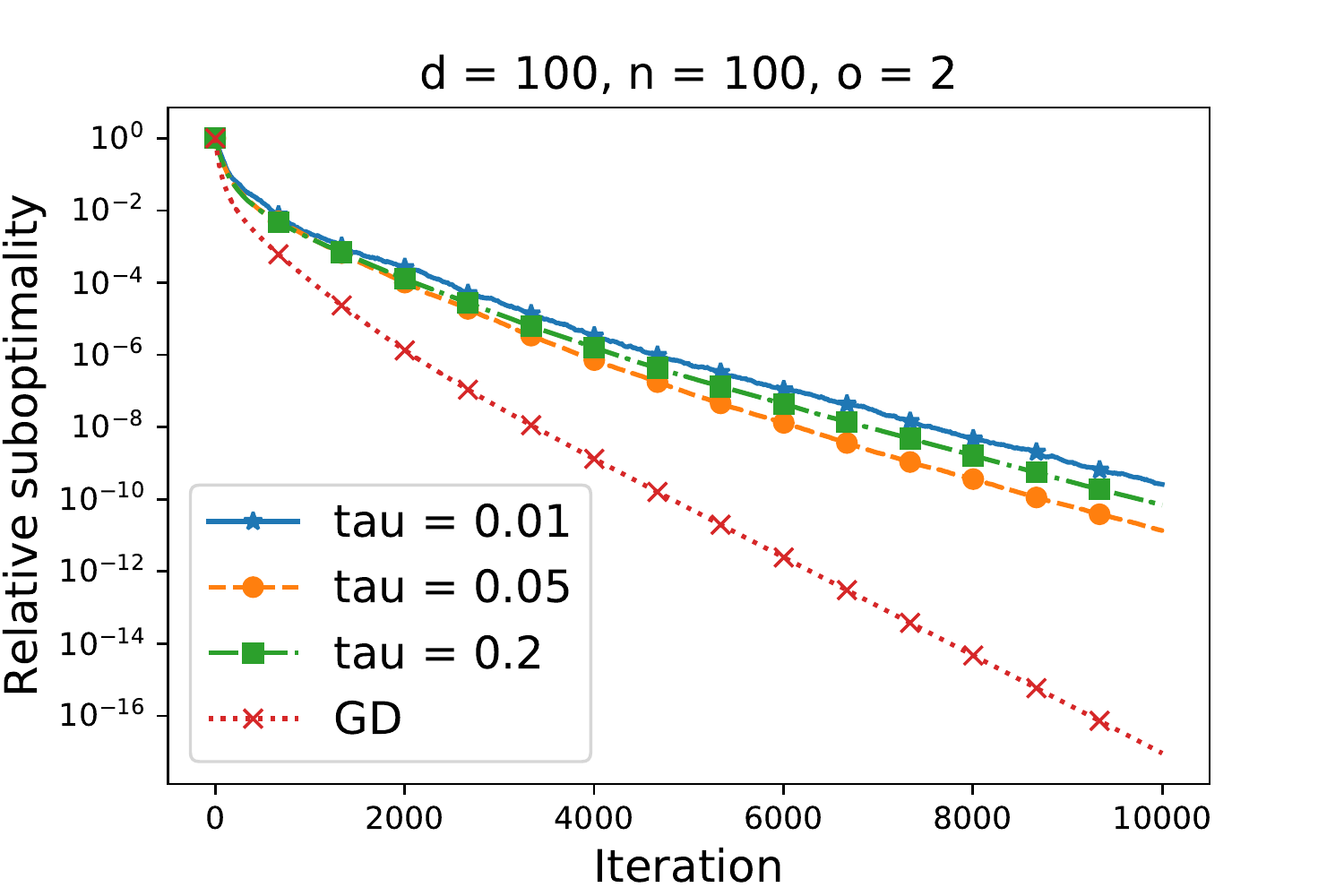}
\end{minipage}%
\begin{minipage}{0.24\textwidth}
  \centering
\includegraphics[width =  \textwidth ]{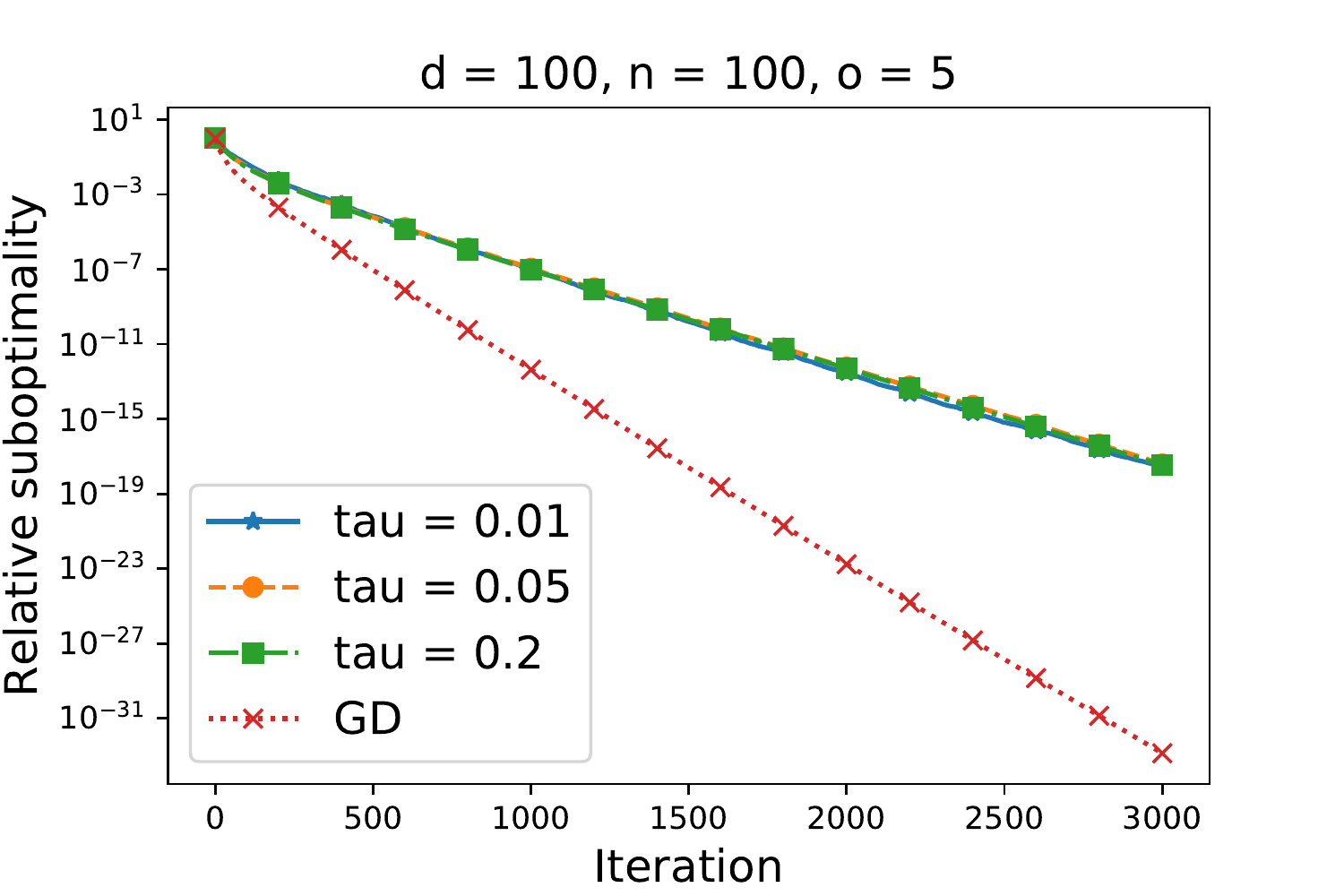}
\end{minipage}%
\begin{minipage}{0.24\textwidth}
  \centering
\includegraphics[width =  \textwidth ]{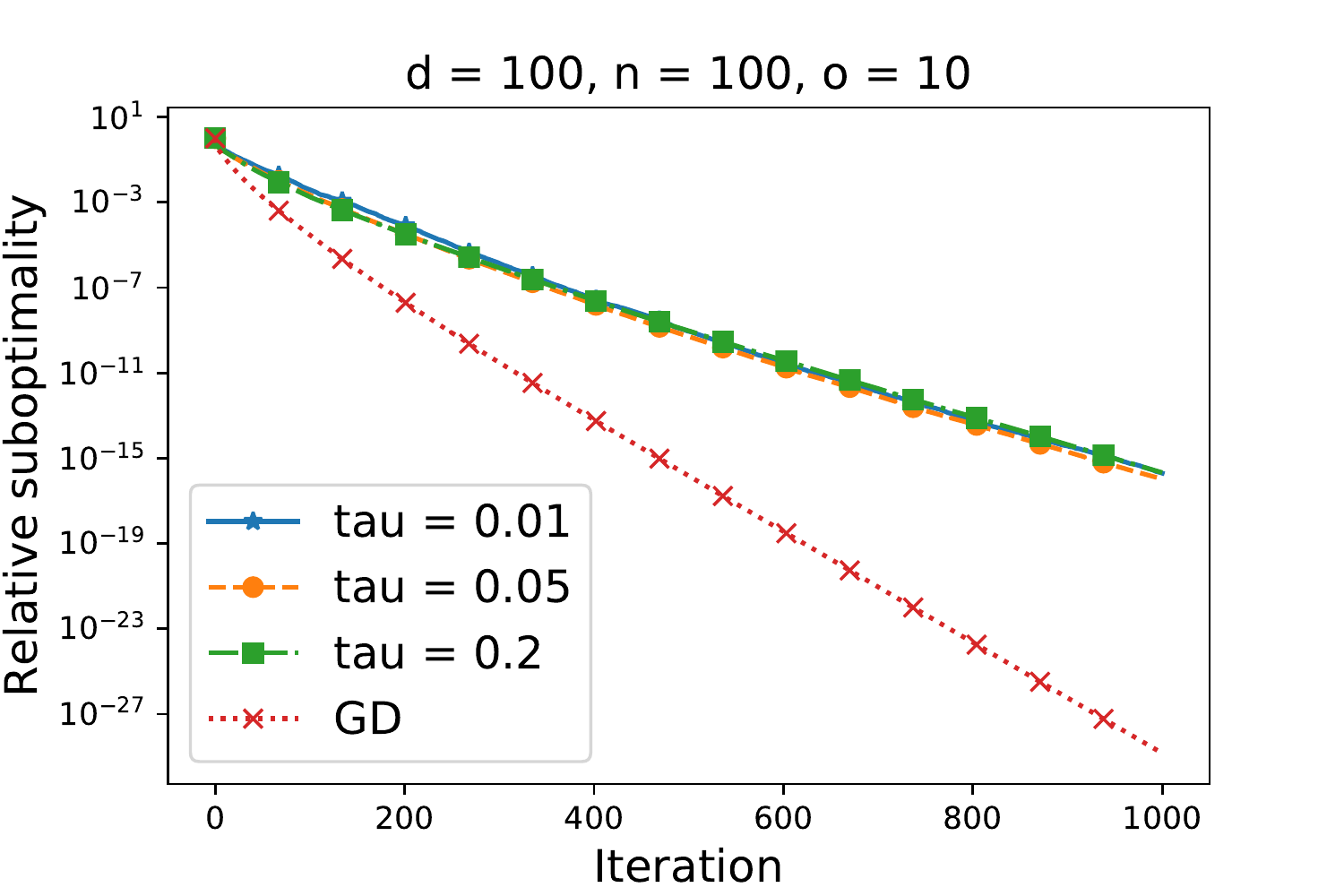}
\end{minipage}%
\begin{minipage}{0.24\textwidth}
  \centering
\includegraphics[width =  \textwidth ]{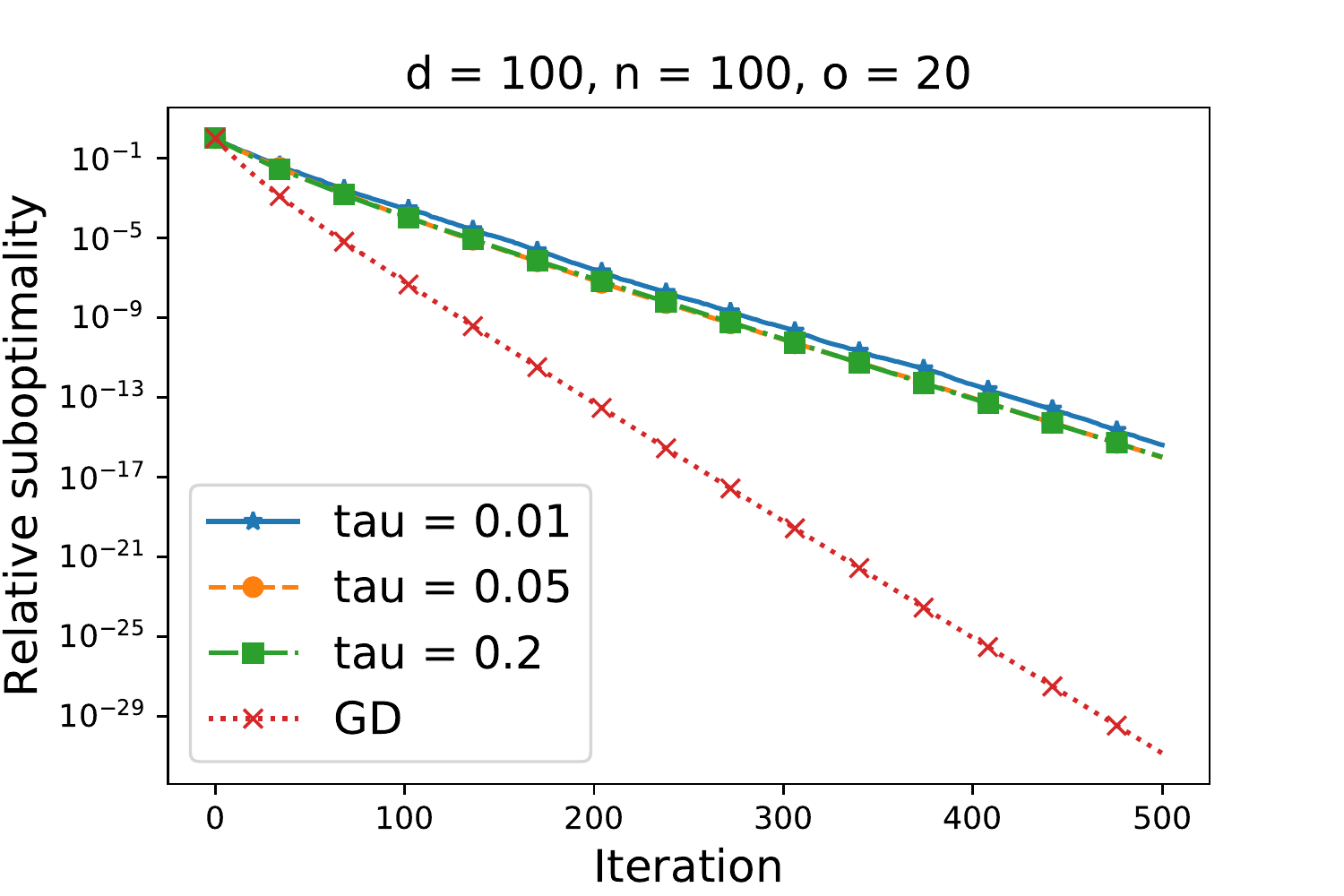}
\end{minipage}%
\\
\caption{Behavior of Algorithm~\ref{alg:cd} for different $\tau$ on a simple artificial quadratic problem~\eqref{eq:quadratic}.}\label{fig:artif_2}
\end{figure}

\subsection{ISGD \label{sec:exp_sgd}}
In this section we numerically test Algorithm~\ref{alg:sgd} for logistic regression problem. As mentioned, $f_i$ consists of set of (uniformly distributed) rows of $A$ from~\eqref{eq:logreg}. We consider the most natural unbiased stochastic oracle for the $\nabla f_i$ -- gradient computed on a subset data points from $f_i$. 

In all experiments of this section, we consider constant step sizes in order to keep the setting as simple as possible and gain as much insight from the experiments as possible. Therefore, one can not expect convergence to the exact optimum. 

In the first experiment, we compare standard SGD (stochastic gradient is computed on single, randomly chosen datapoint every iteration) against Algorithm~\ref{alg:sgd} varying $n$ and choosing $\tau=\frac{1}{n}$ for each $n$. The results are presented by Figure~\ref{fig:sgd1}. We see that, as our theory suggests, SGD and Algorithm~\ref{alg:sgd} have always very similar performance.

\begin{figure}[H]
\centering
\begin{minipage}{0.33\textwidth}
  \centering
\includegraphics[width =  \textwidth ]{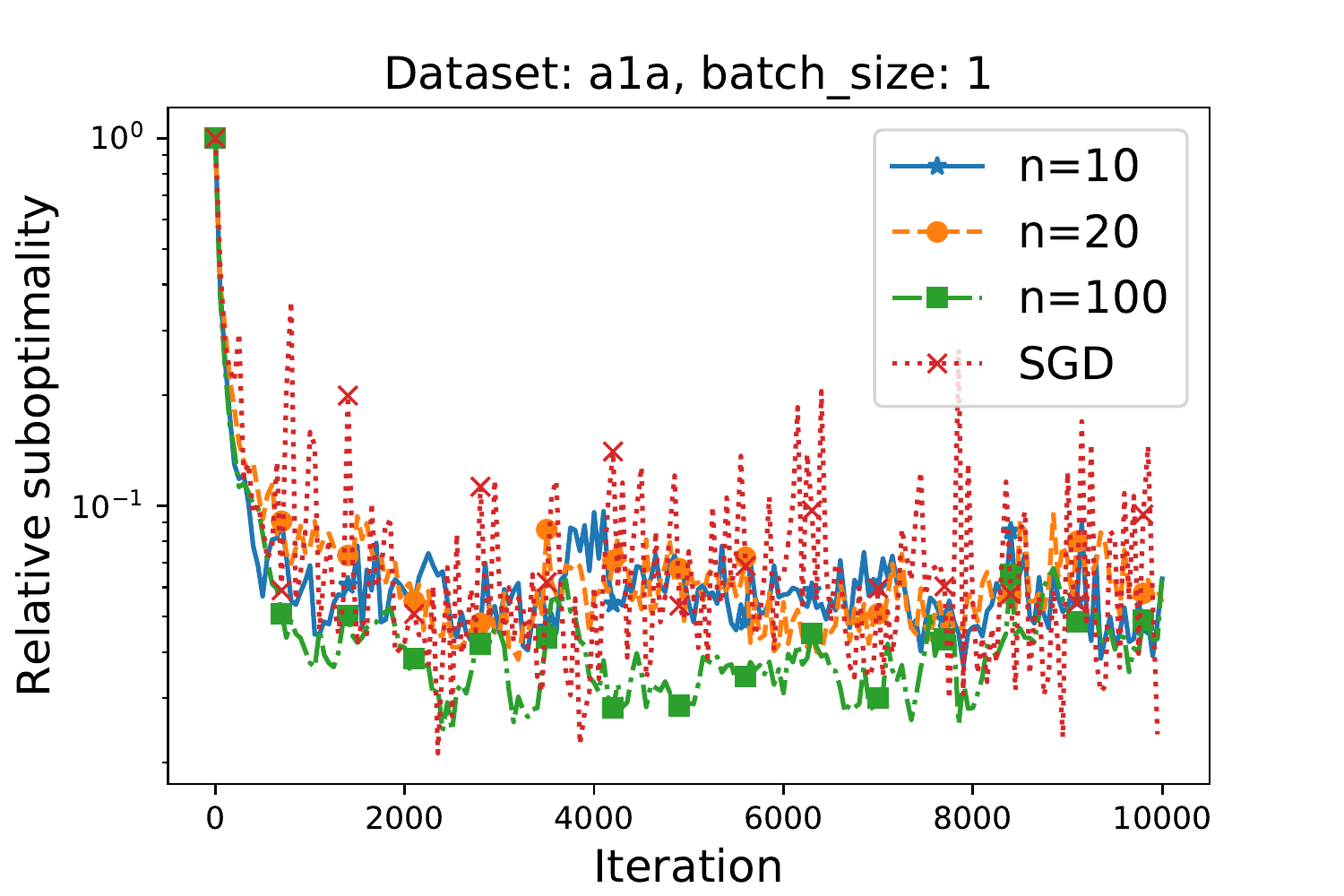}
\end{minipage}%
\begin{minipage}{0.33\textwidth}
  \centering
\includegraphics[width =  \textwidth ]{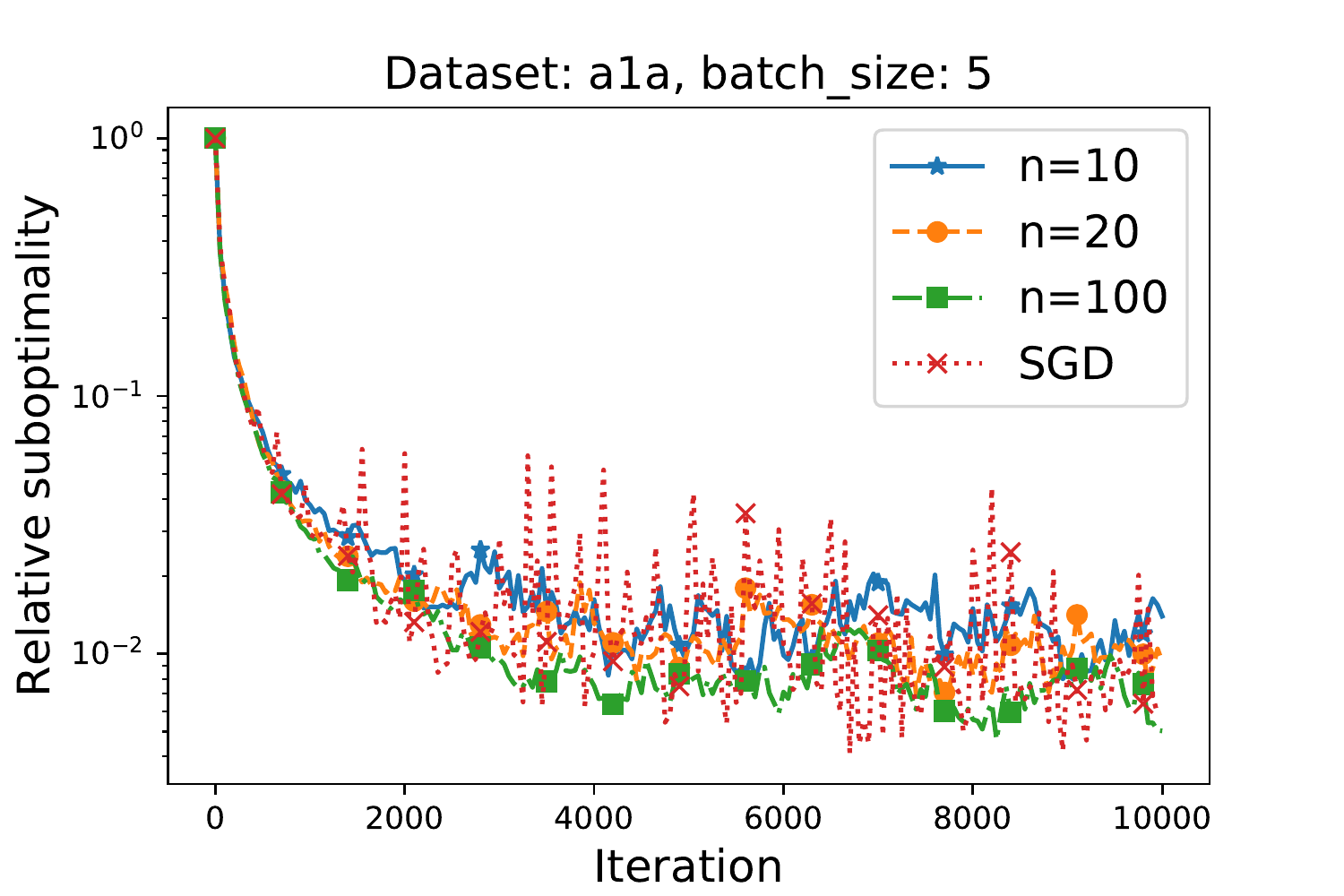}
\end{minipage}%
\begin{minipage}{0.33\textwidth}
  \centering
\includegraphics[width =  \textwidth ]{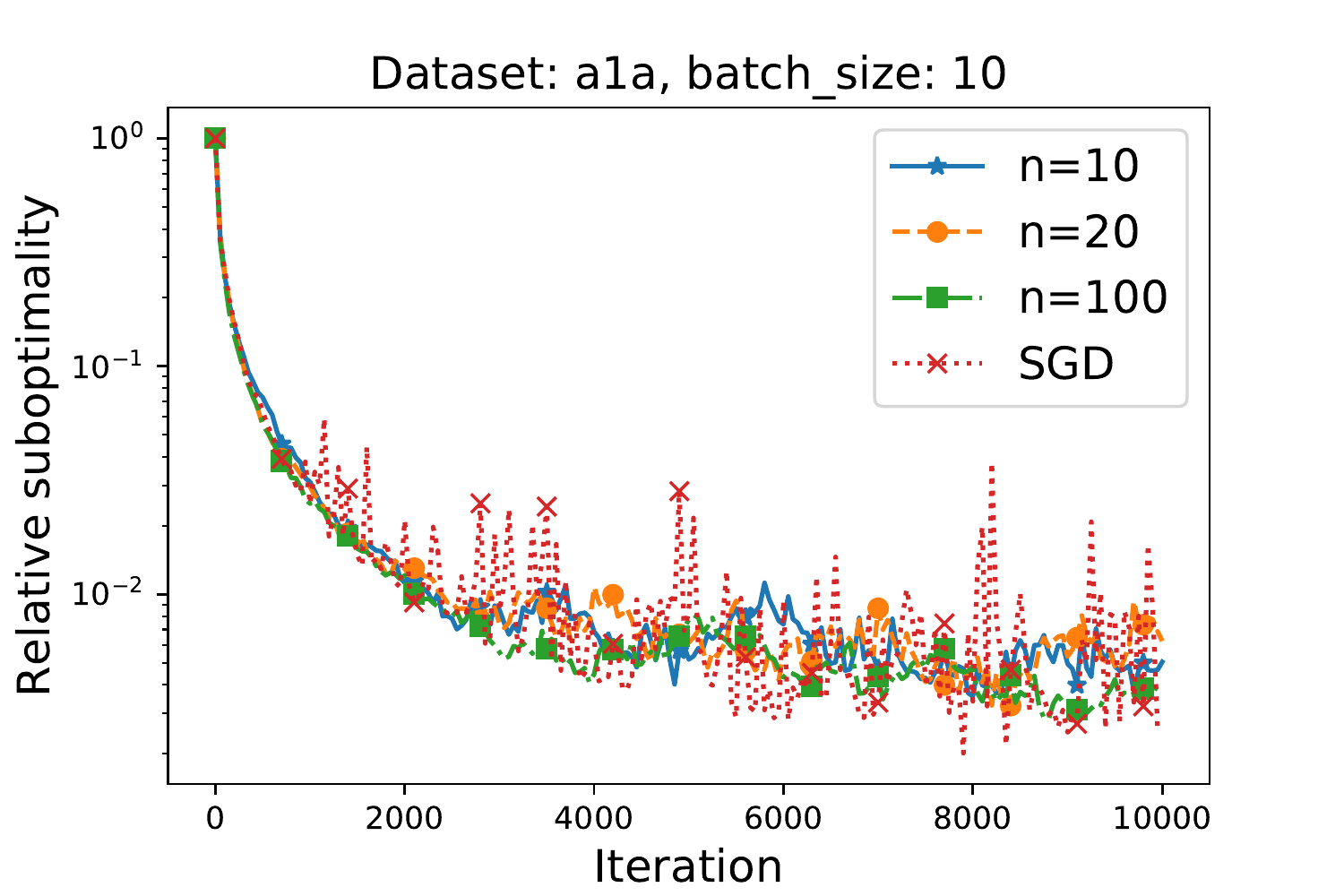}
\end{minipage}%
\\
\begin{minipage}{0.33\textwidth}
  \centering
\includegraphics[width =  \textwidth ]{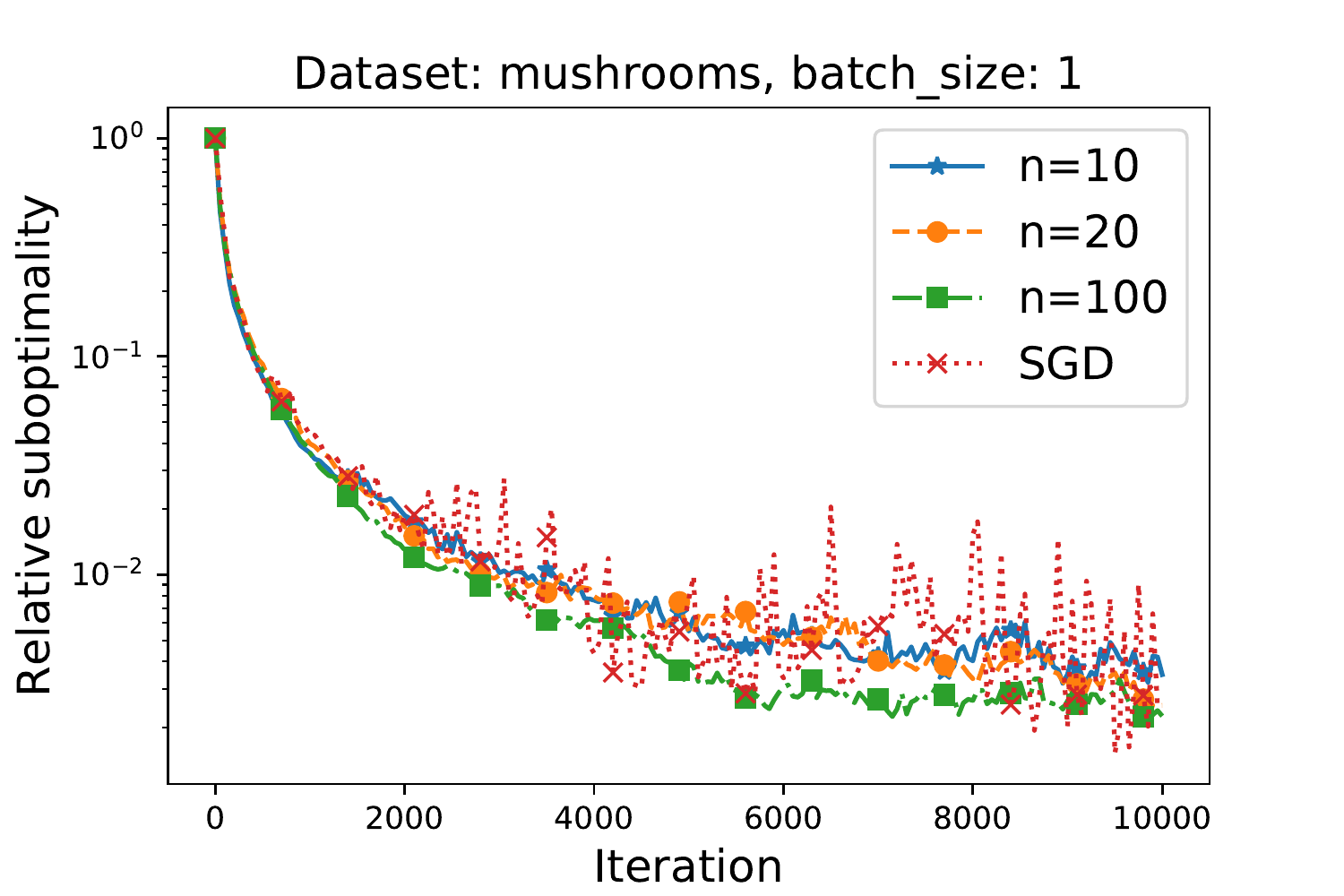}
\end{minipage}%
\begin{minipage}{0.33\textwidth}
  \centering
\includegraphics[width =  \textwidth ]{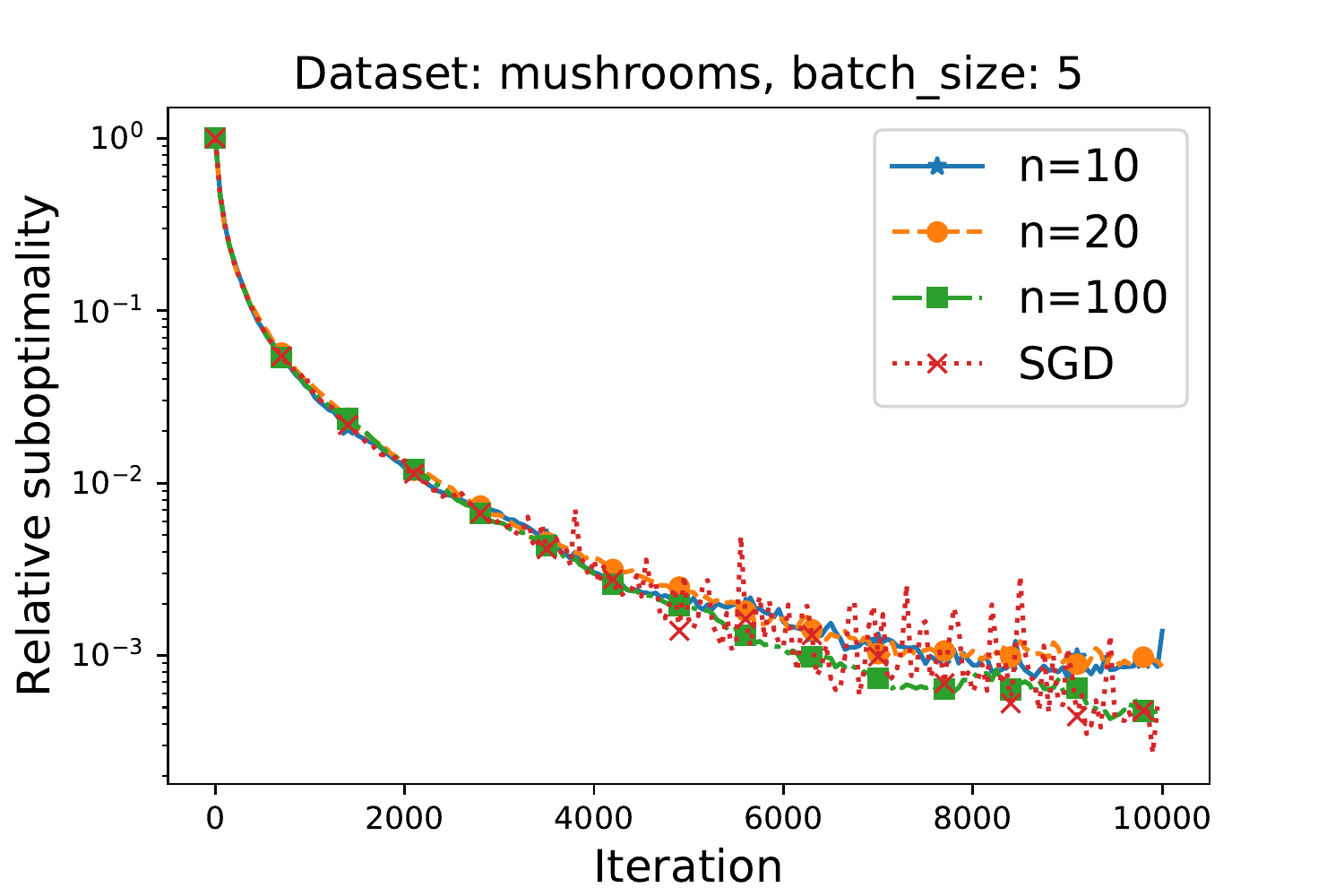}
\end{minipage}%
\begin{minipage}{0.33\textwidth}
  \centering
\includegraphics[width =  \textwidth ]{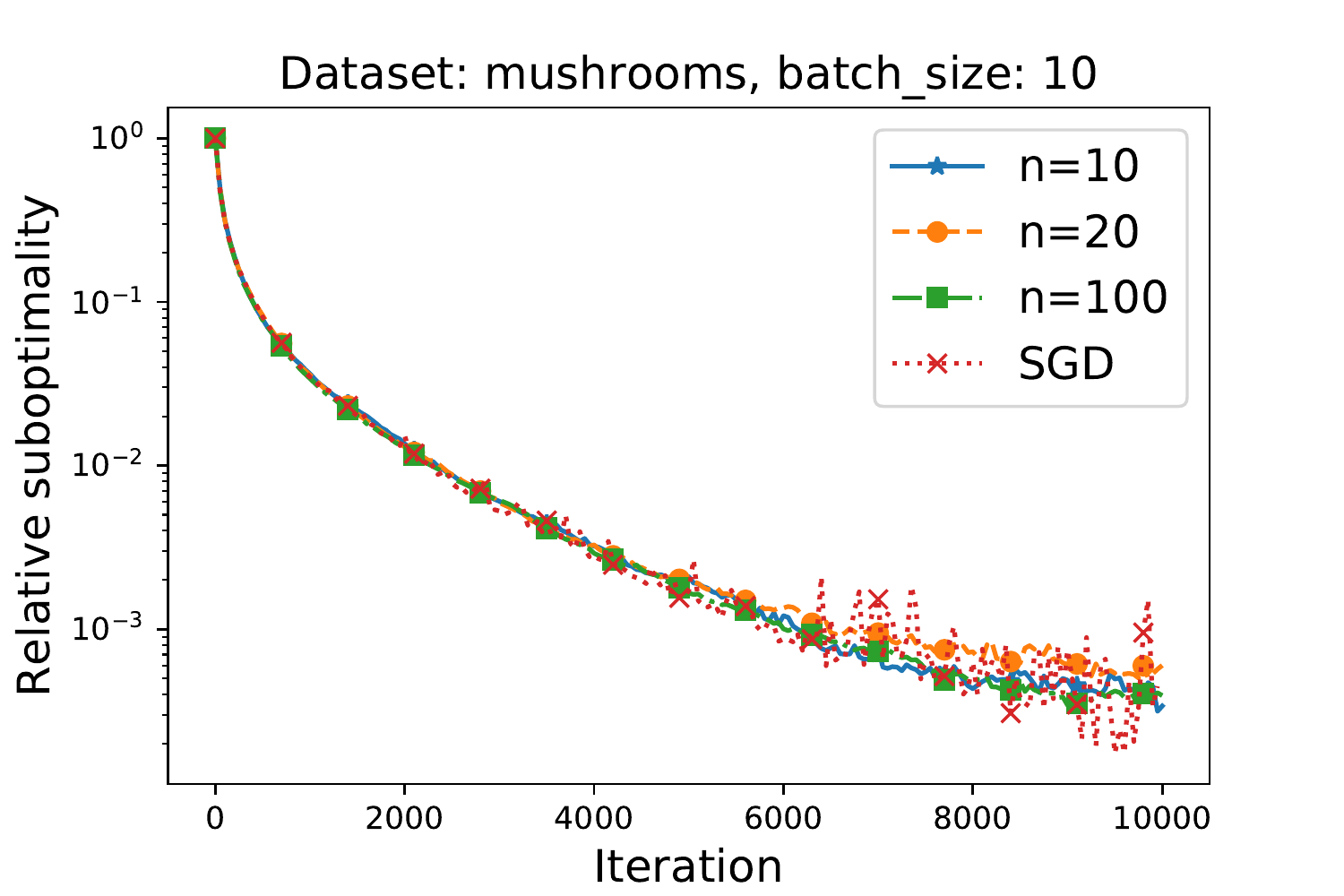}
\end{minipage}%
\\
\begin{minipage}{0.33\textwidth}
  \centering
\includegraphics[width =  \textwidth ]{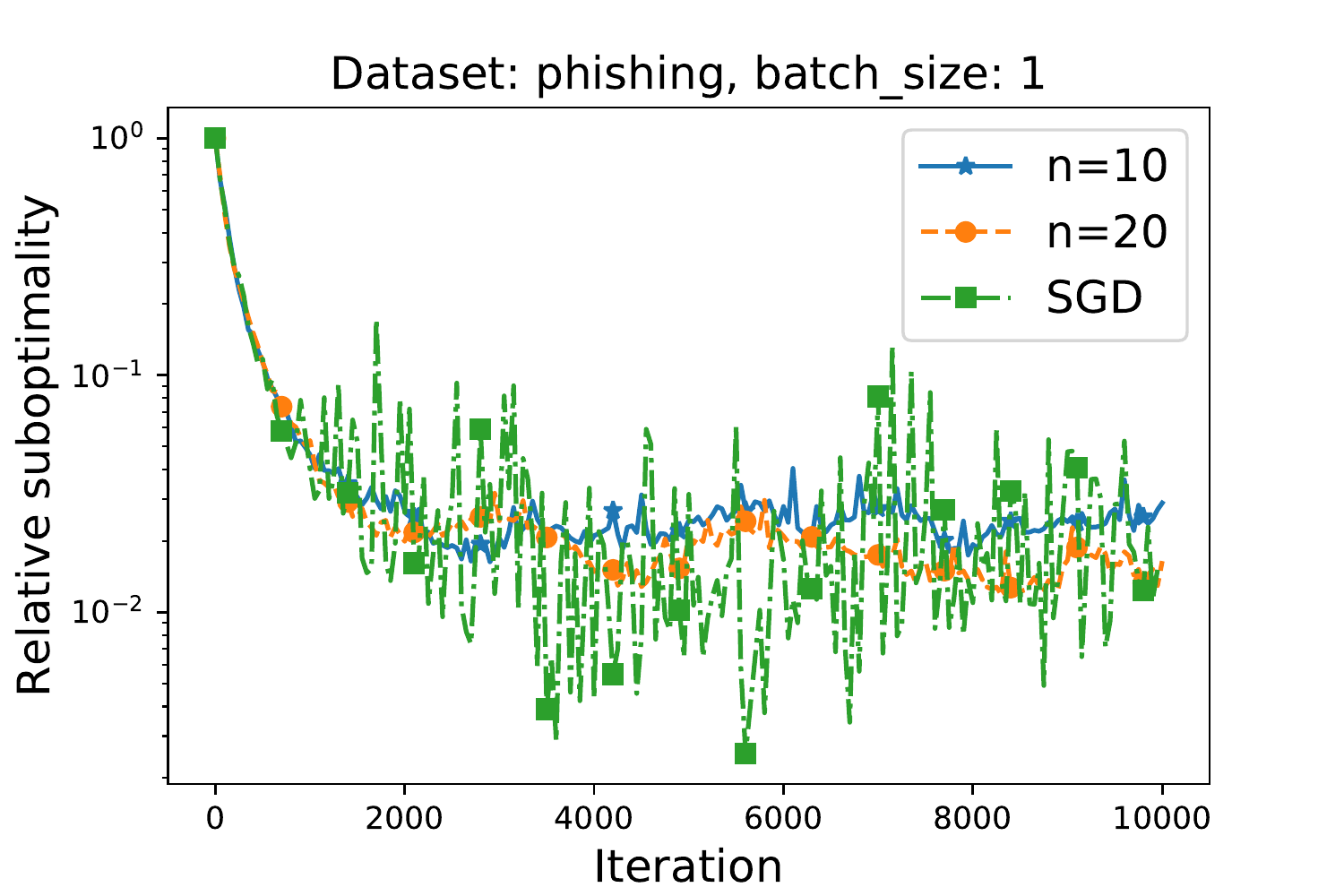}
\end{minipage}%
\begin{minipage}{0.33\textwidth}
  \centering
\includegraphics[width =  \textwidth ]{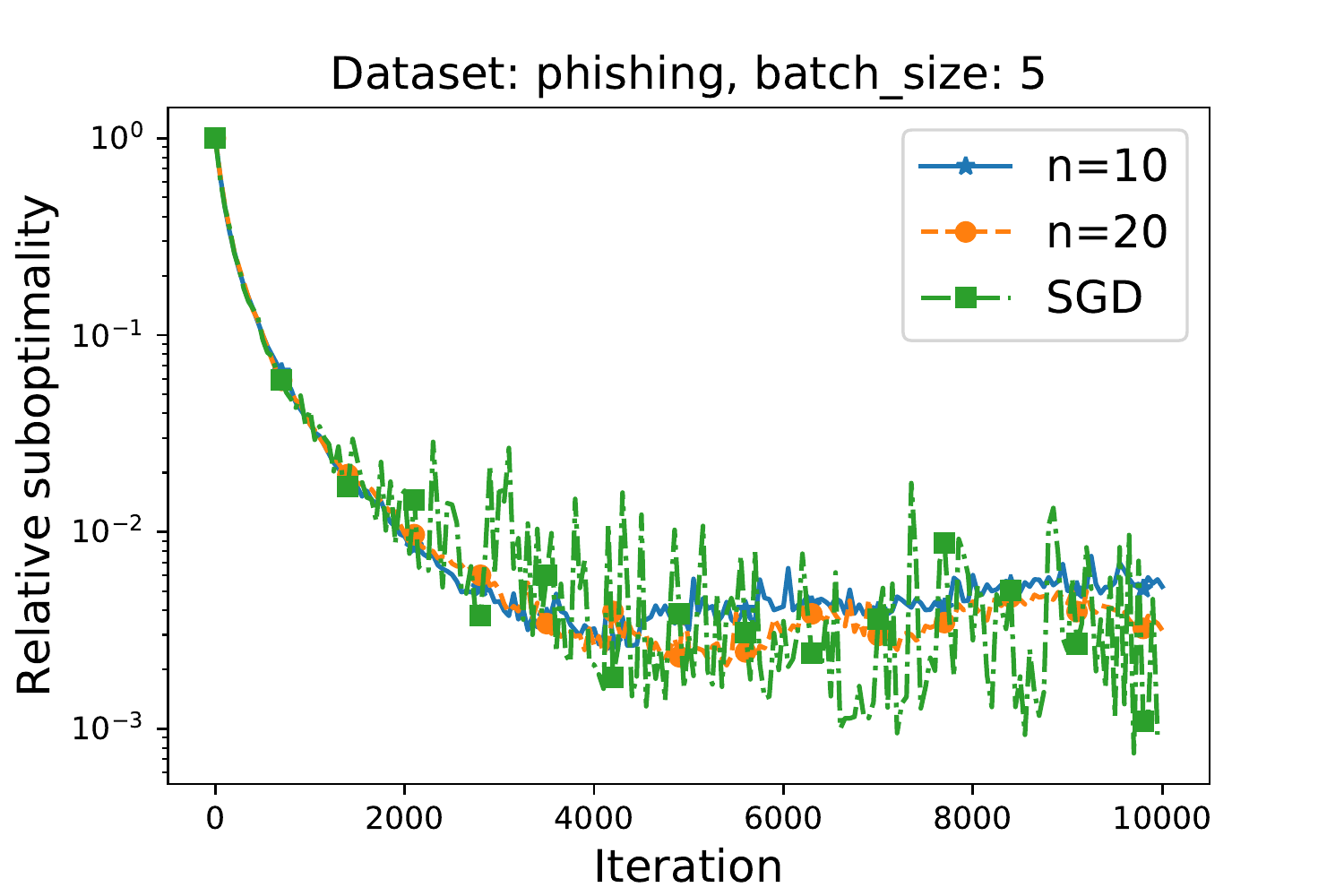}
\end{minipage}%
\begin{minipage}{0.33\textwidth}
  \centering
\includegraphics[width =  \textwidth ]{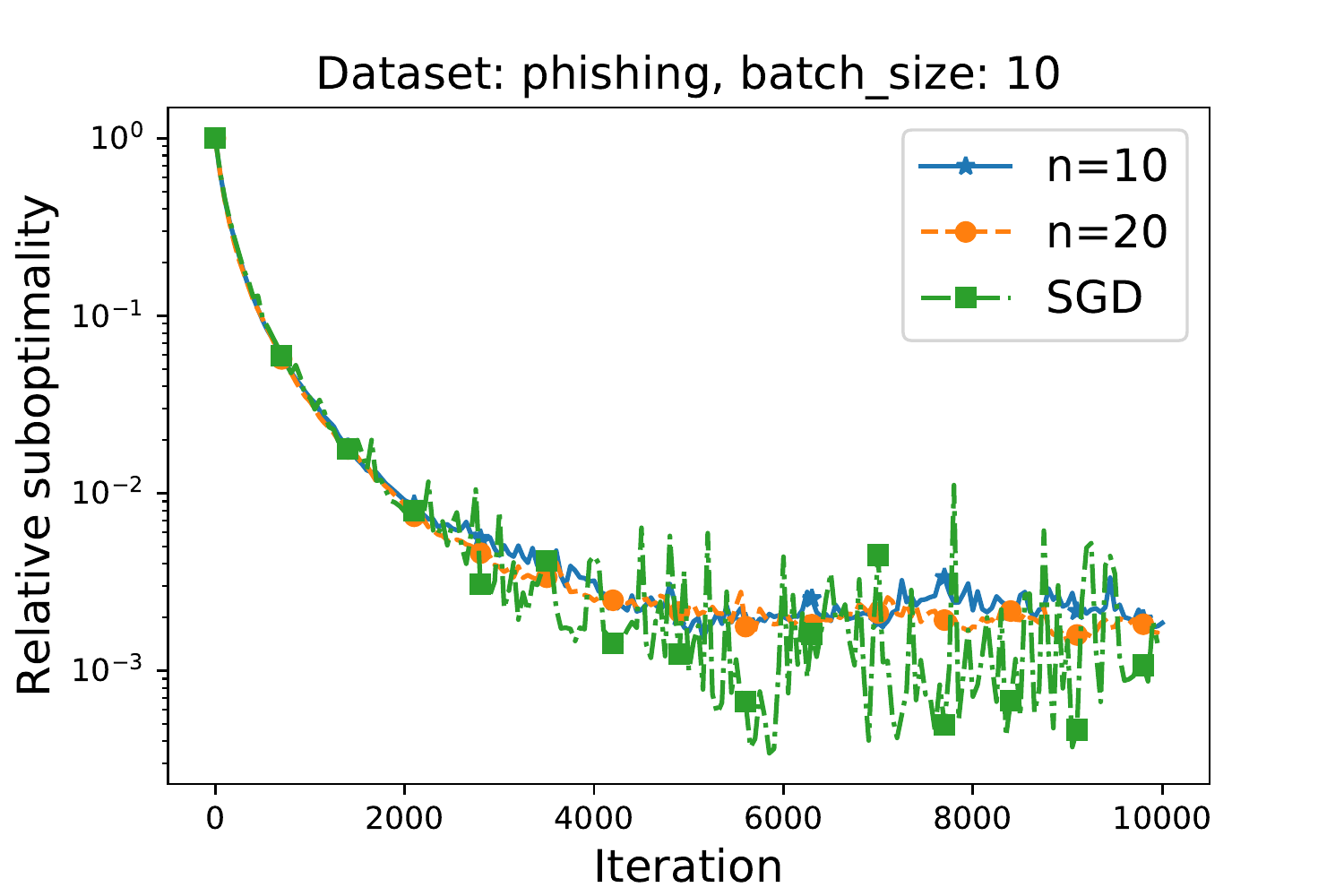}
\end{minipage}%
\\
\begin{minipage}{0.33\textwidth}
  \centering
\includegraphics[width =  \textwidth ]{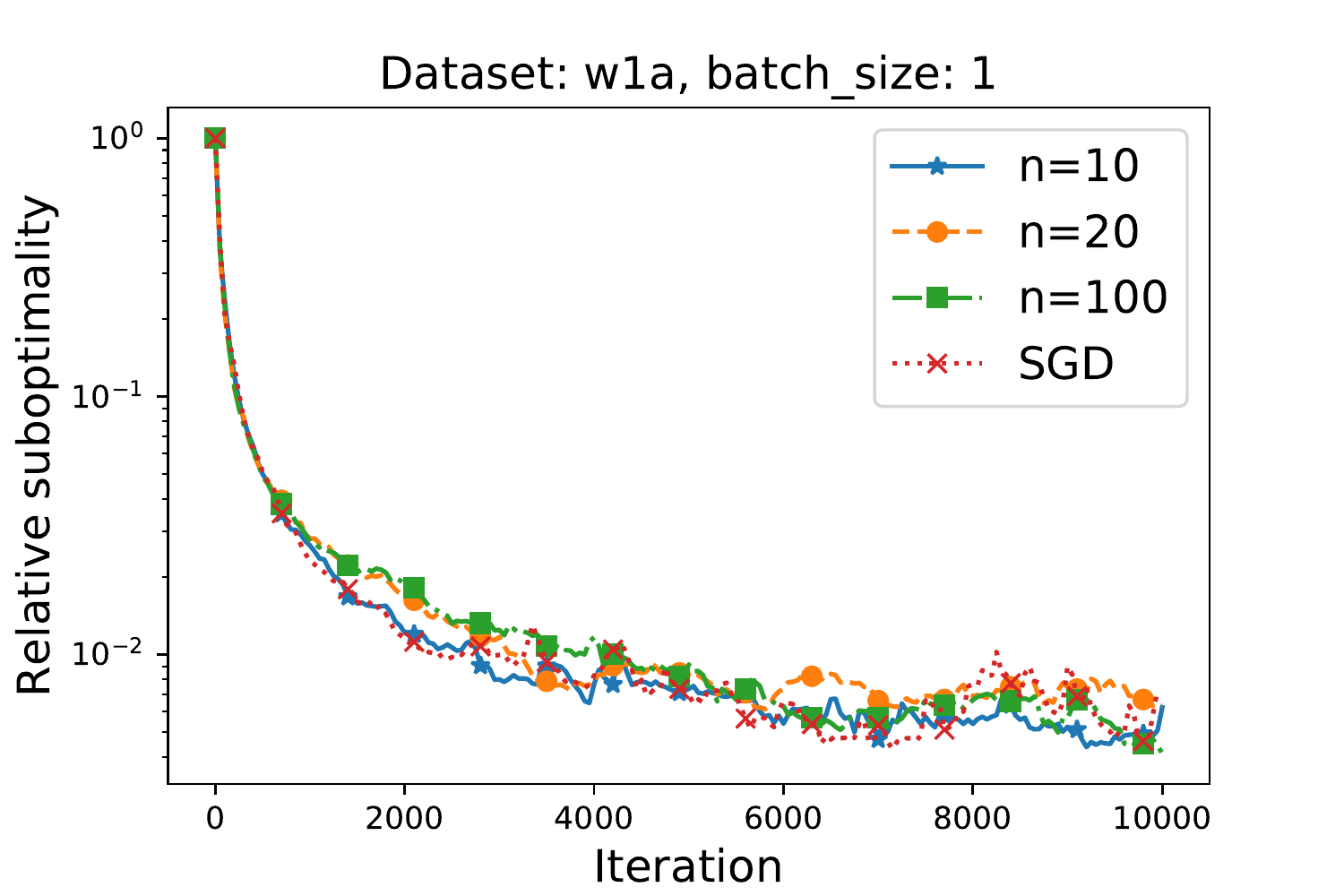}
\end{minipage}%
\begin{minipage}{0.33\textwidth}
  \centering
\includegraphics[width =  \textwidth ]{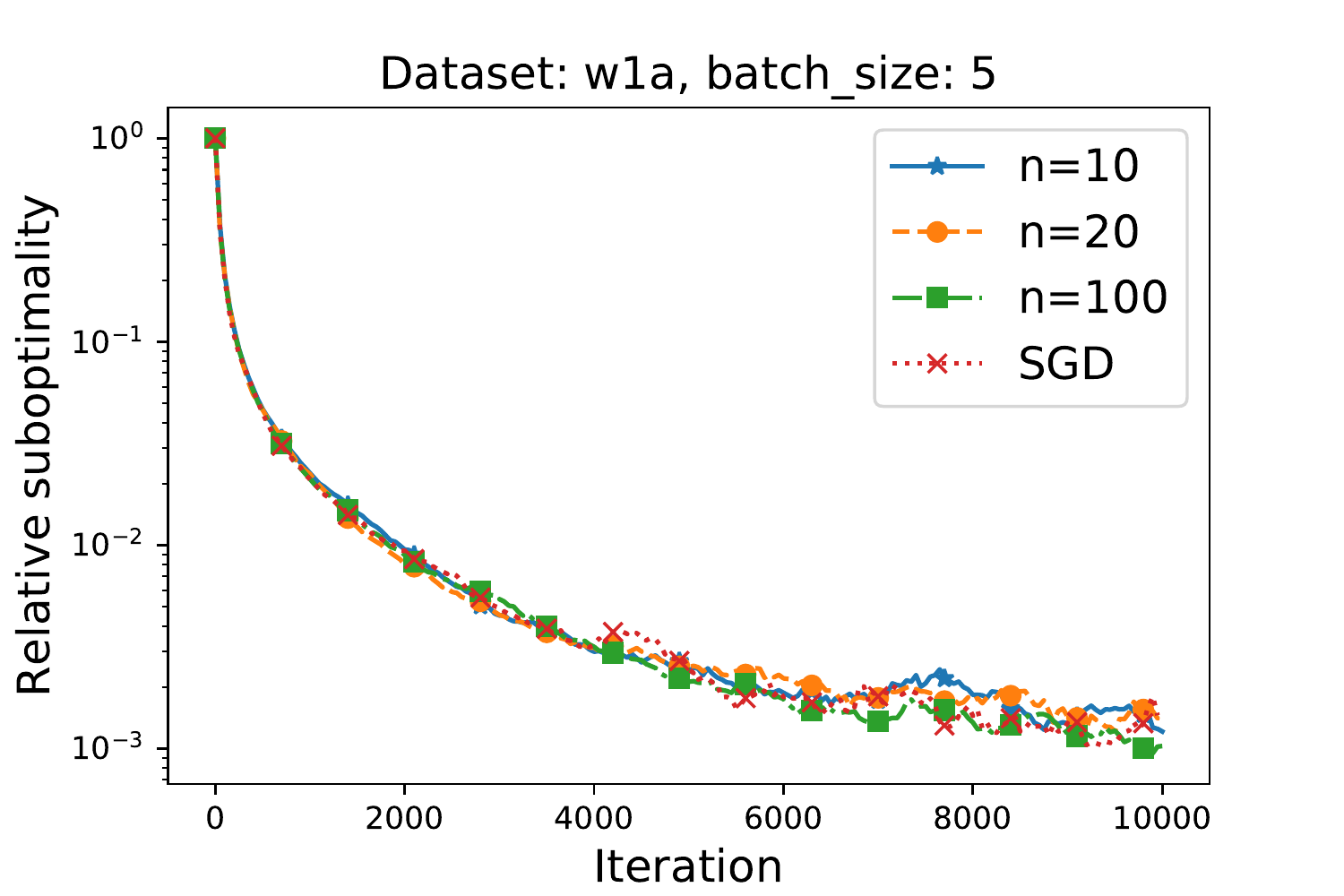}
\end{minipage}%
\begin{minipage}{0.33\textwidth}
  \centering
\includegraphics[width =  \textwidth ]{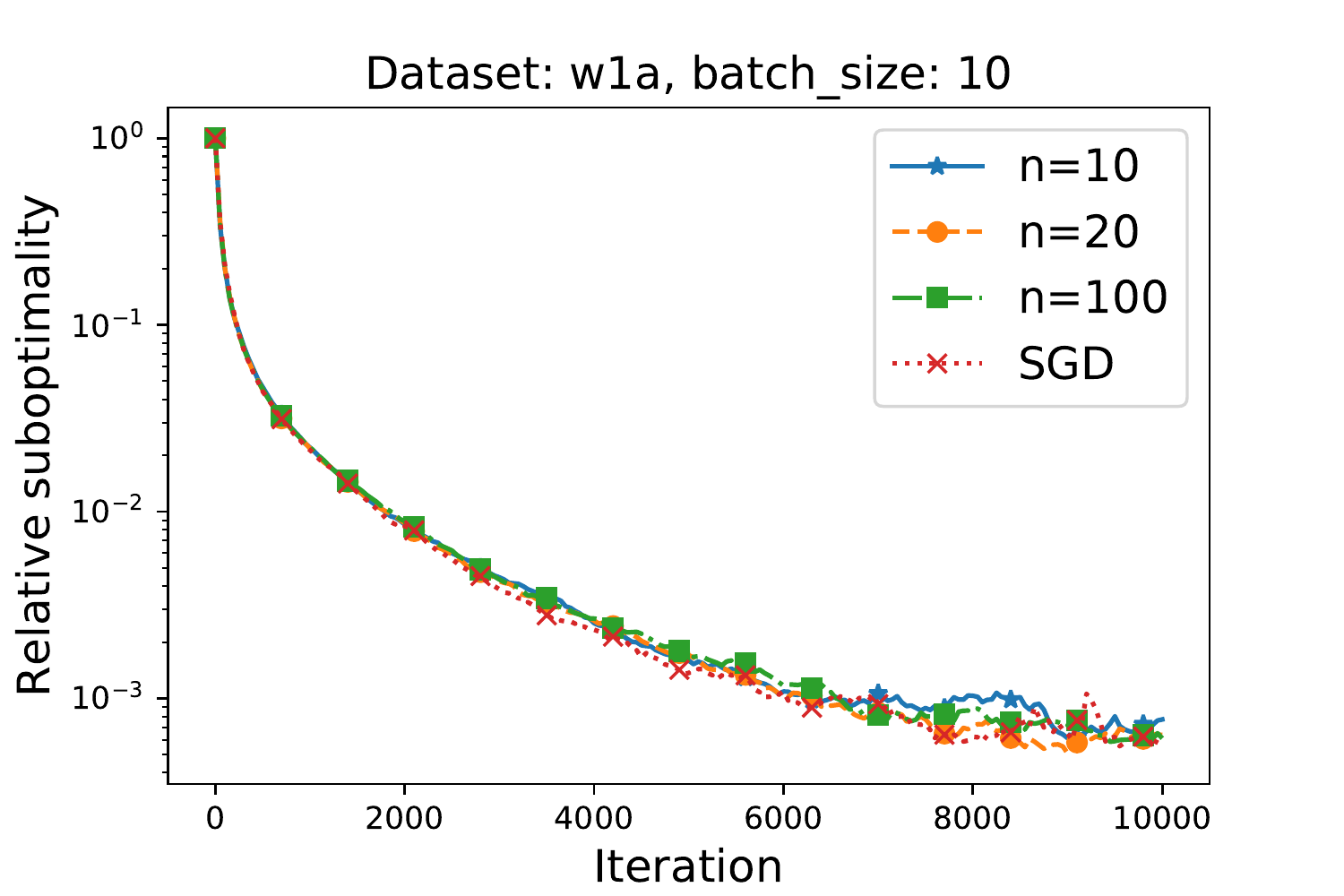}
\end{minipage}%
\\
\caption{Comparison of SGD (gradient evaluated on a single datapoint) and Algorithm~\ref{alg:sgd} with $n\tau=1$. Constant $\gamma  = \frac{1}{5L}$ was used for each algorithm. Label ``batch\_size'' indicates how big minibatch was chosen for stochastic gradient of each worker's objective.} \label{fig:sgd1}
\end{figure}

Next, we study the dependence of the convergence speed on $\tau$ for various values of $n$. Figure~\ref{fig:sgd2} presents the results. In each case, $\tau$ influences the convergence rate (or the region where the iterates oscillate) significantly, however, the effect is much weaker for larger $n$. This is in correspondence with Corollary~\ref{cor:sgd}.

\begin{figure}[H]
\centering
\begin{minipage}{0.33\textwidth}
  \centering
\includegraphics[width =  \textwidth ]{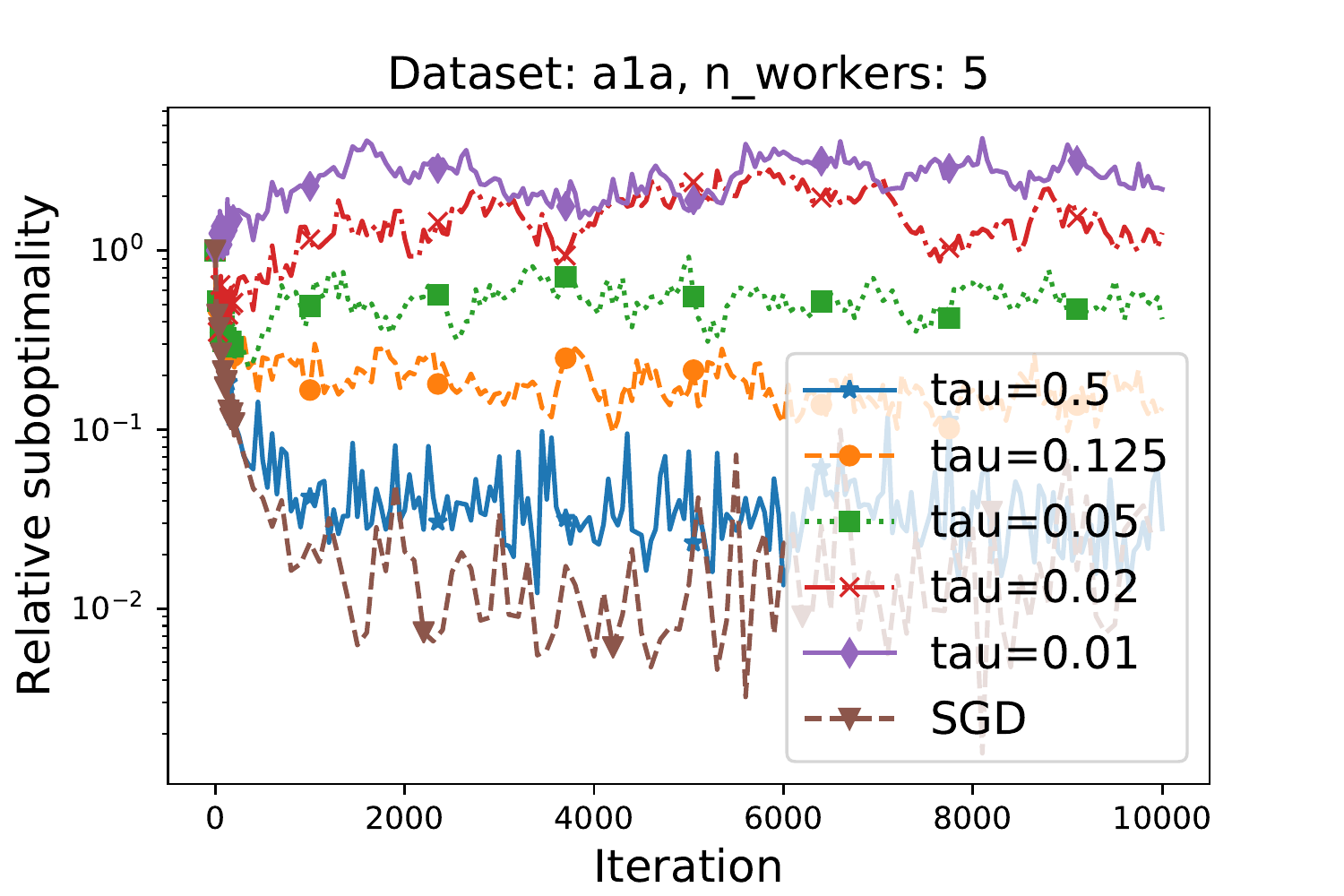}
\end{minipage}%
\begin{minipage}{0.33\textwidth}
  \centering
\includegraphics[width =  \textwidth ]{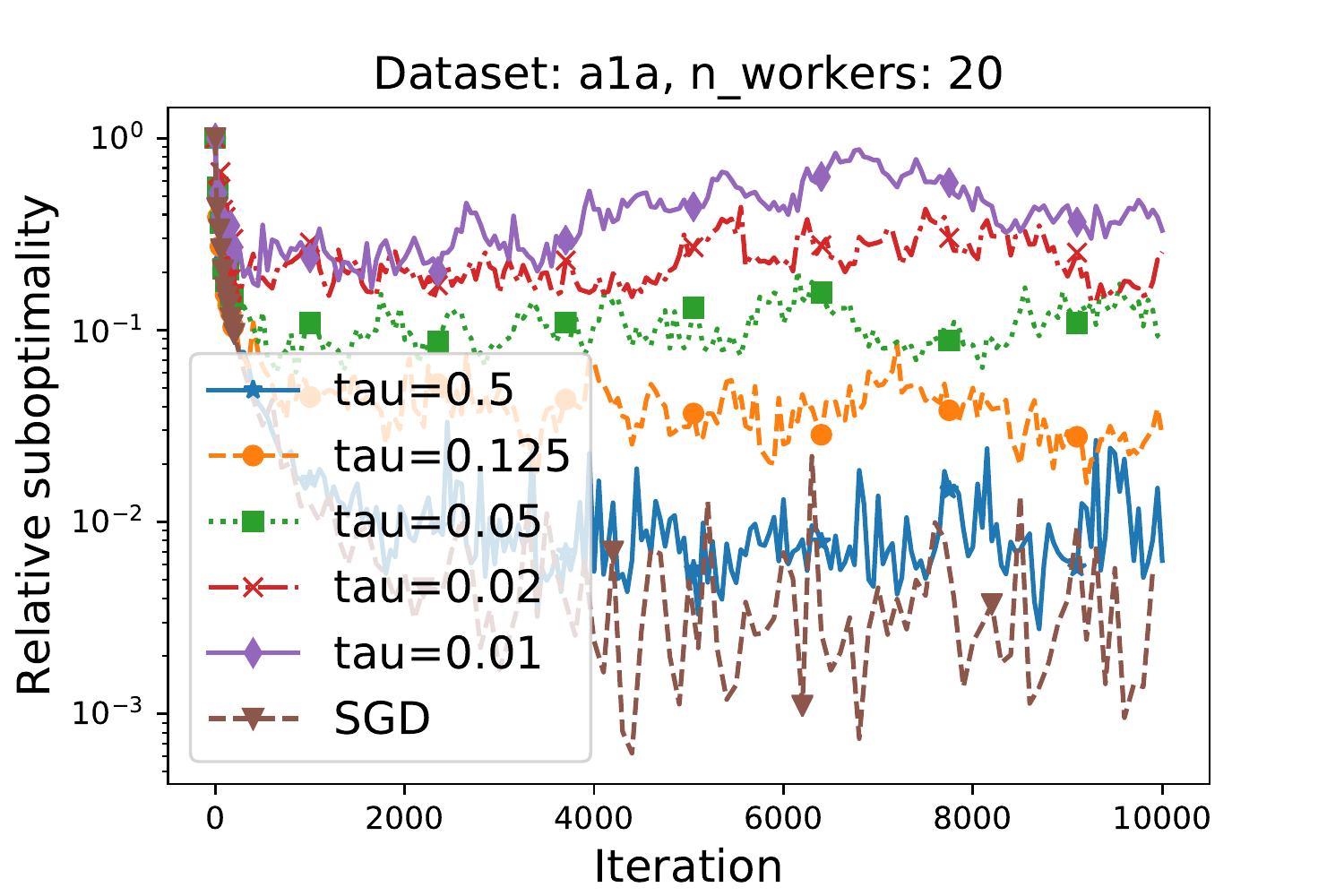}
\end{minipage}%
\begin{minipage}{0.33\textwidth}
  \centering
\includegraphics[width =  \textwidth ]{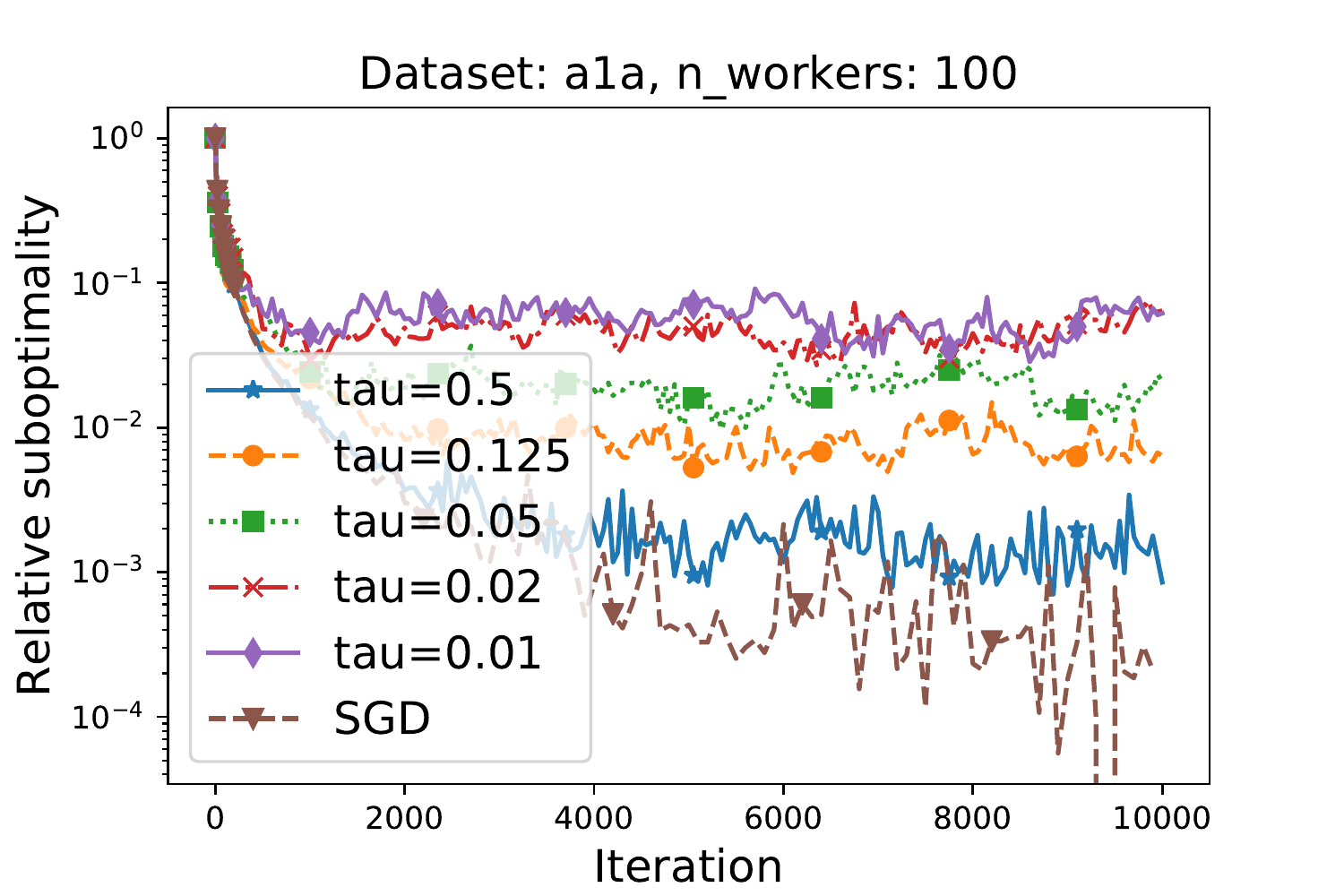}
\end{minipage}%
\\
\begin{minipage}{0.33\textwidth}
  \centering
\includegraphics[width =  \textwidth ]{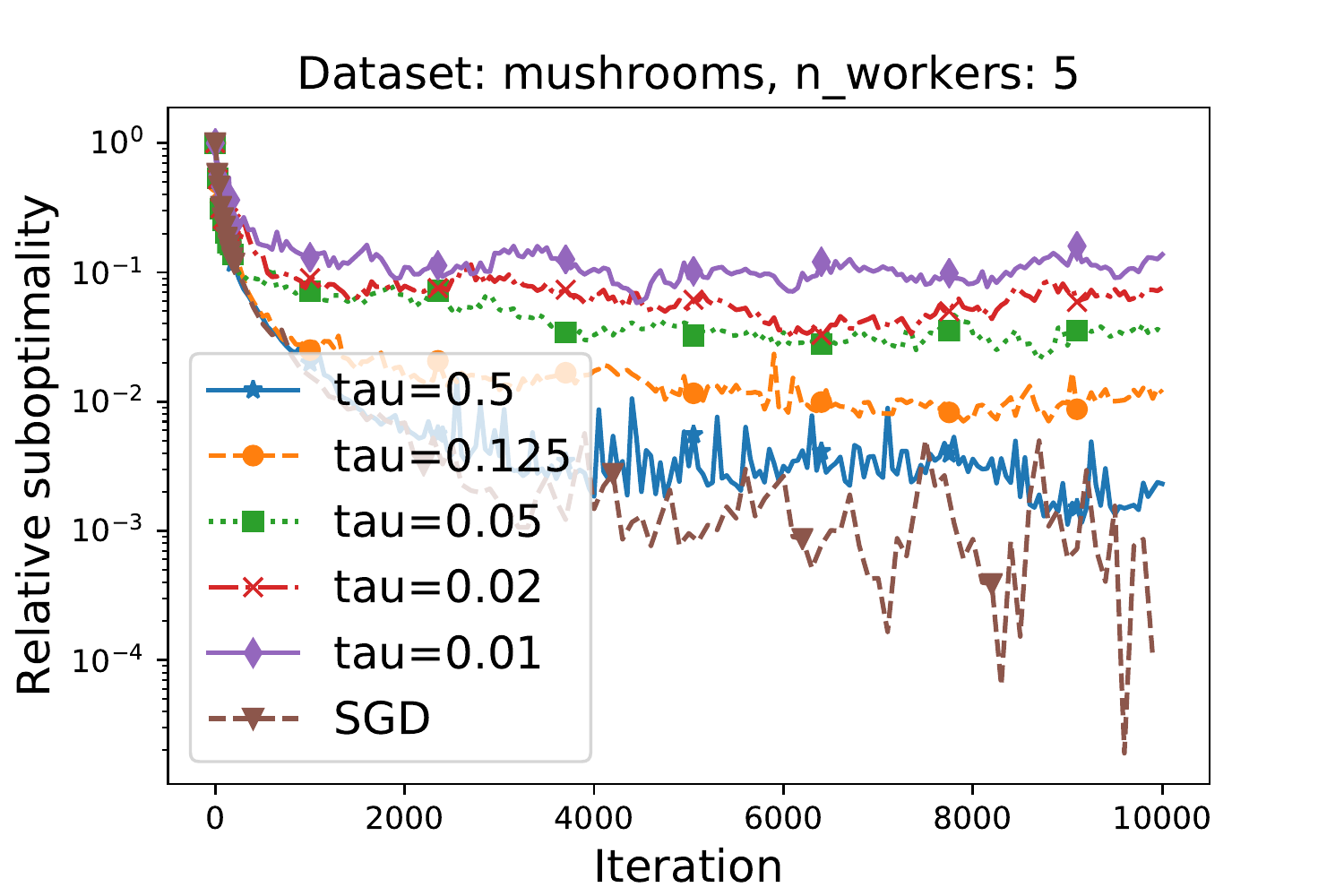}
\end{minipage}%
\begin{minipage}{0.33\textwidth}
  \centering
\includegraphics[width =  \textwidth ]{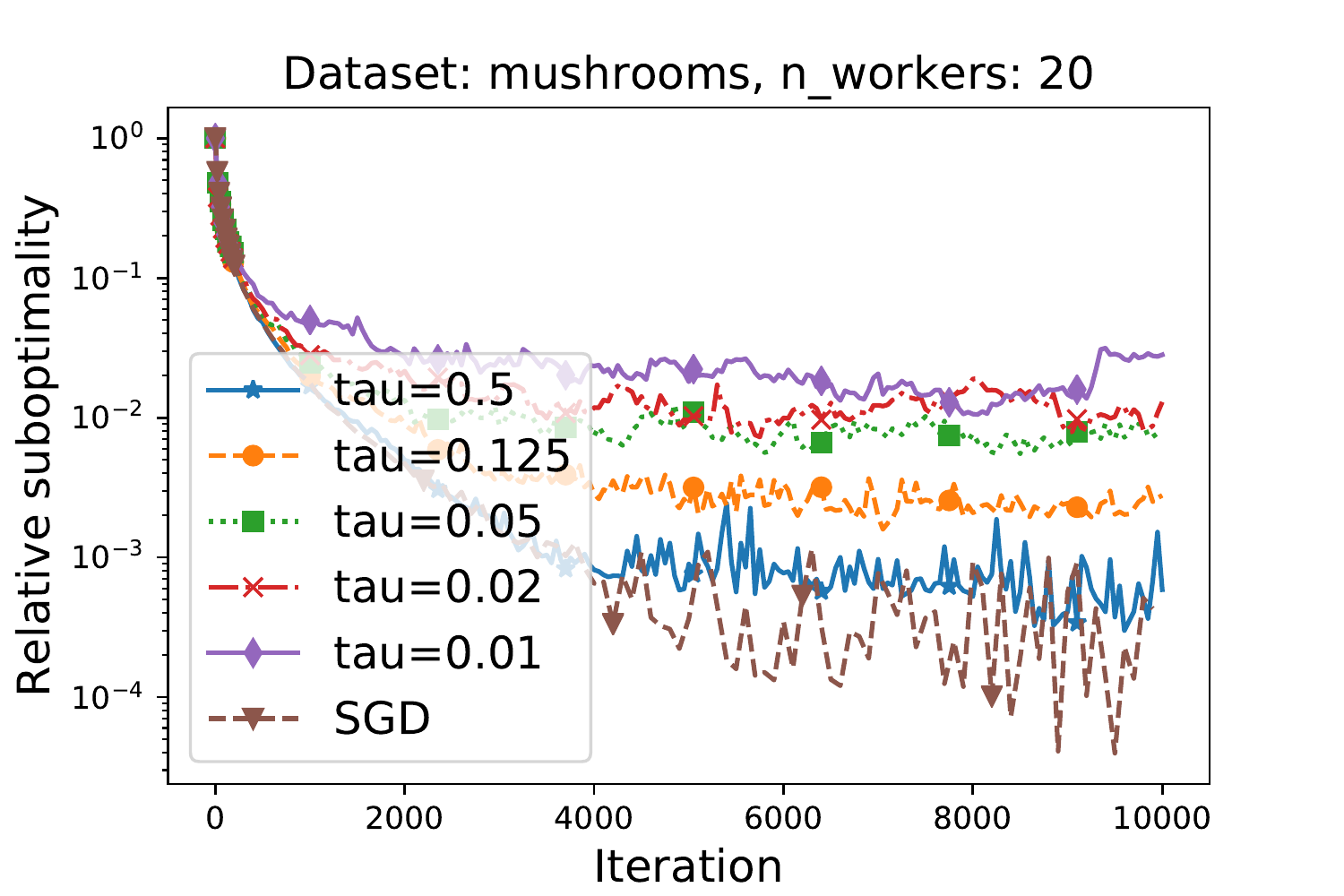}
\end{minipage}%
\begin{minipage}{0.33\textwidth}
  \centering
\includegraphics[width =  \textwidth ]{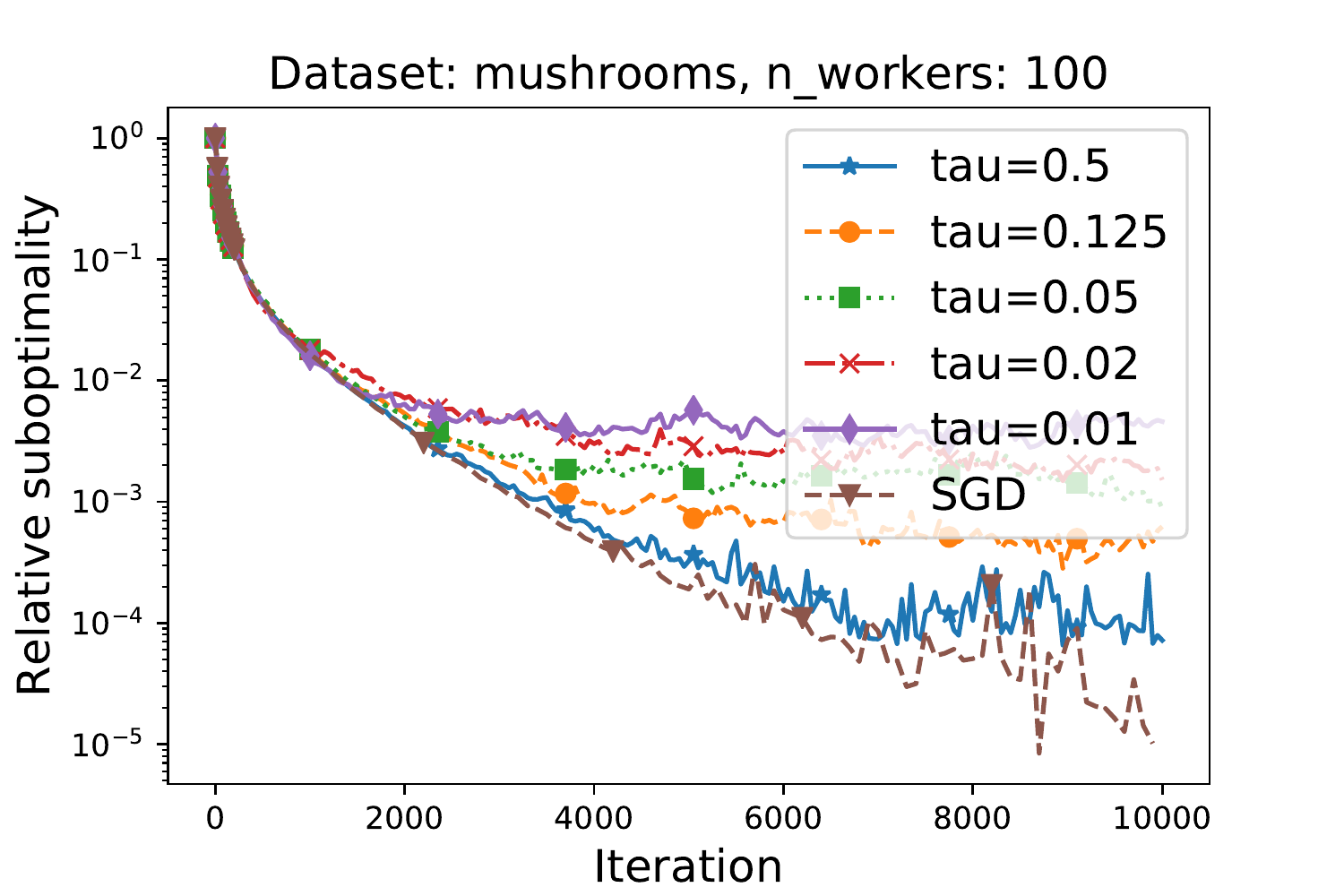}
\end{minipage}%
\\
\begin{minipage}{0.33\textwidth}
  \centering
\includegraphics[width =  \textwidth ]{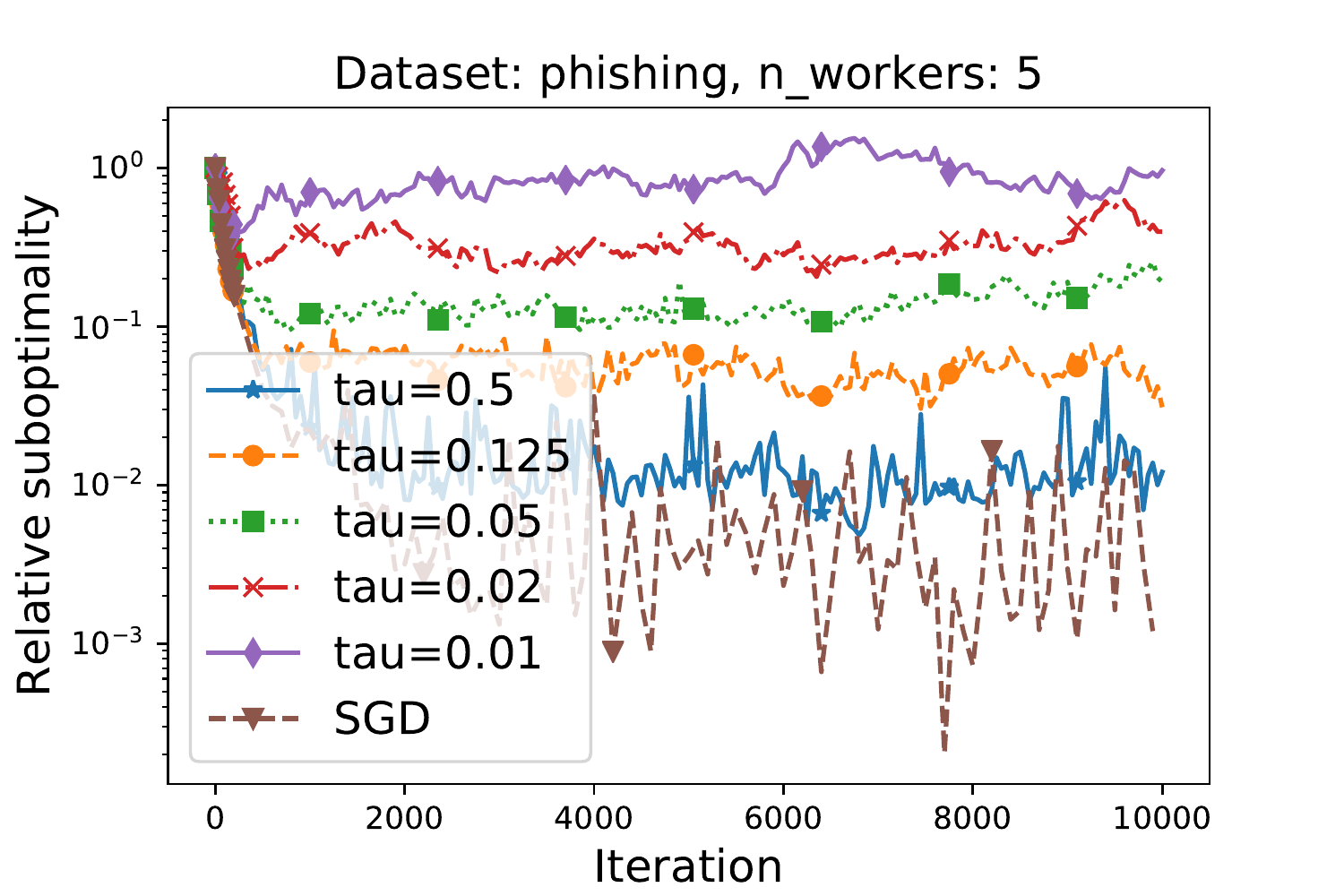}
\end{minipage}%
\begin{minipage}{0.33\textwidth}
  \centering
\includegraphics[width =  \textwidth ]{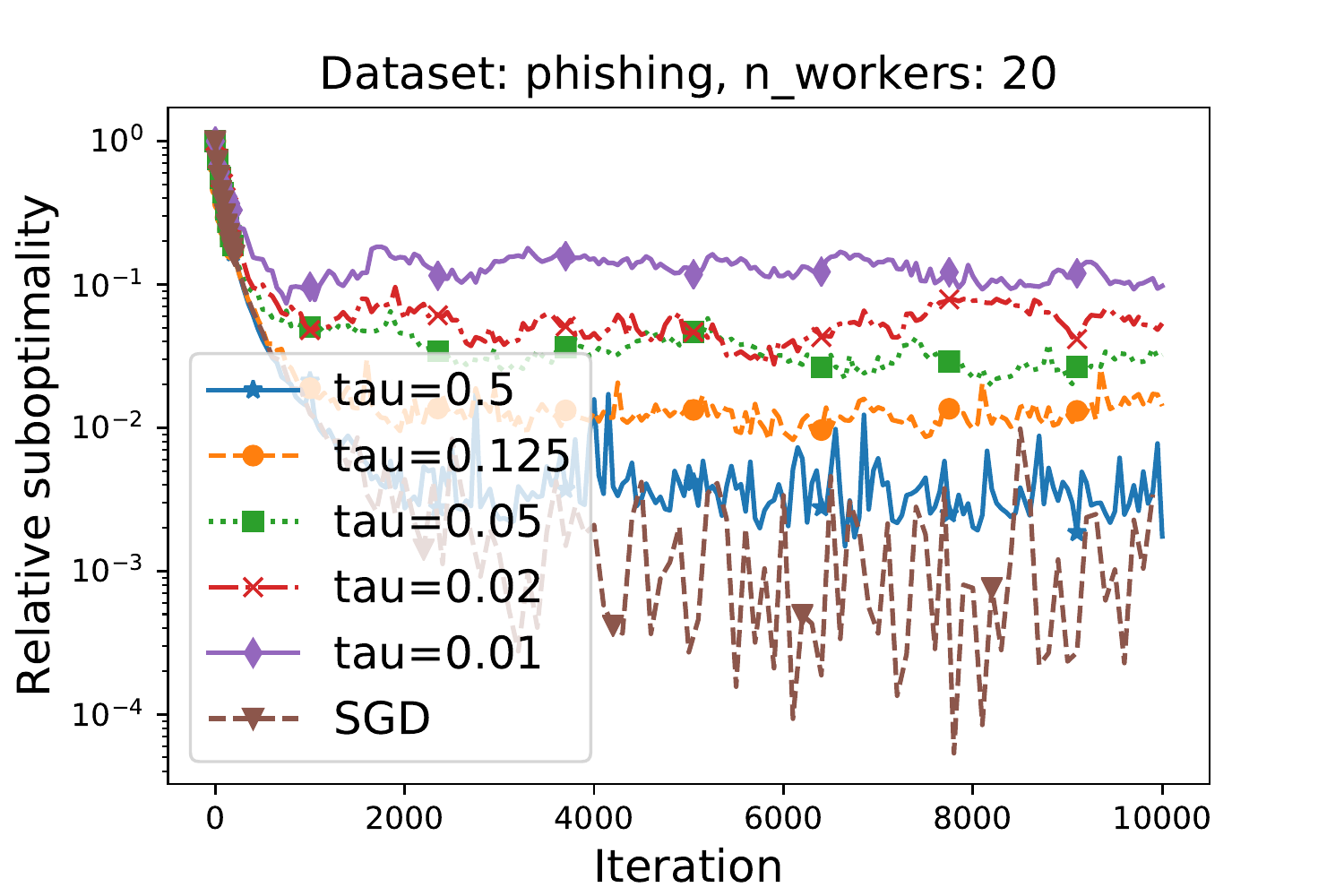}
\end{minipage}%
\\
\begin{minipage}{0.33\textwidth}
  \centering
\includegraphics[width =  \textwidth ]{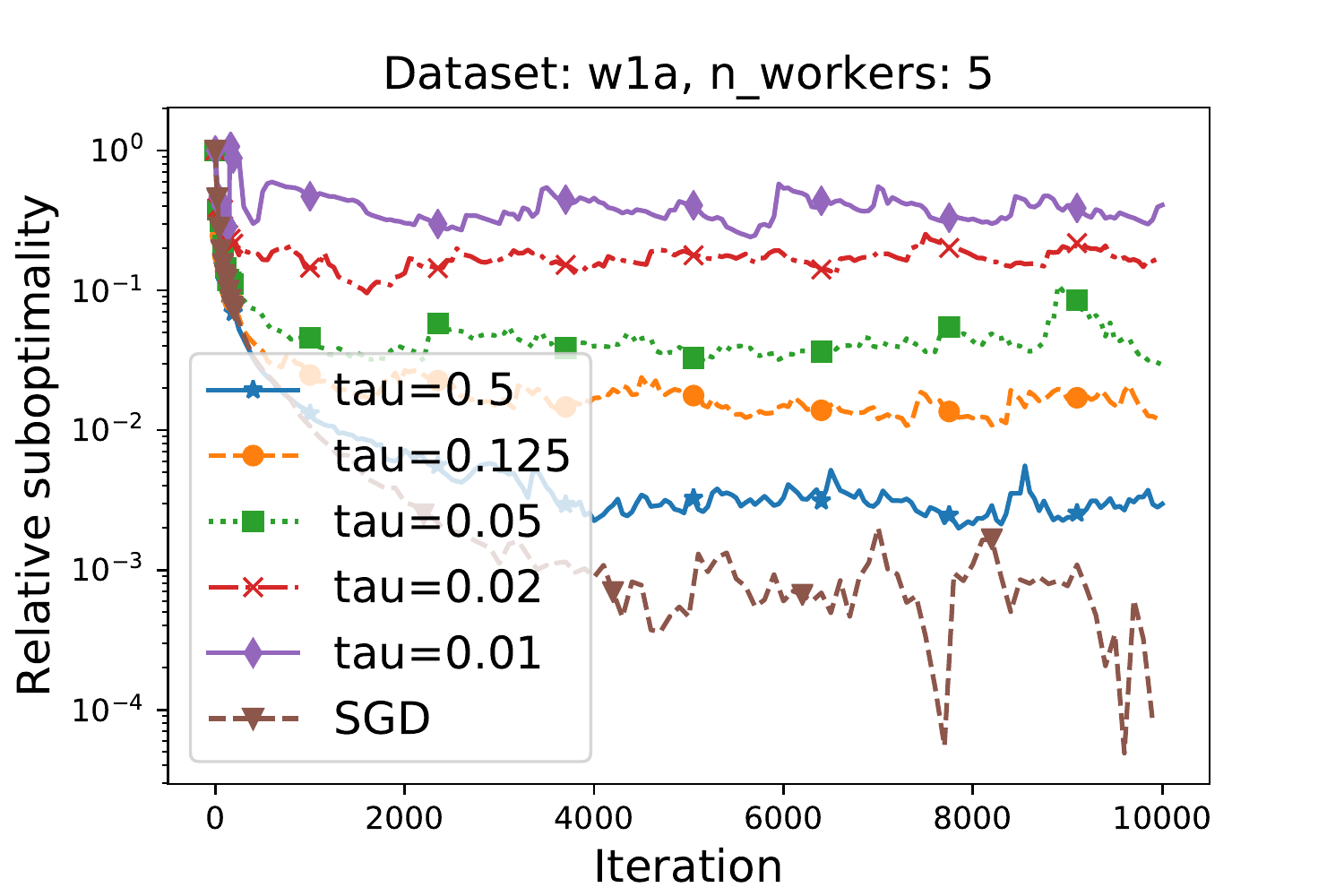}
\end{minipage}%
\begin{minipage}{0.33\textwidth}
  \centering
\includegraphics[width =  \textwidth ]{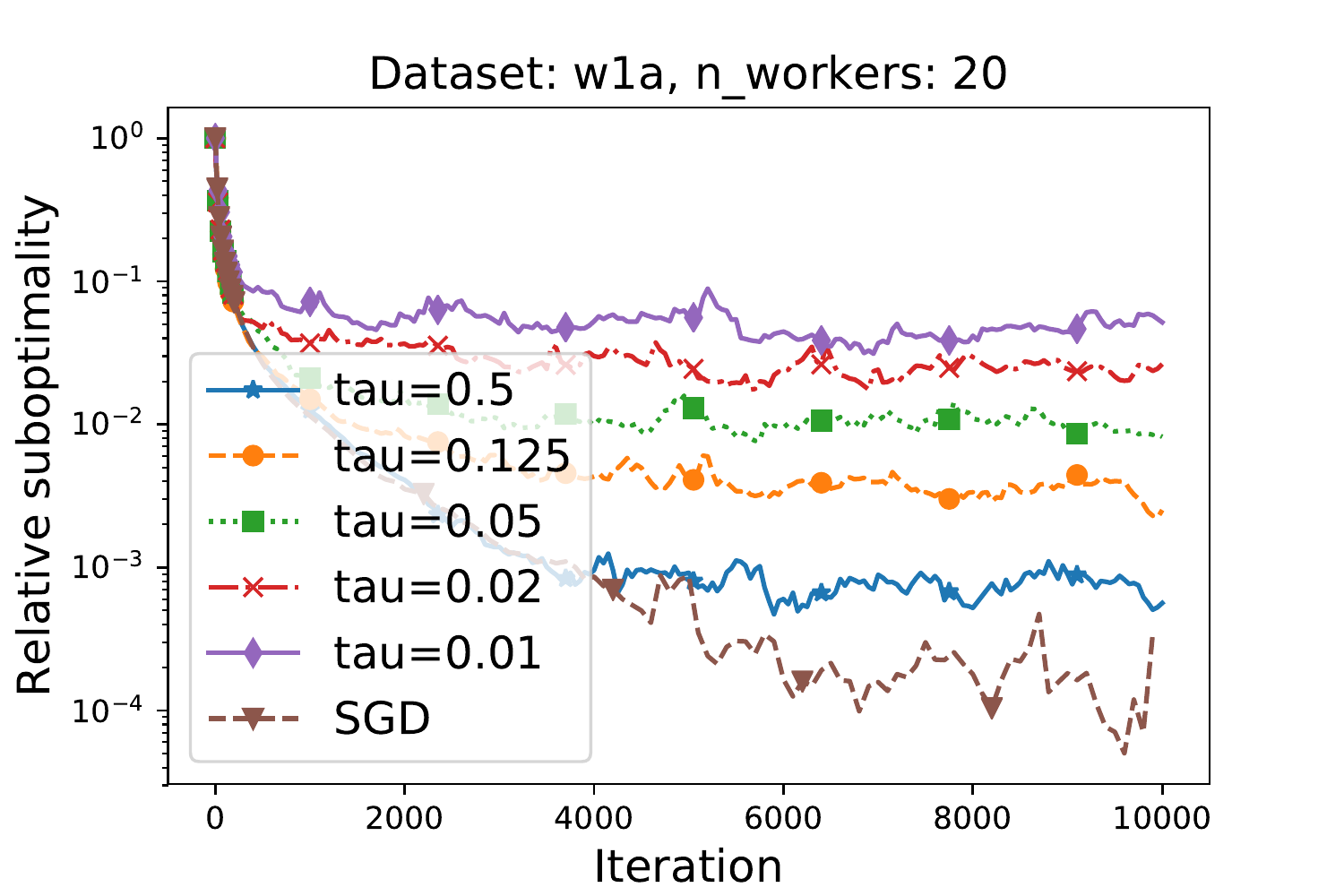}
\end{minipage}%
\begin{minipage}{0.33\textwidth}
  \centering
\includegraphics[width =  \textwidth ]{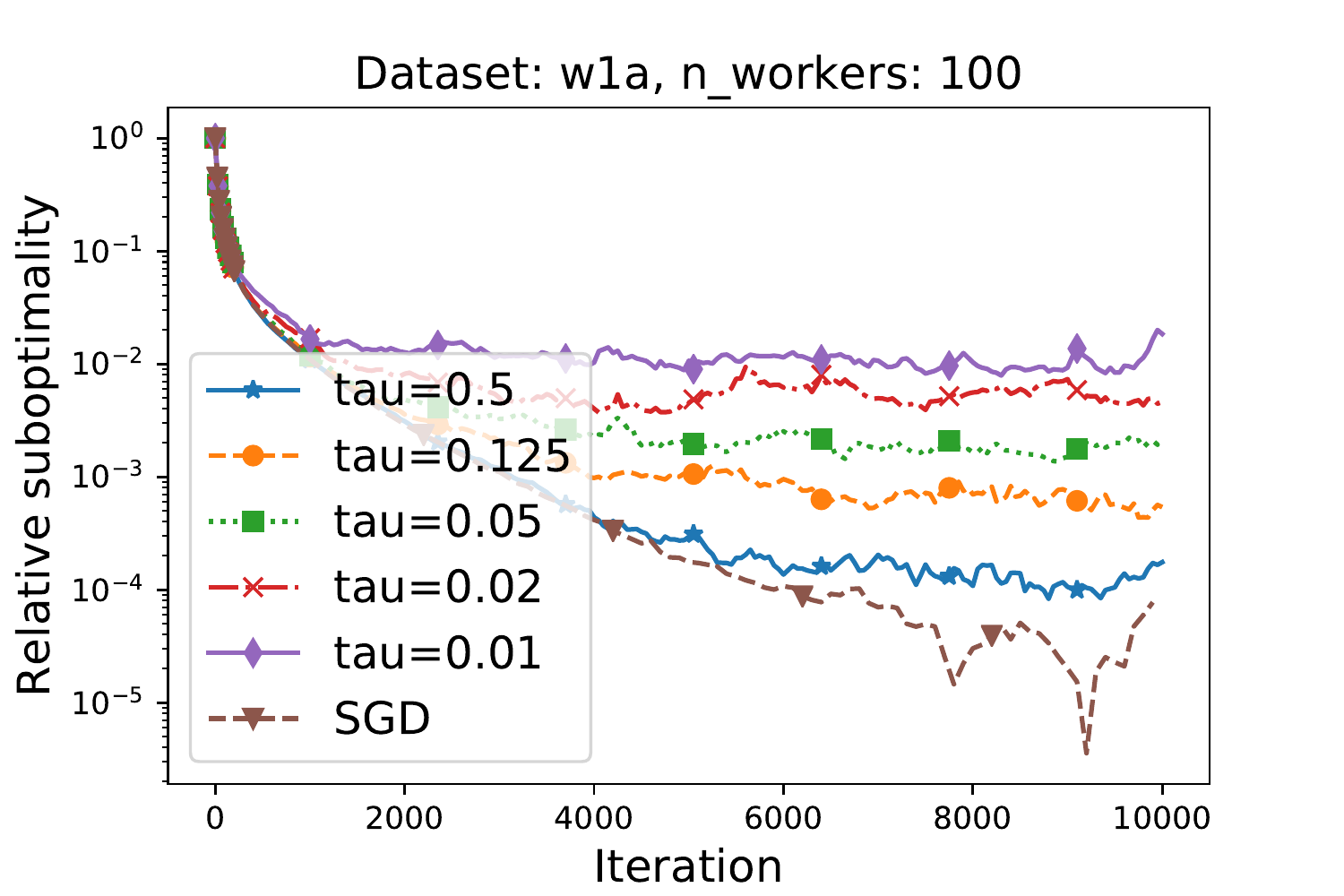}
\end{minipage}%
\\
\caption{Behavior of Algorithm~\ref{alg:sgd} while varying $\tau$. Label ``SGD'' corresponds to the choice $n=1, \tau = 1$. Stepsize $\gamma = \frac1{3L}$ was used in every case.}\label{fig:sgd2}
\end{figure}

\subsection{IASGD \label{sec:exp_asgd}}
In this section we numerically test Algorithm~\ref{alg:acc} for logistic regression problem. As in the last section, $f_i$ consists of set of (uniformly distributed) rows of $A$ from~\eqref{eq:logreg}. The stochastic gradient is taken as a gradient on a subset data points from each $f_i$. Note that Algorithm~\ref{alg:acc} depends on a priori unknown strong growth parameter $\hat{\rho}$ of unbiased stochastic gradient $q$\footnote{Formulas to obtain parameters of Algorithm~\ref{alg:acc} are given in~\cite{vaswani2018fast}. }.  Therefore, we first find empirically optimal $\hat{\rho}$ for each algorithm run by grid search and report only the best performance for each algorithm. 

The first experiment (Figure~\ref{fig:acc1}) verifies the linearity claim -- we vary $(n,\tau)$ such that $n\tau=1$. As predicted by theory, the behavior of presented algorithms is almost indistinguishable. 

\begin{figure}[H]
\centering
\begin{minipage}{0.33\textwidth}
  \centering
\includegraphics[width =  \textwidth ]{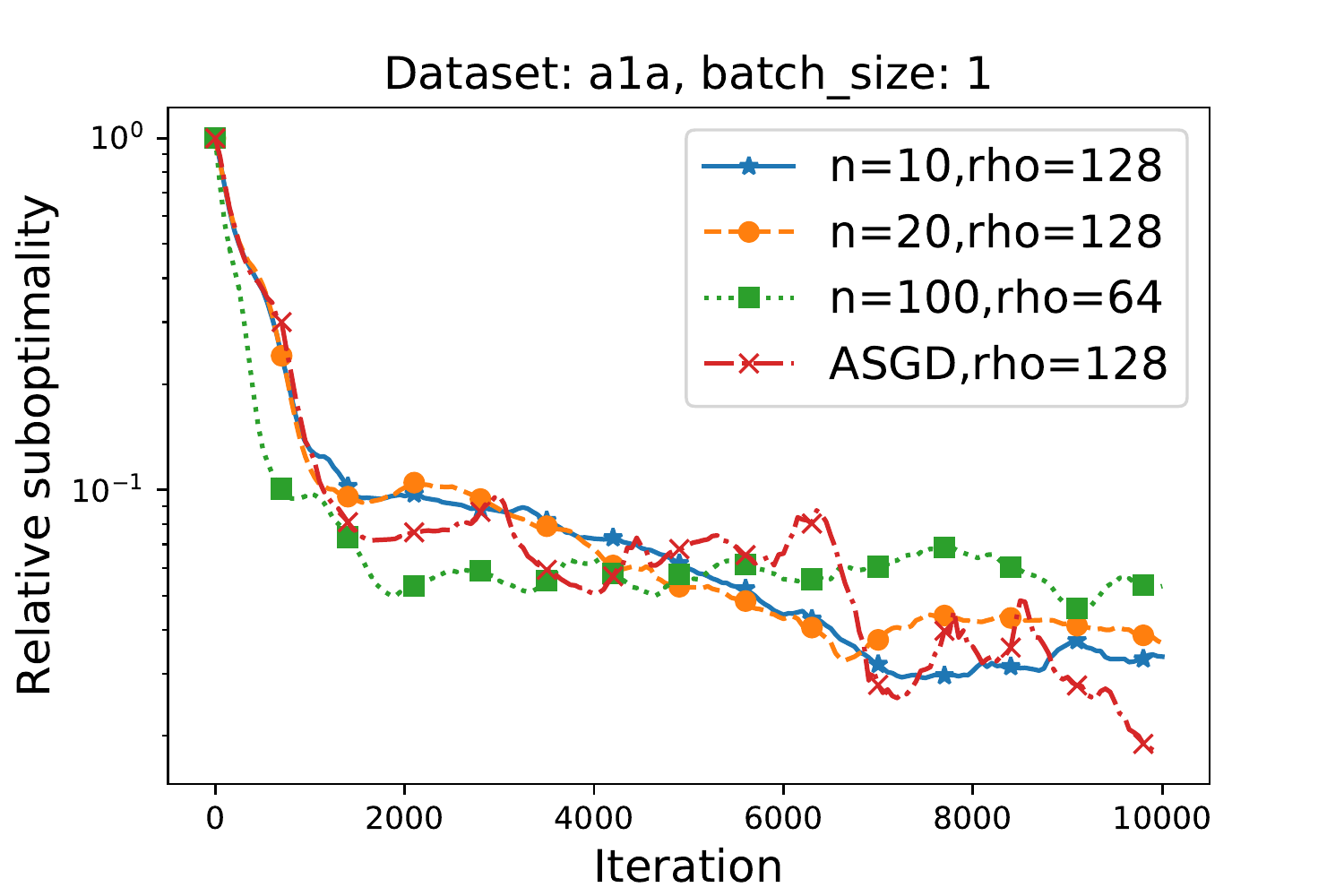}
\end{minipage}%
\begin{minipage}{0.33\textwidth}
  \centering
\includegraphics[width =  \textwidth ]{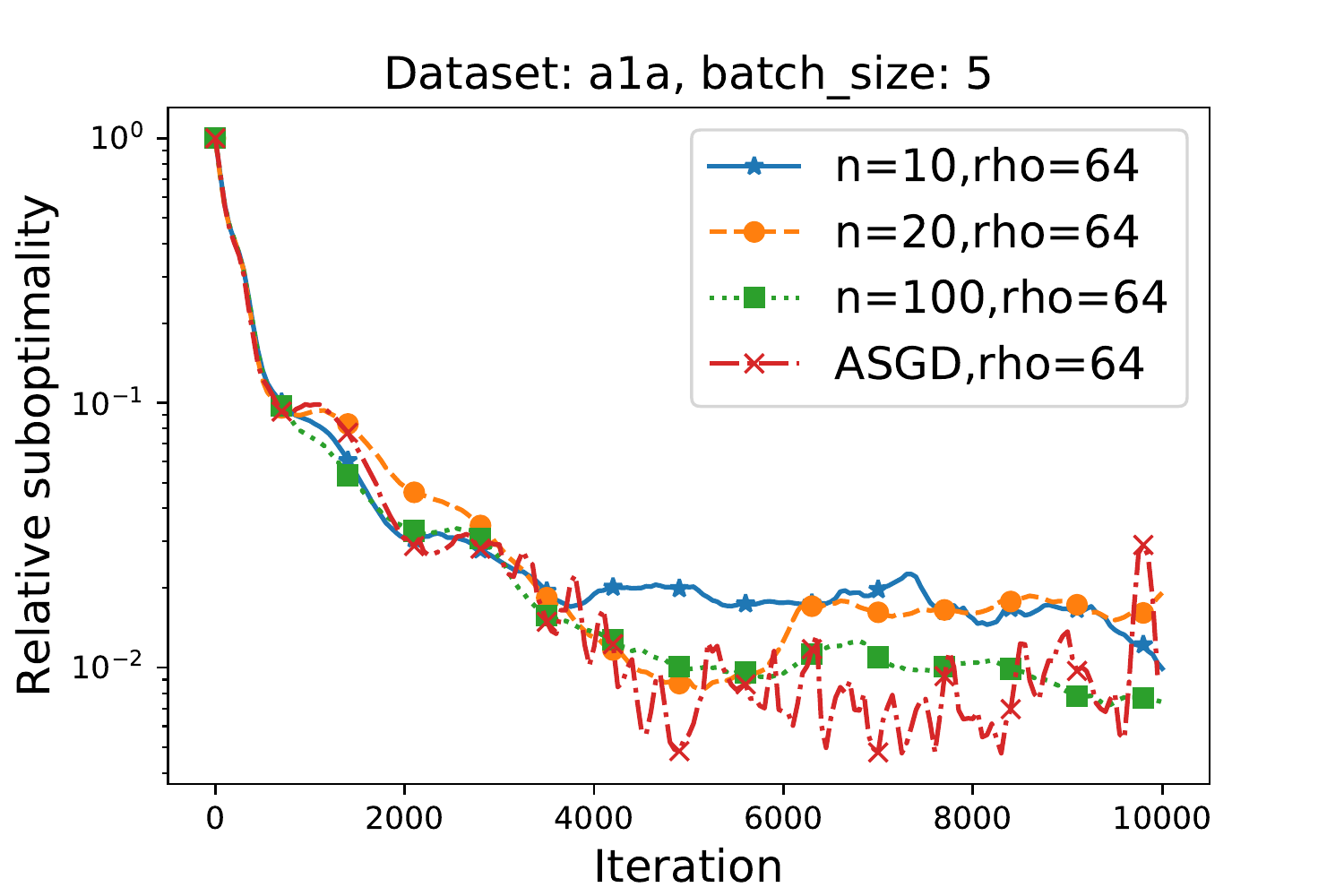}
\end{minipage}%
\begin{minipage}{0.33\textwidth}
  \centering
\includegraphics[width =  \textwidth ]{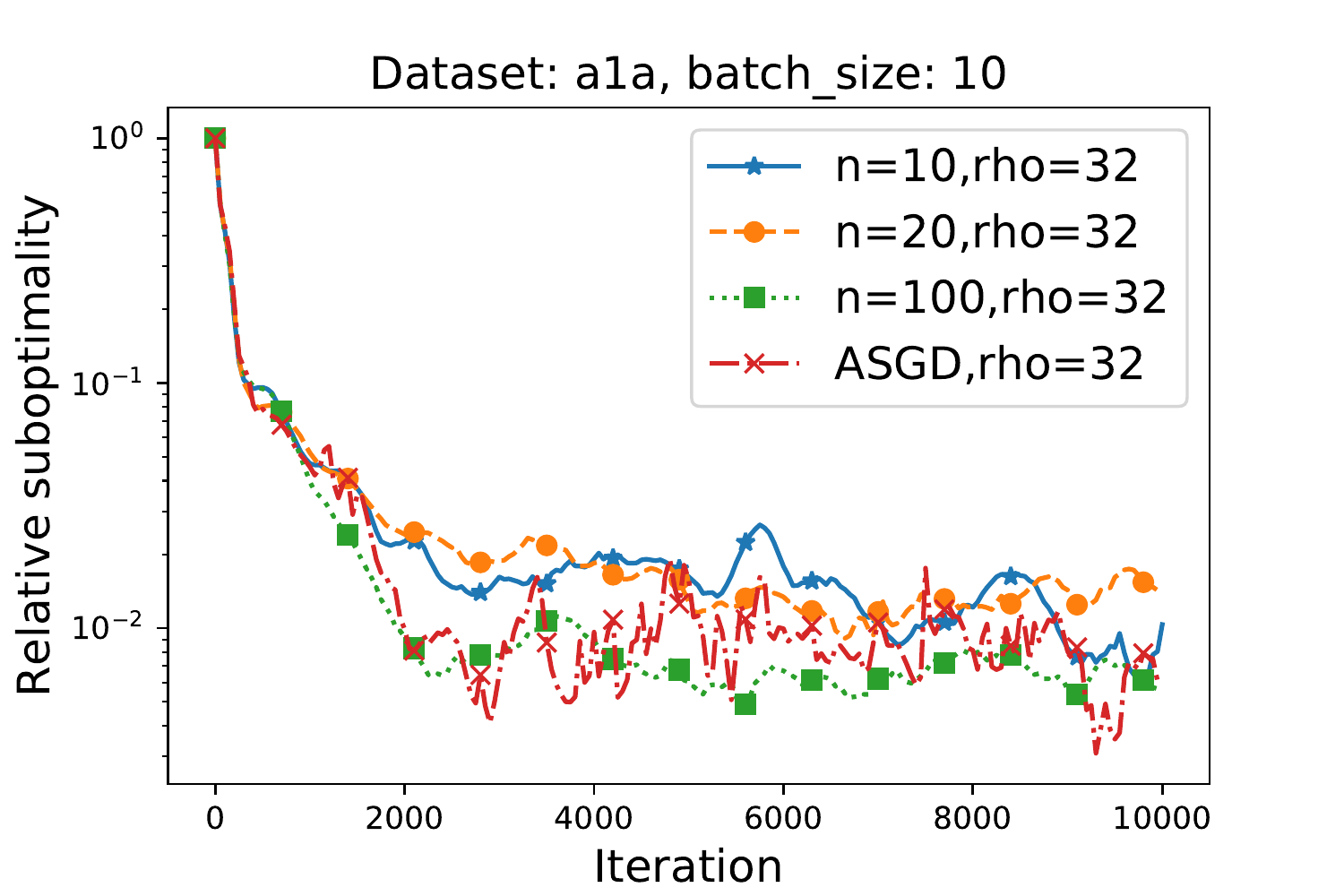}
\end{minipage}%
\\
\begin{minipage}{0.33\textwidth}
  \centering
\includegraphics[width =  \textwidth ]{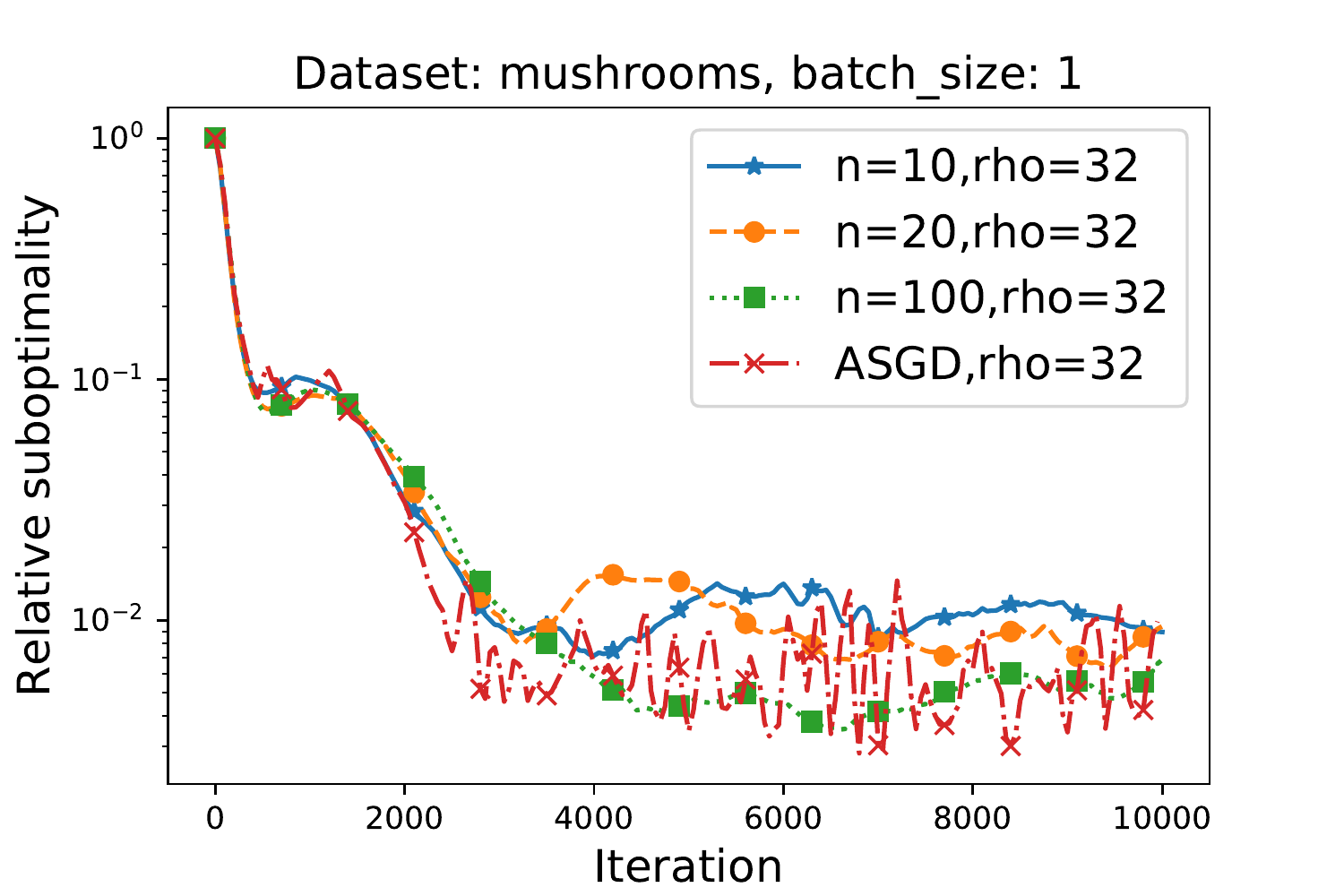}
\end{minipage}%
\begin{minipage}{0.33\textwidth}
  \centering
\includegraphics[width =  \textwidth ]{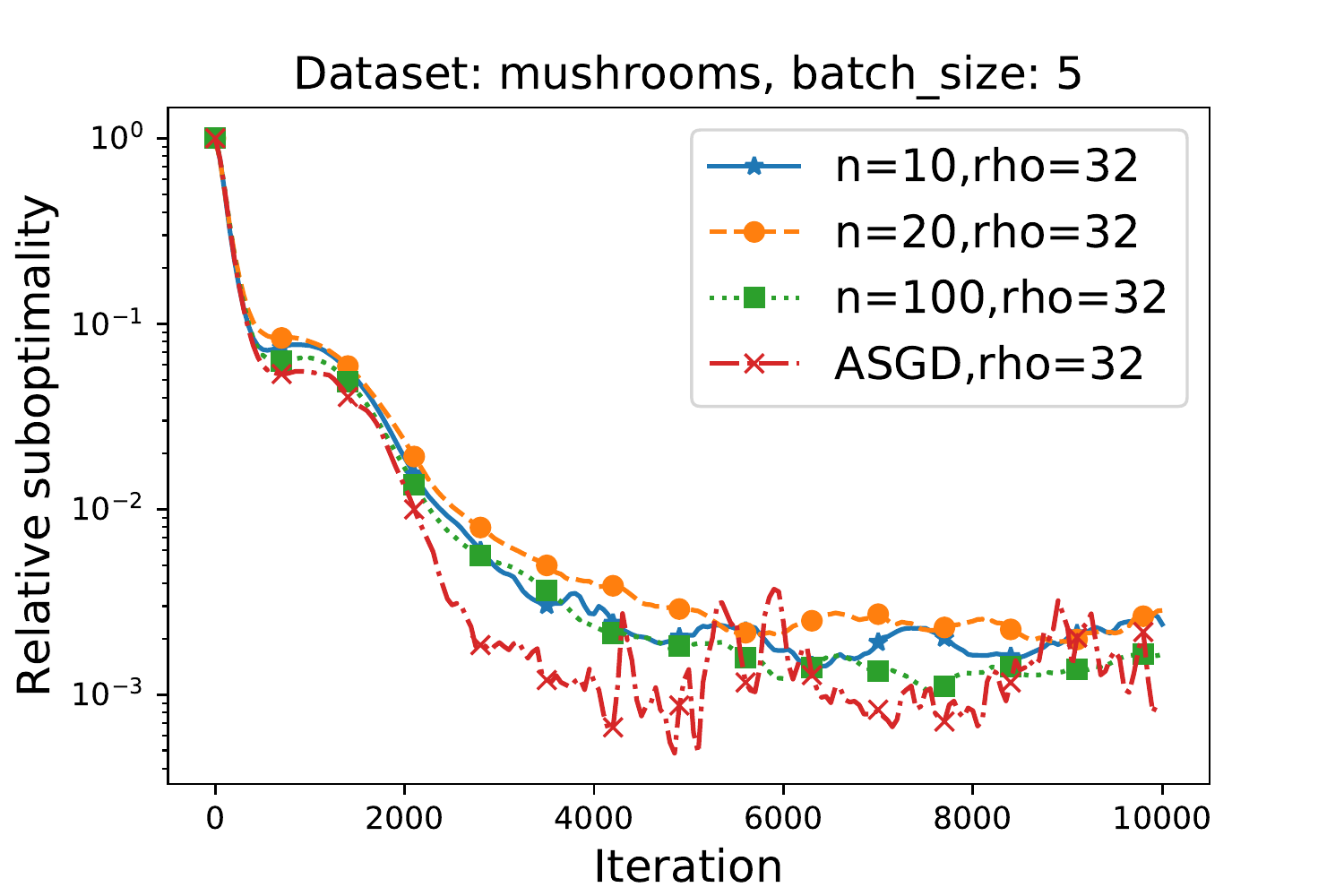}
\end{minipage}%
\begin{minipage}{0.33\textwidth}
  \centering
\includegraphics[width =  \textwidth ]{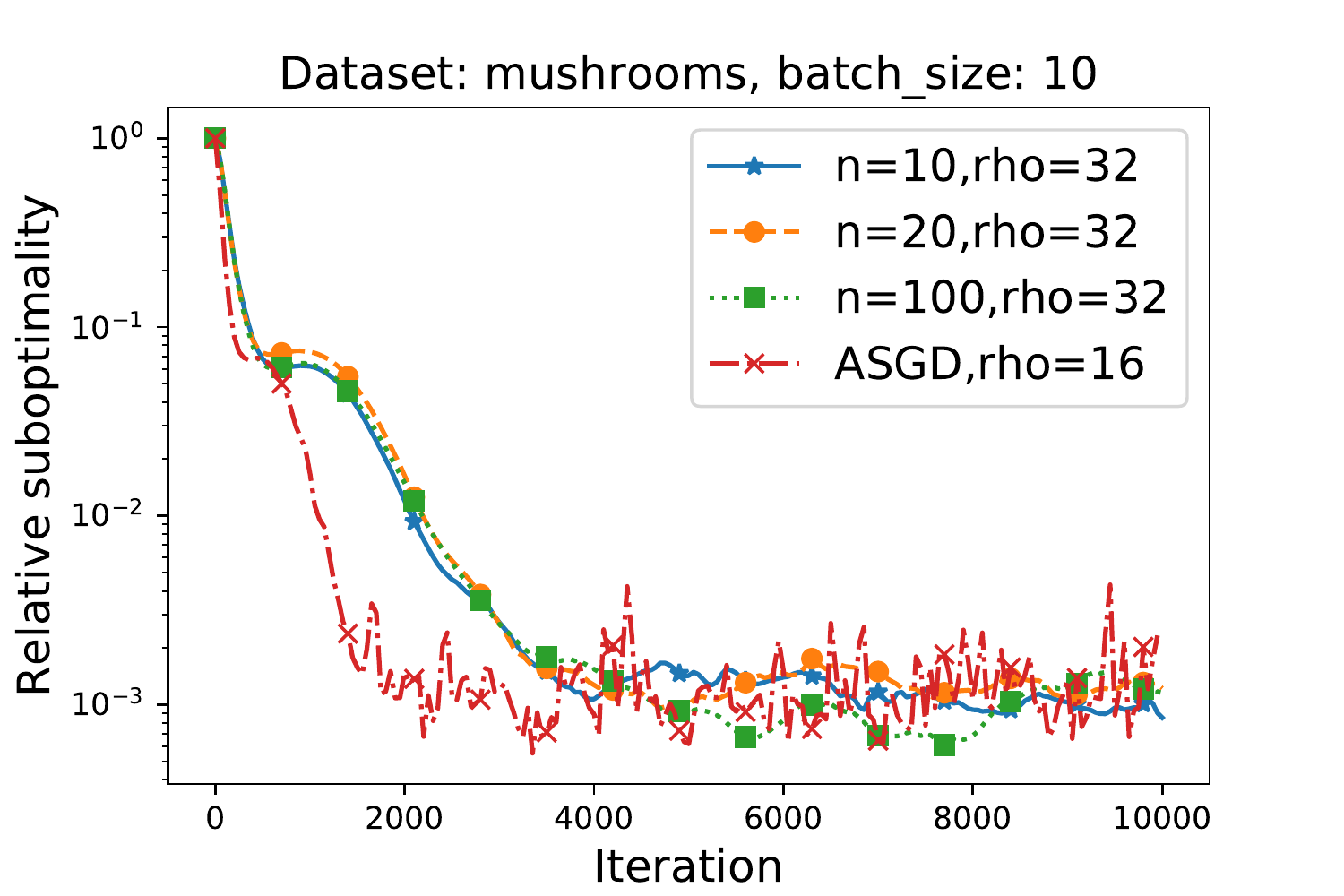}
\end{minipage}%
\\
\begin{minipage}{0.33\textwidth}
  \centering
\includegraphics[width =  \textwidth ]{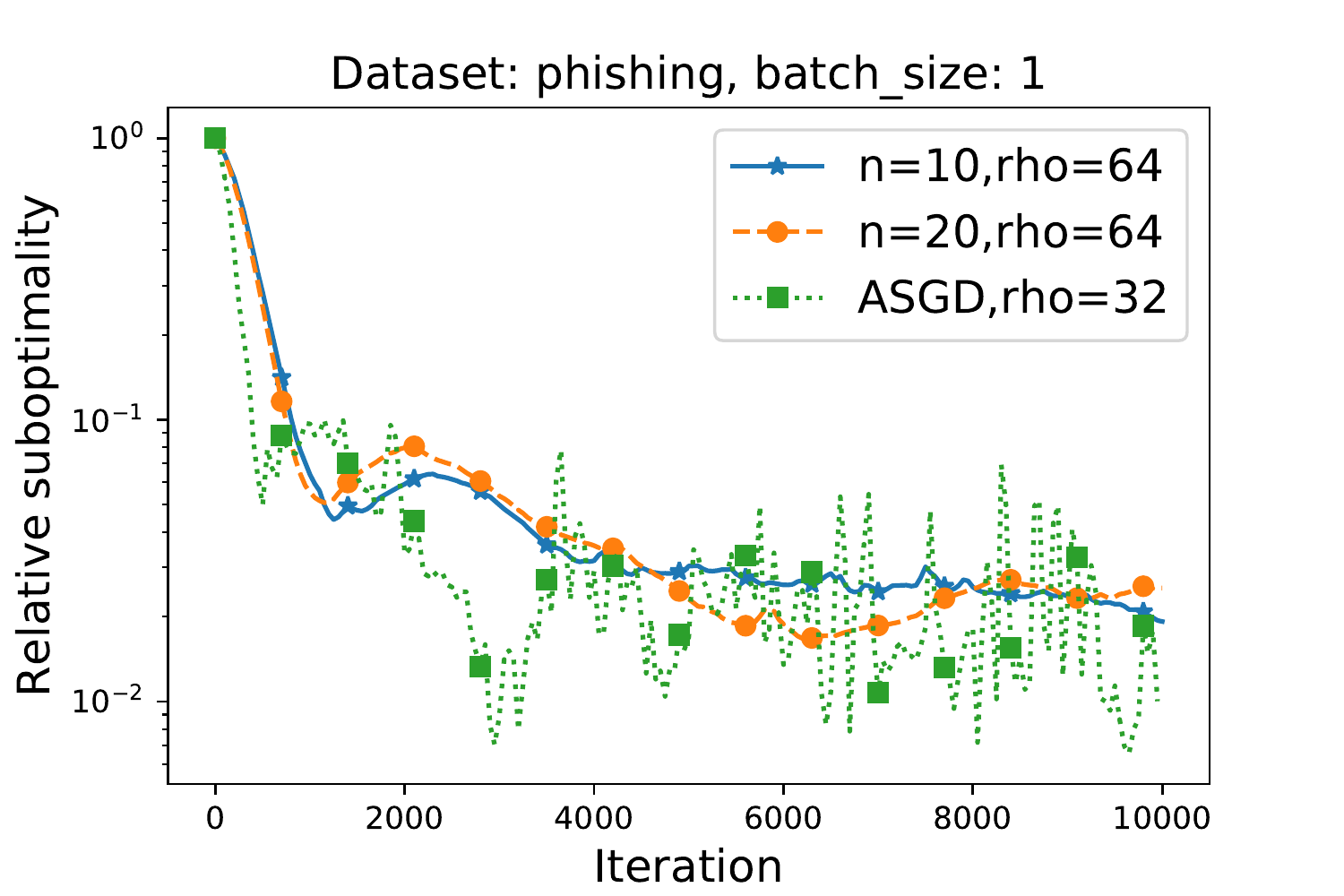}
\end{minipage}%
\begin{minipage}{0.33\textwidth}
  \centering
\includegraphics[width =  \textwidth ]{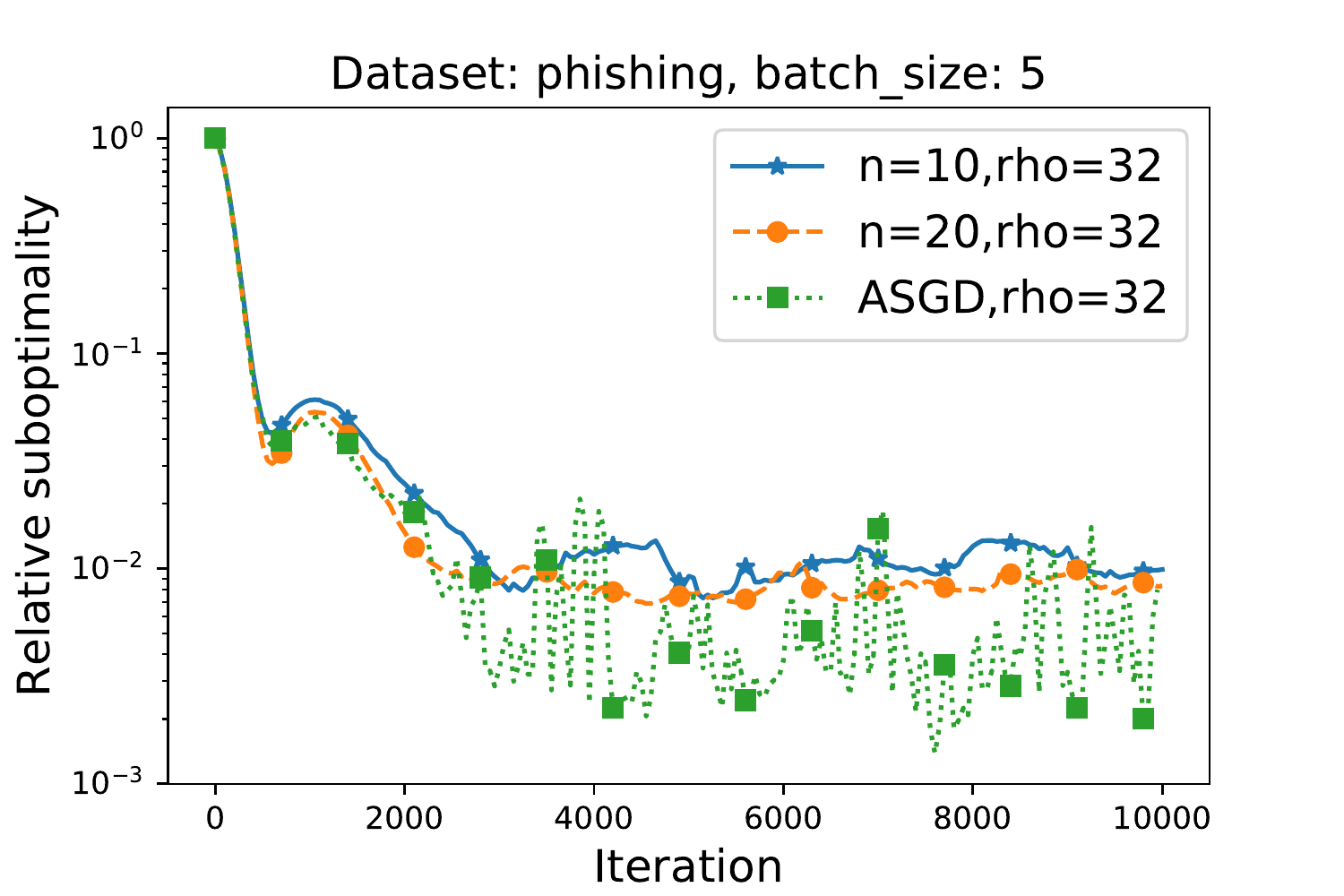}
\end{minipage}%
\begin{minipage}{0.33\textwidth}
  \centering
\includegraphics[width =  \textwidth ]{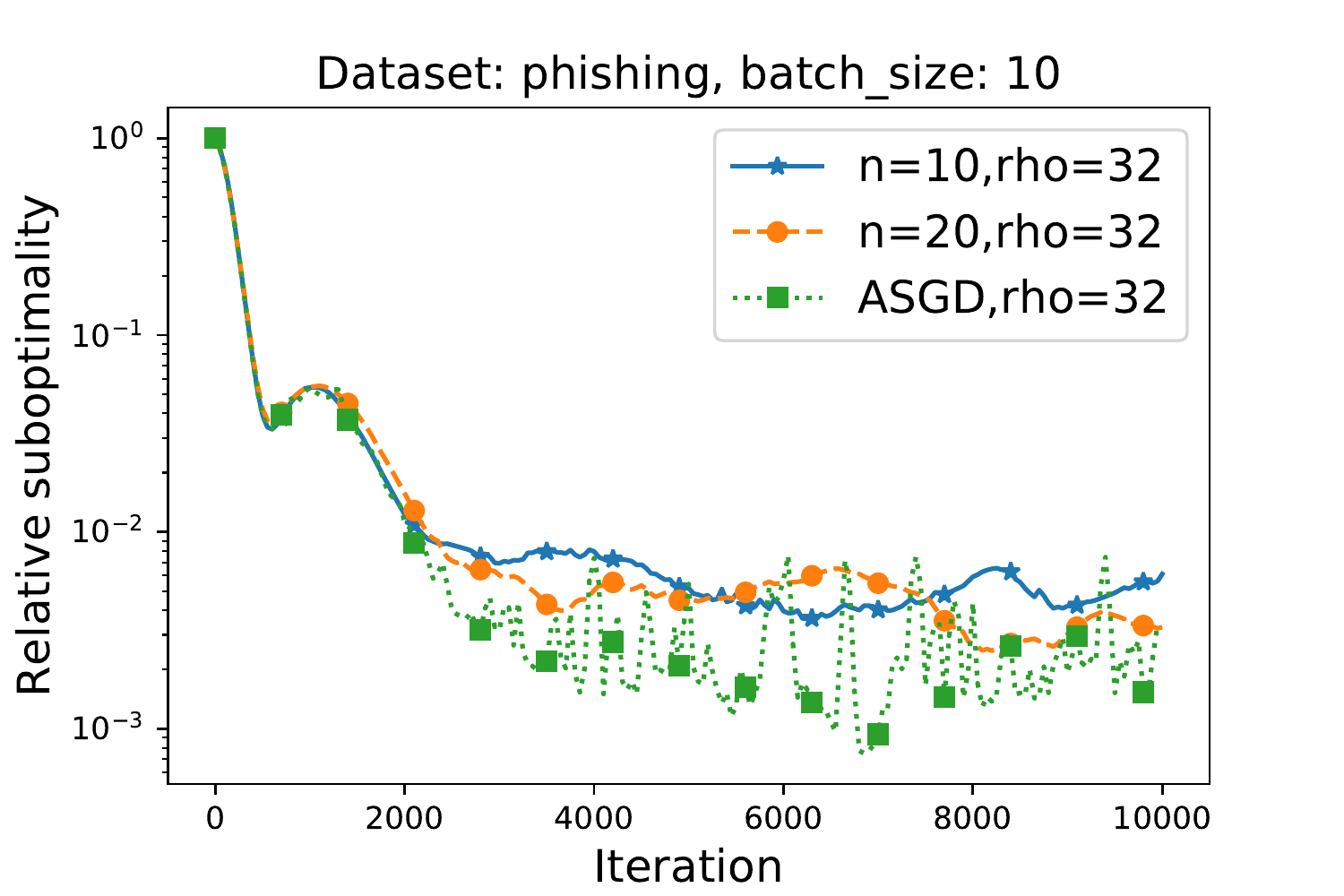}
\end{minipage}%
\\
\begin{minipage}{0.33\textwidth}
  \centering
\includegraphics[width =  \textwidth ]{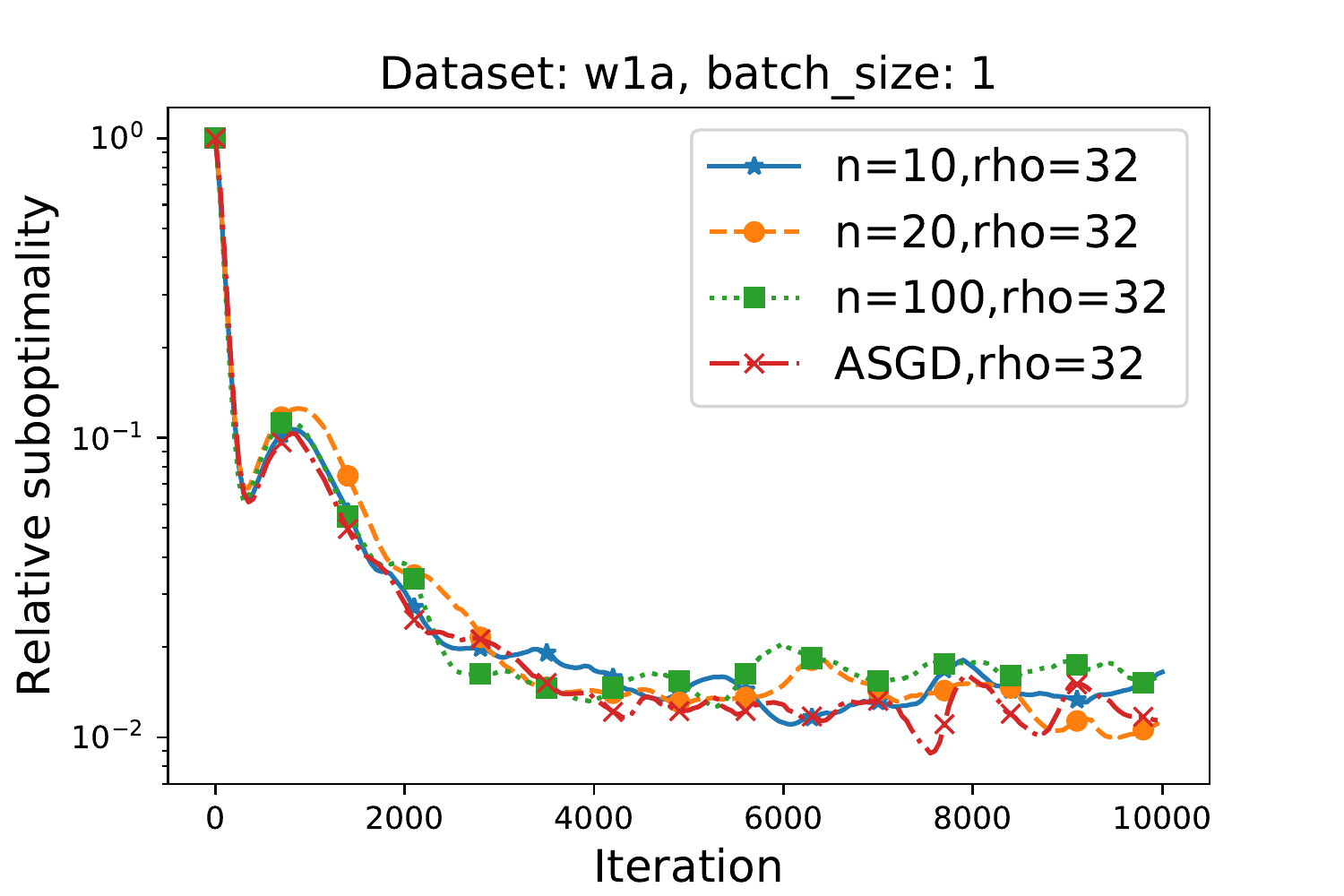}
\end{minipage}%
\begin{minipage}{0.33\textwidth}
  \centering
\includegraphics[width =  \textwidth ]{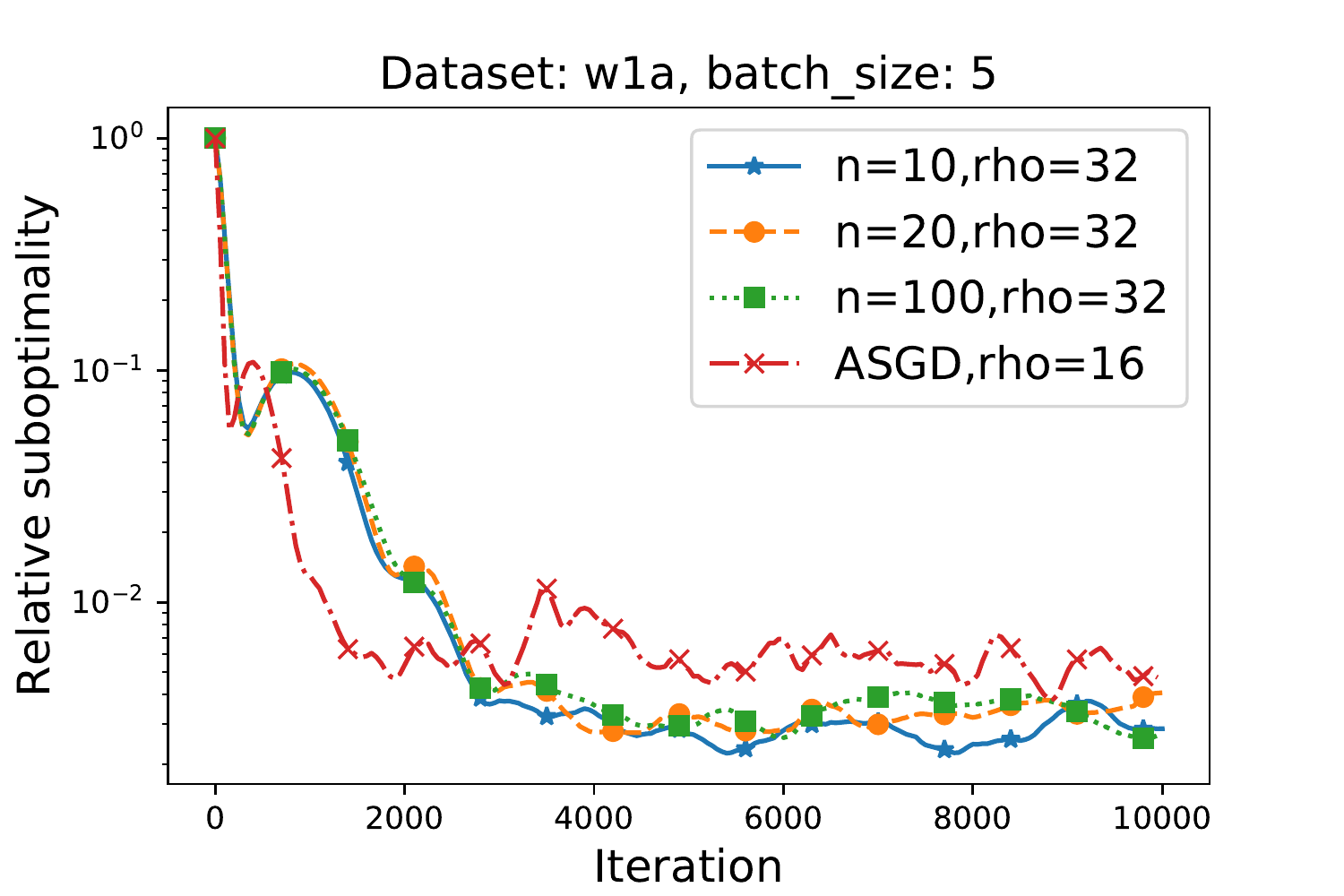}
\end{minipage}%
\begin{minipage}{0.33\textwidth}
  \centering
\includegraphics[width =  \textwidth ]{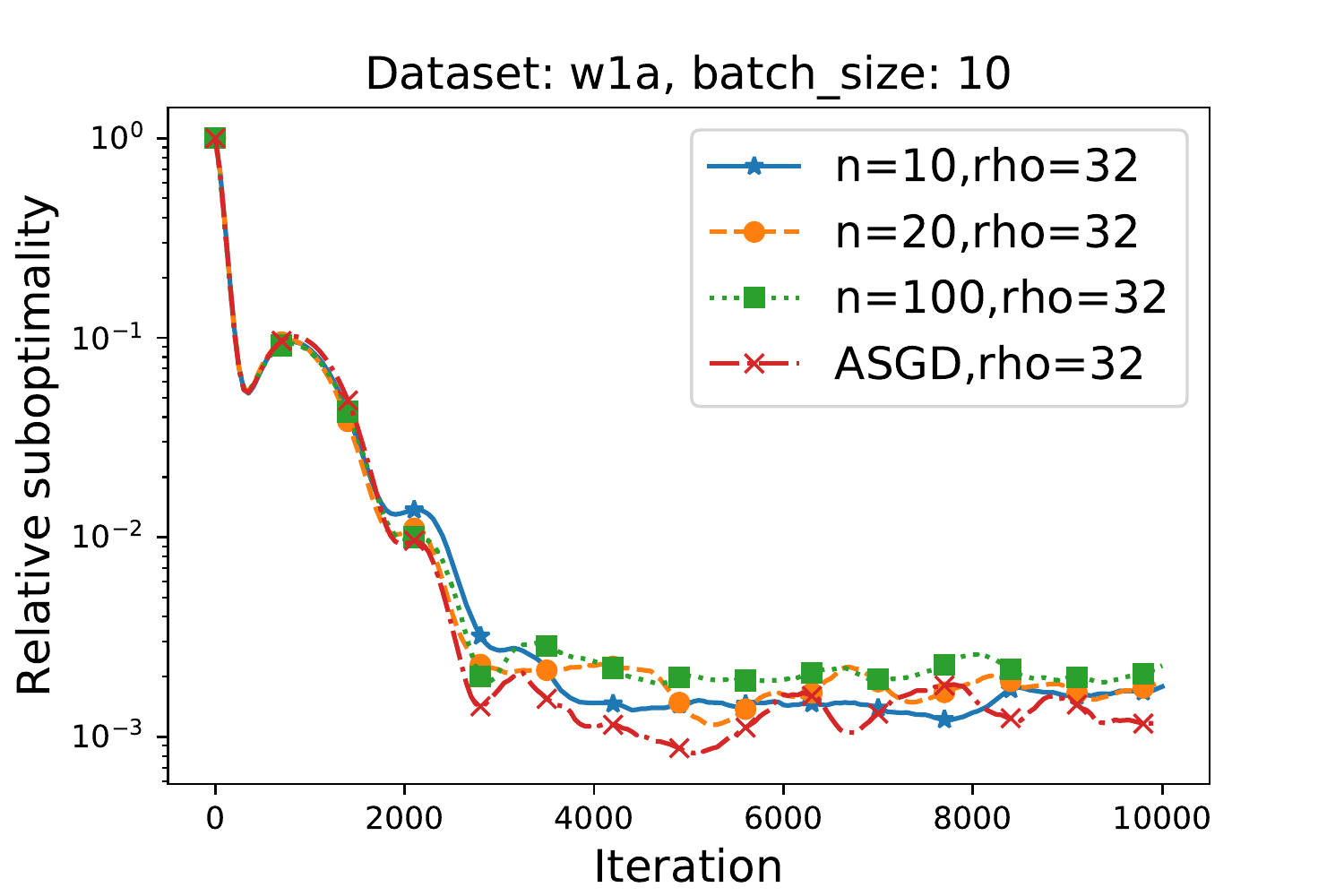}
\end{minipage}%
\\
\caption{Comparison of Algorithm~\ref{alg:acc} for various $(n,\tau)$ such that $n\tau=1$. Label ``ASGD'' corresponds to the choice $n=1, \tau = 1$. Label ``batch\_size'' indicates how big minibatch was chosen for stochastic gradient of each worker's objective. Parameter $\rho$ was chosen by grid search.  } \label{fig:acc1}
\end{figure}

Now, we once again check how different values of $\tau$ affect the convergence speed for several values of $n$. Figure~\ref{fig:acc2} presents the results. In every case, $\tau$ slightly influences the convergence rate (or the region where the iterates oscillate), although the effect is weaker for larger $n$. Note that theory predicts diminishing effect of $\tau$ only above $\frac{\bar{\rho}\tilde{\rho}}{n}$, in contrast to other sections, where the limit is $\frac1n$.

\begin{figure}[H]
\centering
\begin{minipage}{0.33\textwidth}
  \centering
\includegraphics[width =  \textwidth ]{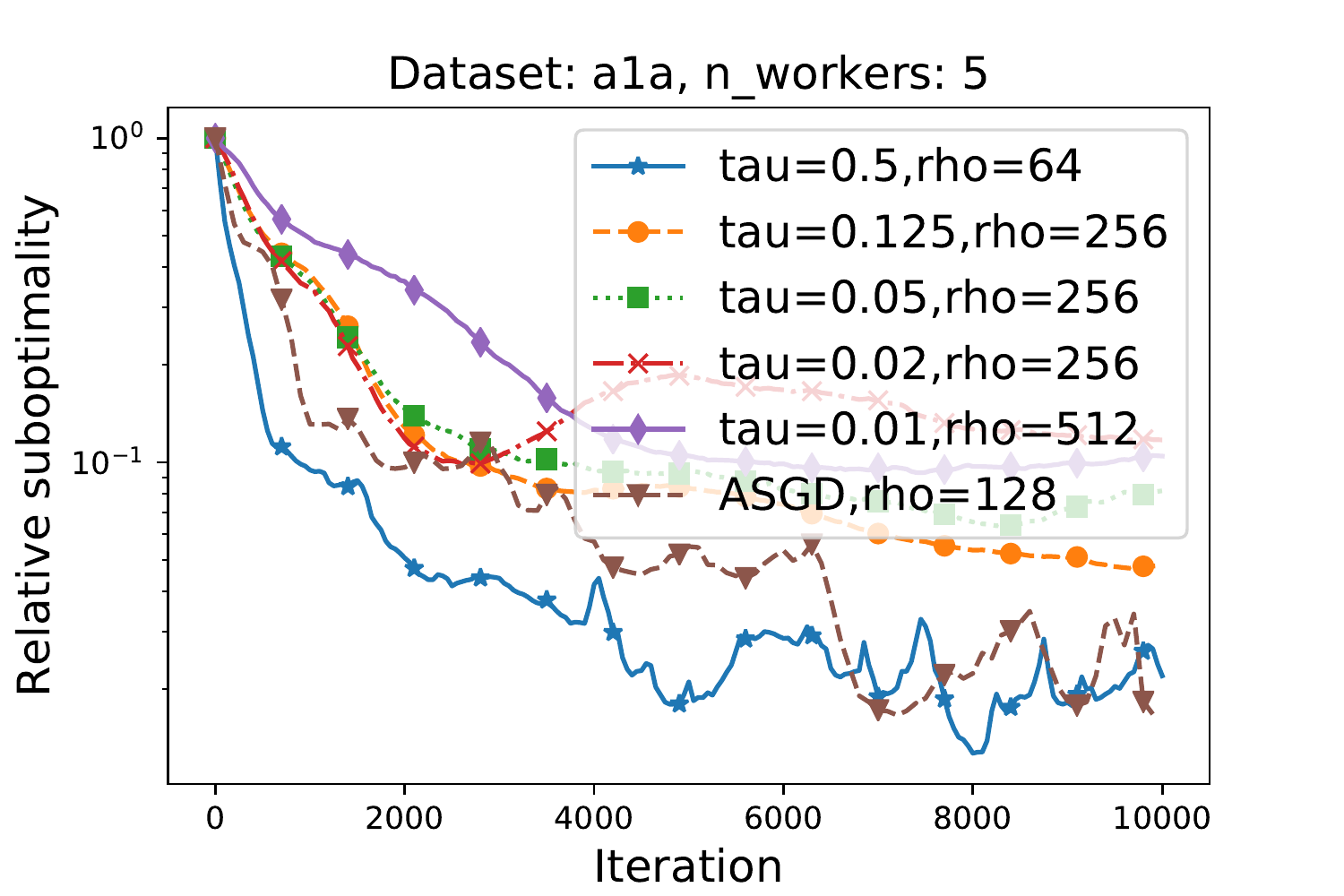}
\end{minipage}%
\begin{minipage}{0.33\textwidth}
  \centering
\includegraphics[width =  \textwidth ]{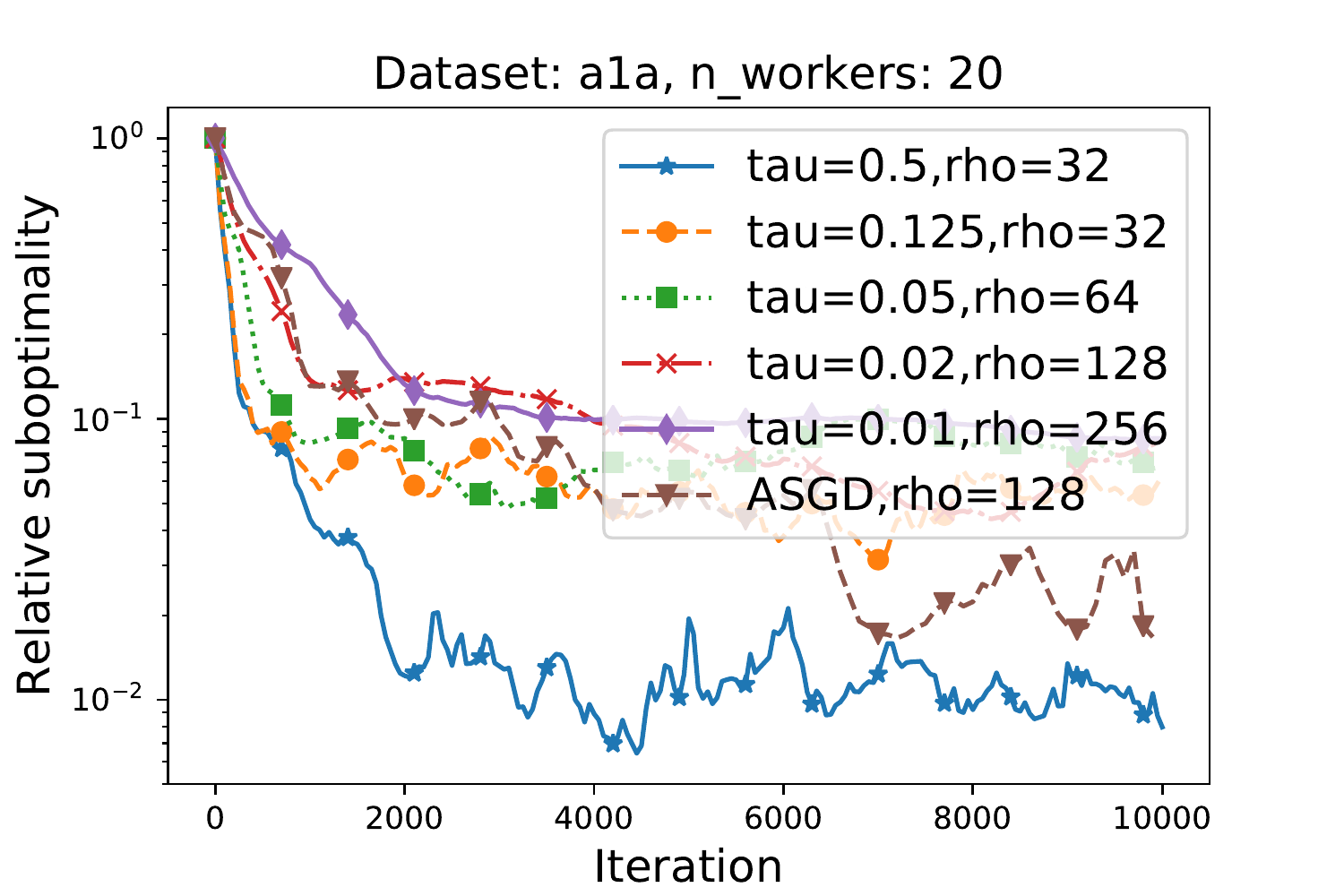}
\end{minipage}%
\begin{minipage}{0.33\textwidth}
  \centering
\includegraphics[width =  \textwidth ]{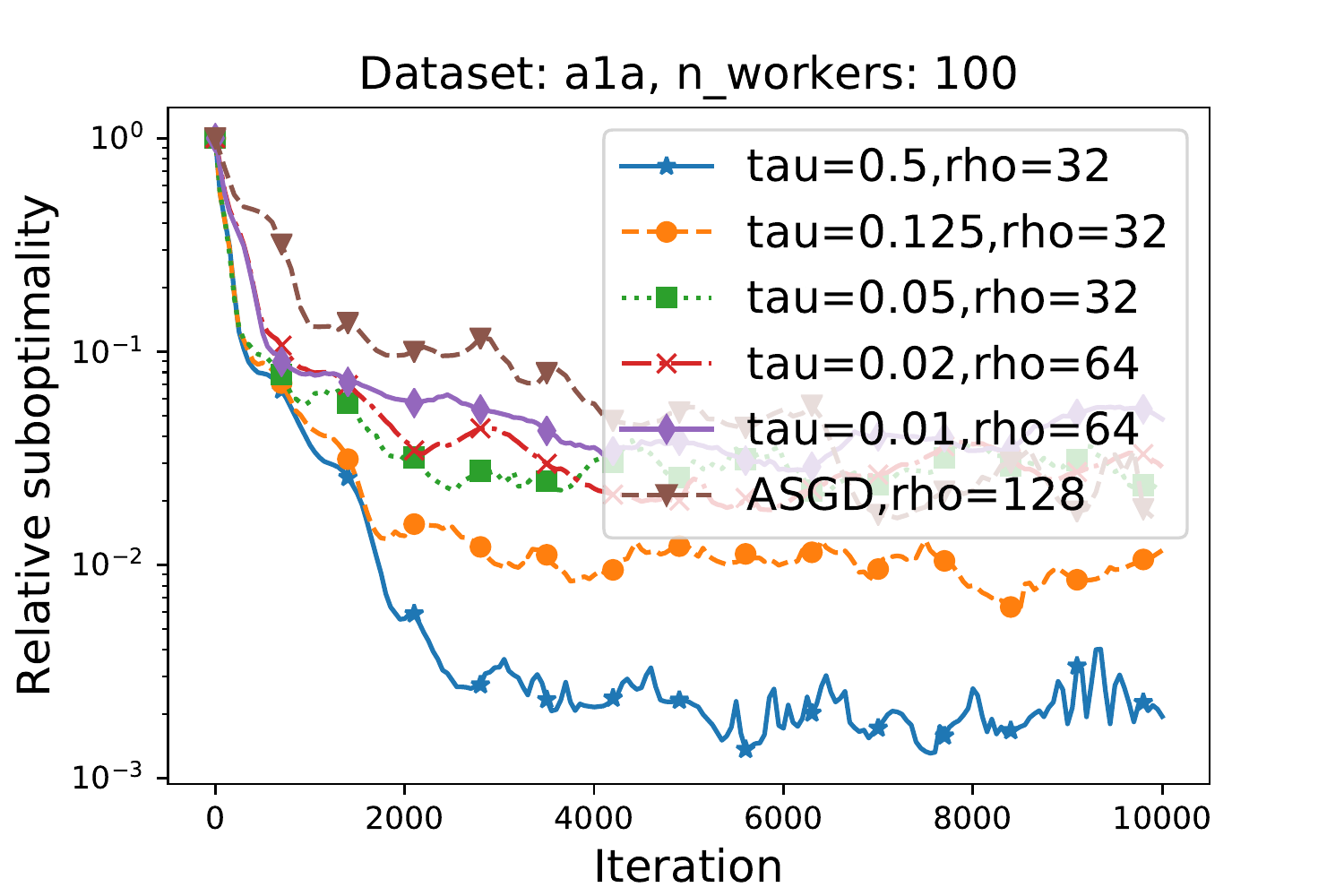}
\end{minipage}%
\\
\begin{minipage}{0.33\textwidth}
  \centering
\includegraphics[width =  \textwidth ]{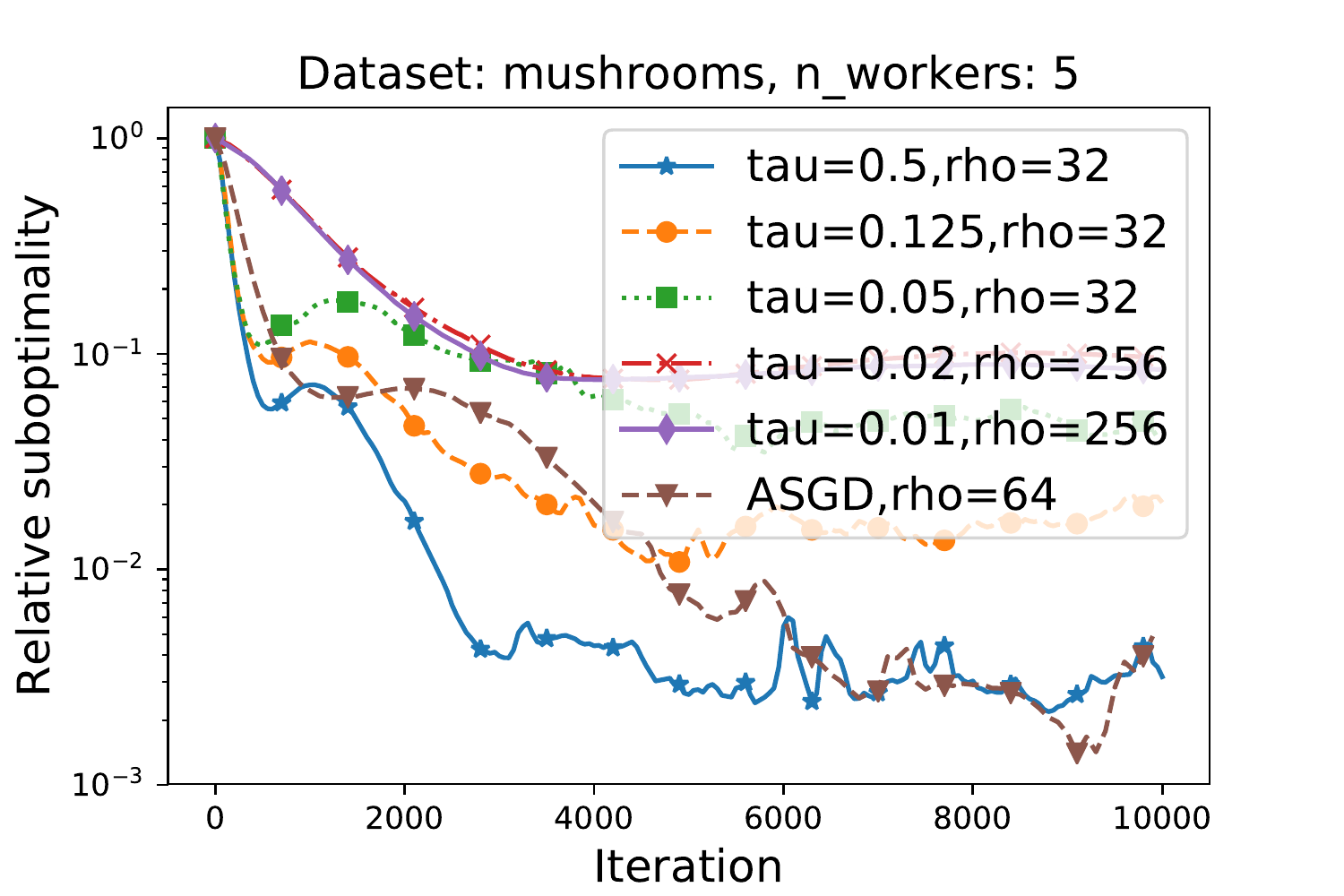}
\end{minipage}%
\begin{minipage}{0.33\textwidth}
  \centering
\includegraphics[width =  \textwidth ]{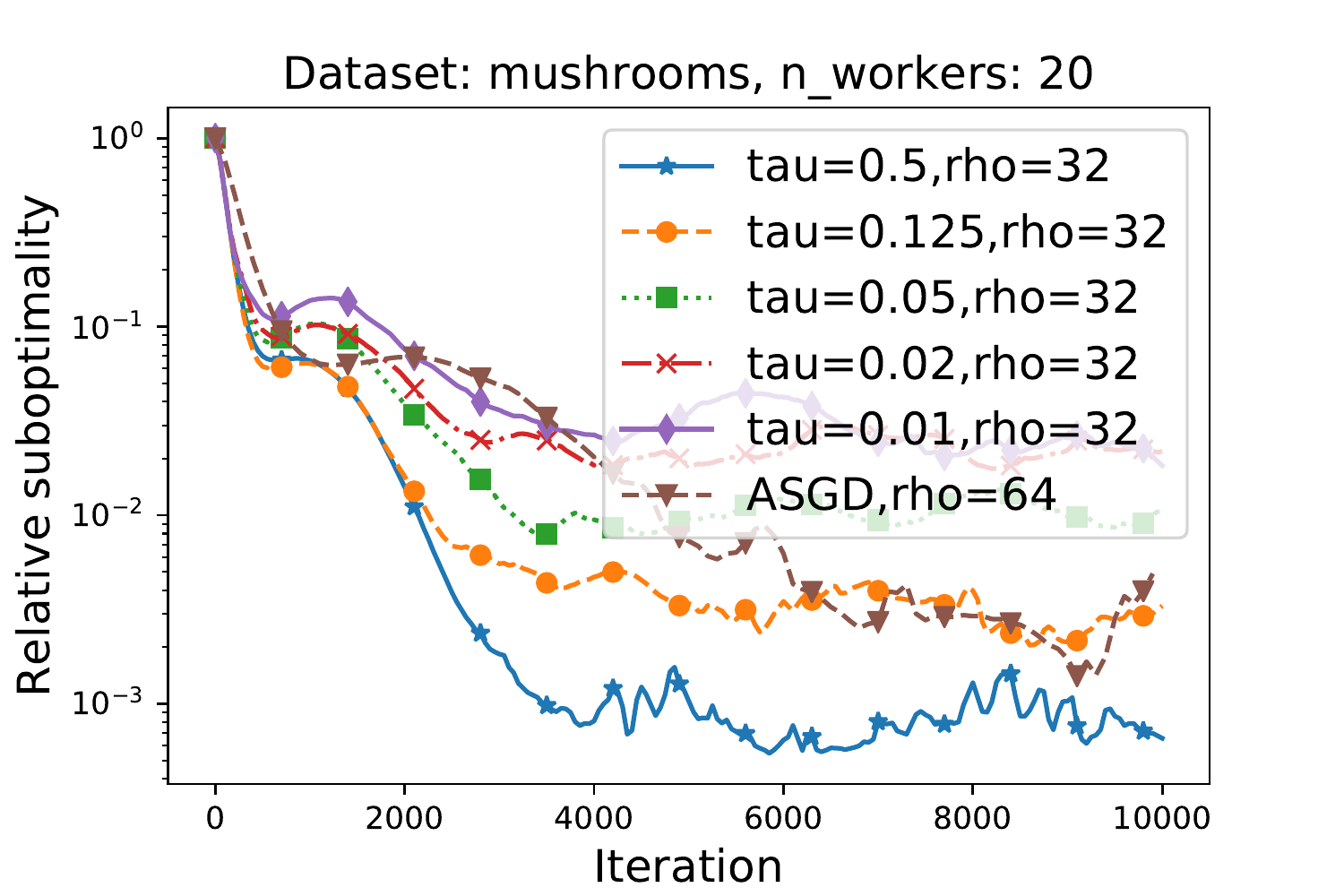}
\end{minipage}%
\begin{minipage}{0.33\textwidth}
  \centering
\includegraphics[width =  \textwidth ]{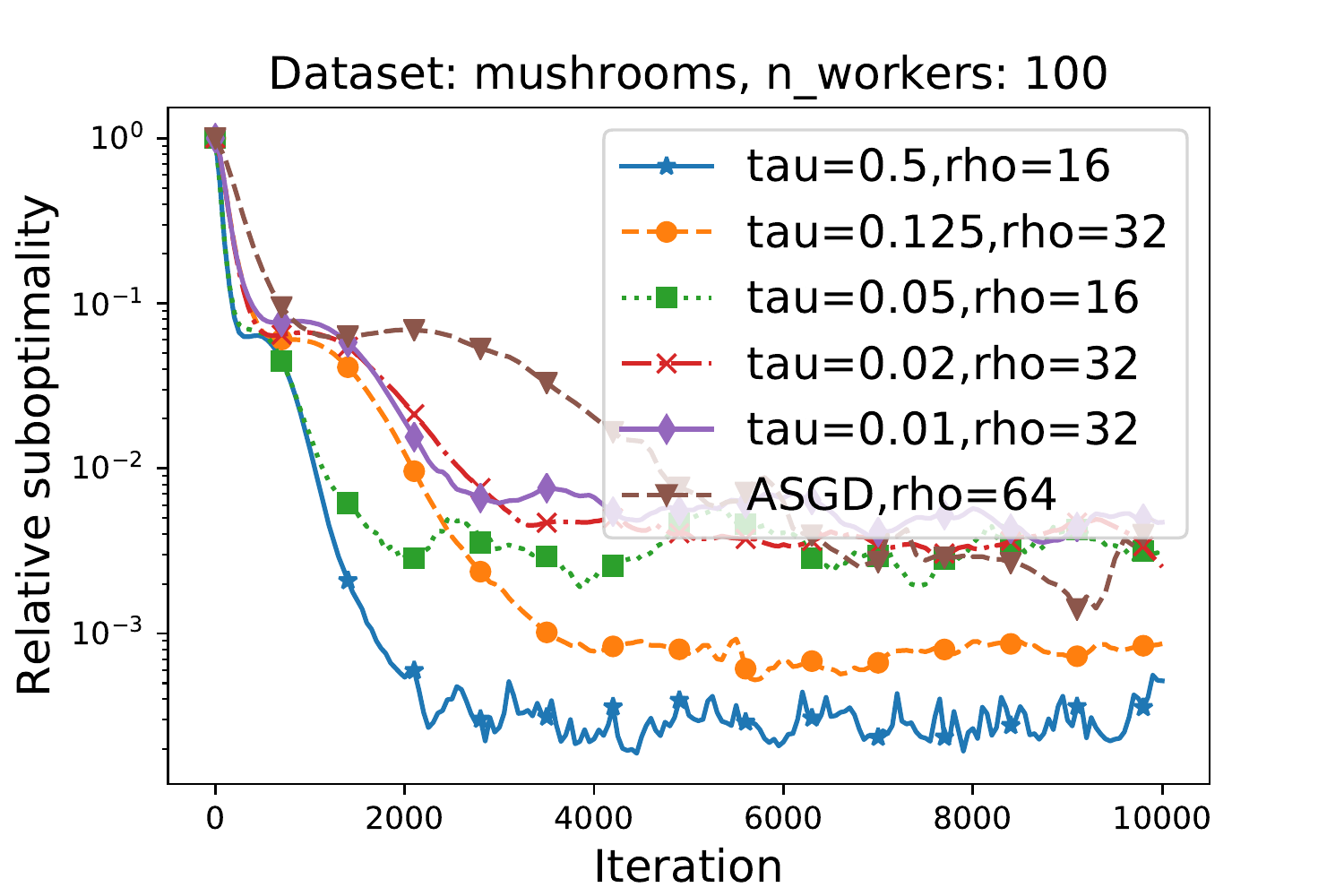}
\end{minipage}%
\\
\begin{minipage}{0.33\textwidth}
  \centering
\includegraphics[width =  \textwidth ]{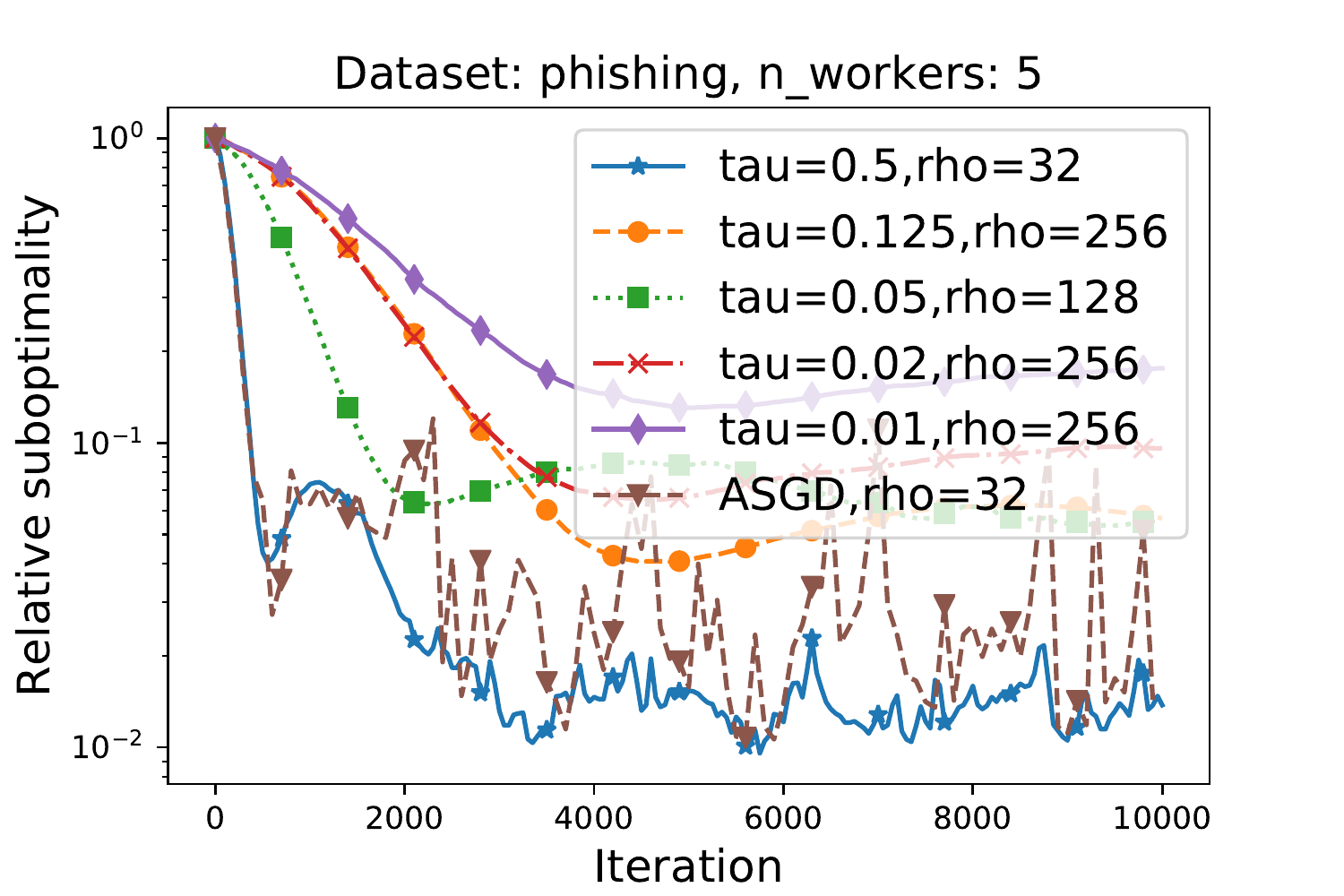}
\end{minipage}%
\begin{minipage}{0.33\textwidth}
  \centering
\includegraphics[width =  \textwidth ]{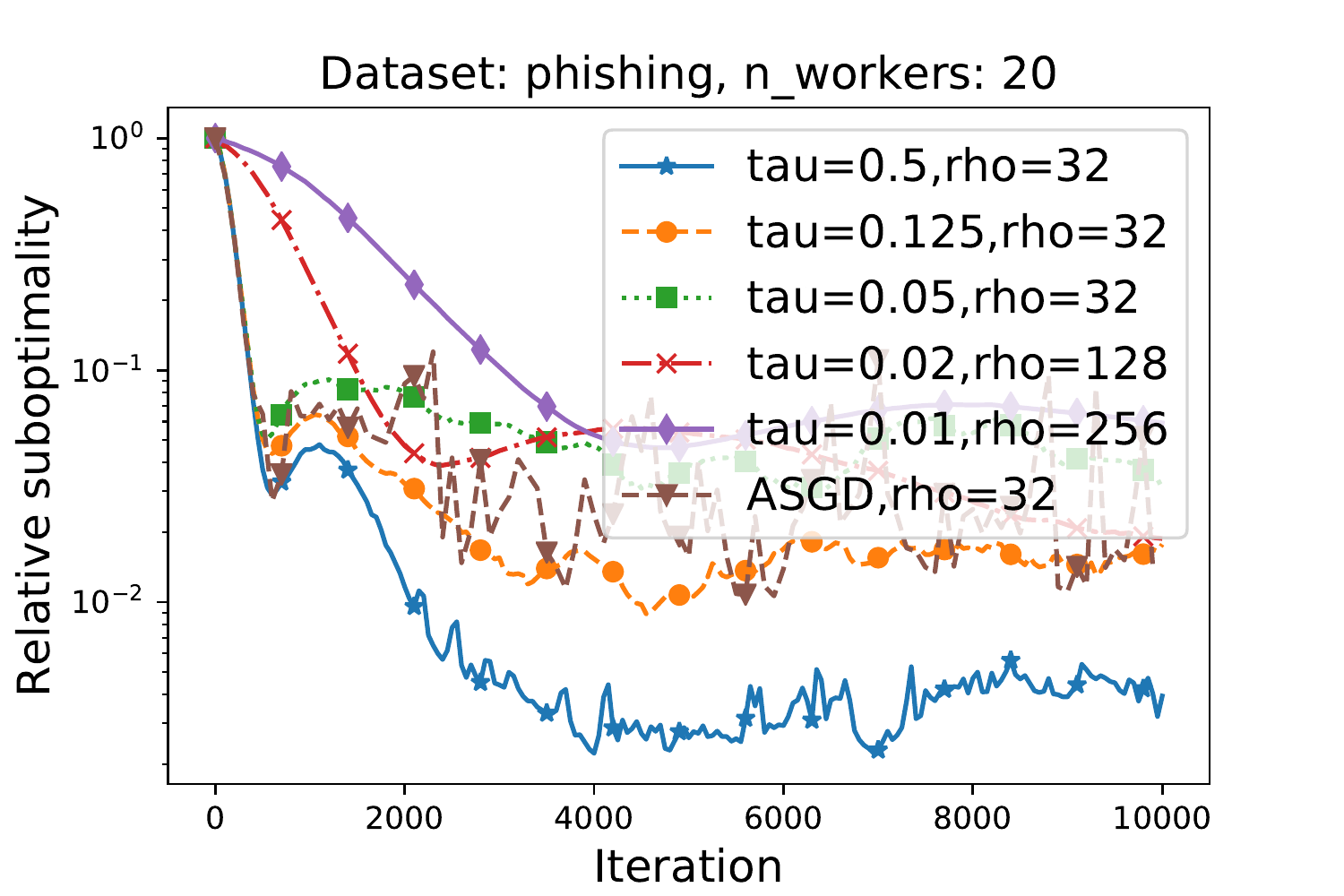}
\end{minipage}%
\\
\begin{minipage}{0.33\textwidth}
  \centering
\includegraphics[width =  \textwidth ]{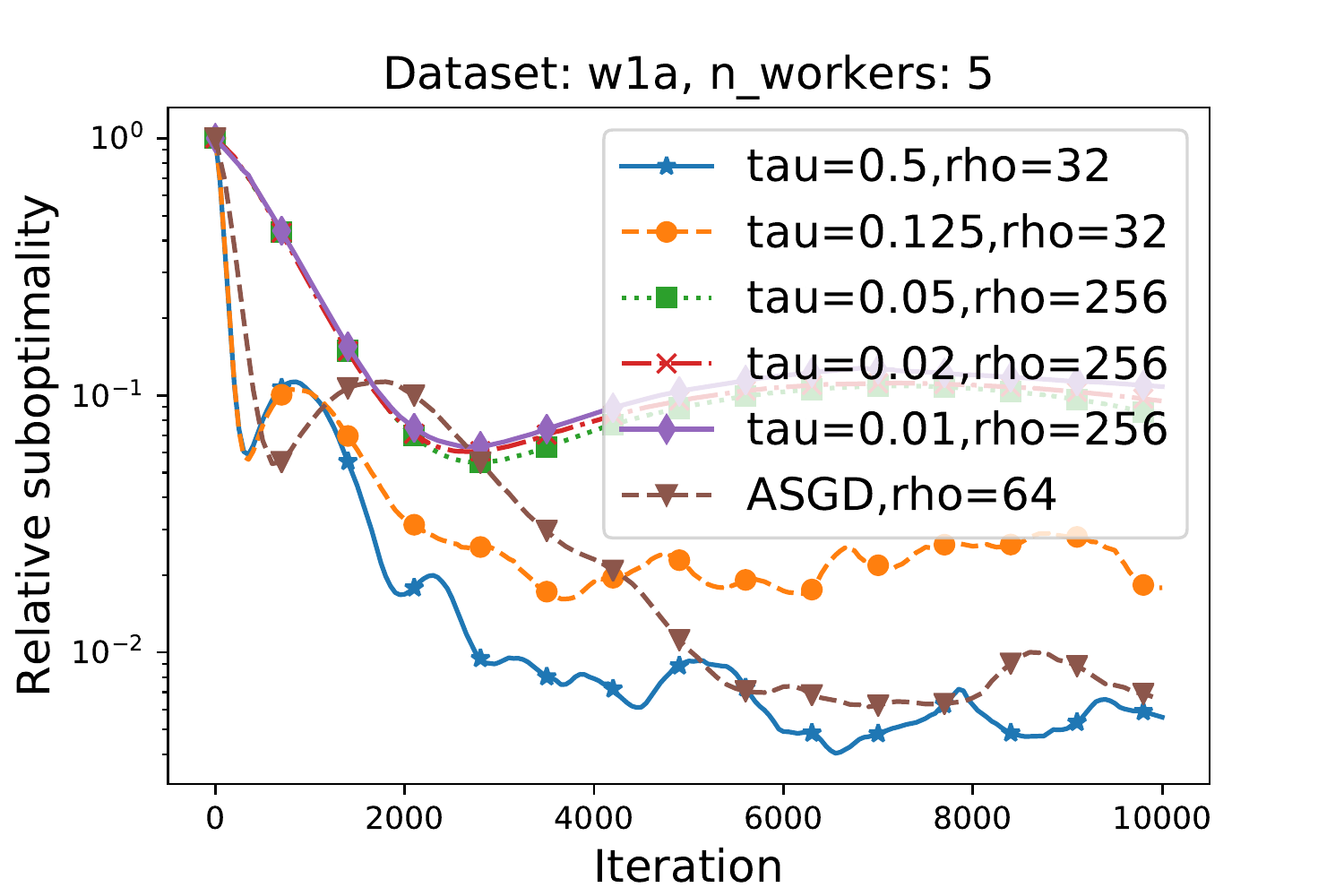}
\end{minipage}%
\begin{minipage}{0.33\textwidth}
  \centering
\includegraphics[width =  \textwidth ]{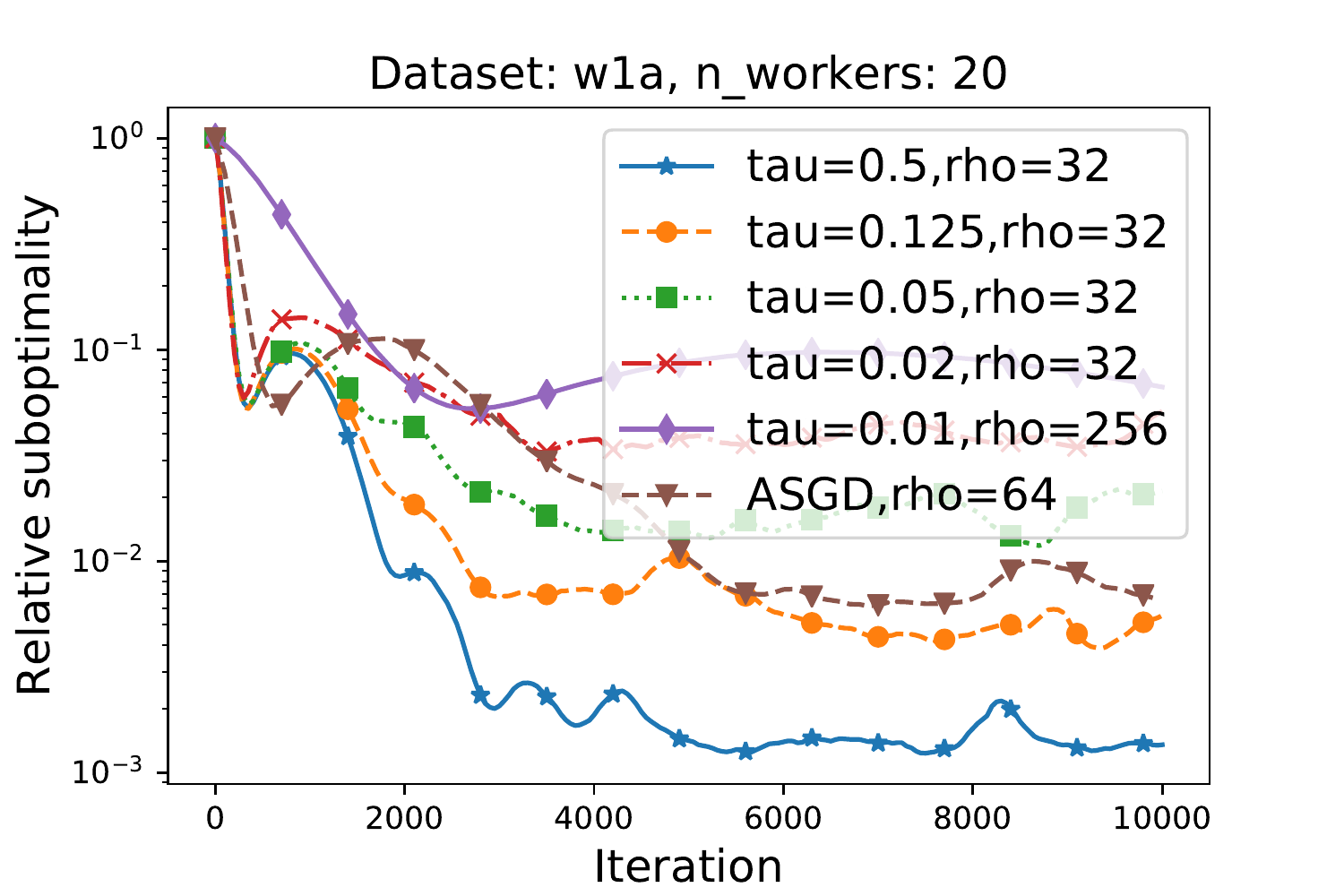}
\end{minipage}%
\begin{minipage}{0.33\textwidth}
  \centering
\includegraphics[width =  \textwidth ]{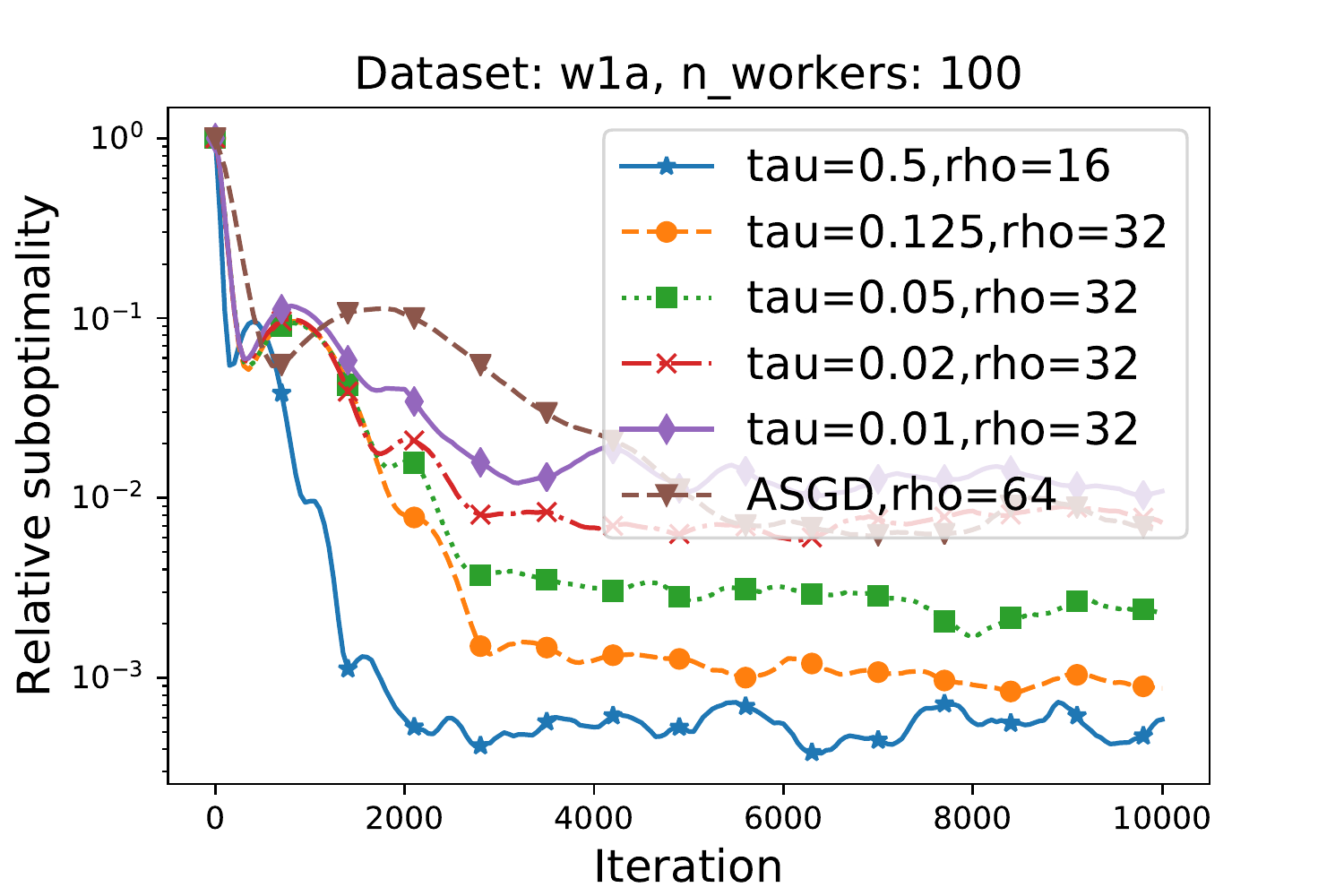}
\end{minipage}%
\\
\caption{Behavior of Algorithm~\ref{alg:acc} while varying $\tau$. Label ``ASGD'' corresponds to the choice $n=1, \tau = 1$. Parameter $\rho$ was chosen by grid search.  } \label{fig:acc2}
\end{figure}

\subsection{ISAGA \label{sec:exp_saga}}
We also study the convergence of shared data ISAGA -- Algorithm~\ref{alg:saga}. As previously, first experiment compares default SAGA against Algorithm~\ref{alg:saga} for various values of $n$ with $\tau = n^{-1}$. Again, the results (Figure~\ref{fig:saga1}\footnote{Figure~\ref{fig:saga1} is identical to Figure~\ref{fig:saga_main}. We present it again for completeness.}) shows what theory claims -- setting $n\tau=1$ does not violate a convergence rate of the original SAGA. 

\begin{figure}[H]
\centering
\begin{minipage}{0.24\textwidth}
  \centering
\includegraphics[width =  \textwidth ]{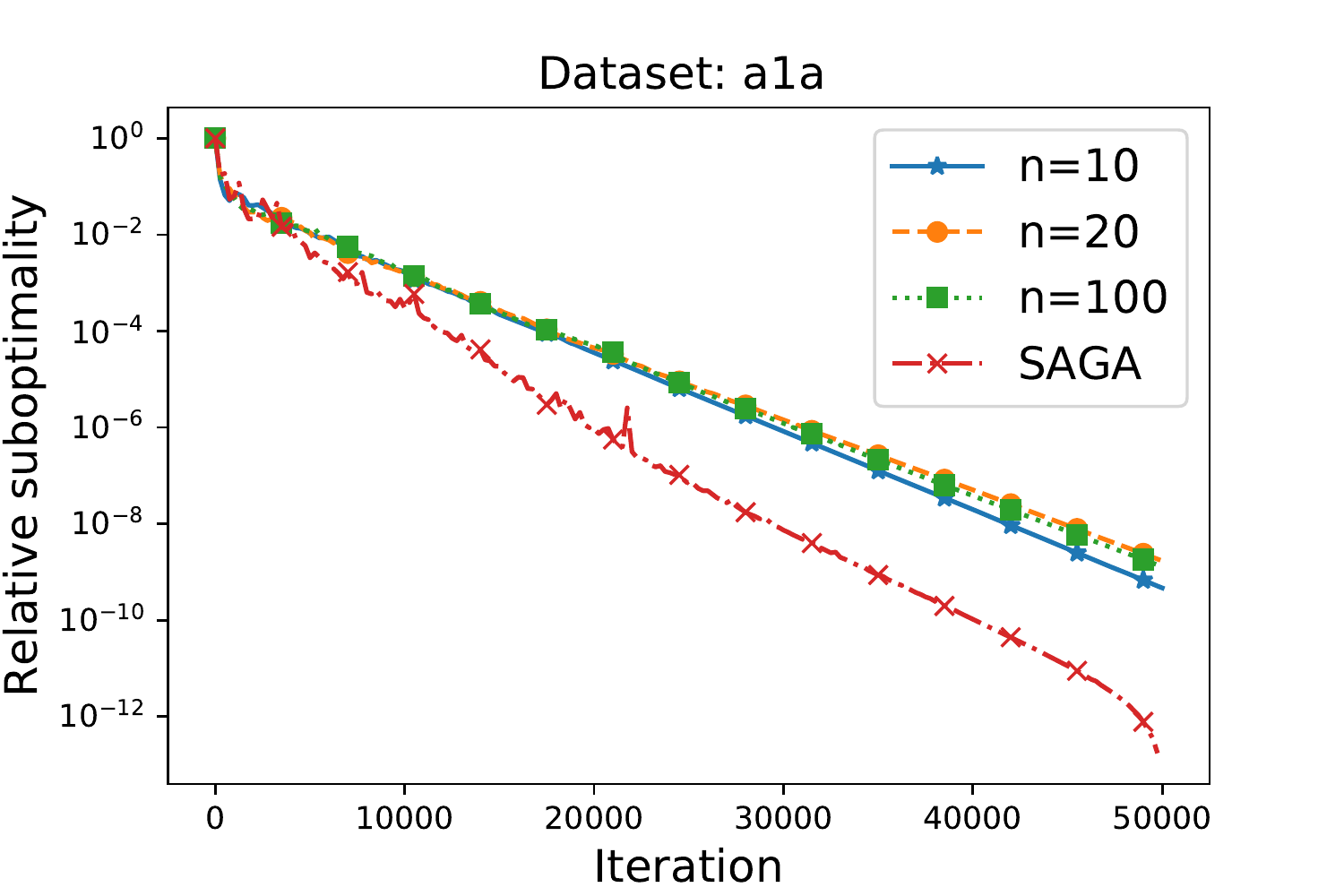}
\end{minipage}%
\begin{minipage}{0.24\textwidth}
  \centering
\includegraphics[width =  \textwidth ]{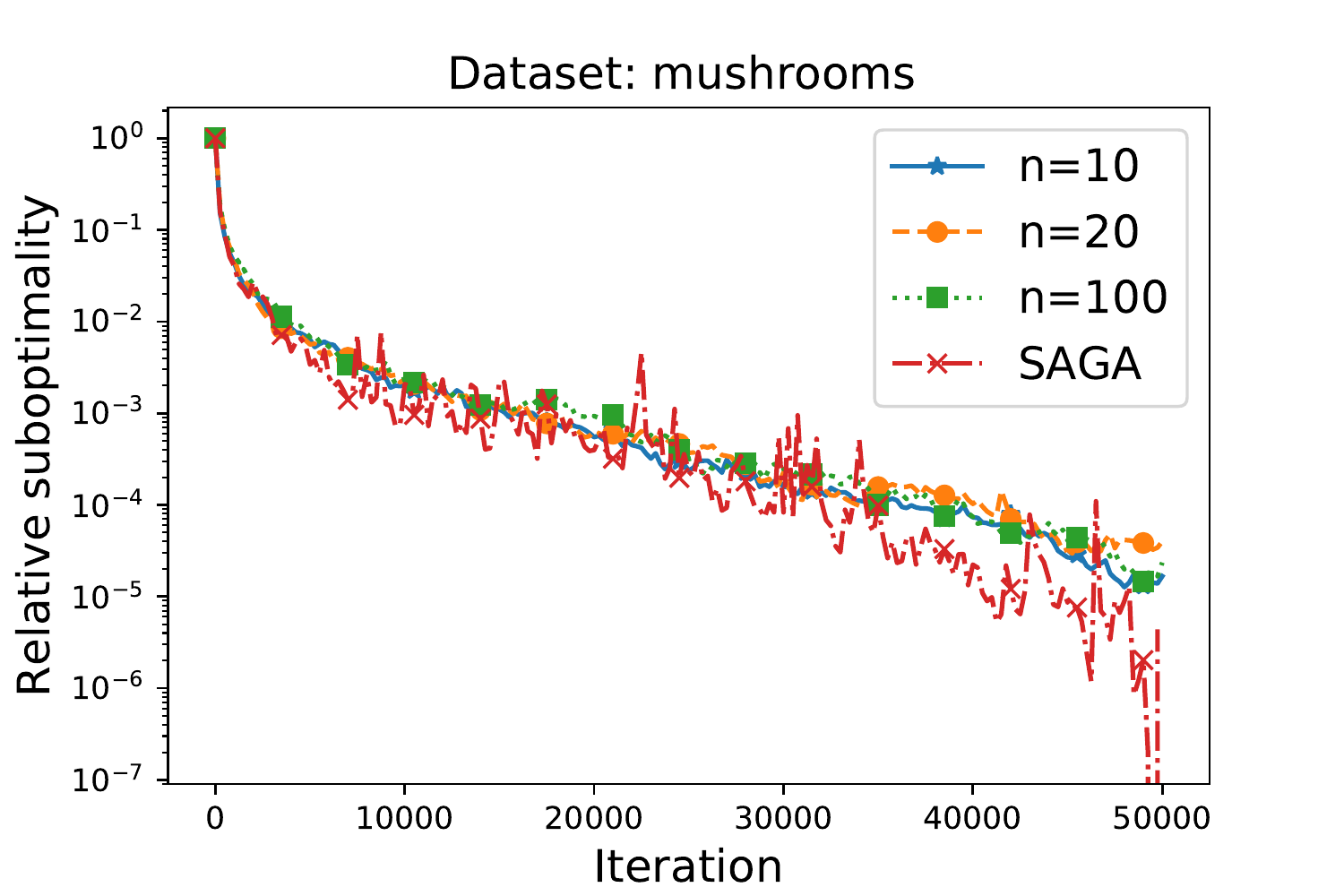}
\end{minipage}%
\begin{minipage}{0.24\textwidth}
  \centering
\includegraphics[width =  \textwidth ]{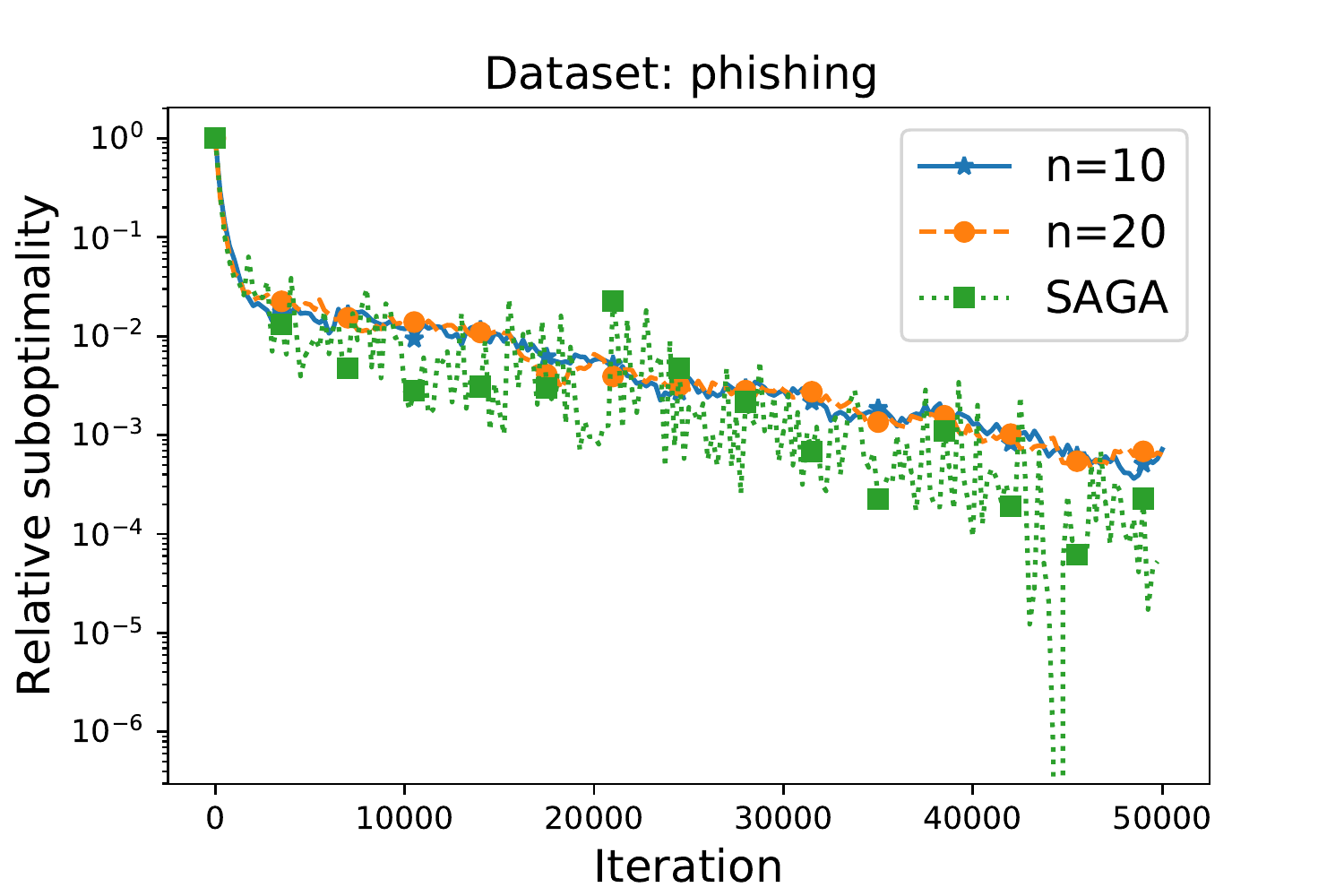}
\end{minipage}%
\begin{minipage}{0.24\textwidth}
  \centering
\includegraphics[width =  \textwidth ]{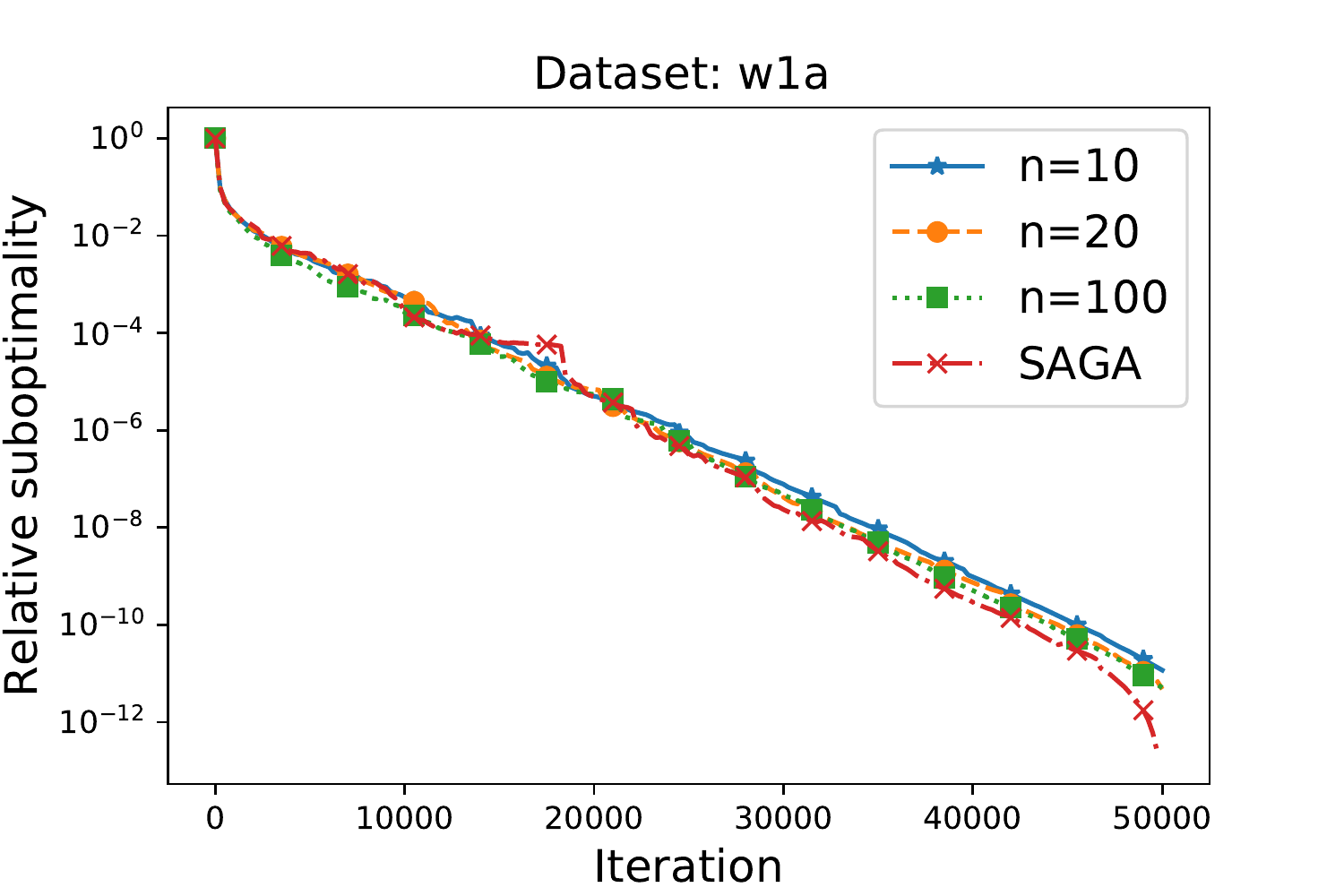}
\end{minipage}%
\\
\caption{Comparison of SAGA and Algorithm~\ref{alg:saga} for various values $n$ and $\tau=n^{-1}$. Stepsize $\gamma = \frac{1}{L(3n^{-1}+\tau)}$ is chosen in each case. } \label{fig:saga1}
\end{figure}

The second experiment of this section shows the convergence behavior for varying $\tau$ of Algorithm~\ref{alg:saga}. The results (Figure~\ref{fig:saga2}) show that, for small $n$, the ratio of coordinates $\tau$ affects the speed heavily. However, as $n$ increases, the effect of $\tau$ is diminishing. 

\begin{figure}[H]
\centering
\begin{minipage}{0.33\textwidth}
  \centering
\includegraphics[width =  \textwidth ]{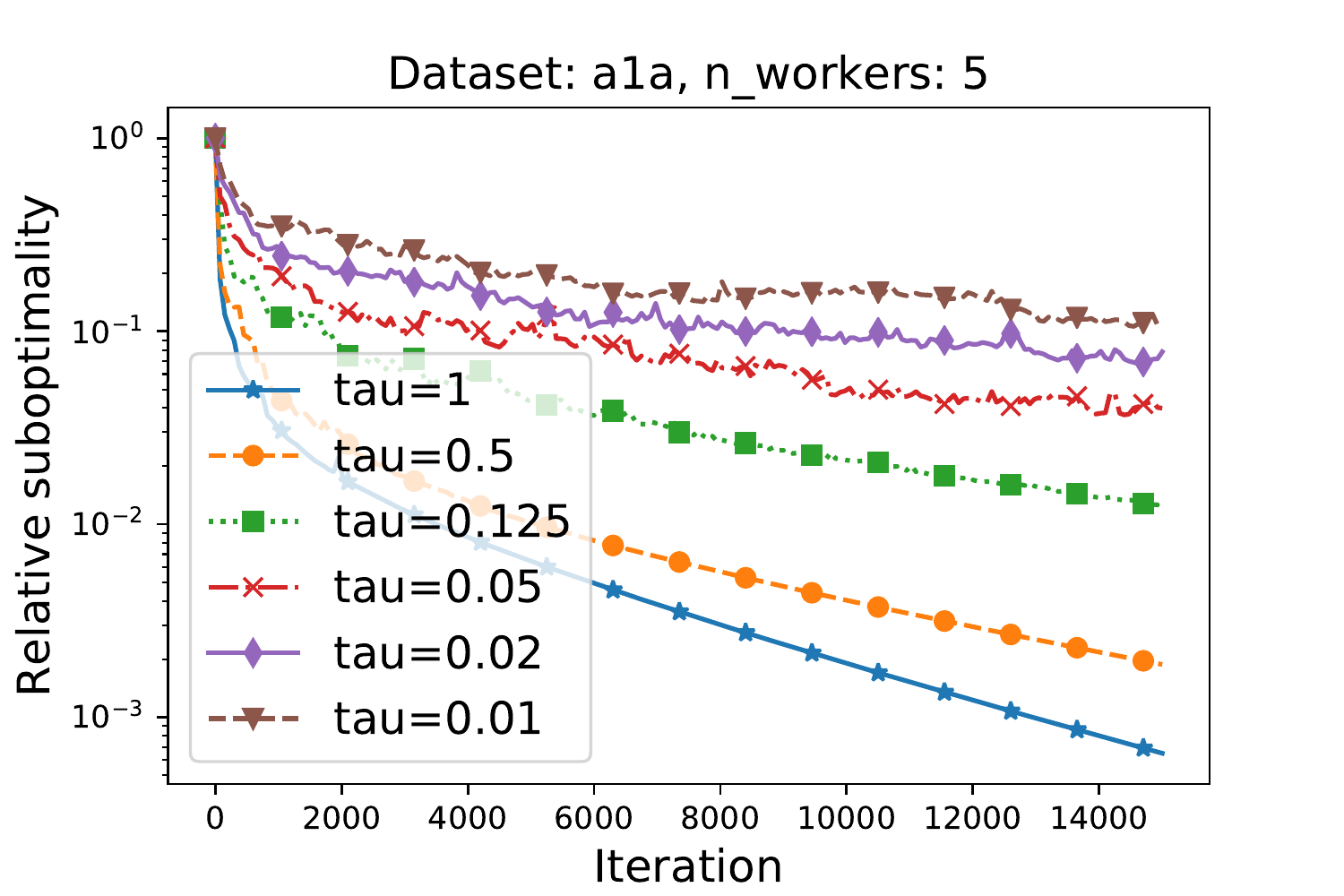}
\end{minipage}%
\begin{minipage}{0.33\textwidth}
  \centering
\includegraphics[width =  \textwidth ]{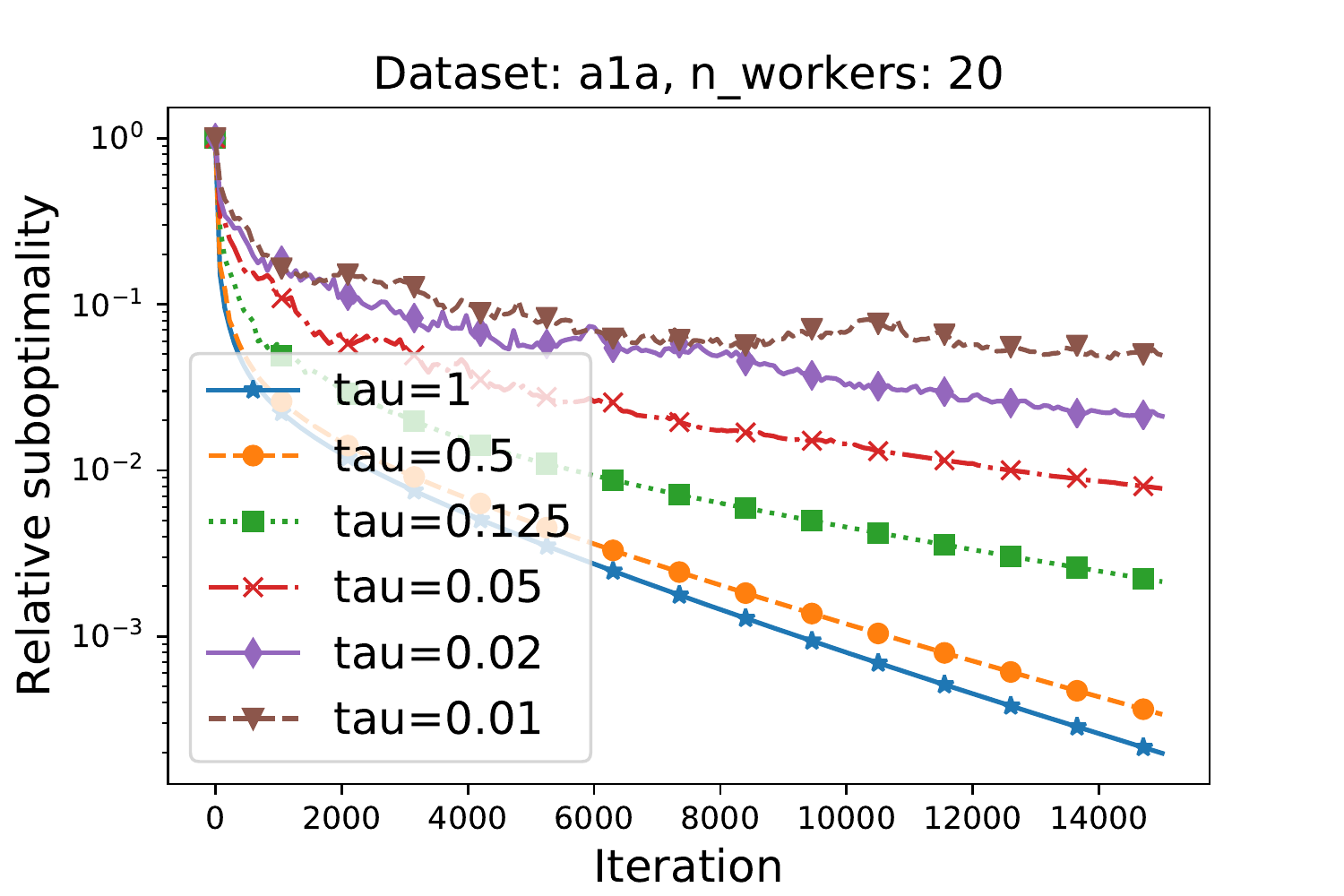}
\end{minipage}%
\begin{minipage}{0.33\textwidth}
  \centering
\includegraphics[width =  \textwidth ]{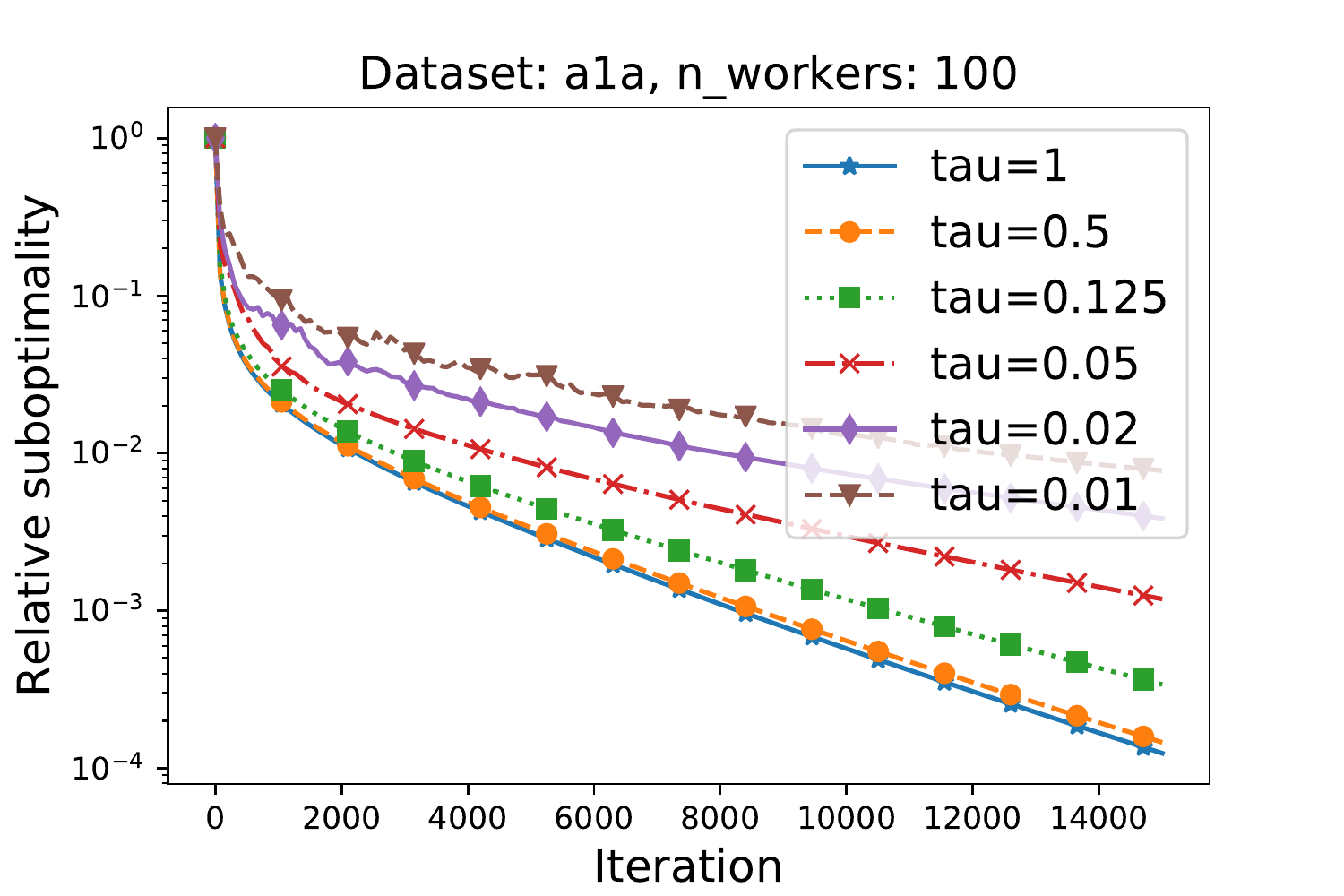}
\end{minipage}%
\\
\begin{minipage}{0.33\textwidth}
  \centering
\includegraphics[width =  \textwidth ]{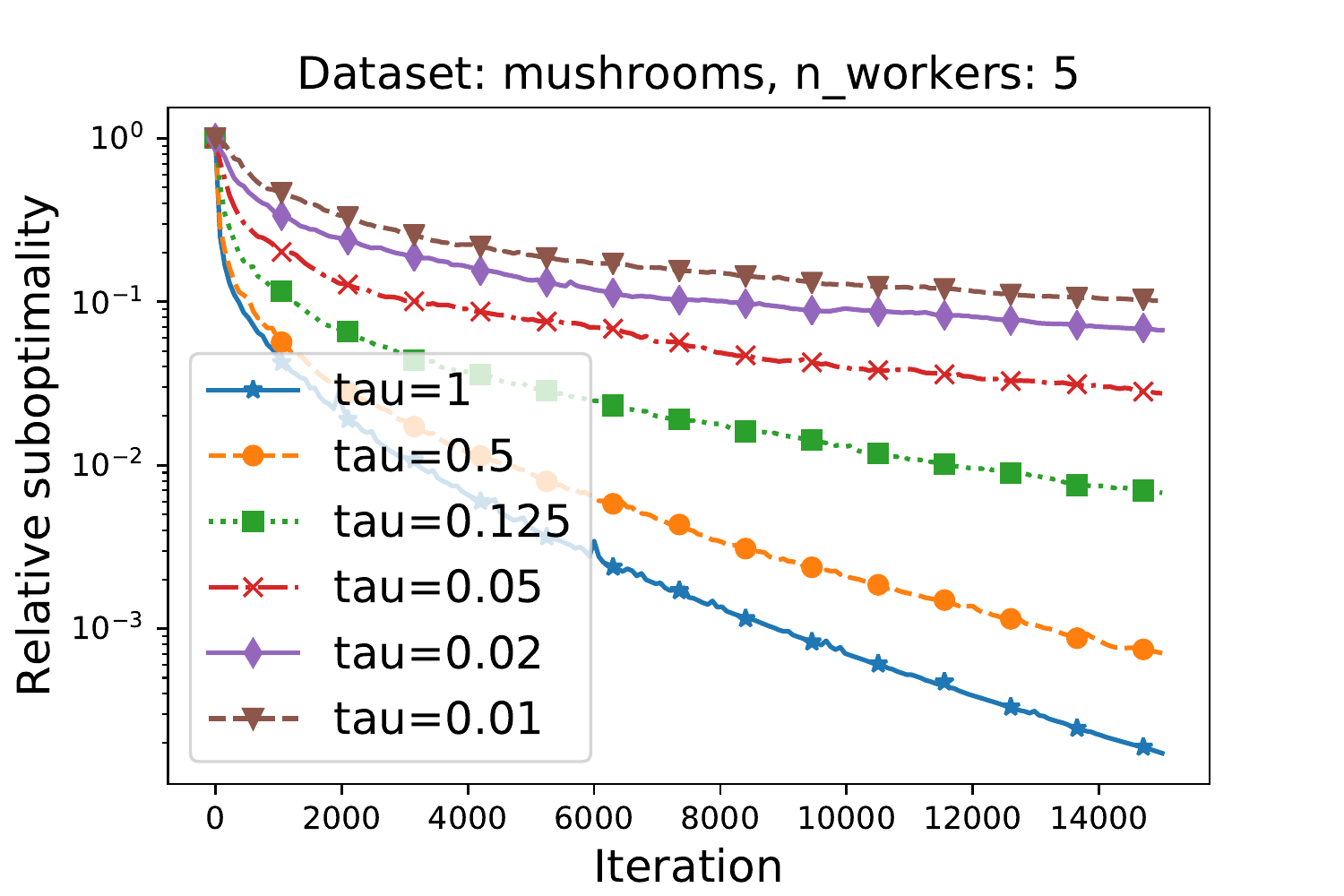}
\end{minipage}%
\begin{minipage}{0.33\textwidth}
  \centering
\includegraphics[width =  \textwidth ]{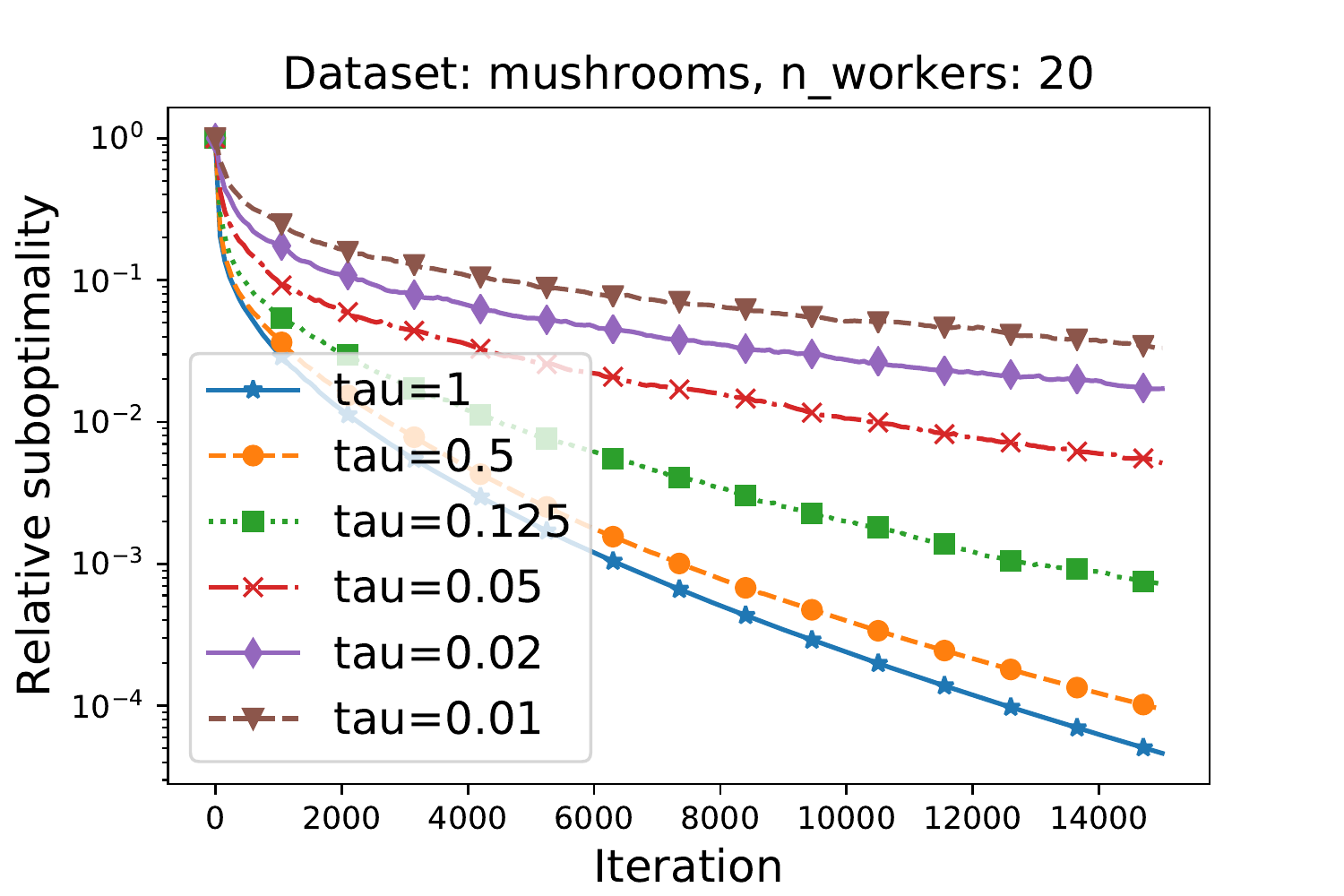}
\end{minipage}%
\begin{minipage}{0.33\textwidth}
  \centering
\includegraphics[width =  \textwidth ]{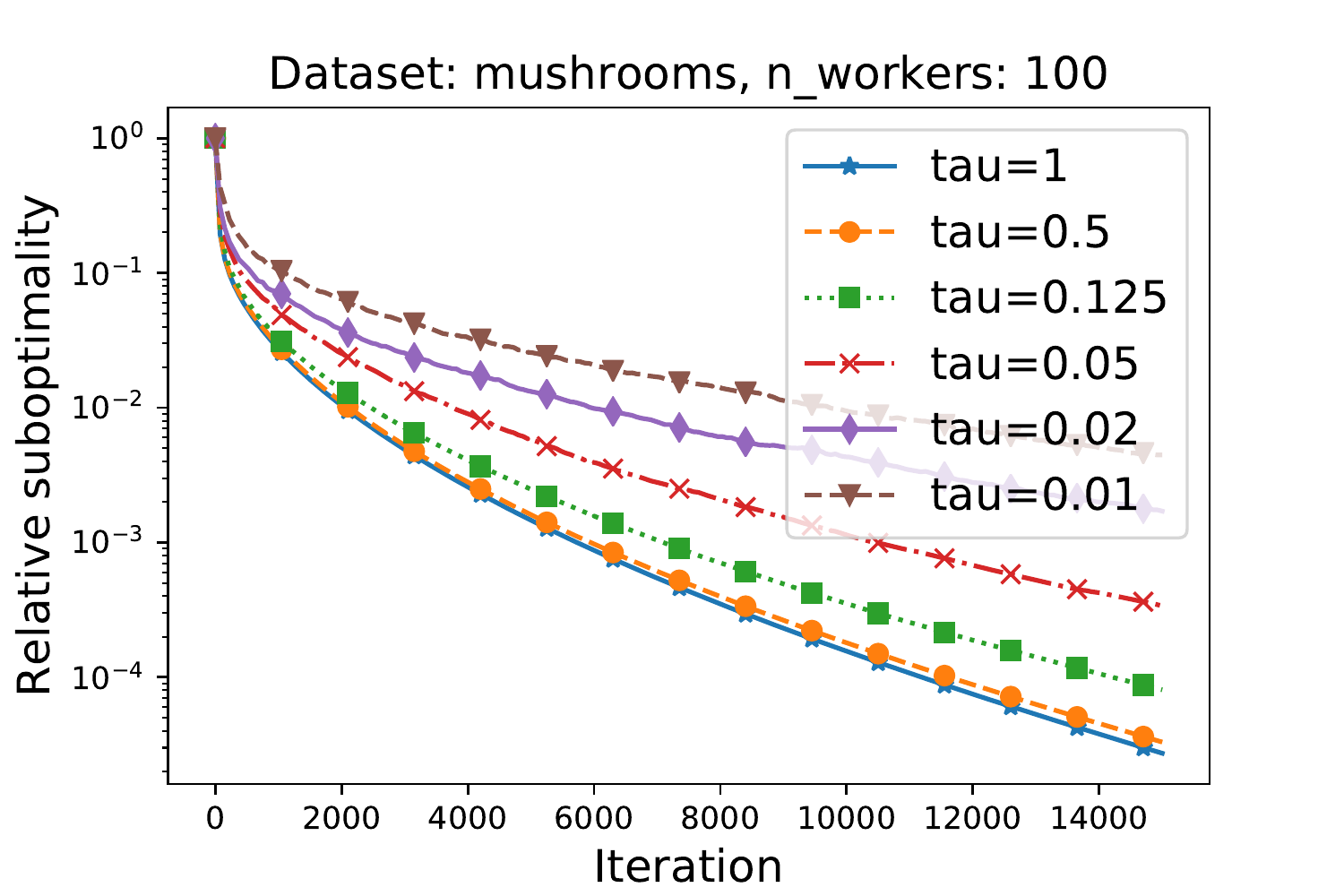}
\end{minipage}%
\\
\begin{minipage}{0.33\textwidth}
  \centering
\includegraphics[width =  \textwidth ]{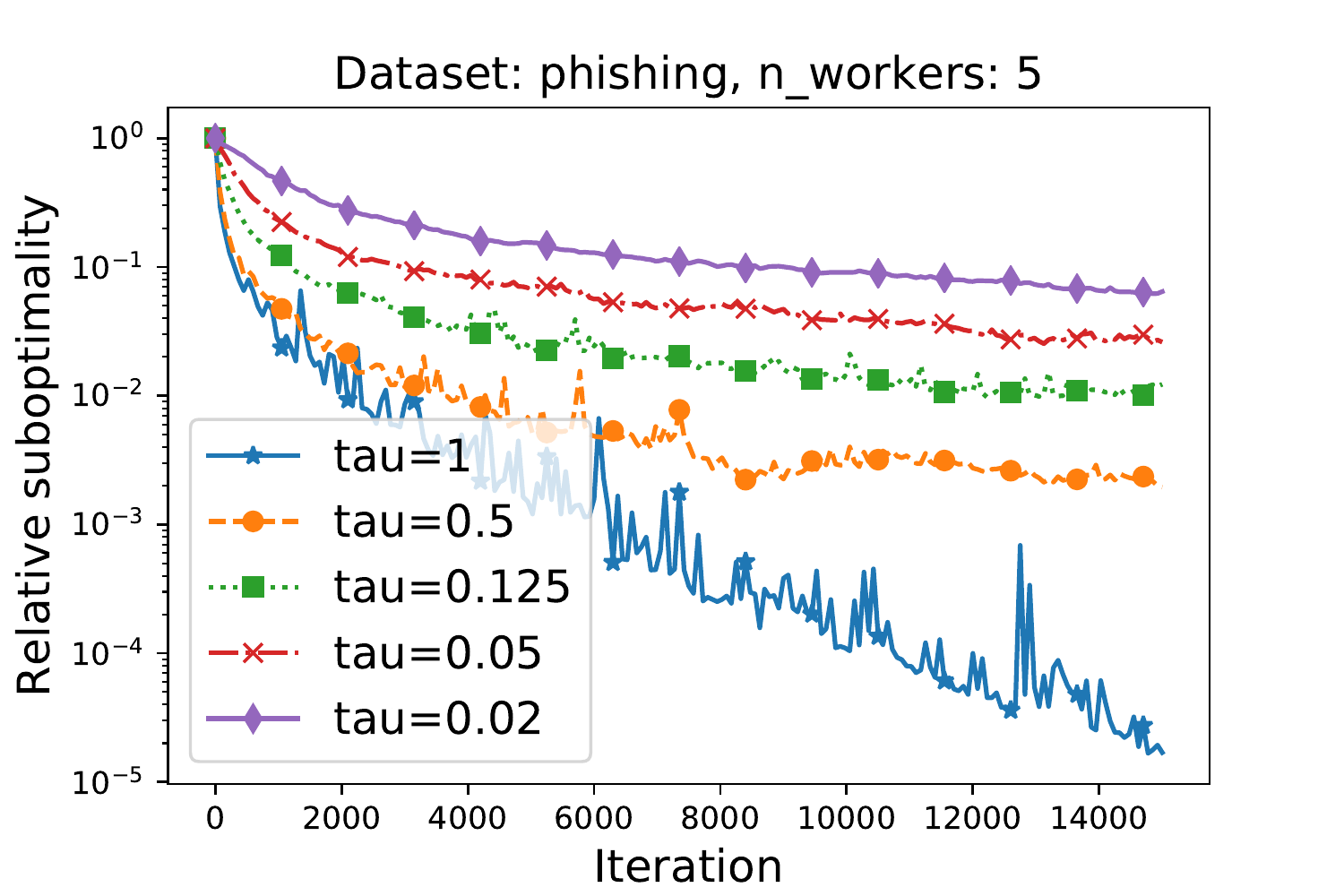}
\end{minipage}%
\begin{minipage}{0.33\textwidth}
  \centering
\includegraphics[width =  \textwidth ]{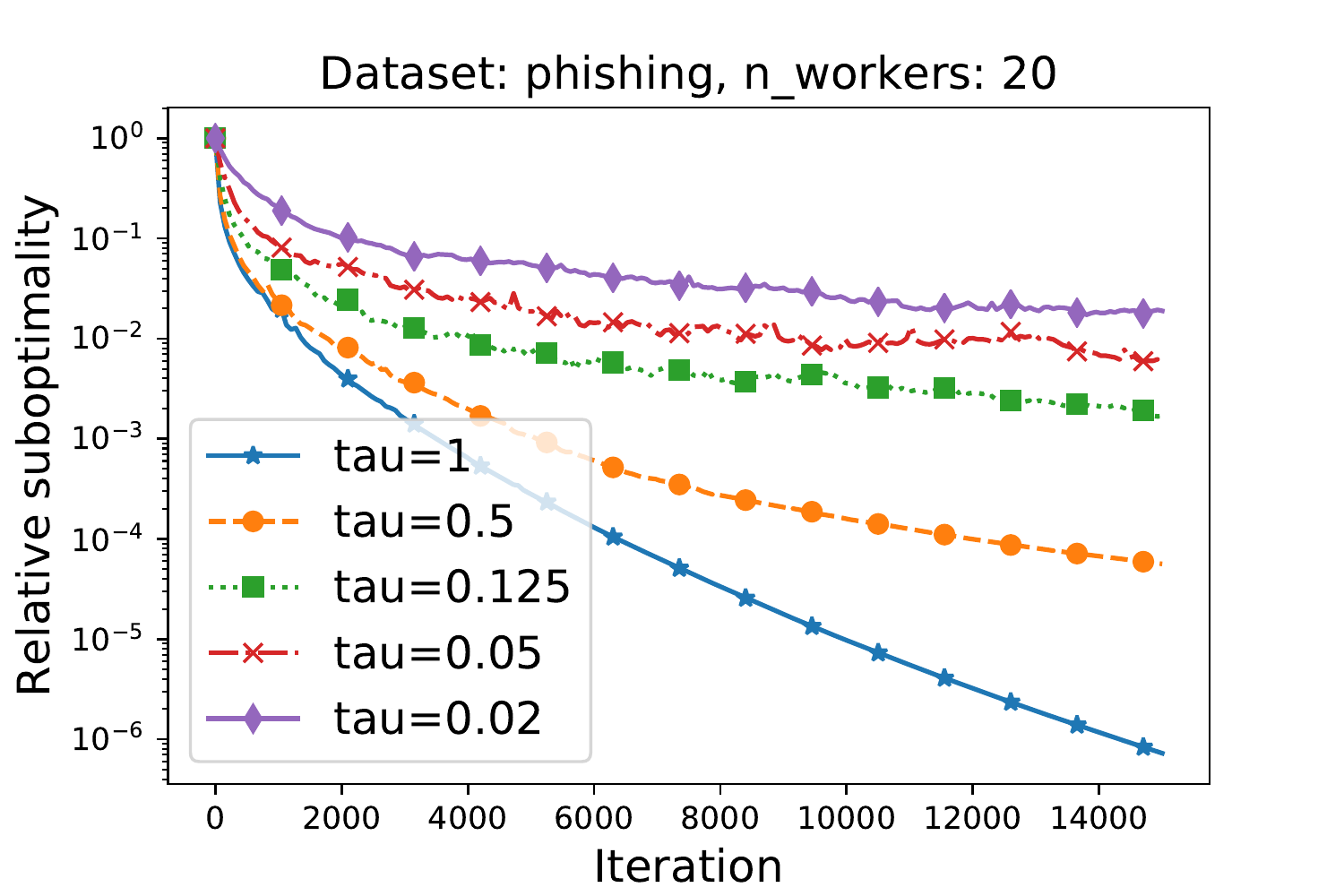}
\end{minipage}%
\begin{minipage}{0.33\textwidth}
  \centering
\includegraphics[width =  \textwidth ]{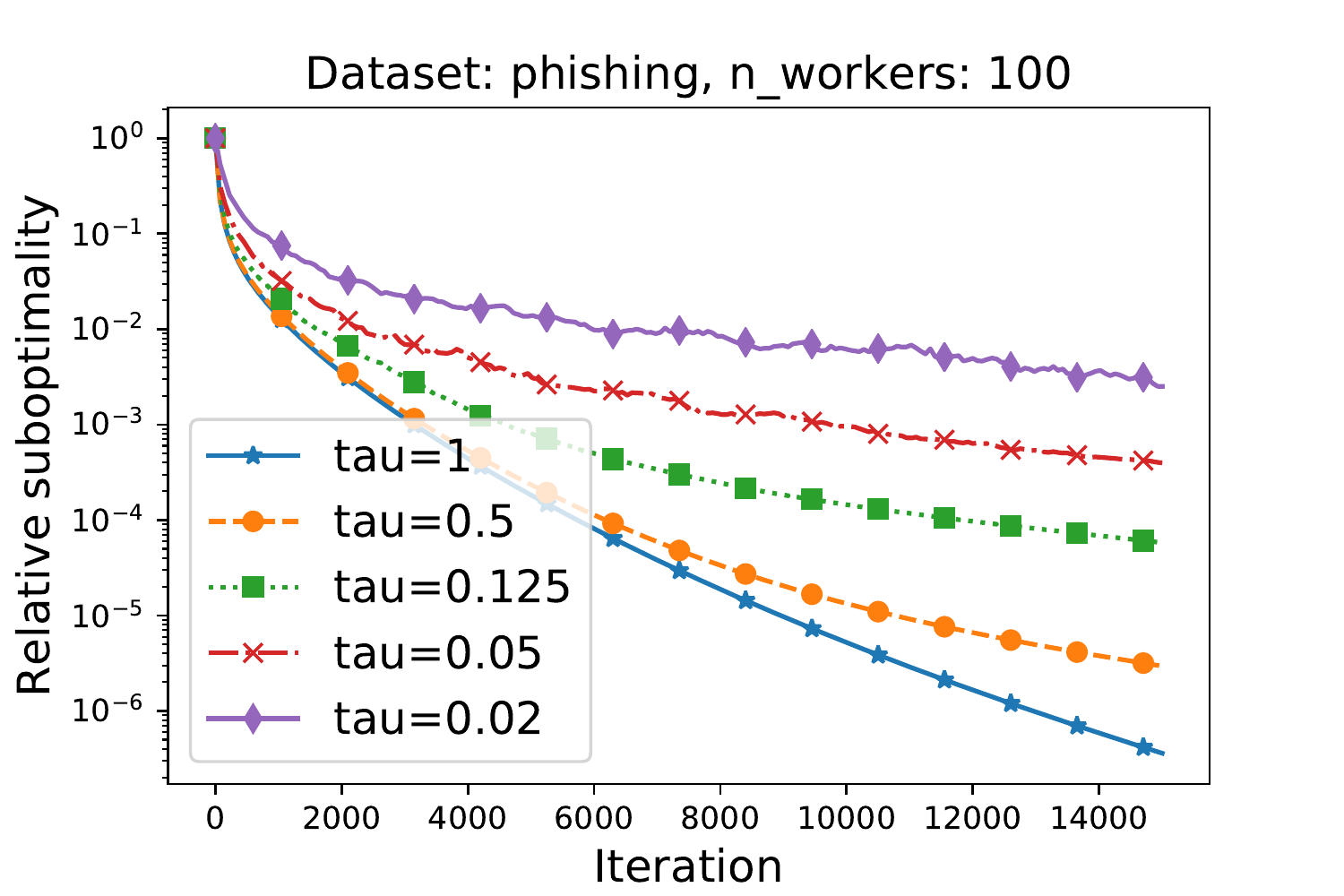}
\end{minipage}%
\\
\begin{minipage}{0.33\textwidth}
  \centering
\includegraphics[width =  \textwidth ]{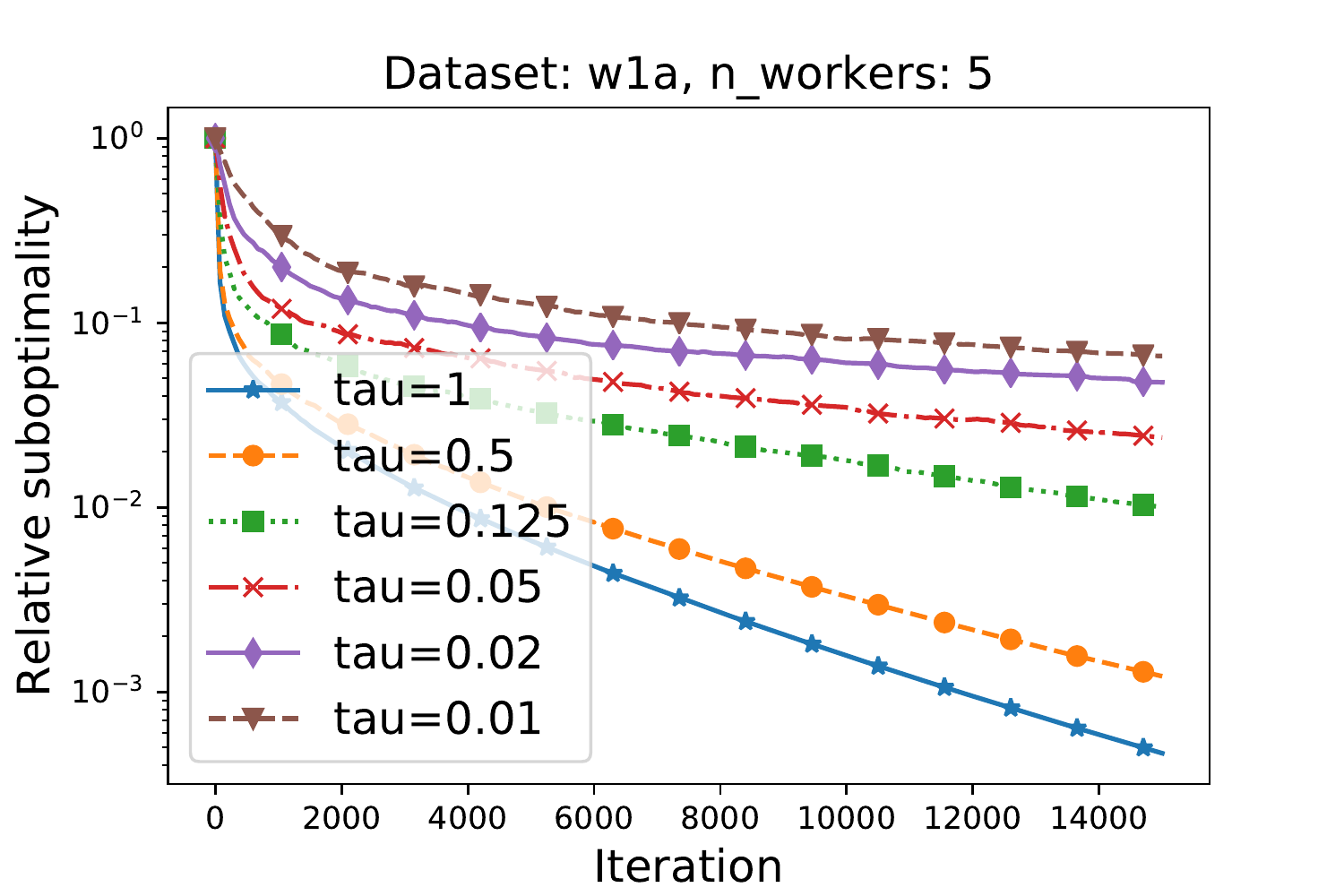}
\end{minipage}%
\begin{minipage}{0.33\textwidth}
  \centering
\includegraphics[width =  \textwidth ]{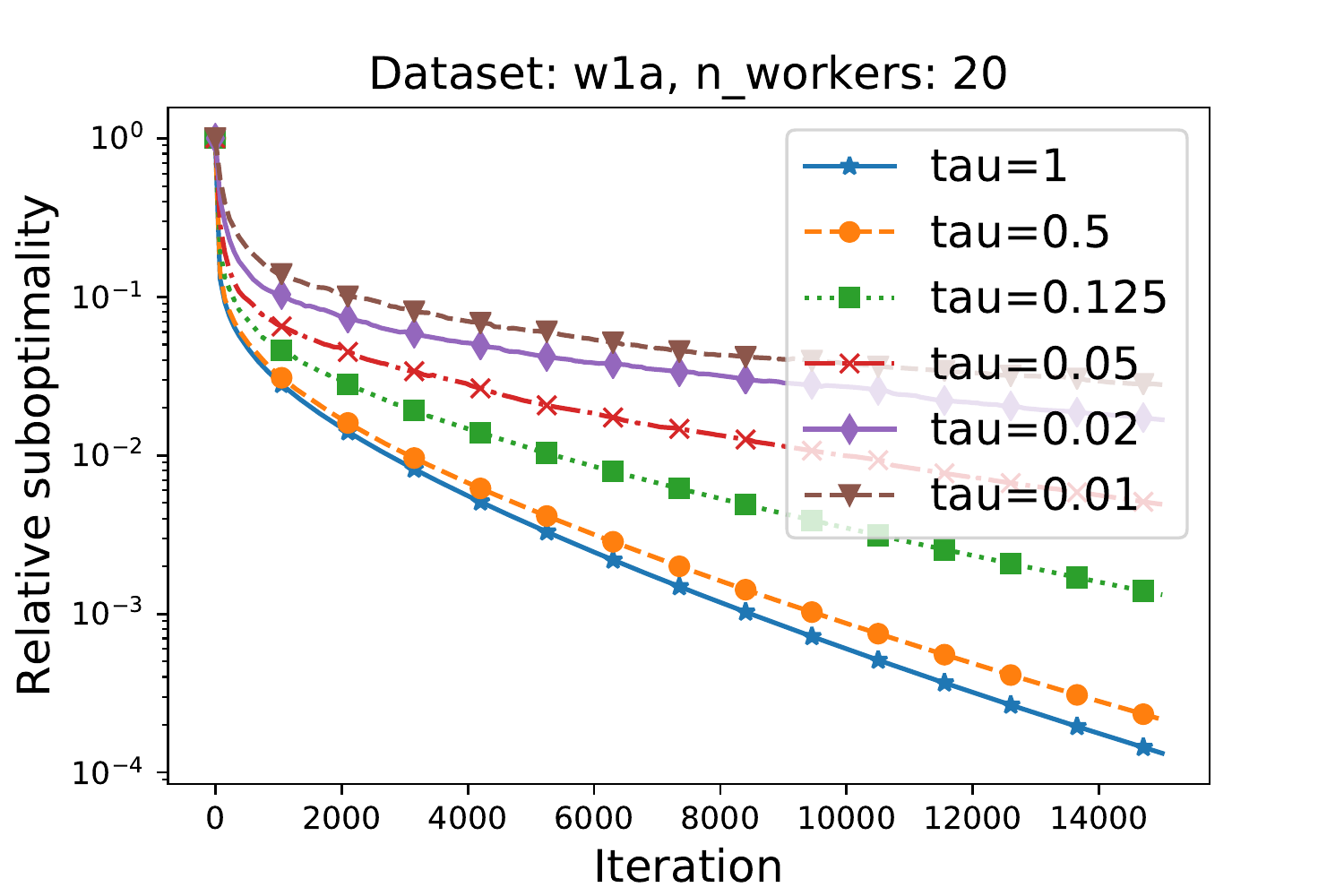}
\end{minipage}%
\begin{minipage}{0.33\textwidth}
  \centering
\includegraphics[width =  \textwidth ]{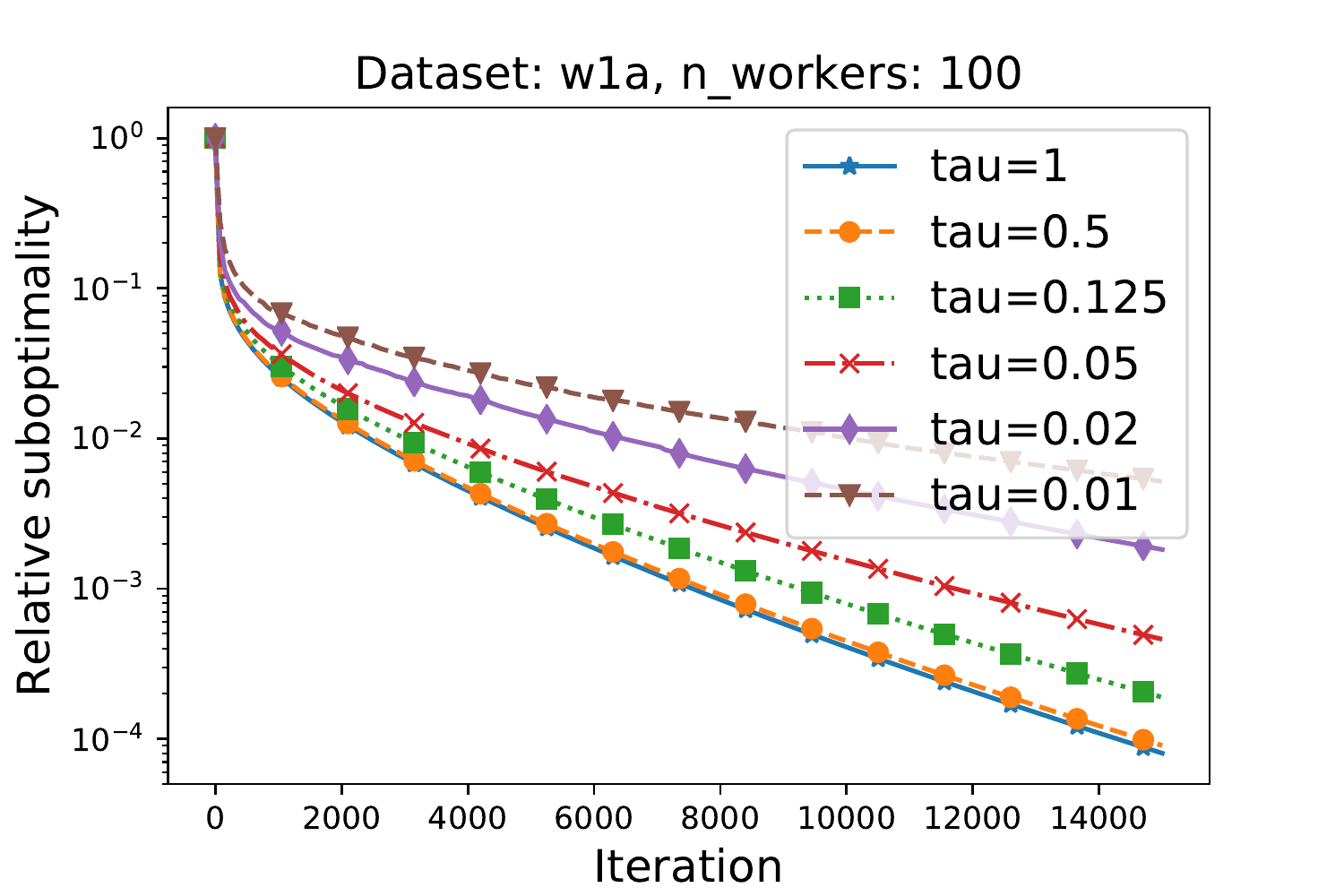}
\end{minipage}%
\\
\caption{Comparison of Algorithm~\ref{alg:saga} for different values of $\tau$. Stepsize $\gamma = \frac{1}{L(3n^{-1}+\tau)}$ is chosen in each case. For this experiment, we choose smaller regularization; $\ell_2  = 0.000025 $. }\label{fig:saga2}
\end{figure}

\subsection{ISEGA \label{sec:exp_sega}}
Lastly, we numerically test Algorithm~\ref{alg:sega}, and its linear convergence without Assumption~\ref{as:zero_grads}. For simplicity, we consider $R(x)=0$ in~\eqref{eq:problem_sega}. 

In the first experiment (Figure~\ref{fig:sega1}), we compare Algorithm~\ref{alg:sega} for various $(n,\tau)$ such that $n\tau=1$. For illustration, we also plot convergence of gradient descent with the analogous stepsize. As theory predicts, the method has almost same convergence speed.\footnote{We have chosen stepsize $\gamma = \frac{1}{2L}$ for GD, as this is the baseline to Algorithm~\ref{alg:sega} with zero variance. One can in fact set $\gamma = \frac{1}{L}$ for GD and get 2 times faster convergence. However, this is still only a constant factor. }

\begin{figure}[H]
\centering
\begin{minipage}{0.24\textwidth}
  \centering
\includegraphics[width =  \textwidth ]{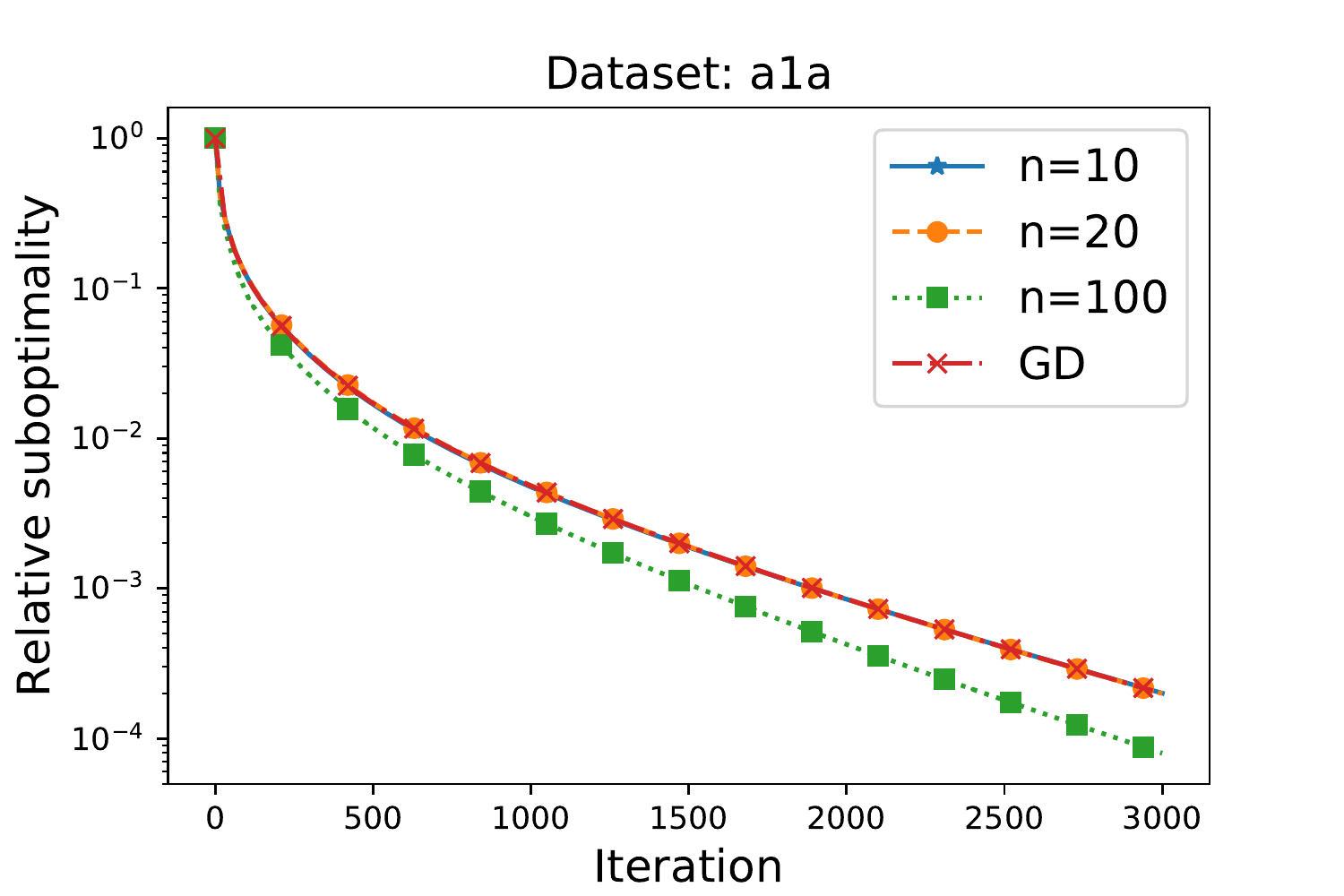}
\end{minipage}%
\begin{minipage}{0.24\textwidth}
  \centering
\includegraphics[width =  \textwidth ]{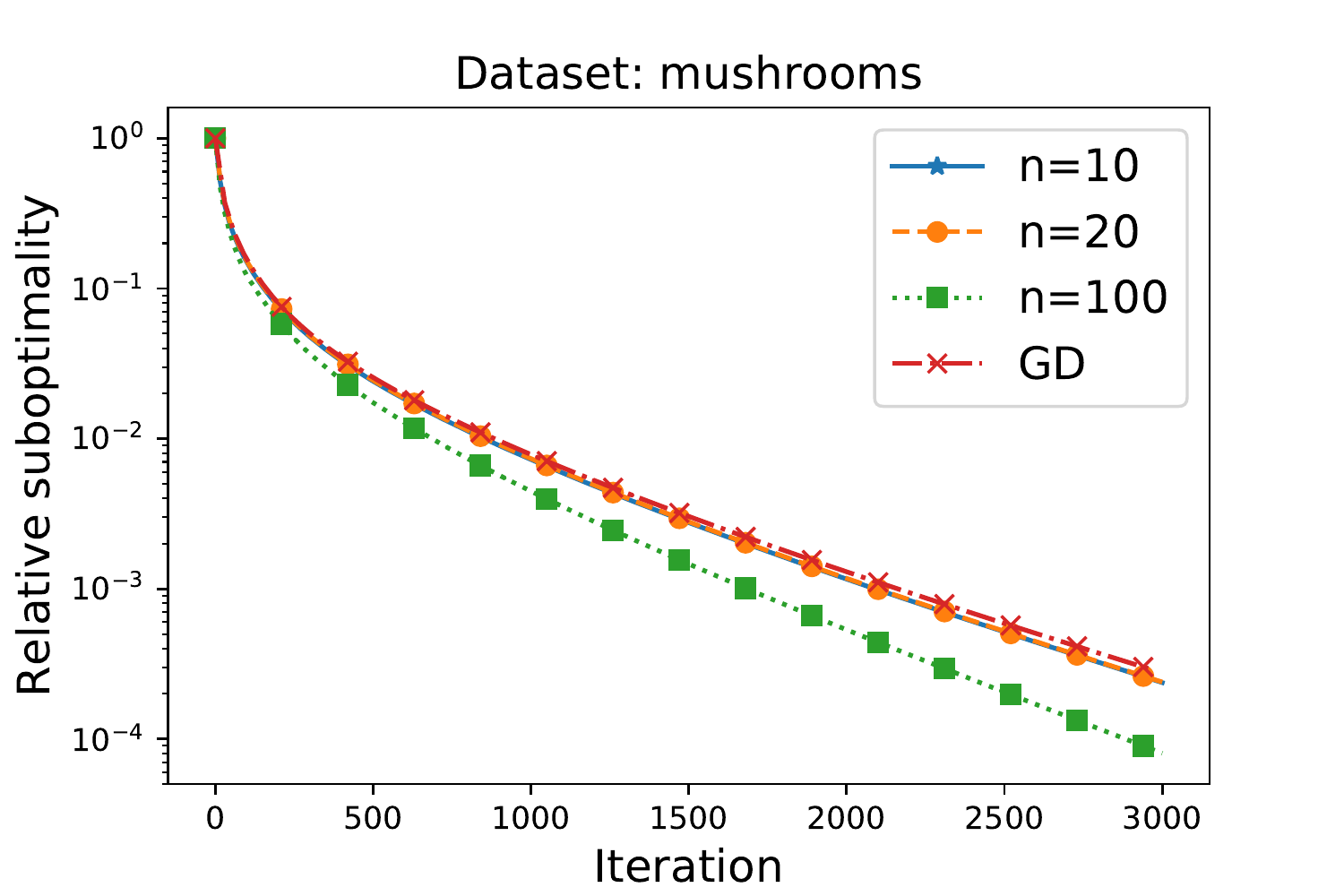}
\end{minipage}%
\begin{minipage}{0.24\textwidth}
  \centering
\includegraphics[width =  \textwidth ]{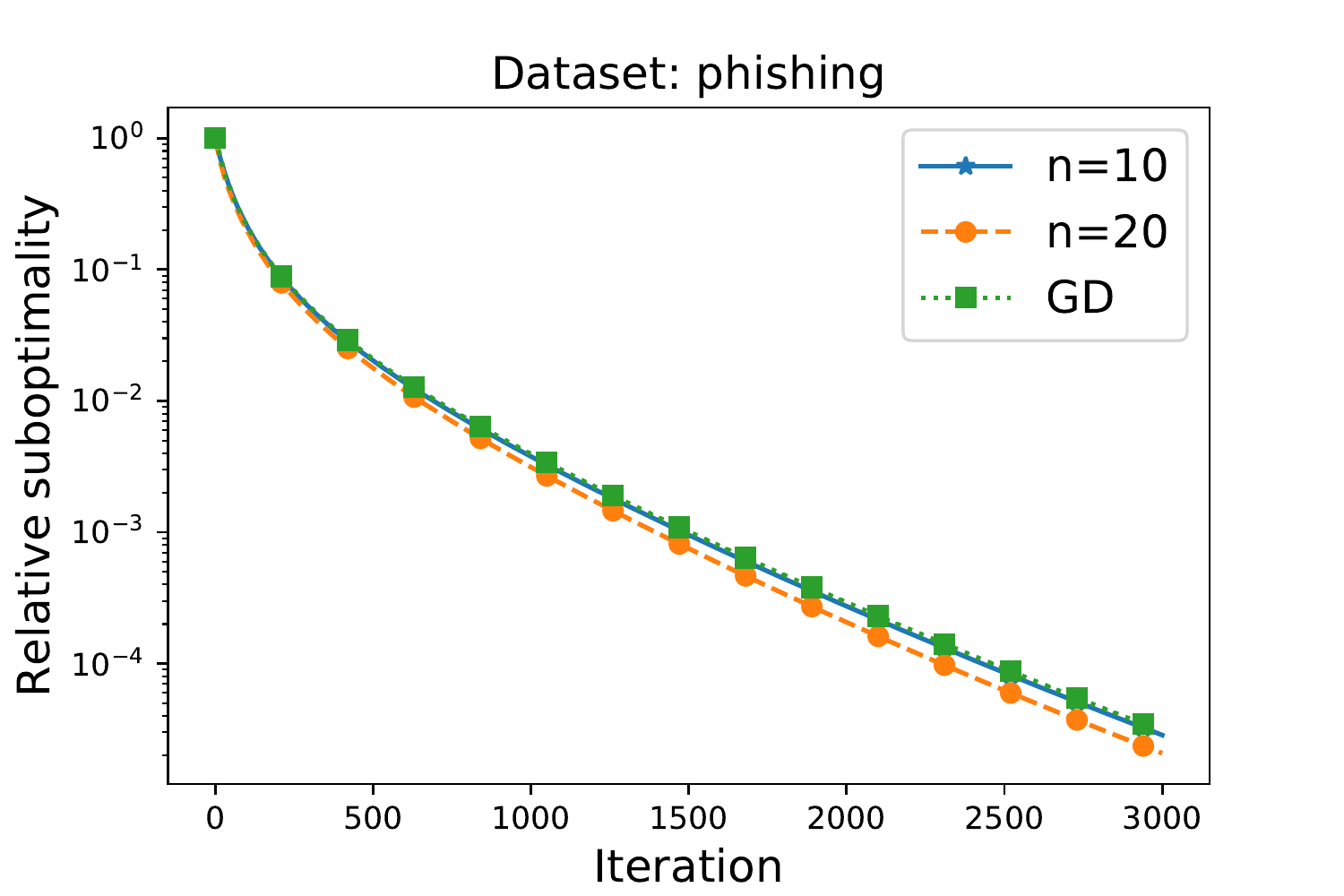}
\end{minipage}%
\begin{minipage}{0.24\textwidth}
  \centering
\includegraphics[width =  \textwidth ]{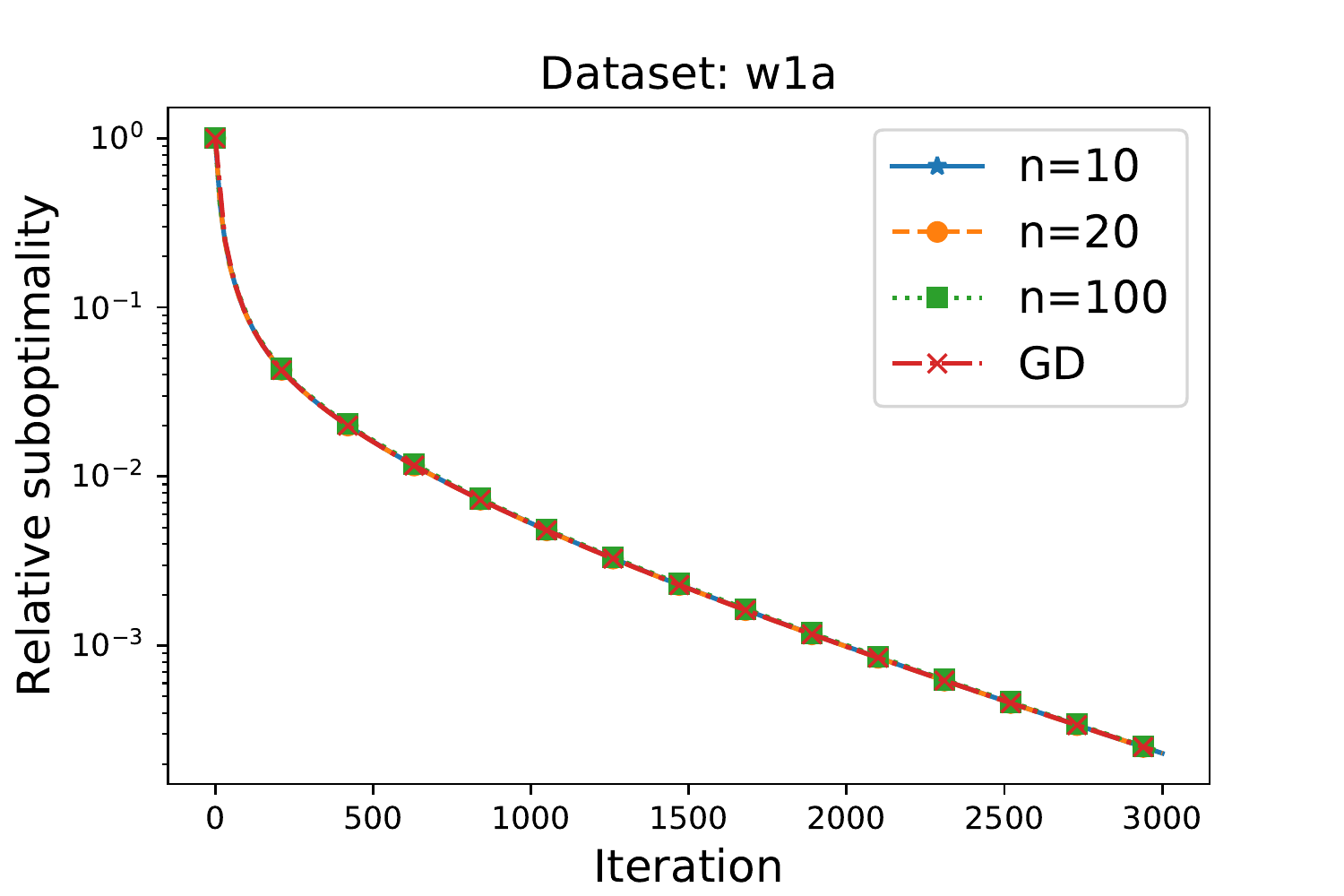}
\end{minipage}%
\\
\caption{Comparison of Algorithm~\ref{alg:sega} for various $(n,\tau)$ such that $n\tau=1$ and GD. Stepsize $\frac{1}{L\left(1+\frac{1}{n\tau}\right)} $ was chosen for Algorithm~\ref{alg:sega} and $\frac1{2L}$ for GD.}\label{fig:sega1}
\end{figure}

The second experiment of this section shows the convergence behavior for varying $\tau$ of Algorithm~\ref{alg:sega}. Again, the results (Figure~\ref{fig:sega2}) indicate that $\tau$ has a heavy impact on the convergence speed for small $n$. However, as $n$ increases, the effect of $\tau$ is diminishing. In particular, for increasing $\tau$ beyond $n^{-1}$ does not yield a significant speedup.

\begin{figure}[H]
\centering
\begin{minipage}{0.33\textwidth}
  \centering
\includegraphics[width =  \textwidth ]{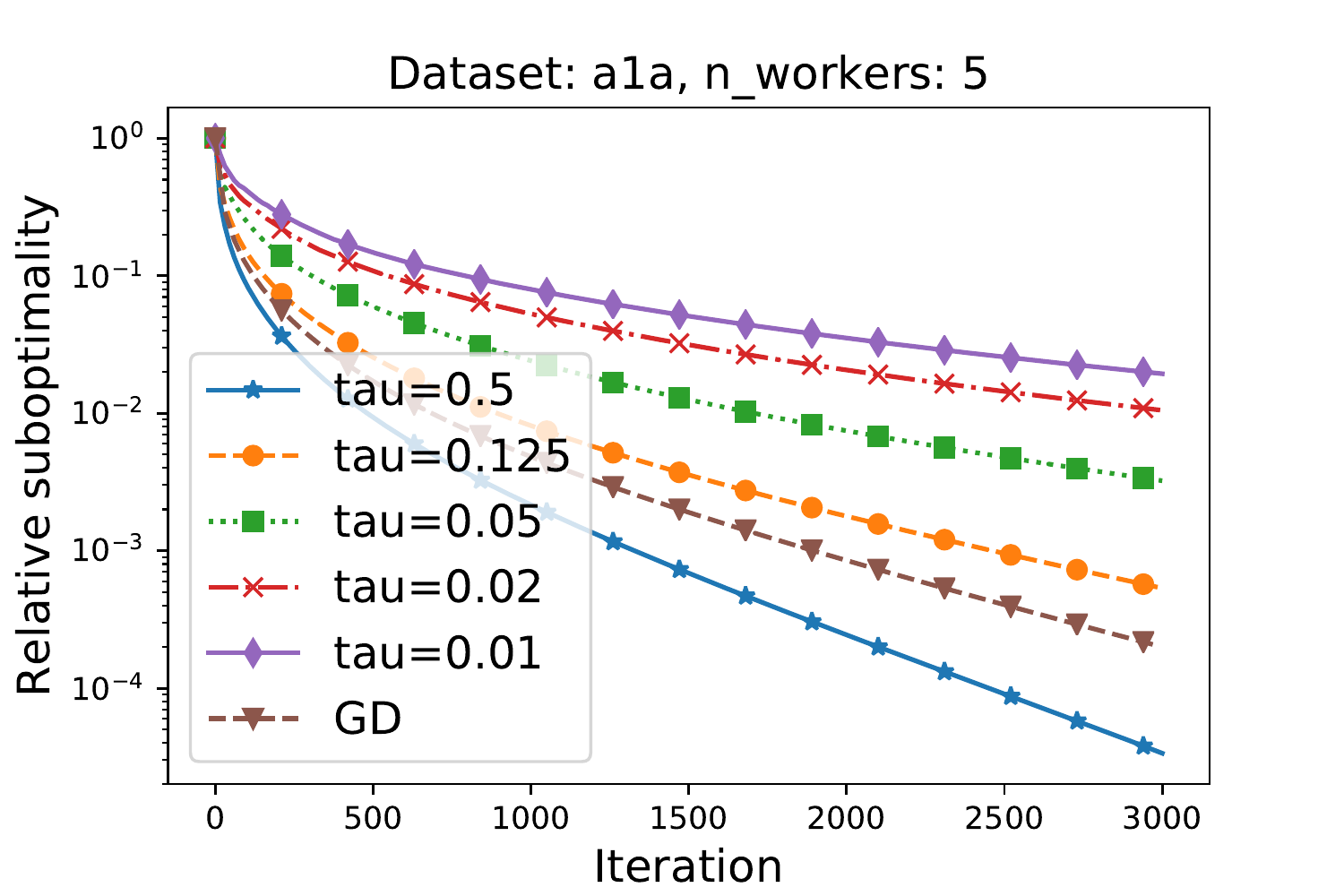}
\end{minipage}%
\begin{minipage}{0.33\textwidth}
  \centering
\includegraphics[width =  \textwidth ]{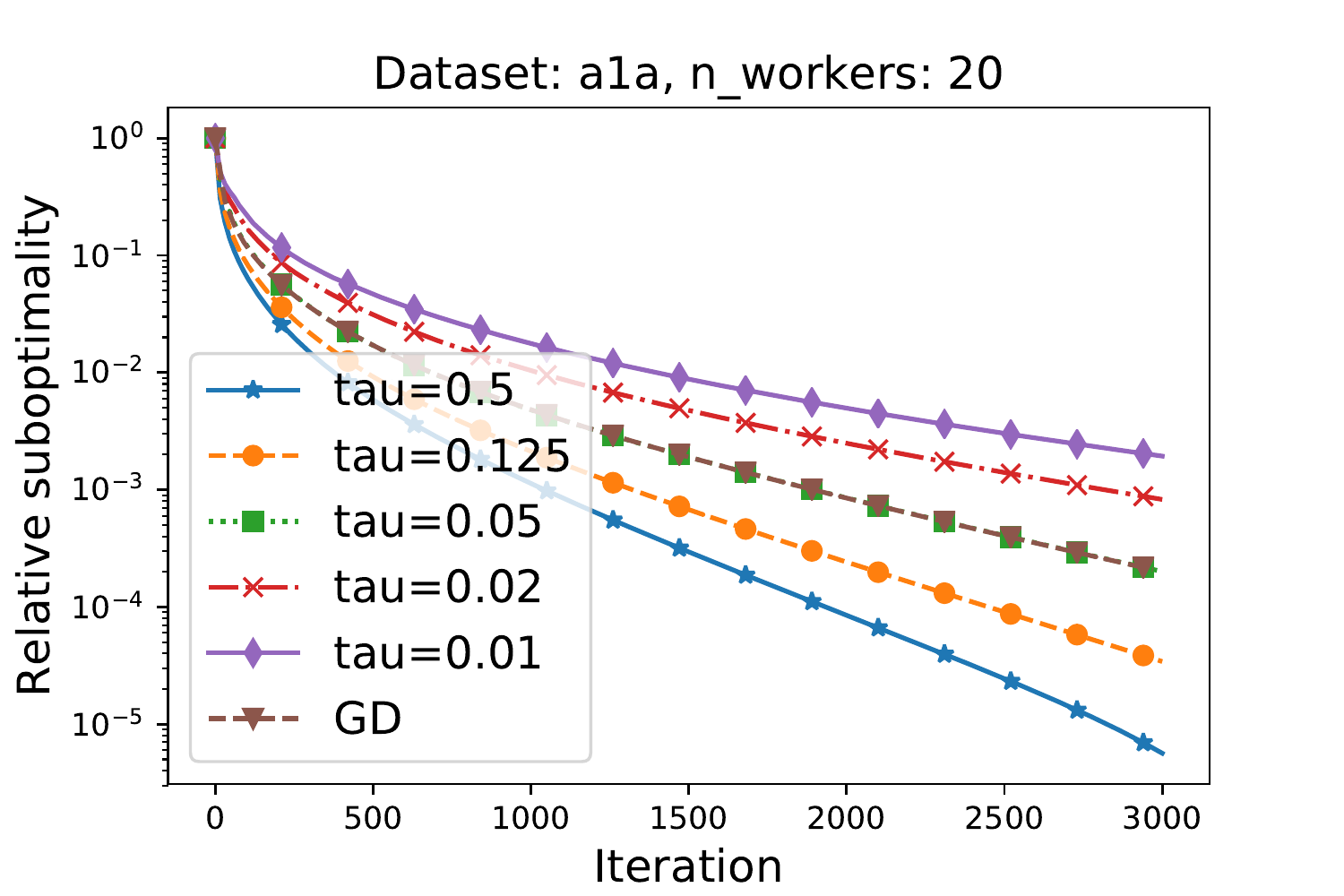}
\end{minipage}%
\begin{minipage}{0.33\textwidth}
  \centering
\includegraphics[width =  \textwidth ]{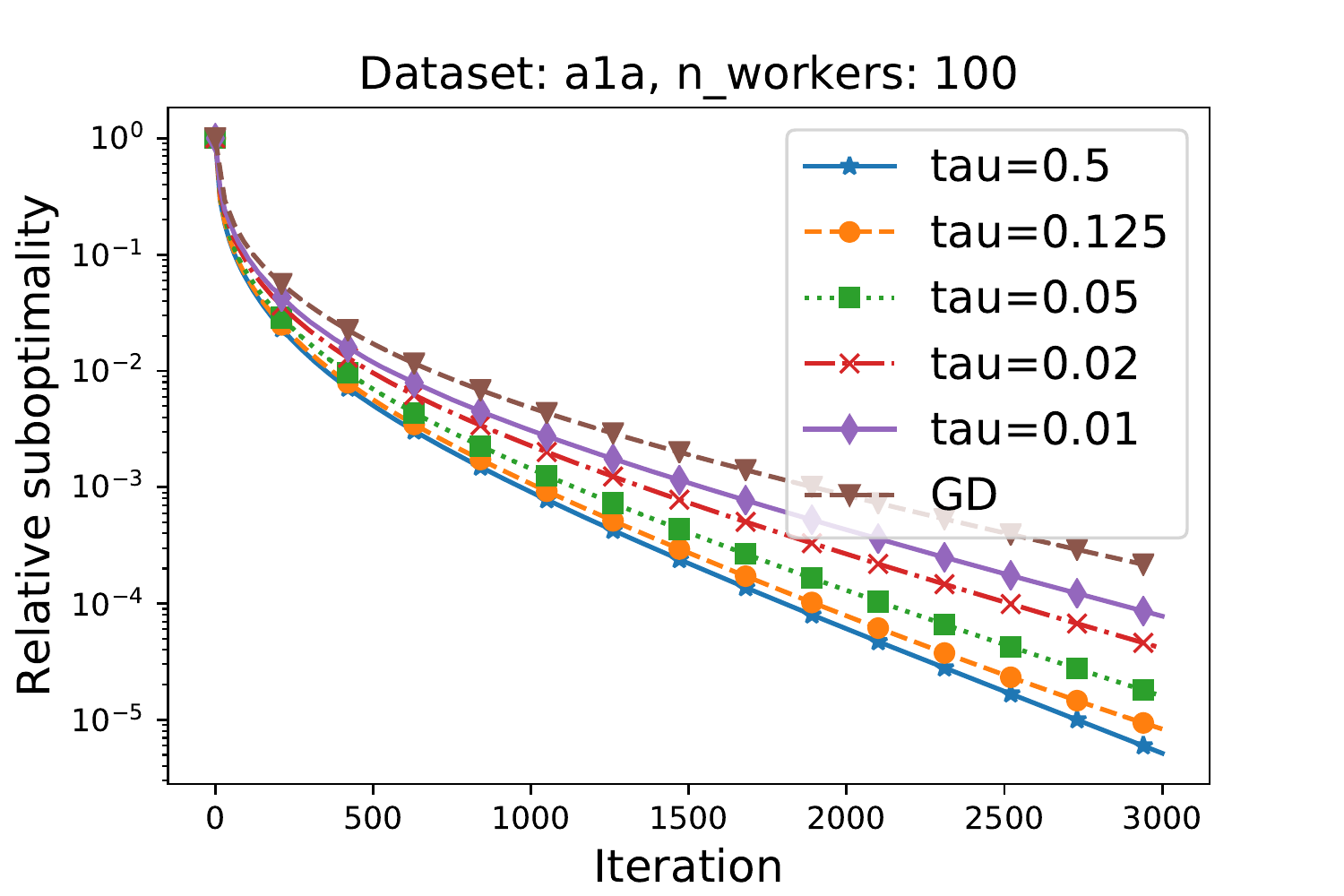}
\end{minipage}%
\\
\begin{minipage}{0.33\textwidth}
  \centering
\includegraphics[width =  \textwidth ]{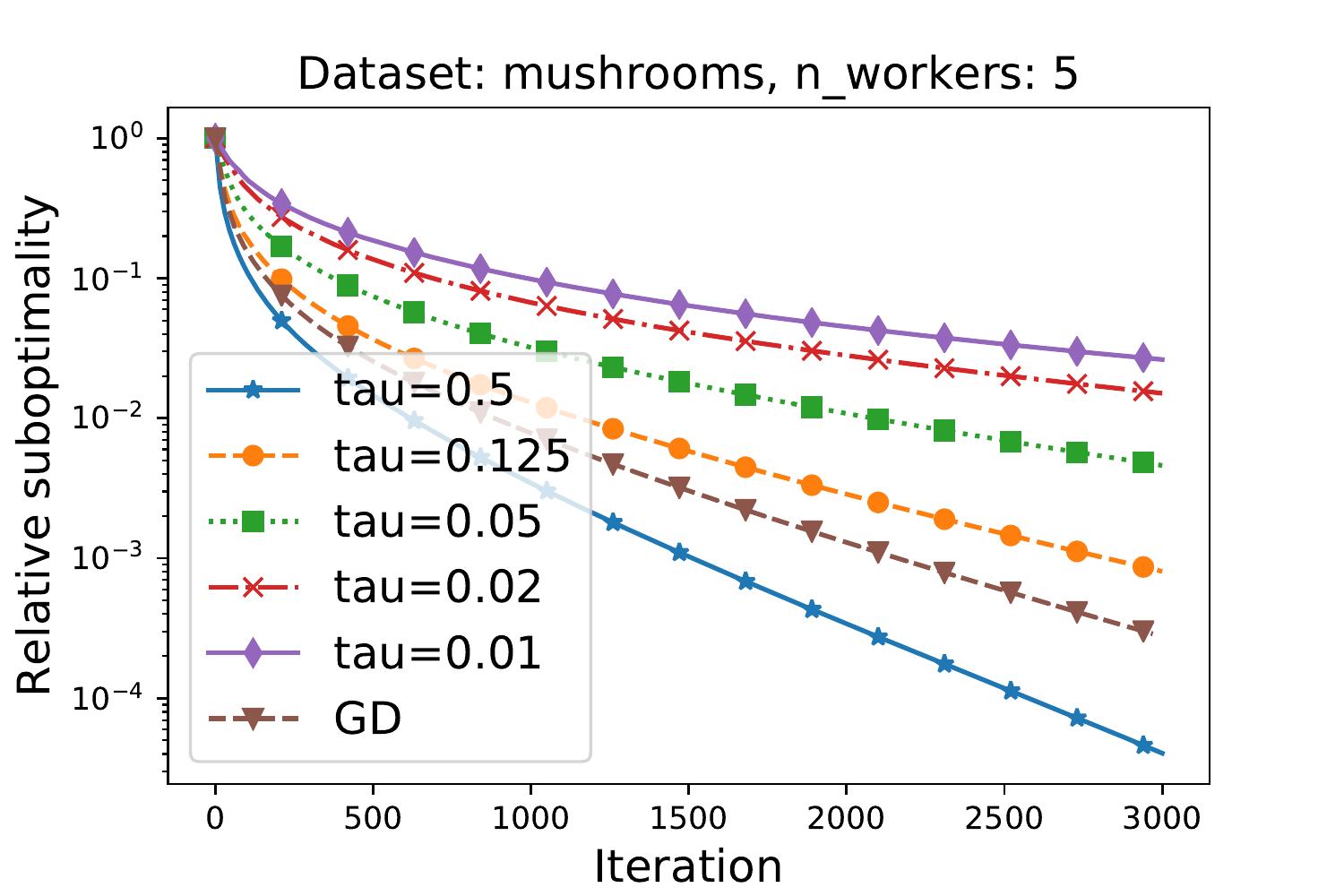}
\end{minipage}%
\begin{minipage}{0.33\textwidth}
  \centering
\includegraphics[width =  \textwidth ]{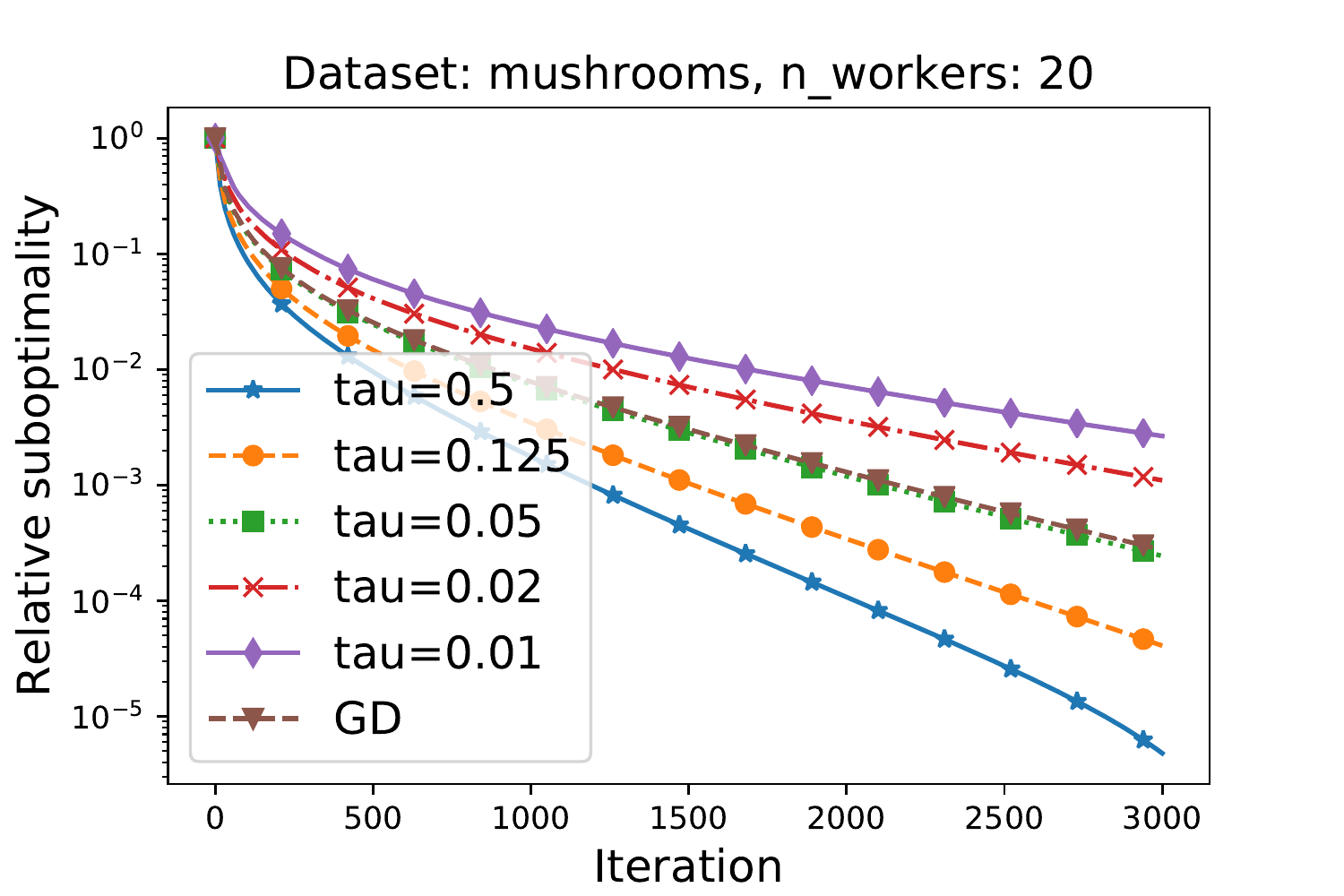}
\end{minipage}%
\begin{minipage}{0.33\textwidth}
  \centering
\includegraphics[width =  \textwidth ]{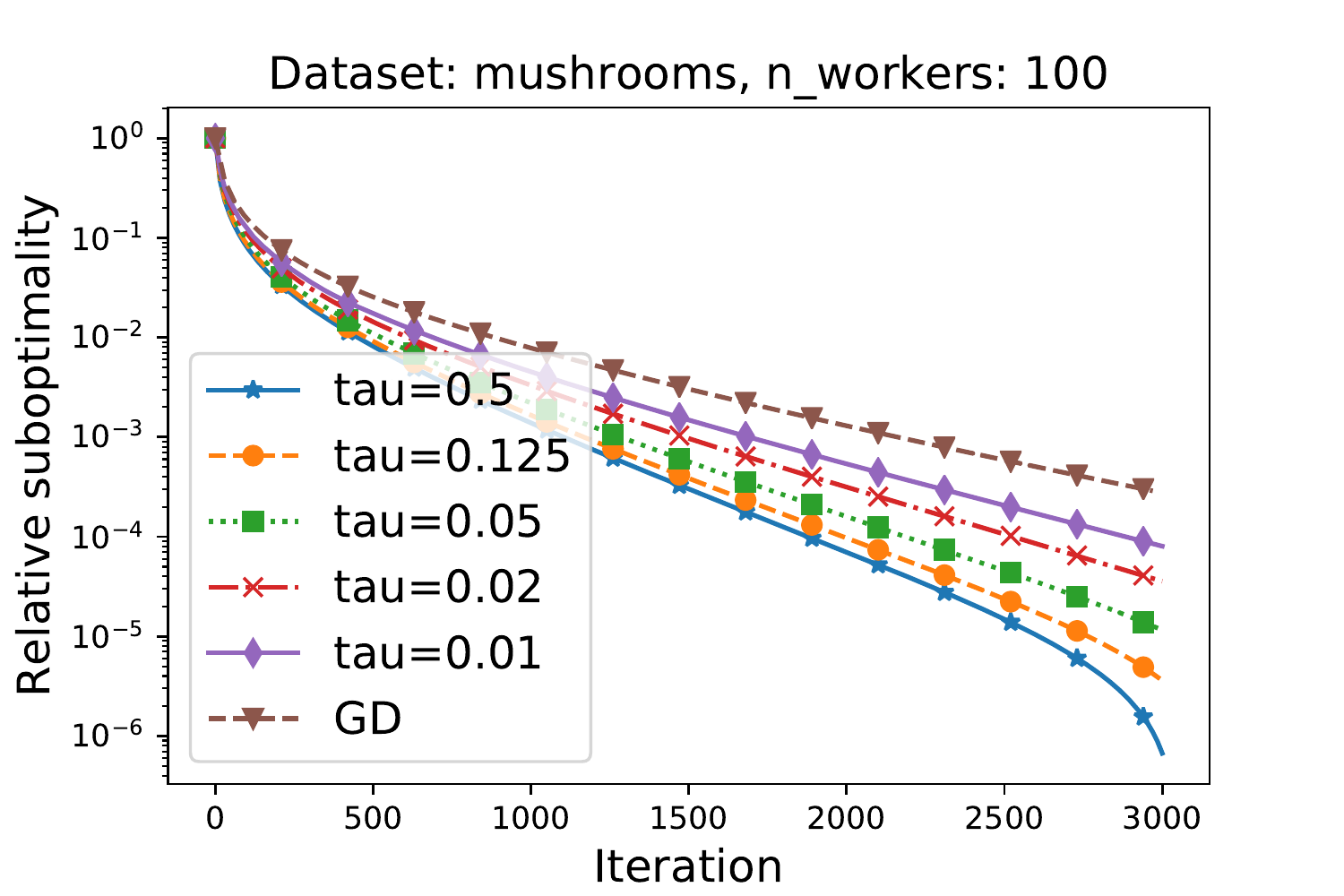}
\end{minipage}%
\\
\begin{minipage}{0.33\textwidth}
  \centering
\includegraphics[width =  \textwidth ]{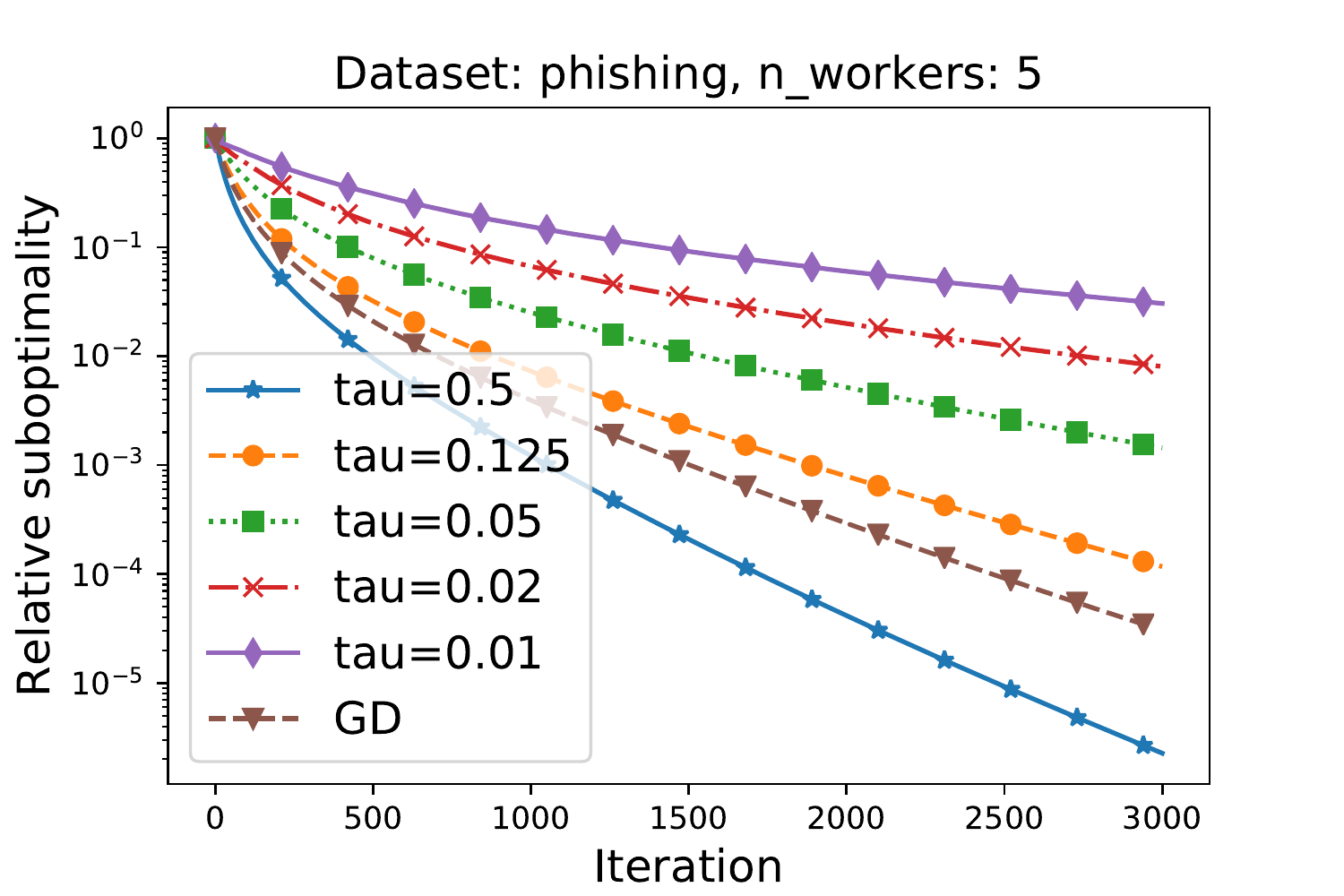}
\end{minipage}%
\begin{minipage}{0.33\textwidth}
  \centering
\includegraphics[width =  \textwidth ]{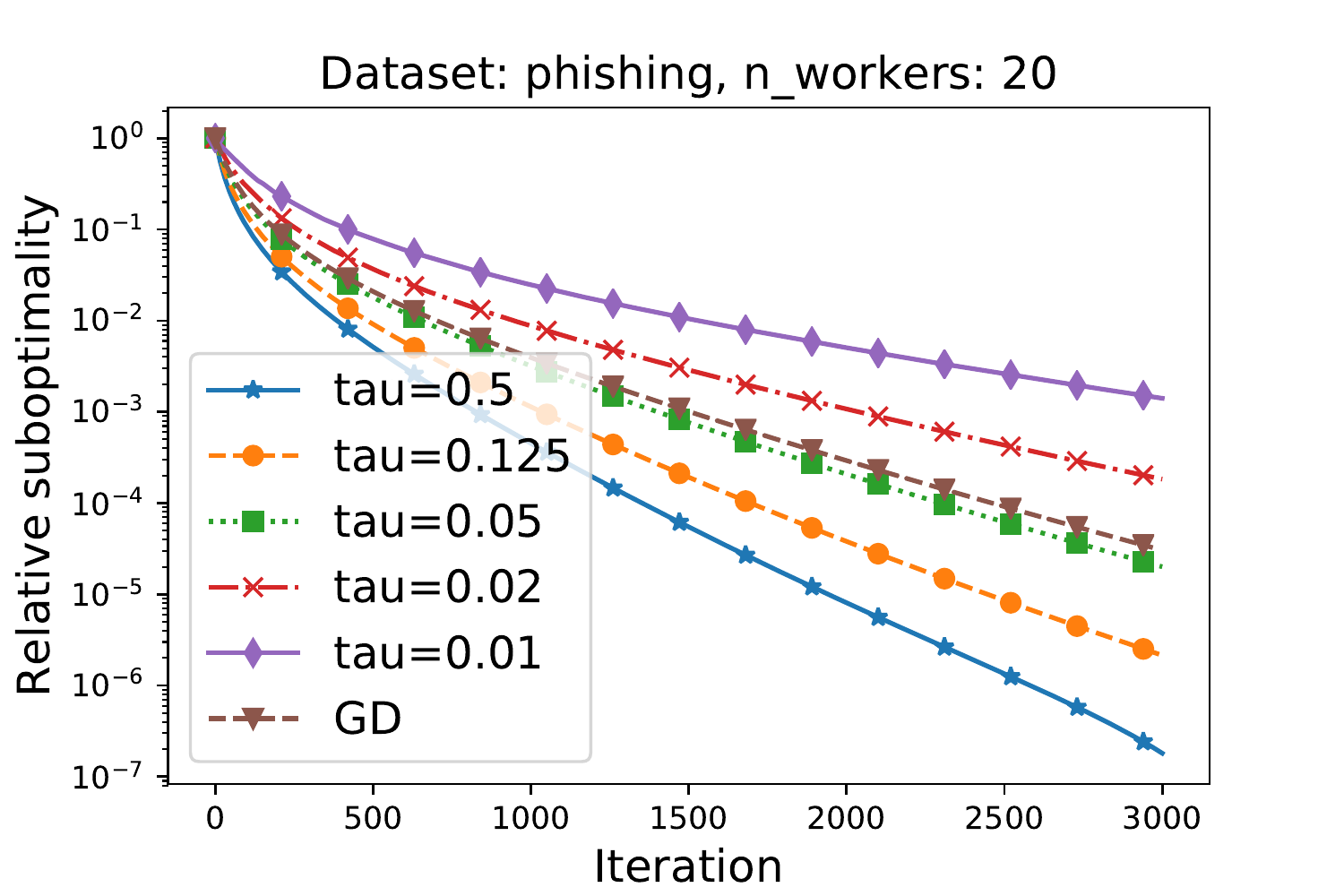}
\end{minipage}%
\\
\begin{minipage}{0.33\textwidth}
  \centering
\includegraphics[width =  \textwidth ]{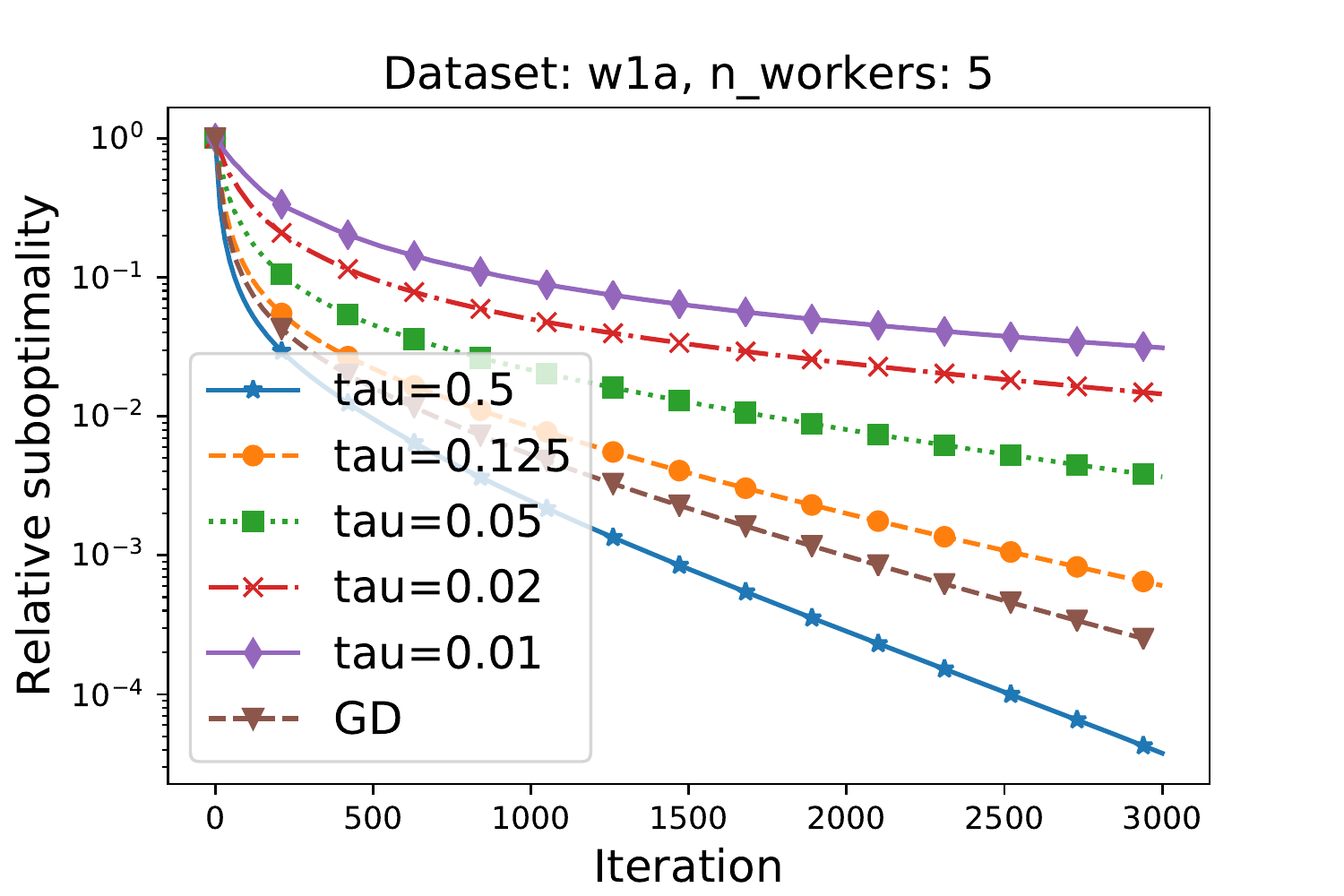}
\end{minipage}%
\begin{minipage}{0.33\textwidth}
  \centering
\includegraphics[width =  \textwidth ]{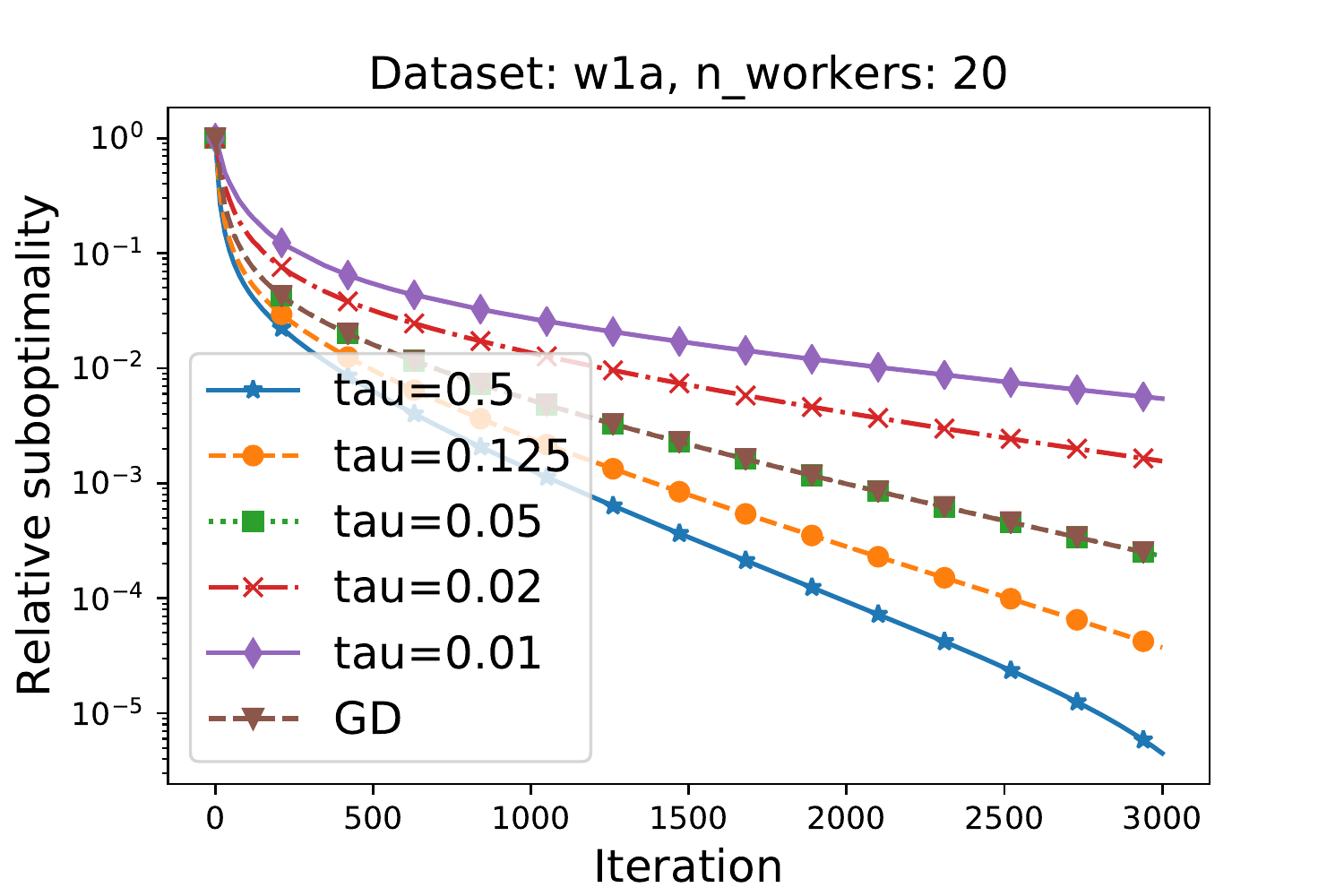}
\end{minipage}%
\begin{minipage}{0.33\textwidth}
  \centering
\includegraphics[width =  \textwidth ]{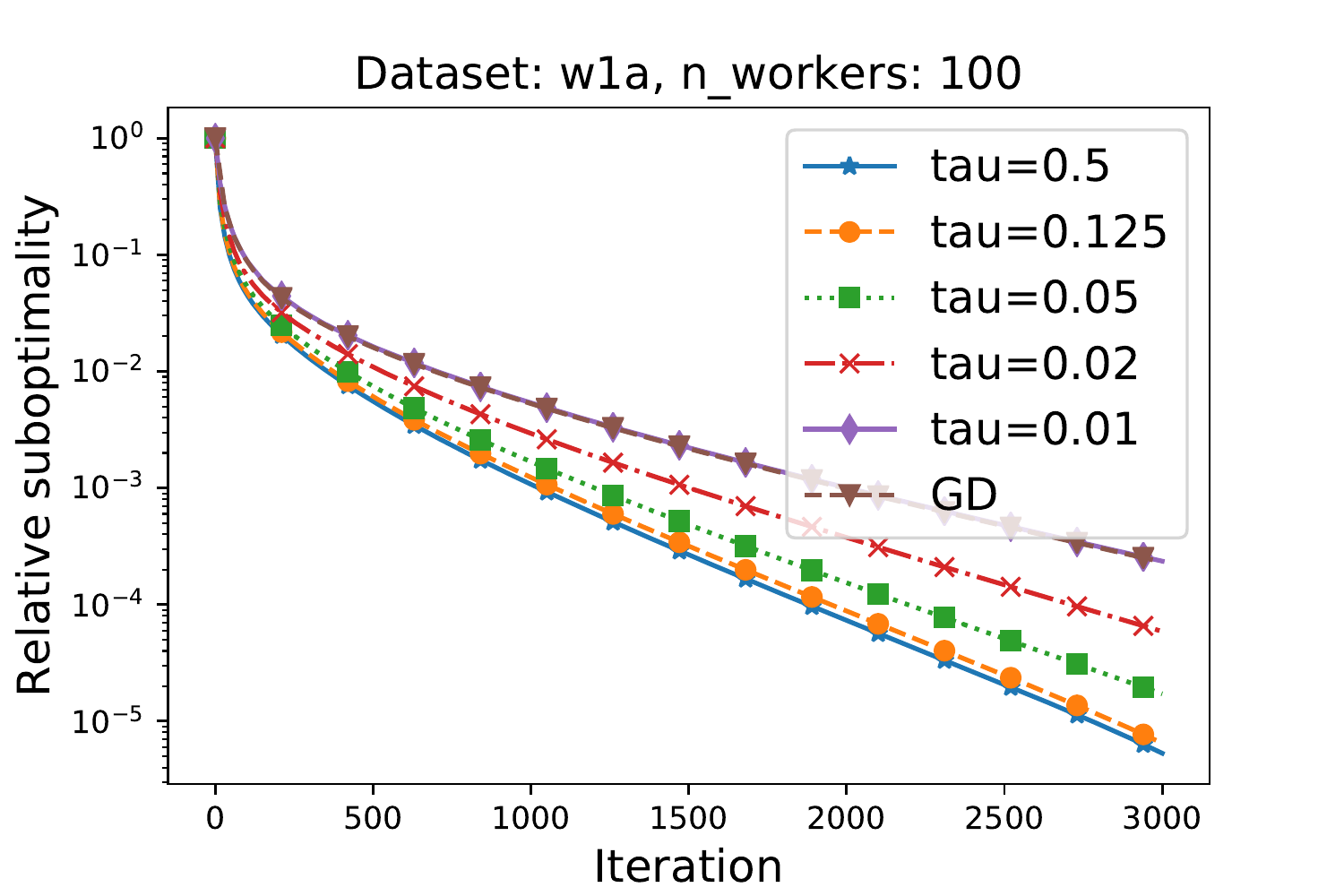}
\end{minipage}%
\caption{Comparison of Algorithm~\ref{alg:sega} for different values of $\tau$. Stepsize $\gamma = \frac{1}{L\left(1+ \frac{1}{n\tau}\right)}$ is chosen in each case.}\label{fig:sega2}
\end{figure}

\section{Proofs for Section~\ref{sec:basic}}

\subsection{Key techniques}

The most important equality used many times to prove the results of this paper is a simple decomposition of expected distances into the distance of expectation and variance:
\begin{align}
    \EE \|X - a\|_2^2 = \|\EE X - a\|_2^2 + \EE \|X - \EE X\|_2^2,\label{eq:variance_decompos}
\end{align}
where $X$ is any random vector with finite variance and $a$ is an arbitrary vector from $\RR^d$. 

As almost every algorithm we propose average all updates coming from workers, it will be useful to bound the expected distance of mean of $n$ random variables from the optimum. Lemma~\ref{lem:expected_distance} provides the result.  

\begin{lemma}\label{lem:expected_distance}
    Suppose that $x^{t+1} = \frac{1}{n}\sumin x_i^t$. Then, we have
    \begin{align*}
        \EE \|x^{t+1} - x^*\|_2^2
        \le & \left\|\frac{1}{n}\sumin \EE x_i^{t+1} - x^* \right\|_2^2 
         + \frac{1}{n^2}\sumin \EE \|x_i^{t+1} - \EE x_i^{t+1}\|_2^2.
    \end{align*}
\end{lemma}
\begin{proof}
    First of all, we have
    \begin{align*}
        \|x^{t + 1} - x^*\|_2
        = \left\|\frac{1}{n}\sumin x_i^t - (x^* - \gamma \nabla f(x^*)) \right\|_2.
    \end{align*}
    Now let us proceed to expectations. Note that for any random vector $X$ we have
    \begin{align*}
        \EE \|X\|_2^2  = \|\EE X\|_2^2 + \EE \|X - \EE X\|_2^2.
    \end{align*}
    Applying this to random vector $X \eqdef \frac{1}{n}\sumin x_i^t - x^* $, we get
    \begin{align*}
        \EE \left\|\frac{1}{n}\sumin x_i^t - x^* \right\|_2^2 = \left\|\EE \frac{1}{n}\sumin x_i^t - x^* \right\|_2^2 + \EE \left\|\frac{1}{n}\sumin x_i^t - \EE \frac{1}{n}\sumin x_i^t\right\|_2^2.
    \end{align*}
    In addition, in all minibatching schemes  $x_i^{t+1}$ are conditionally independent given $x^t$. Therefore, for the variance term we get 
    \begin{align}
    \EE \left \|\frac{1}{n}\sumin x_i^t - \EE \frac{1}{n}\sumin x_i^t \right \|_2^2 = \frac{1}{n^2}\sumin \EE \|x_i^{t+1} - \EE x_i^{t+1}\|_2^2. \label{eq:variance_decomposition}
    \end{align} 
    Plugging it into our previous bounds concludes the proof.
\end{proof}

\subsection{Proof of Theorem~\ref{th:cd}}
\begin{proof}
From Lemma~\ref{lem:sgd_recur}, using $\sigma=0$ and $\nabla f_i(x^*)=0$, we immediately obtain
\[
\EE\left[ \|x^{t+1} -x^*\|_2^2 \,\|\, x^t \right]\leq  (1-\mu \gamma \tau)\|x^{t} -x^*\|_2^2=\left(1-\frac{\mu}{2L} \frac{\tau n}{\tau n + 2(1-\tau)}\right)\|x^{t} -x^*\|_2^2 .
\]
It remains to apply the above inequality recursively. 
\end{proof}

\subsection{Proof of Theorem~\ref{th:ibd}}
\begin{proof}
    Clearly, 
    \begin{align*}
        \EE x_i^{t+1}
        = x^t - \gamma\tau\nabla f_i(x^t).
    \end{align*} 
        Let us now elaborate on the second moments.

    Thus,
\[
        \EE \|x_i^{t+1} - \EE x_i^{t+1}\|_2^2 
        = \gamma^2\EE \left[\| g_i^t - \tau \nabla f_i(x^t)\|_2^2 \right] 
        = \tau (1-\tau) \| \nabla f_i(x)\|^2.
\]
 Therefore, the conditional variance of $x^{t+1}$ variance is equal to
    \begin{align*}
        \EE \|x^{t+1} - \EE x^{t+1}\|_2^2 
        = \frac{1}{n^2}\sumin \EE \|x_i^{t+1} - x_i^{t+1}\|_2^2  = \frac{\tau (1-\tau)}{n^2}\sumin  \| \nabla f_i(x)\|^2.
    \end{align*}
Note that the above equality is exactly~\eqref{eq:x_moments} with $\sigma=0$. Thus, one can use Lemma~\ref{lem:sgd_recur} (with using $\sigma=0$ and $\nabla f_i(x^*)=0$) obtaining
\[
\EE\left[ \|x^{t+1} -x^*\|_2^2 \,\|\, x^t \right]\leq  (1-\mu \gamma \tau)\|x^{t} -x^*\|_2^2=\left(1-\frac{\mu}{2L} \frac{\tau n}{\tau n + 2(1-\tau)}\right)\|x^{t} -x^*\|_2^2 .
\]
It remains to apply the above inequality recursively. 
\end{proof}

\section{Missing Parts from Sections~\ref{sec:saga} and~\ref{sec:saga_dist}}

\subsection{Useful Lemmata}

Let us start with a variance bound, which will be useful for both Algorithm~\ref{alg:saga_dist} and Algorithm~\ref{alg:saga}. 
Define $\Phi(x) = \frac1k \sum_{i=1}^k \phi_i(x)$, and define $x^+ =x - \gamma (\nabla f_j(x)-\alpha_j+\bar{\alpha})_U$ for (uniformly) randomly chosen index $1\leq j \leq k$ and subset of blocks $U$ of size $\tau m $. Define also $\bar{\alpha} = \frac1k \sum_{i=1}^k \alpha_i$.

\begin{lemma}[Variance bound]\label{lem:saga_variance}
Assume $\phi$ is $\mu$-strongly convex and $\phi_j$ is $L$-smooth and convex for all $j$. Suppose that $x^\star = \argmin \Phi(x)$. Then, for any $x$ we have
    \begin{align}\label{eq:saga_variance}
        \EE \|x^{+} - \EE x^{+}\|_2^2
        \le 2\gamma^2\tau\left(2L(\phi(x) - \phi(x^\star) + \frac{1}{k}\sum_{j=1}^k \|\alpha_j - \nabla \phi_j(x^\star)\|_2^2 \right).
    \end{align}
\end{lemma}
\begin{proof}
Since $x^{+} = x - \gamma (\nabla \phi_{j}(x) - \alpha_j + \overline \alpha)_{U}$ and $\EE x^{+} = x - \gamma \tau \nabla \phi(x)$, we get
    \begin{align*}
        \EE \|x^{+} - \EE x^{+}\|_2^2
        &= \gamma^2 \EE \left\|\tau \nabla \phi(x) -(\nabla \phi_{j}(x) - \alpha_j + \overline \alpha)_{U} \right\|_2^2 \\
        &= \gamma^2 \EE  \|(\tau \nabla \phi(x) - (\nabla \phi_{j}(x) - \alpha_j + \overline \alpha))_{U}\|_2^2
         + \gamma^2 \EE \|\tau\nabla \phi(x)- (\tau \nabla \phi(x))_{U}\|_2^2 \\
        &= \gamma^2 \tau \EE \|\tau \nabla \phi(x) - (\nabla \phi_{j}(x) - \alpha_j + \overline \alpha)\|_2^2
         + \gamma^2(1 - \tau)\tau^2\|\nabla \phi(x)\|_2^2.
    \end{align*}
    We will leave the second term as is for now and obtain a bound for the first one. Note that the expression inside the norm is now biased: $\EE[\tau \nabla \phi(x) - (\nabla \phi_{j}(x) - \alpha_j + \overline \alpha)]=(\tau - 1)\nabla \phi(x)$. Therefore,
    \begin{align*}
        \EE \|\tau \nabla \phi(x) - (\nabla \phi_{j}(x) - \alpha_j + \overline \alpha)\|_2^2
        = (1 - \tau)^2 \|\nabla \phi(x)\|_2^2 + \EE\|\nabla \phi(x) - (\nabla \phi_{j}(x) - \alpha_j + \overline \alpha)\|_2^2.
    \end{align*}
    Now, since $\nabla \phi_{j}(x)$ and $\alpha_j$ are not independent, we shall decouple them using inequality $\|a + b\|_2^2 \le 2\|a\|_2^2 + 2\|b\|_2^2$. In particular,
    \begin{align*}
        \EE \|\nabla \phi(x) - (\nabla \phi_{j}(x) - \alpha_j + \overline \alpha)\|_2^2
        &= \EE \|\nabla \phi(x) - \nabla \phi_{j}(x) + \nabla \phi_{j}(x^\star) - \nabla \phi_{j}(x^\star) + \alpha_j - \overline \alpha\|_2^2\\
        &\le 2\EE \|\nabla \phi(x) - \nabla \phi_{j}(x) + \nabla \phi_{j}(x^\star)\|_2^2 + 2\EE\|\alpha_j - \nabla \phi_{j}(x^\star) - \overline \alpha\|_2^2.
    \end{align*}
    Both terms can be simplified by expanding the squares. For the first one we have:
    \begin{align*}
        \EE \|\nabla \phi(x) - \nabla \phi_{j}(x) + \nabla \phi_{j}(x^\star)\|_2^2
        &= \|\nabla \phi(x)\|_2^2 - 2\< \nabla \phi(x), \EE\left[\nabla \phi_{j}(x) - \nabla \phi_{j}(x^\star) \right]> \\
        &\quad + \EE \|\nabla \phi_{j}(x) - \nabla \phi_{j}(x^\star)\|_2^2 \\
        &= -\|\nabla \phi(x)\|_2^2 + \frac{1}{k}\sum_{j=1}^k \|\nabla \phi_j(x) - \nabla \phi_j(x^\star)\|_2^2.
    \end{align*}
    Similarly,
    \begin{align*}
        \EE\|\alpha_j - \nabla \phi_{j}(x^\star) - \overline \alpha\|_2^2
        &= \frac{1}{k}\sum_{j=1}^k\|\alpha_j - \nabla \phi_{j}(x^\star)\|_2^2 - 2\EE\<\alpha_j - \nabla \phi_{j}(x^\star), \overline \alpha> + \|\overline \alpha\|_2^2 \\
        &=\frac{1}{k}\sum_{j=1}^k\|\alpha_j - \nabla \phi_{j}(x^\star)\|_2^2 - \|\overline \alpha\|_2^2 \\
        &\le \frac{1}{k}\sum_{j=1}^k\|\alpha_j - \nabla \phi_{j}(x^\star)\|_2^2.
    \end{align*}
    Coming back to the first bound that we obtained for this lemma, we deduce
    \begin{align*}
        \EE \|x^{+} - \EE x^{+}\|_2^2
        &\le \gamma^2\tau\left((1 - \tau)^2\|\nabla \phi(x)\|_2^2 - 2\|\nabla \phi(x)\|_2^2 + \frac{2}{k}\sum_{j=1}^k\|\nabla \phi_j(x) - \nabla \phi_j(x^\star)\|_2^2  \right)\\
        &\quad +\gamma^2\tau\frac{2}{k}\sum_{j=1}^k \|\alpha_j - \nabla \phi_j(x^\star)\|_2^2 + \gamma^2(1 - \tau)\tau^2\|\nabla \phi(x)\|_2^2.
    \end{align*}
    The coefficient before $\|\nabla \phi(x)\|_2^2$ is equal to $\gamma^2\tau((1 - \tau)^2 - 2 + (1 - \tau)\tau) = \gamma^2\tau(1 - \tau - 2) < 0$, so we can drop this term. By smoothness of each $\phi_j$,
    \begin{eqnarray}\nonumber
        \frac{1}{k}\sum_{j=1}^k \|\nabla \phi_j(x) - \nabla \phi(x^\star)\|_2^2
        &\stackrel{\eqref{eq:smooth}}{\le}& \frac{2L}{k}\sum_{j=1}^k \left(\phi_j(x) - \phi(x^\star) - \<\nabla \phi_j(x^\star), x - x^\star> \right) \\
        &=& 2L (\phi(x) - \phi(x^\star)), \label{eq:saga_grad_sum}
    \end{eqnarray}
    where in the last step we used $\frac{1}{k}\sum_{j=1}^k\nabla \phi_j(x^\star)=0$.
\end{proof}

\begin{lemma}\label{lem:saga_grads}
For ISAGA with shared data we have (given the setting from Theorem~\ref{th:saga_shared})
    \begin{align*}
        \EE \left[\sum_{j=1}^N \|\alpha_j^{t+1} - \nabla f_j(x^*)\|_2^2 \, \Big| \, x^t \right]
        \le 2\tau Ln (f(x^t) - f(x^*)) + \left(1 - \frac{\tau n}{N}\right) \sum_{j=1}^N \|\alpha_j^{t} - \nabla f_j(x^*)\|_2^2.
    \end{align*}
    On the other hand, for distributed ISAGA we have for all $i$ (given the setting from Theorem~\ref{th:saga_dist}):
        \begin{align*}
        \EE \left[\sum_{j=1}^{l} \|\alpha_{ij}^{t+1} - \nabla f_{ij}(x^*)\|_2^2 \, \Big| \, x^t  \right]
        \le 2\tau L (f_i(x^t) - f_i(x^*)) + \left(1 - \frac{\tau }{l}\right) \sum_{j=1}^{l} \|\alpha_{ij}^{t} - \nabla f_{ij}(x^*)\|_2^2.
    \end{align*}
\end{lemma}
\begin{proof}
Consider all expectations throughout this proof to be conditioned on $x^t$. 
Let $j^t$ be the function index used to obtain $x_i^{t+1}$ from $x^t$. Then we have $(\alpha_{j^t}^{t+1})_{U_i^t} = (\nabla f_{j^t}(x^t))_{U_i^t}$. In the rest of the blocks, $\alpha_{j^t}^{t+1}$ coincides with its previous value. This implies
    \begin{align}
        \EE \left[\|\alpha_{j^t}^{t+1} - \nabla f_{j^t}(x^*)\|_2^2\mid j^t \right]
        = \tau \|\nabla f_{j^t}(x^t) - \nabla f_{j^t}(x^*)\|_2^2 + (1 - \tau) \|\alpha_{j^t}^{t} - \nabla f_{j^t}(x^*)\|_2^2. \label{eq:saga_134278238213698}
    \end{align}
    Taking expectation with respect to sampling of $j^t$ we obtain for shared data setup:
    \begin{eqnarray*}
        \EE \left[\sum_{j=1}^N \|\alpha_j^{t+1} - \nabla f_j(x^*)\|_2^2 \right]
        & =& 
        \sum_{j=1}^N \EE \left[ \|\alpha_j^{t+1} - \nabla f_j(x^*)\|_2^2 \right] \\ 
       & \stackrel{\eqref{eq:saga_134278238213698}}{=}& \sum_{j=1}^N \frac{n}{N} \left( \tau   \|\nabla f_j(x^t) - \nabla f_j(x^*)\|_2^2 + \left(1 - \tau \right)  \|\alpha_j^{t} - \nabla f_j(x^*)\|_2^2 \right) \\
       && \qquad +
 \sum_{j=1}^N \left(1-\frac{n}{N}\right)    \left(   \left(1 - \tau \right)  \|\alpha_j^{t} - \nabla f_j(x^*)\|_2^2 \right)\\
       & =& \tau \frac{n}{N} \sum_{j=1}^N \|\nabla f_j(x^t) - \nabla f_j(x^*)\|_2^2 + \left(1 - \frac{\tau n}{N}\right) \sum_{j=1}^N \|\alpha_j^{t} - \nabla f_j(x^*)\|_2^2.
    \end{eqnarray*}
Similarly, for distributed setup we get     
   \begin{align*}
        \EE \left[\sum_{j=1}^{l} \|\alpha_{ij}^{t+1} - \nabla f_{ij}(x^*)\|_2^2 \right]
        \le \tau \frac{1}{l} \sum_{j=1}^l \|\nabla f_{ij}(x^t) - \nabla f_{ij}(x^*)\|_2^2+ \left(1 - \frac{\tau }{l}\right) \sum_{j=1}^{l} \|\alpha_{ij}^{t} - \nabla f_{ij}(x^*)\|_2^2
    \end{align*}
    Using~\eqref{eq:saga_grad_sum}, the first sum of right hand side can be bounded by $2LN(f(x^t) - f(x^*))$ or $2Ll(f(x^t) - f(x^*))$. 
\end{proof}
\subsection{Proof of Theorem~\ref{th:saga_dist}}
\begin{proof}
    First of all, let us verify that it indeed holds $c>0$ and $\rho\ge 0$. As $\gamma \le \frac{1}{L\left(\frac{3}{n} + \tau\right)}$, we have $c = \frac{1}{n}\left( \frac{1}{\gamma L} - \frac{1}{n} - \tau\right) \ge \frac{1}{n}\left( \frac{3}{n} + \tau - \frac{1}{n} - \tau\right) > 0$. Furthermore, $\gamma\mu \ge 0$, so to show $\rho\ge 0$ it is enough to mention $\frac{1}{l} - \frac{2}{n^2lc} = \frac{1}{l} - \frac{2}{n^2l\left(\frac{1}{\gamma L} - \frac{1}{n} - \tau \right)} \ge \frac{1}{l} - \frac{2}{n^2l\left(\frac{3}{n} + \tau - \frac{1}{n} - \tau \right)} = 0$.
    
    Now we proceed to the proof of convergence. We are going to decompose the expected distance from $x^{t+1}$ to $x^*$ into its variance and the distance of expected iterates, so let us analyze them separately. The variance can be bounded as follows:
    \begin{eqnarray} 
               \EE \left[\|x^{t+1} - \EE [x^{t+1} \mid x^t]\|_2^2\mid x^t \right]
        &\stackrel{\eqref{eq:variance_decomposition}+\eqref{eq:saga_variance}}{\le} &
2\frac{\gamma^2\tau}{n}\left(2L(f(x) - f(x^*) + \frac{1}{ln}\sum_{j=1}^n\sum_{j=1}^l \|\alpha_{ij} - \nabla f_{ij}(x^*)\|_2^2 \right). \label{eq:saga_distrib_variance}
    \end{eqnarray}
    For the distance of the expected iterates we write
    \begin{align*}
        \|\EE [x^{t+1} \mid x^t] - x^*\|_2^2
        &= \|x^t - \gamma\tau \nabla f(x^t) - x^*\|_2^2 \\
        &\le (1 - \gamma\tau\mu)\|x^t - x^*\|_2^2 - 2\gamma\tau (f(x^t) - f(x^*)) + 2\gamma^2\tau^2 L(f(x^t) - f(x^*)).
    \end{align*}
    As is usually done for SAGA, we are going to prove convergence using a Lyapunov function. Namely, let us define
    \begin{align}\label{eq:saga_lyapunov}
        \cL^t \eqdef \EE \left[\|x^t - x^*\|_2^2 + c \gamma^2\sum_{i=1}^n\sum_{j=1}^l \|\alpha_{ij}^t - \nabla f_{ij}(x^*)\|_2^2 \right],
    \end{align}
    where $c =  \frac{1}{n}\left( \frac{1}{\gamma L} - \frac{1}{n} - \tau\right)$. Using Lemma~\ref{lem:saga_grads} together with the bounds above, we get
    \begin{align*}
        \cL^{t+1} 
        &\le \EE \left[(1 - \gamma\tau\mu)\|x^t - x^*\|_2^2 + \left(\frac{2\gamma^2\tau}{n^2l} + c\gamma^2\left(1 - \frac{\tau }{l}\right)\right)\sum_{i=1}^n\sum_{j=1}^l \|\alpha_{ij}^t - \nabla f_{ij}(x^*)\|_2^2 \right] \\
        &\quad + 2\gamma\tau\EE\Bigl[\Bigl(\underbrace{\gamma\tau L + \frac{\gamma L}{n} + c\gamma L n - 1}_{= 0\text{ by our choice of } c } \Bigr) (f(x^t) - f(x^*)) \Bigr].
    \end{align*}
    In fact, we chose $c$ exactly to make the last expression equal to zero.
    After dropping it, we reduce the bound to
    \begin{align*}
        \cL^{t+1} 
        \le (1 - \rho)\EE \left[ \|x^t - x^*\|_2^2 + c\gamma^2\sum_{i=1}^n\sum_{j=1}^l \|\alpha_{ij}^t - \nabla f_{ij}(x^*)\|_2^2 \right]
        = (1 - \rho)\cL^{t},
    \end{align*}
    where $\rho = \min\left\{\gamma\tau\mu, \frac{\tau }{l} - \frac{2\tau}{n^2lc} \right\}$. Note that $\EE \|x^t - x^*\|_2^2\le \cL^t \le (1 - \rho)^t\cL^0$ by induction, so we have the stated linear rate.
\end{proof}
\subsection{Proof of Theorem~\ref{th:saga_shared}}
\begin{proof}
    First of all, let us verify that it indeed holds $c>0$ and $\rho\ge 0$. As $\gamma \le \frac{1}{L\left(\frac{3}{n} + \tau\right)}$, we have $c = \frac{1}{n}\left( \frac{1}{\gamma L} - \frac{1}{n} - \tau\right) \ge \frac{1}{n}\left( \frac{3}{n} + \tau - \frac{1}{n} - \tau\right) > 0$. Furthermore, $\gamma\mu \ge 0$, so to show $\rho\ge 0$ it is enough to mention $\frac{n}{N} - \frac{2}{nNc} = \frac{n}{N} - \frac{2}{N\left(\frac{1}{\gamma L} - \frac{1}{n} - \tau \right)} \ge \frac{n}{N} - \frac{2}{N\left(\frac{3}{n} + \tau - \frac{1}{n} - \tau \right)} = 0$.
    
    Now we proceed to the proof of convergence. We are going to decompose the expected distance from $x^{t+1}$ to $x^*$ into its variance and the distance of expected iterates, so let us analyze them separately. The variance term can be bounded as follows
    \begin{eqnarray} \label{eq:saga_shared_variance}
        \EE \left[\|x^{t+1} - \EE [x^{t+1} \mid x^t]\|_2^2\mid x^t \right]
        &\stackrel{\eqref{eq:variance_decomposition}+\eqref{eq:saga_variance}}{\le} &
        \frac{2\gamma^2\tau}{n}\left(2L(f(x^t) - f(x^*) + \frac{1}{N}\sum_{j=1}^N \|\alpha_j^t - \nabla f_j(x^*)\|_2^2 \right) 
    \end{eqnarray}
    For the distance of the expected iterates we write
    \begin{align*}
        \|\EE [x^{t+1} \mid x^t] - x^*\|_2^2
        &= \|x^t - \gamma\tau \nabla f(x^t) - x^*\|_2^2 \\
        &\le (1 - \gamma\tau\mu)\|x^t - x^*\|_2^2 - 2\gamma\tau (f(x^t) - f(x^*)) + 2\gamma^2\tau^2 L(f(x^t) - f(x^*)).
    \end{align*}
    As is usually done for SAGA, we are going to prove convergence using a Lyapunov function. Namely, let us define
    \begin{align*}
        \cL^t \eqdef \EE \left[\|x^t - x^*\|_2^2 + c \gamma^2\sum_{j=1}^N \|\alpha_j^t - \nabla f_j(x^*)\|_2^2 \right],
    \end{align*}
    where $c =  \frac{1}{n}\left( \frac{1}{\gamma L} - \frac{1}{n} - \tau\right)$. Using Lemma~\ref{lem:saga_grads} together with the bounds above, we get
    \begin{align*}
        \cL^{t+1} 
        &\le \EE \left[(1 - \gamma\tau\mu)\|x^t - x^*\|_2^2 + \left(\frac{2\gamma^2\tau}{nN} + c\gamma^2\left(1 - \frac{\tau n}{N}\right)\right)\sum_{j=1}^N\|\alpha_j^t - \nabla f_j(x^*)\|_2^2 \right] \\
        &\quad + 2\gamma\tau\EE\Bigl[\Bigl(\underbrace{\gamma\tau L + \frac{\gamma L}{n} + c\gamma L n - 1}_{= 0\text{ by our choice of } c } \Bigr) (f(x^t) - f(x^*)) \Bigr].
    \end{align*}
    In fact, we chose $c$ exactly to make the last expression equal to zero.
    After dropping it, we reduce the bound to
    \begin{align*}
        \cL^{t+1} 
        \le (1 - \rho)\EE \left[ \|x^t - x^*\|_2^2 + c\gamma^2\sum_{j=1}^N\|\alpha_j^t - \nabla f_j(x^*)\|_2^2\right]
        = (1 - \rho)\cL^{t},
    \end{align*}
    where $\rho = \min\left\{\gamma\tau\mu, \frac{\tau n}{N} - \frac{2\tau}{nNc} \right\}$. Note that $\EE \|x^t - x^*\|_2^2\le \cL^t \le (1 - \rho)^t\cL^0$ by induction, so we have the stated linear rate.
\end{proof}

\section{Proofs for Section~\ref{sec:sgd}}
\subsection{Useful Lemmata}

The next lemma is a key technical tool to analyze Algorithm~\ref{alg:sgd}. It provides a better expression for first and second moments of algorithm iterates.

\begin{lemma}[SGD moments]\label{lem:moments}
    Consider the randomness of the update of Algorithm~\ref{alg:sgd} at moment $t$. The first moments of the generated iterates are simply $\EE x_i^{t+1} = x^t - \gamma\tau \nabla f_i(x^t)$ and $\EE x^{t+1} = x^t - \gamma\tau \nabla f(x^t)$, while their second moments are:
    \begin{align}
        \EE\left[ \|x_i^{t+1} - \EE x_i^{t+1}\|_2^2 \, | \, x^t\right]&= \gamma^2 \tau\Big(\left(1 - \tau \right) \|\nabla f_i(x^t)\|_2^2  + \EE \|g_i^t - \nabla f_i(x^t)\|_2^2 \Big), \label{eq:x_i_moments}\\
        \EE\left[\|x^{t+1} - \EE x^{t+1}\|_2^2 \, | \, x^t\right] &= \gamma^2 \frac{\tau}{n^2}\sumin \Big((1 - \tau)\|\nabla f_i(x^t)\|_2^2  + \EE \|g_i^t - \nabla f_i(x^t)\|_2^2 \Big) \label{eq:x_moments}.
    \end{align}
\end{lemma}

\begin{proof}
    Clearly, 
    \begin{align*}
        \EE x_i^{t+1}
        = x^t - \gamma \EE \left[(g_i^t)_{U_i^t} \right]
        = x^t - \gamma \EE \left[ \left( \nabla f_i(x^t)\right)_{U_i^t} \right]
        = x^t - \gamma\tau\nabla f_i(x^t)
    \end{align*}
    and, therefore, $\EE x^{t+1} = x^t - \gamma\tau\nabla f(x^t)$.
    Let us now elaborate on the second moments. Using the obtained formula for $\EE x_i^{t+1}$, we get $(x_i^{t+1} - \EE x_i^{t+1})_{U_i^t} = -\gamma (g_i^t - \tau \nabla f_i(x^t))_{U_i^t}$ and $(x_i^{t+1} - \EE x_i^{t+1})_{\bar{U}_i^t} = \gamma\tau\left( \nabla f_i(x^t)\right)_{\bar{U}_i^t}$ where $\bar{U}_i^t$ is a set of blocks not contained in $U_i^t$. Thus,
    \begin{align*}
        \EE \|x_i^{t+1} - \EE x_i^{t+1}\|_2^2 
        &= \gamma^2\EE \left[\| (g_i^t - \tau \nabla f_i(x^t))_{U_i^t}\|_2^2 + \tau^2\|\left( \nabla f_i(x^t)\right)_{\bar{U}_i^t}\|_2^2 \right] \\
        &= \gamma^2 \left(\tau\EE\| g_i^t - \tau \nabla f_i(x^t)\|_2^2  + \tau^2\left(1 - \tau\right)   \|\nabla f_i(x^t)\|_2^2\right).
    \end{align*}
    Note that $\EE g_i^t - \tau \nabla f_i(x^t) = (1 - \tau)\nabla f_i(x^t)$, so we can use decomposition~\eqref{eq:variance_decompos} to write $\EE\| g_i^t - \tau \nabla f_i(x^t)\|_2^2 = (1 - \tau)^2\|\nabla f_i(x^t)\|_2^2 + \EE \|g_i^t - \nabla f_i(x^t)\|_2^2$. This develops our previous statement into
    \begin{align*}
        \EE \|x_i^{t+1} - \EE x_i^{t+1}\|_2^2 
        & = \gamma^2 \left(\tau\left((1 - \tau)^2\|\nabla f_i(x^t)\|_2^2 + \EE \|g_i^t - \nabla f_i(x^t)\|_2^2\right)  + \tau^2\left(1 - \tau\right)   \|\nabla f_i(x^t)\|_2^2\right)\\
        &= \gamma^2 \tau\left((1 - \tau)\|\nabla f_i(x^t)\|_2^2 + \EE \|g_i^t - \nabla f_i(x^t)\|_2^2 \right),
    \end{align*}
    which coincides with what we wanted to prove for $x_i^{t+1}$. As for $x^{t+1}$, it is merely the average of independent random variables conditioned on $x^t$. Therefore, its variance is equal to
    \begin{align*}
        \EE \|x^{t+1} - \EE x^{t+1}\|_2^2 
        = \frac{1}{n^2}\sumin \EE \|x_i^{t+1} - x_i^{t+1}\|_2^2 .
    \end{align*}
    This concludes the proof.
\end{proof}

\begin{lemma}\label{lem:expected_squared_grad}
    Let $f_i$ be $L$-smooth and convex for all $i$. Then, 
    \begin{align}
        \frac{1}{n}\sumin \|\nabla f_i(x^t)\|_2^2
        \le 4L(f(x^t) - f(x^*)) + \frac{2}{n}\sumin \|\nabla f_i(x^*)\|_2^2. \label{eq:lemma3}
    \end{align}
   Under Assumption~\ref{as:zero_grads}, the bound improves to
    \begin{align*}
        \frac{1}{n}\sumin \|\nabla f_i(x^t)\|_2^2
        \le 2L(f(x^t) - f(x^*)).
    \end{align*}
\end{lemma}
\begin{proof}
    If $\nabla f_i(x^*)=0$ for all $i$, we can simply write $\|\nabla f_i(x^t)\|_2^2 = \|\nabla f_i(x^t) - \nabla f_i(x^*)\|_2^2 \leq 2L(f(x^t) - f(x^*) - \<\nabla f_i(x^*), x^t - x^*>) = 2L (f(x^t) - f(x^*))$. 
    Otherwise, we have to use inequality $\|a+b\|_2^2 \le 2\|a\|_2^2 + 2\|b\|_2^2$ with $a = \nabla f_i(x^t) - \nabla f_i(x^*)$ and $b = \nabla f_i(x^*)$. We get
    \begin{align*}
        \sumin \|\nabla f_i(x^t)\|_2^2
        &\le 2 \sumin \|\nabla f_i(x^t) - \nabla f_i(x^*)\|_2^2 + 2\sumin\|\nabla f_i(x^*)\|_2^2 \\
        &\le 4L \sumin (f_i(x^t) - f_i(x^*) - \<\nabla f_i(x^*), x^t - x^*>) + 2\sumin\|\nabla f_i(x^*)\|_2^2\\
        &= 4Ln (f(x^t) - f(x^*)) + 2\sumin\|\nabla f_i(x^*)\|_2^2.
    \end{align*}
\end{proof}

\begin{lemma}\label{lem:second_momentum_of_stoch_grad}
	Let $f= \EE f(\cdot ; \xi)$ be $\mu$-strongly and $f(\cdot; \xi)$ be $L$-smooth and convex almost surely. Then, for any $x$ and $y$
	\begin{align*}
		\EE \|\nabla f(x; \xi)\|_2^2
		\le 4L(f(x) - f(y) - \<\nabla f(y), x - y>) + 2\EE \|\nabla f(y; \xi)\|_2^2.
	\end{align*}
\end{lemma}
\begin{proof}
	The proof proceeds exactly the same way as that of Lemma~\ref{lem:expected_squared_grad}.
\end{proof}

\begin{lemma}\label{lem:sgd_recur}
  Suppose that Assumption~\ref{as:smooth_sc} holds. Then, if we have
    \begin{align*}
        2\gamma\tau\left(1 - \gamma\tau L - \frac{2\gamma L(1 - \tau)}{n} \right)\EE[f(x^t) - f(x^*)]
        &\le (1 - \gamma\tau\mu)\EE\|x^t - x^*\|_2^2 - \EE \|x^{t+1} - x^*\|_2^2 \\
        &\quad + \gamma^2\frac{\tau}{n}\left(\sigma^2 + 2\frac{1 - \tau}{n}\sumin \|\nabla f_i(x^*)\|_2^2 \right).
    \end{align*}
\end{lemma}

\begin{proof}
    Substituting Assumption~\ref{as:bounded_noise} into~\eqref{eq:x_moments}, we obtain
    \begin{align}
        \EE \left[\|x^{t+1} - \EE [x^{t+1} \mid x^t]\|_2^2 \mid x^t\right]
        &\le \gamma^2 \frac{\tau}{n^2}\sumin \left((1 - \tau)\|\nabla f_i(x^t)\|_2^2 + \sigma^2 \right).\label{eq:sgd_varxt}
    \end{align}
    We use it together with decomposition~\eqref{eq:variance_decompos} to write
    \begin{align*}
        \EE \left[\|x^{t+1} - x^*\|_2^2 \mid x^t \right]
        &= \|\EE [x^{t+1} \mid x^t]  - x^*\|_2^2 + \EE\left[ \|x^{t+1} - \EE [x^{t+1} \mid x^t]\|_2^2\right]\\
        &\stackrel{\eqref{eq:sgd_varxt}}{\le} \|x^t -\gamma\tau \nabla f(x^t) - x^*\|_2^2 + \gamma^2 \frac{\tau}{n^2}\sumin \left((1 - \tau)\|\nabla f_i(x^t)\|_2^2 + \sigma^2 \right) \\
        & \stackrel{\eqref{eq:lemma3}}{\le} \|x^t -\gamma\tau \nabla f(x^t) - x^*\|_2^2 \\
        & \qquad + \gamma^2 \frac{\tau}{n}\left((1 - \tau)\left(4L(f(x^t) - f(x^*)) + \frac{2}{n}\sumin \|\nabla f_i(x^*)\|_2^2\right) + \sigma^2\right).
    \end{align*}
    Let us expand the first square:
    \begin{align*}
        \|x^t -\gamma\tau \nabla f(x^t) - x^*\|_2^2
        &= \|x^t - x^*\|_2^2 - 2 \gamma\tau\< x^t - x^*, \nabla f(x^t)> + \gamma^2 \tau^2\|\nabla f(x^t)\|_2^2 \\
        &\stackrel{\eqref{eq:smooth}}{\le} \|x^t - x^*\|_2^2 - 2 \gamma\tau\< x^t - x^*, \nabla f(x^t)> + \gamma^2 \tau 2L(f(x^t) - f(x^*)).
    \end{align*}
    The scalar product gives
    \begin{align*}
         \<\nabla f(x^t), x^t - x^*> 
        \stackrel{\eqref{eq:strong_convex}}{\ge} f(x^t) - f(x^*) + \frac{\mu}{2}\|x^t - x^*\|_2^2.
    \end{align*}
    Combining the produced bounds, we show that
    \begin{align*}
        \EE\left[\|x^{t+1} - x^*\|_2^2  \mid x^t\right]
        &\le \left(1 - \gamma\tau\mu\right)\|x^t - x^*\|_2^2 + \left(2\gamma^2 \tau L - 2\gamma\tau + \gamma^2(1 - \tau)\frac{4\tau L}{n}\right) (f(x^t) - f(x^*))\\
        &\quad + \gamma^2\frac{\tau}{n}\left(2(1 - \tau)\frac{1}{n}\sumin \|\nabla f_i(x^*)\|_2^2 + \sigma^2 \right).
    \end{align*}
    This is equivalent to our claim.
\end{proof}

\subsection{Proof of Theorem~\ref{th:sgd}}
\begin{proof}
Only for the purpose of this proof, denote $\gamma_t \eqdef \gamma^t$ in order to not confuse superscript with power. From the choice of $\gamma_t$ we deduce that $2\gamma_t\tau\left(1 - \gamma_t\tau L - \frac{2\gamma_t L(1 - \tau)}{n} \right)\ge \gamma_t\tau$.  Therefore, the result of Lemma~\ref{lem:sgd_recur} simplifies to
    \begin{align} \label{eq:lem5_consequence}
        \EE[f(x^k) - f(x^*)]
        \le \frac{1}{\gamma_k\tau}(1 - \gamma_k\tau\mu)\EE \|x^k - x^*\|_2^2 - \frac{1}{\gamma_k\tau}\EE \|x^{k+1} - x^*\|_2^2 + \gamma_k\frac{E}{n},
    \end{align}
    where $E \eqdef \sigma^2 + (1 - \tau)\frac{2}{n}\sumin\|\nabla f_i(x^*)\|_2^2$. Dividing~\eqref{eq:lem5_consequence} by $\gamma_k$ and summing it for $k=0,\dots, t$ we obtain
     \begin{align*}
        \sum_{k=0}^t \frac{1}{\gamma_k} \EE[f(x^k) - f(x^*)]
        &\le \frac{1}{\gamma_0^2\tau}(1 - \gamma_0\tau\mu)\|x^0 - x^*\|_2^2 - \frac{1}{\gamma_t^2\tau}\EE \|x^{t+1} - x^*\|_2^2 +t \frac{E}{n} \\
        &\quad + \frac1\tau \sum_{k=1}^{t-1}\left( \frac{1}{\gamma_k^2} \left( 1-\gamma_k\tau\mu\right) - \frac{1}{\gamma_{k-1}^2}\right)\EE\|x^k - x^*\|_2^2.\\
    \end{align*}
    Next, notice that 
             \begin{align*}
\frac{1}{\gamma_{k}^2} - \frac{1}{\gamma_{k+1}^2} \left( 1-\gamma_{k+1}\tau\mu\right) &= 
 \frac{1}{\gamma_{k}^2} \left(1-  \frac{\gamma_{k}^2}{\gamma_{k+1}^2} \left( 1-\gamma_{k+1}\tau\mu\right)\right)
 =
  \frac{1}{\gamma_{k}^2} \left(1- \left(1+ \frac{c}{a+ck} \right)^2  \left( 1-\frac{\tau\mu}{a+c(k+1)}\right)\right)
  \\
 &\stackrel{ (*)}{\geq} 
   \frac{1}{\gamma_{k}^2} \left(1- \left(1+ \frac{2.125\,c}{a+ck} \right) \left( 1-\frac{\tau\mu}{a+c(k+1)}\right)\right)
   \\
   &=
     \frac{1}{\gamma_{k}^2} \left(1- \left(1+ \frac{2.125}{4} \frac{1  }{\frac{a}{\tau \mu }+  \frac14 k} \right) \left( 1-\frac{1}{\frac{a}{\tau \mu }+ \frac14 k +\frac14}\right)\right) 
       \\
          &\stackrel{(**)}{\geq }0.
    \end{align*}
    Above $(*)$ holds since $\frac{c}{a+ck}\leq \frac18$ and $(1+\epsilon)^2\leq (1+2.125\epsilon)$ for $\epsilon \leq \frac18$. Next, inequality $(**)$ holds since function $\varphi(y) = (1+\frac{2.125}{4y})\left(1-\frac{1}{y+\frac14}\right)$ is upper bounded by 1 on $[0,\infty)$. Thus, we have

       \begin{align*}
        \sum_{k=0}^t \frac{1}{\gamma_k} \EE[f(x^k) - f(x^*)]
\le \frac{a^2}{\tau}\left(1 - \frac{\tau\mu}{a}\right)\|x^0 - x^*\|_2^2 +t \frac{E}{n}.
    \end{align*}
    
    All that remains is to mention that by Jensen's inequality $\EE f(\hat x^t) \le \frac{1}{(t+1)a+ \frac{c}{2}t(t+1)}\sum_{k=0}^t(a+ck)\EE f(x^k)=  \frac{1}{\sum_{k=0}^t\gamma_k^{-1}}\sum_{k=0}^t\gamma_k^{-1} \EE f(x^k) $.

\end{proof}

\subsection{Proof of Theorem~\ref{th:sgd_ncvx}}
It will be useful to establish a technical lemma first.
\begin{lemma}\label{lem:func_impr}
    Let $f$ be $L$-smooth and assume that $\frac{1}{n}\sumin \|\nabla f_i(x) - \nabla f(x)\|_2^2 \le \nu^2$ for all $x$. Then, considering only randomness from iteration $t$ of Algorithm~\ref{alg:sgd},
    \begin{align*}
        \EE f(x^{t+1})
        \le f(x^t) - \gamma\tau\left(1 - \frac{\gamma\tau L }{2} - \gamma L\left(1 - \tau\right)\frac{1}{n}\right)\|\nabla f(x^t)\|_2^2 + 
       \gamma^2L\tau\frac{\left(1 - \tau\right)\nu^2+\frac12\sigma^2}{n}.
    \end{align*}
\end{lemma}
\begin{proof}
    Using smoothness of $f$ and assuming $x^t$ is fixed, we write
    \begin{align*}
        \EE f(x^{t+1}) 
        &\le f(x^t) + \< \nabla f(x^t), \EE\ x^{t+1} - x^t> + \frac{L}{2}\EE\|x^{t+1} - x^t\|_2^2 \\
        &= f(x^t) - \gamma\tau \|\nabla f(x^t)\|_2^2 + \frac{L}{2}\EE\|x^{t+1} - x^t\|_2^2 .
    \end{align*}
    It holds 
    \begin{eqnarray*}
        \EE \|x^{t+1} - x^t\|_2^2
        &=& \left\|\EE x^{t+1} - x^t \right\|_2^2 + \EE\left\|x^{t+1} - \EE\left[x^{t+1} \mid x^t \right]\right\|_2^2
        \\
        &\overset{\eqref{eq:x_moments}}{=} &
        \gamma^2\tau^2  \left\|\nabla f(x^t) \right\|_2^2 + \gamma^2\tau\frac{1}{n^2}\sumin \left((1 - \tau) \|\nabla f_i(x^t)\|_2^2 +\EE \|g_i^t - \nabla f_i(x^t)\|_2^2\right) 
        \\
        &\overset{\text{As}.~\ref{as:bounded_noise}}{\leq}&
         \gamma^2\tau^2  \left\|\nabla f(x^t) \right\|_2^2 + \gamma^2\tau\frac{1}{n^2}\sumin \left((1 - \tau) \|\nabla f_i(x^t)\|_2^2 +\sigma^2\right).
    \end{eqnarray*}
    Using inequality $\|a+b\|_2^2 \le \|a\|_2^2 + \|b\|_2^2$ with $a= \nabla f_i(x^t) - \nabla f(x^t)$ and $b = \nabla f(x^t)$ yields
    \begin{align*}
        \frac{1}{n}\sumin \|\nabla f_i(x^t)\|_2^2
        \le \frac{2}{n}\sumin \|\nabla f_i(x^t) - \nabla f(x^t)\|_2^2 + 2 \|\nabla f(x^t)\|_2^2
        \le 2\nu^2 + 2 \|\nabla f(x^t)\|_2^2.
    \end{align*}
    Putting the pieces together, we prove the claim.
\end{proof}
We now proceed with Proof of Theorem~\ref{th:sgd_ncvx}.
    \begin{proof}
    Taking full expectation in Lemma~\ref{lem:func_impr} and telescoping this inequality from 0 to $t$, we obtain
    \begin{align*}
        0\le \EE f(x^{t+1}) - f^* \le f(x^0) - f^* - \gamma\tau\left(1 - \frac{\gamma\tau L}{2} - \gamma L\left(1 - \tau\right)\frac{1}{n}\right)\sum_{k=0}^t\|\nabla f(x^k)\|_2^2  + t\gamma^2 L\tau\frac{\left(1 - \tau\right)\nu^2+\frac12\sigma^2}{n}.
    \end{align*}
    Rearranging the gradients and dividing by the coefficient before it, we get the result.
    \end{proof}

\clearpage
\section{Missing Parts from Section~\ref{sec:ABCDE}}

\subsection{Proof of Lemma~\ref{lem:stronggrowth}}
\begin{proof}

Let us first bound a variance of $\frac1\tau (g_i)_{U_i}$ -- an unbiased estimate of $\nabla f_i(x)$, as it will appear later in the derivations:
\begin{eqnarray}
\EE\left[ \left\|\frac{1}{\tau} (g_i)_{U_i} -\nabla f_i(x)\right \|^2_2 \right]
&=&
\EE_g\left[\EE_U\left[ \left\|\frac{1}{\tau} (g_i)_{U_i} -\nabla f_i(x)\right \|^2_2 \right]\right]
\nonumber \\
&=&
\EE_g\left[  (1-\tau) \|\nabla f_i(x) \|^2_2+ \tau \left\|  \frac1\tau g_i- \nabla f_i(x)\right\|^2_2   
\nonumber\right] \\
&\stackrel{\eqref{eq:variance_decompos}}{=}&
(1-\tau) \|\nabla f_i(x) \|^2_2+ \tau \left\|  \left( \frac1\tau -1\right) \nabla f_i(x)\right\|^2_2 +\tau \EE_g\left[\left\|  \frac1\tau (g_i- \nabla f_i(x))\right\|^2_2 \right]
 \nonumber \\
&=&
(1-\tau) \|\nabla f_i(x) \|^2_2+ \tau \left( \frac1\tau -1\right)^2\left\|  \nabla f_i(x)\right\|^2_2+\frac1\tau \left\|  g_i- \nabla f_i(x)\right\|^2_2 \nonumber \\
&\stackrel{\eqref{eq:acc_sg_fi}}{\leq}&
(1-\tau) \|\nabla f_i(x) \|^2_2+ \tau \left( \frac1\tau -1\right)^2\left\|  \nabla f_i(x)\right\|^2_2 +\frac{\bar{\rho}}{\tau} \left\|  \nabla f_i(x)\right\|^2_2 + \frac{\bar{\sigma}^2}{\tau} \nonumber \\
&=&
\left(\frac1\tau-1+\frac{\bar{\rho}}{\tau} \right) \|\nabla f_i(x) \|^2_2  + \frac{\bar{\sigma}^2}{\tau}.
\label{eq:acc_gi_bound}
\end{eqnarray}
Next we proceed with bounding the second moment of gradient estimator: 
\begin{eqnarray*}
\EE\left[\|q\|^2_2\right] &=& 
\EE\left[ \left \|\frac{1}{n\tau} \sum_{i=1}^n(g_i)_{U_i} \right\|^2_2\right] 
\\
&\stackrel{\eqref{eq:variance_decompos}}{=}&
\|\nabla f(x) \|_2^2 + \EE\left[ \left\|\frac{1}{n\tau} \sum_{i=1}^n\left((g_i)_{U_i} -\nabla f_i(x)\right)\right\|^2_2 \right]
\\
&\stackrel{(*)}{=}&
\|\nabla f(x) \|^2_2 + \frac{1}{n^2} \sum_{i=1}^n\EE\left[ \left\|\frac{1}{\tau} (g_i)_{U_i} -\nabla f_i(x)\right \|^2_2 \right]
\\
&\stackrel{\eqref{eq:acc_gi_bound}}{\leq}&
\|\nabla f(x) \|^2_2 + \frac{1}{n^2} \sum_{i=1}^n\left( \left(\frac1\tau-1+\frac{\bar{\rho}}{\tau} \right) \|\nabla f_i(x) \|^2_2  + \frac{\bar{\sigma}^2}{\tau} 
 \right) 
 \\ 
  &=&
\|\nabla f(x) \|^2_2 + \frac{\bar{\sigma}^2}{n\tau} + \left(\frac1\tau-1+\frac{\bar{\rho}}{\tau} \right) \frac{1}{n^2} \sum_{i=1}^n \|\nabla f_i(x) \|^2_2
\\ 
  &\stackrel{\eqref{eq:acc_sg_f}}{\leq}&
\|\nabla f(x) \|^2_2 + \frac{\bar{\sigma}^2}{n\tau} + \left(\frac1\tau-1+\frac{\bar{\rho}}{\tau} \right) \frac{1}{n} \left(  \tilde{\rho}\| \nabla f(x)\|^2_2 + \tilde{\sigma}^2\right)
\\ 
  &=&
\left(1+ \frac{\tilde{\rho}}{n}   \left(\frac1\tau-1+\frac{\bar{\rho}}{\tau} \right) \right)\|\nabla f(x) \|^2_2 + \frac{\bar{\sigma}^2}{n\tau} + \frac{\tilde{\sigma}^2}{n}\left(\frac1\tau-1+\frac{\bar{\rho}}{\tau} \right) 
\\ 
  &\stackrel{\eqref{eq:acc_rho}+ \eqref{eq:acc_sigma}}{=}&
\hat{\rho}\|\nabla f(x) \|^2_2 + \frac{\bar{\sigma}^2}{n\tau} + \hat{\sigma}^2.
\end{eqnarray*}
Above, $(*)$ holds since $\frac{1}{\tau} (g_i)_{U_i} -\nabla f_i(x)$ is zero mean for all $i$ and $U_i, U_j$ are independent for $i\neq j$.

\end{proof}

\section{Proofs for Section~\ref{sec:sega}}
\subsection{Useful Lemmata}
First, we mention a basic property of the proximal operator. 
\begin{proposition}
    \label{pr:prox_contraction}
    Let $R$ be a closed and convex function. Then for any $x, y\in\RR^d$
    \begin{align} \label{eq:prox_contraction}
        \|\proxR(x) - \proxR(y)\|_2
        \le \|x - y\|_2.
    \end{align}
\end{proposition}
The next lemma shows a basic recurrence of sequences $\{ h_i^t\}_{t=1}^\infty$ from ISEGA. 
\begin{lemma} \cite[Lemma B.3]{SEGA} If $h_i^{t+1} \eqdef h_i^t + \tau (g_i^t - h^t)$, where $g_i^t \eqdef h_i^t + \frac{1}{\tau} (\nabla f_i(x^t) - h_i^t)_{U_i^t}$, then
\begin{equation} \label{eq:sega_h_bound}
\EE\left[ \|h^{t+1}_i - \nabla f_i(x^*)\|^2_2 \right] =(1-\tau) \|h^t_i - \nabla f_i(x^*) \|^2_2+\tau \| \nabla f_i(x^t)-\nabla f_i(x^*)\|^2_2.
\end{equation}
\end{lemma}

We will also require a recurrent bound on sequence $\{ g^t\}_{t=1}^\infty$ from ISEGA. 
\begin{lemma}
Consider any vectors $v_i$ and set $v\eqdef \frac1n \sum_{i=1}^n v_i$. Then, we have
\begin{equation} \label{eq:sega_g_bound}
\EE \left[ \| g^t - v \|^2_2\right] \leq 
\frac{2}{n^2} \sum_{i=1}^n \left( \left( \frac{1}{\tau}+ (n-1)\right) \| \nabla f_i(x^t) - v_i\|^2_2
+    \left( \frac{1}{\tau}-1\right)\|h^t_i - v_i\|^2_2\right) .
\end{equation}
\end{lemma}
\begin{proof}
   Writing $g^t - v = a + b$, where 
   \[a\eqdef \frac1n \sum_{i=1}^n \left( h_i^t - v_i - \tau^{-1}(h_i^t - v_i)_{U_i^t}  \right)\] and 
   \[b\eqdef \frac1n \sum_{i=1}^n \tau^{-1}  (\nabla f_i(x^t) - v_i)_{U_i^t} \]
 we get $\|g^t-v\|^2_2=\|a+b\|_2^2=\leq 2(\|a\|^2_2 + \|b\|^2_2)$. 
  
 Let us bound $\EE \left[\|b\|^2_2\right]$ using Young's inequality $2\<x, y> \le \|x\|_2^2 + \|y\|_2^2$:
\begin{eqnarray*}
\EE \left[\|b\|^2_2\right]
&=& 
\frac{1}{n^2} \EE \left[ \< \sum_{i=1}^n  \tau^{-1}(\nabla f_i(x^t) - v_i)_{U_i^t} ,  \sum_{i=1}^n  \tau^{-1}  (\nabla f_i(x^t) - v_i)_{U_i^t} >\right]
\\
&=& 
\frac{1}{\tau^2 n^2} \EE \left[\sum_{i=1}^n \left\| (\nabla f_i(x^t) - v_i)_{U_i^t} \right\|^2_2  \right] 
 + 
\frac{2}{\tau^2n^2} \EE \left[\sum_{i\neq j} \<(\nabla f_i(x^t) - v_i)_{U_i^t}  , (\nabla f_j(x^t) - v_j)_{U_i^t} > \right]  
\\
&=& 
\frac{1}{\tau n^2}\sum_{i=1}^n \left\| \nabla f_i(x^t) - v_i \right\|^2_2
 + 
\frac{2}{n^2} \sum_{i\neq j}  \<\nabla f_i(x^t) - v_i, \nabla f_j(x^t) - v_j> 
\\
&\leq& 
\frac{1}{\tau n^2}\sum_{i=1}^n \left\| \nabla f_i(x^t) - v_i \right\|^2_2 
+
 \frac{1}{n^2} \sum_{i\neq j} \left( \| \nabla f_i(x^t) - v_i\|^2_2 +\| \nabla f_j(x^t) - v_j\|^2_2\right)
\\
&= & 
\frac{1}{n^2}\left( \frac1\tau+ n-1 \right)  \sum_{i=1}^n \| \nabla f_i(x^t) - v_i\|^2_2.
\end{eqnarray*}
Similarly we bound $\EE \left[\|a\|^2_2\right]$:
\begin{eqnarray*}
\EE \left[\|a\|^2_2\right]
&=& 
\frac{1}{n^2} \EE \left[\<  \sum_{i=1}^n \left( h_i^t - v_i - \tau^{-1}(h_i^t - v_i)_{U_i^t}  \right),   \sum_{i=1}^n \left( h_i^t - v_i - \tau^{-1} (h_i^t - v_i)_{U_i^t}  \right)> \right]
\\
&=& 
\frac{1}{n^2} \EE \left[\sum_{i=1}^n \< \left( h_i^t - v_i - \tau^{-1} (h_i^t - v_i)_{U_i^t} \right),\left( h_i^t - v_i - \tau^{-1}(h_i^t - v_i)_{U_i^t}  \right)> \right] 
\\
&& \qquad + 
\frac{2}{n^2}\EE \left[\sum_{i\neq j}\< \left( h_i^t - v_i - \tau^{-1} (h_i^t - v_i)_{U_i^t}  \right),\left( h_j^t - v_j - \tau^{-1} (h_j^t - v_j)_{U_i^t}  \right)> \right]  
\\
&=& 
\frac{\tau^{-1}-1}{n^2}\sum_{i=1}^n  \|h^t_i - v_i\|^2_2.
\end{eqnarray*}
It remains to combine the above results.
\end{proof}
\subsection{Proof of Theorem~\ref{thm:sega}}
\begin{proof}
For convenience, denote $g^t \eqdef \frac1n \sum_{i=1}^n g_i^t$. It holds
\begin{eqnarray}
\nonumber
\EE[\|x^{t+1} - x^*\|^2_2]
&=& \EE\left[\|\prox_{\gamma R}(x^k - \gamma g^k) - \prox_{\gamma R}(x^* - \gamma \nabla f(x^*))\|_2^2\right]\\ \nonumber
&\stackrel{\eqref{eq:prox_contraction}}{\leq}& \EE\left[\|x^k - \gamma g^k - (x^* - \gamma \nabla f(x^*))\|^2_2\right]\\ \nonumber
&=& \|x^k - x^*\|^2_2 - 2\gamma \<\nabla f(x^k) - \nabla f(x^*),  x^k - x^*>+ \gamma^2\EE\left[ \|g^k - \nabla f(x^*)\|^2_2\right] 
\\ \nonumber
&\stackrel{\eqref{eq:sega_g_bound}}{\leq}&
\|x^k - x^*\|^2_2 - 2\gamma \<\nabla f(x^k) - \nabla f(x^*),x^k - x^*>
\\ \nonumber && \qquad +
\gamma^2 \frac{2}{n^2} \sum_{i=1}^n \left( \left( \frac{1}{\tau}+ (n-1)\right) \| \nabla f_i(x^t) - \nabla f_i(x^*)\|^2_2
+    \left( \frac{1}{\tau}-1\right)\|h^t_i - \nabla f_i(x^*)\|^2_2\right)
\\ \nonumber
&\leq&
\|x^k - x^*\|^2_2 -  \gamma \mu \|x^t-x^* \|^2_2  -2\gamma D_f(x^t,x^*)
\\ && \qquad +
\frac{2}{n^2} \sum_{i=1}^n \left( \left( \frac{1}{\tau}+ (n-1)\right) \| \nabla f_i(x^t) - \nabla f_i(x^*)\|^2_2
+    \left( \frac{1}{\tau}-1\right)\|h^t_i - \nabla f_i(x^*)\|^2_2\right).
\label{eq:sega_first}
\end{eqnarray}
Moreover, we have from smoothness and convexity of $f_i$
\begin{equation}\label{eq:sega_smoothness}
-2D_{f_i}(x^t,x^*)\leq -\frac{1}{L} \|\nabla f_i(x^t)-\nabla f_i(x^*) \|^2_2. 
\end{equation}
Combining the above, for any $\seganu \geq 0$ (which we choose later) we get
\begin{eqnarray*}
\EE[\|x^{t+1} - x^*\|^2_2] &+& \gamma \seganu  \frac1n\sum_{i=1}^n \EE\left[ \|h^{t+1}_i - \nabla f_i(x^*)\|^2_2 \right]  \\
&& \quad 
\stackrel{\eqref{eq:sega_first}+\eqref{eq:sega_h_bound}}{\leq}
\|x^t - x^*\|^2_2 -  \gamma \mu \|x^t-x^* \|^2_2  -2\gamma D_f(x^t,x^*)
\\ && \qquad \qquad +
\gamma^2 \frac{2}{n^2} \sum_{i=1}^n \left( \left( \frac{1}{\tau}+ (n-1)\right) \| \nabla f_i(x^t) - \nabla f_i(x^*)\|^2_2
+    \left( \frac{1}{\tau}-1\right)\|h^t_i - \nabla f_i(x^*)\|^2_2\right)
\\ && \qquad \qquad + 
 \gamma \seganu  \frac1n\sum_{i=1}^n \left((1-\tau) \|h^t_i - \nabla f_i(x^*) \|^2_2+\tau \| \nabla f_i(x^t)-\nabla f_i(x^*)\|^2_2\right)\\
 && \qquad 
 \stackrel{\eqref{eq:sega_smoothness}}{\leq}
\|x^t - x^*\|^2_2 -  \gamma \mu \|x^t-x^* \|^2_2  - \frac{\gamma}{nL} \sum_{i=1}^n \|\nabla f_i(x^t)-\nabla f_i(x^*) \|^2_2
\\ && \qquad \qquad +
\gamma^2\frac{2}{n^2} \sum_{i=1}^n \left( \left( \frac{1}{\tau}+ (n-1)\right) \| \nabla f_i(x^t) - \nabla f_i(x^*)\|^2_2
+    \left( \frac{1}{\tau}-1\right)\|h^t_i - \nabla f_i(x^*)\|^2_2\right)
\\ && \qquad \qquad + 
 \gamma \seganu  \frac1n\sum_{i=1}^n \left((1-\tau) \|h^t_i - \nabla f_i(x^*) \|^2_2+\tau \| \nabla f_i(x^t)-\nabla f_i(x^*)\|^2_2\right)\\
 && \qquad 
 =
 (1-\gamma \mu)\|x^t - x^*\|^2_2 + \left(\seganu\tau + \frac{2\gamma}{n}\left(\frac1\tau+n-1\right)-\frac{1}{L}\right)\frac{\gamma}{n}\sum_{i=1}^n\| \nabla f_i(x^t)-\nabla f_i(x^*)\|^2_2
 \\ && \qquad \qquad +
  \left( \frac{2\gamma}{n}\left( \frac1\tau -1\right)+ \seganu(1-\tau) \right) \frac{\gamma}{n}\sum_{i=1}^n \|h^t_i - \nabla f_i(x^*) \|^2_2.
\end{eqnarray*}
To get rid of gradient differences in this bound, we want to obtain $\frac{1}{L}\geq \frac{2\gamma}{n}(\frac{1}{\tau}+n-1)+ \seganu\tau$, which, in turn, is satisfied if
\begin{eqnarray*}
\gamma &=& {\cal O}\left(\frac{1+\tau n}{L}\right),\\
\omega &=& {\cal O}\left(\frac{1}{L\tau}\right).
\end{eqnarray*} 
Next, we want to prove contraction with factor $(1 - \gamma\mu)$ in terms of $\|h_i^t-\nabla f_i(x^*)\|^2_2$, so we require
\[
(1-\gamma\mu ) \nu\geq \omega (1-\tau)+\frac{2\gamma}{n}\left(\frac1\tau-1\right)
\]
we shall choose $\gamma$ such that the following two properties hold:
\begin{eqnarray*}
\gamma &=& {\cal O}\left(\frac{\tau}{\mu}\right),\\
\gamma &=& {\cal O}\left(\frac{n\tau^2\omega}{1-\tau}\right) = {\cal O}\left(\frac{n\tau}{(1-\tau)L}\right)\geq  {\cal O}\left(\frac{n\tau}{L}\right).
\end{eqnarray*} 
In particular, the choice $\omega = \frac{1}{2L\tau}$ and $\gamma= \min\left( \frac{1}{4L\left( 1+\frac{1}{n\tau }\right)}, \frac{1}{\frac{\mu}{\tau}+ \frac{4L}{n\tau}}\right) $ works. 
\end{proof}
\section{Proofs for Section~\ref{sec:asynch}}
One way to analyze a delayed algorithm is to define sequence of epoch start moments $T_0, T_1, \dotsc$ such that $T_0=0$ and $T_{k+1} = \min\{t: t - \max_{i=1,\dotsc,n} d_i^t \ge T_k\}$. In case delays are bounded uniformly, i.e.\ for some number $M$ it holds $d_i^t \le M$ for all $i$ and $t$, one can show by induction~\cite{mishchenko2018distributed} that $T_k\le Mk$.

In addition, we define for every $i$ sequence
\begin{align*}
	z_i^t = x^{t - d_i^t}.
\end{align*}
For notational simplicity, we will assume that if worker $i$ does not perform an update at iteration $t$, then all related vectors increase their counter without changing their value,  i.e.\ $g_i^{t+1} = g_i^t$, $U_i^{t+1}=U_i^t$, $z_i^{t+1} = z_i^t$ and $x_i^{t+1} = x_i^t$. Then, we can write a simple identity for $x_i^t$ that holds for any $i$ and $t$,
\begin{align}
	x_i^t
	= x^{t - d_i^t}  - \gamma (g_i^{t})_{U_i^{t}}
	= z_i^t  - \gamma (g_i^{t})_{U_i^{t}}. \label{eq:delayed_recurrence}
\end{align}
\subsection{Useful Lemmata}
\begin{lemma}\label{lem:asynch}
	Let Assumption~\ref{as:bounded_noise_at_opt} be satisfied and assume without loss of generality that $d_1^t<\dotsc< d_n^t$. Then, for any $i$
	\begin{align}
		\EE\|x_i^t - \EE[x_i^t \mid z_i^t, x_{i+1}^t, \dotsc, x_n^t]\|_2^2
		&\le 4\gamma^2\tau\EE\left[\sigma^2 + 2L(f_i(z_i^t) - f_i(x^*) - \<\nabla f_i(x^*), z_i^t -x^*>)\right].\label{eq:conditioned_variance}
	\end{align}
\end{lemma}
\begin{proof}
	Denote by $\cF_i^t$ the sigma-algebra generated by $z_i^t, x_{i+1}^t, \dotsc, x_n^t$. Then, $\EE\left[\cdot \mid z_i^t, x_{i+1}^t, \dotsc, x_n^t\right] = \EE\left[\cdot \mid \cF_i^t\right]$.

	Since $d_1^t<\dotsc< d_n^t$, $x_i^t$ is independent of the randomness in $x_1^t, \dotsc, x_{i-1}^t$ as those vectors were obtained after $x_i^t$. Recall that
	\begin{align*}
		x_i^t 
		\overset{\eqref{eq:delayed_recurrence}}{=} z_i^t  - \gamma (g_i^{t})_{U_i^{t}}
	\end{align*}
	and denote $\tilde x_i^t \eqdef z_i^{t} - \nabla f_i(z_i^t)$. Clearly, by uniform sampling of the blocks $\EE[ x_i^t\mid \cF_i^t] = z_i^{t} - \tau\gamma\EE[g_i^t\mid \cF_i^t] = z_i^t - \gamma\tau \nabla f_i(z_i^t)$. Thus,
	\begin{align*}
		\EE\|x_i^t - \EE[x_i^t \mid \cF_i^t]\|_2^2
		&= \gamma^2 \EE\|(g_i^t)_{U_i^t} - \tau \nabla f_i(z_i^t)\|^2\\
		&= (1 - \tau)\gamma^2\EE\|\tau \nabla f_i(z_i^t)\|^2 +\tau \gamma^2\EE \| g_i^t - \tau \nabla f_i(z_i^t)\|^2 \\
		&= (1 - \tau)\gamma^2\tau^2\EE\|\nabla f_i(z_i^t)\|^2 +\tau \gamma^2\EE \left[\| \nabla f_i(z_i^t) - \tau \nabla f_i(z_i^t)\|^2 + \|g_i^t - \nabla f_i(z_i^t)\|^2 \right]\\
		&\le \tau \gamma^2\EE\left[\|\nabla f_i(z_i^t)\|^2 + \|g_i^t - \nabla f_i(z_i^t)\|^2\right] \\
		&\overset{\eqref{eq:sgd_variance}}{\le} \tau \gamma^2\EE\left[\|\nabla f_i(z_i^t)\|^2 + 2\sigma^2 + 4L(f_i(z_i^t) - f_i(x^*) - \<\nabla f_i(x^*), z_i^t - x^*>)\right] .
	\end{align*}
	In addition,
	\begin{align*}
		\|\nabla f_i(z_i^t)\|^2
		\le 2\|\nabla f_i(z_i^t) - \nabla f_i(x^*)\|^2 + 2 \|\nabla f_i(x^*)\|^2
		\le 4L(f_i(z_i^t) - f_i(x^*) - \<\nabla f_i(x^*), z_i^t - x^*>) + 2 \sigma^2.
	\end{align*}
\end{proof}
We will use in the proof of Theorem~\ref{th:asynch} Yensen's inequality for a set of vectors $a_1,\dotsc, a_n\in\RR^d$ in the form
\begin{align*}
	\left\|\avein a_i \right\|_2^2
	\le \avein \|a_i\|_2^2.
\end{align*}
\begin{lemma}\label{lem:asynch_contraction}
	Assume that $f_i$ is $L$-smooth and $\mu$-strongly convex. If $\tilde x_i^t \eqdef z_i^t - \tau\gamma \nabla f_i(z_i^t)$ and $x_i^* \eqdef x^* - \tau\gamma \nabla f_i(x^*)$, we have
	\begin{align*}
		\|\tilde x_i^t - x_i^*\|_2^2
		\le (1 - \tau\gamma\mu)\|z_i^t - x^*\|_2^2 - 2\gamma\tau(1 - \tau\gamma L)(f_i(z_i^t) - f_i(x^*) - \<\nabla f_i(z_i^t), z_i^t - x^*>).
	\end{align*}
\end{lemma}
\begin{proof}
	It holds
	\begin{align*}
		 \|\tilde x_i^t - x_i^*\|_2^2
		= \|z_i^t - x^*\|_2^2 - 2\gamma\tau\<\nabla f_i(z_i^t) - \nabla f_i(x^*), z_i^t - x^*> + \gamma^2\tau^2\|\nabla f_i(z_i^t) - \nabla f_i(x^*)\|_2^2.
	\end{align*}
	Moreover, by strong convexity and smoothness of $f_i$ (see e.g.~\cite{NesterovBook})
	\begin{align*}
		2\<\nabla f_i(z_i^t) - \nabla f_i(x^*), z_i^t - x^*>
		\ge \mu \|z_i^t - x^*\|_2^2 + 2(f_i(z_i^t) - f_i(x^*) - \<\nabla f_i(z_i^t), z_i^t - x^*>).
	\end{align*}
	On the other hand, convexity and smoothness of $f_i$ together imply
	\begin{align*}
		\|\nabla f_i(z_i^t) - \nabla f_i(x^*)\|_2^2
		\le 2L(f_i(z_i^t) - f_i(x^*) - \<\nabla f_i(z_i^t), z_i^t - x^*>).
	\end{align*}
	Consequently,
	\begin{align*}
		 \|\tilde x_i^t - x_i^*\|_2^2
		\le (1 - \tau\gamma\mu)\|z_i^t - x^*\|_2^2 - 2\tau\gamma(1 - \tau\gamma L)(f_i(z_i^t) - f_i(x^*) - \<\nabla f_i(z_i^t), z_i^t - x^*>).
	\end{align*}
\end{proof}

\subsection{Proof of Theorem~\ref{th:asynch}}
We are going to prove a more general result that does not need uniform boundedness of delays over time. Theorem~\ref{th:asynch} will follow as a special case of the more general theorem.
\begin{theorem}\label{th:asynch_epoch}
	Assume that every $f_i$ is $L$-smooth and $\mu$-strongly convex and also that the gradients noise has bounded variance at $x^*$ as in Assumption~\ref{as:bounded_noise_at_opt}. If also $\gamma \le \frac{1}{2L(\tau + \frac{2}{n})}$, then for any $t\in[T_k, T_{k+1})$
	\begin{align*}
		\EE\|x^t - x^*\|_2^2
		\le \left(1 - \tau\gamma\mu\right)^k\max_{i=1,\dotsc, n}\|x^0 - x_i^*\|_2^2 + 4\gamma\frac{\sigma^2}{\mu n}.
	\end{align*}
\end{theorem}
\begin{proof}
	Recall that we use in the Algorithm $w^t = \avein x_i^t$ and that $x^t =\proxR(w^t)$. Next, by non-expansiveness of the proximal operator it holds for all $t$
	\begin{align*}
		\|x^t - x^*\|^2
		&= \|\proxR(w^t) - \proxR(x^* - \gamma \nabla f(x^*)\|^2\\
		&\le \|w^t - (x^* - \gamma \nabla f(x^*))\|^2.
	\end{align*}
	Denote for simplicity $w^* \eqdef x^* - \gamma \nabla f(x^*)$. Then, we have shown $\|x^t - x^*\|^2 \le \|w^t - w^*\|^2$.
	
	Fix any $t$ and assume without loss of generality that $d_1^t < d_2^t < \dotsb < d_n^t$. Then, using the tower property of expectation
	\begin{align*}
		\EE \|w^t - w^*\|_2^2
		= \EE\left[\EE\left[\|w^t - w^*\|_2^2 \mid x_2^t, \dotsc, x_n^t\right] \right].
	\end{align*}
	At the same time, conditioned on $z_1^t, x_2^t, \dotsc, x_n^t$ the only randomness in $w^t$ is from $x_1^t$, so
	\begin{align*}
		\EE\left[\|w^t - w^*\|_2^2 \mid z_1^t, x_2^t, \dotsc, x_n^t\right]
		&= \|\EE\left[w^t - w^* \mid z_1^t, x_2^t, \dotsc, x_n^t\right]\|_2^2 + \frac{1}{n^2}\EE\|x_1^t - \EE[x_1^t \mid z_1^t,x_2^t, \dotsc, x_n^t]\|_2^2.
	\end{align*}
	By continuing unrolling the first term in the right-hand side we arrive at
	\begin{align*}
		\EE \|w^t - w^*\|_2^2
		&\le \|\EE w^t - w^*\|_2^2 + \frac{1}{n^2}\sum_{i=1}^n \EE\|x_i^t - \EE[x_i^t \mid z_i^t, x_{i+1}^t, \dotsc, x_n^t]\|_2^2\\
		&= \|\EE w^t - w^*\|_2^2 + \frac{1}{n^2}\sum_{i=1}^n \EE\|x_i^t - \EE[x_i^t \mid z_i^t, x_{i+1}^t, \dotsc, x_n^t]\|_2^2\\
		&\overset{\eqref{eq:conditioned_variance}}{\le } \|\EE w^t - w^*\|_2^2 + 4\tau\gamma^2\frac{\sigma^2}{n}  +\frac{8\tau \gamma^2L}{n^2}\sumin(f_i(z_i^t) - f_i(x^*) - \<\nabla f_i(x^*), z_i^t -x^*>).
	\end{align*}
	Moreover, by Jensen's inequality
	\begin{align*}
		\|\EE w^t - w^*\|_2^2
		&= \left\| \avein \EE[x_i^t - x_i^*]\right\|_2^2 \\
		&\le \avein\left\|\EE [x_i^t - x_i^*]\right\|_2^2 \\
		&= \avein\left\|\EE [\EE[x_i^t - x_i^*\mid \cF_i^t]\right\|_2^2 \\
		&\le \avein\EE\left\|\EE[x_i^t - x_i^*\mid \cF_i^t]\right\|_2^2.
	\end{align*}
	Combining it with our older results, we get
	\begin{align*}
		\EE\|w^t - w^*\|_2^2
		\le \avein\EE\left\|\tilde x_i^t - x_i^*\right\|_2^2 + 4\tau\gamma^2\frac{\sigma^2}{n}  +\frac{8\tau \gamma^2L}{n^2}\sumin(f_i(z_i^t) - f_i(x^*) - \<\nabla f_i(x^*), z_i^t -x^*>).
	\end{align*}
	Let us apply Lemma~\ref{lem:asynch_contraction} to verify that
	\begin{align*}
		&\avein \|\tilde x_i^t  - x_i^*\|_2^2 + \frac{8\tau\gamma^2 L}{n^2}\sumin (f_i(z_i^t) - f_i(x^*) - \<\nabla f_i(x^*), z_i^t - x^*>) \\
		&\qquad \le (1 - \tau\gamma\mu)\avein \|z_i^t - x^*\|_2^2 - 2\tau\gamma\underbrace{\left(1 - 2\tau\gamma L - \frac{4 \gamma L}{n}\right)}_{\ge 0}\avein (f_i(z_i^t) - f_i(x^*) - \<\nabla f_i(x^*), z_i^t - x^*>) \\
		&\qquad \le (1 - \tau\gamma\mu)\avein \|z_i^t - x^*\|_2^2.
	\end{align*}
	Since $z_i^t =  x^{t-d_i^t}$, we have proved
	\begin{align*}
		\EE\|x^t - x^*\|_2^2
		&\le (1 - \tau\gamma\mu)\avein \EE\|z_i^t-x^*\|_2^2 + 4\tau\gamma^2\frac{\sigma^2}{n} \\
		&=  (1 - \tau\gamma\mu)\avein \EE\|x^{t-d_i^t}-x^*\|_2^2 + 4\tau\gamma^2\frac{\sigma^2}{n}\\
		&\le (1 - \tau\gamma\mu)\max_i \EE\|x^{t-d_i^t}-x^*\|_2^2 + 4\tau\gamma^2\frac{\sigma^2}{n}.
	\end{align*}
	If we define sequences $\psi^t\eqdef \max_{i=1,\dotsc, n}\EE\|x^{t-d_i^t} - x^*\|_2^2$ and $\Psi^k \eqdef \max_{t\in [T_k, T_{k+1})} \left\{\max\{0, \psi^t -  4\gamma\frac{\sigma^2}{\mu n}\}\right\}$, it follows from the above that
	\begin{align*}
		\EE\|x^t - x^*\|_2^2 - 4\gamma\frac{\sigma^2}{\mu n}
		\le (1 - \tau\gamma\mu)(\max_i \EE\|x^{t-d_i^t}-x^*\|_2^2 - 4\gamma\frac{\sigma^2}{\mu n}).
	\end{align*}
	Therefore, if $\Psi^{k_0}=0$ for some $k_0$ then $\Psi^k=0$ for all $k\ge k_0$. Otherwise,
	 $\Psi^{k+1} \le (1 - \tau\gamma\mu)\Psi^k$ and for any $t\in[T_k, T_{k+1})$
	\begin{align*}
		\EE\|x^t - x^*\|^2
		\le \psi^{t}
		\le \Psi^k + 4\gamma\frac{\sigma^2}{\mu n}
		\le (1 - \tau\gamma\mu)^k \|x^0 - x^*\|^2 + 4\gamma\frac{\sigma^2}{\mu n}.
	\end{align*}
\end{proof}
\begin{proposition}[\cite{mishchenko2018distributed}]\label{pr:epoch}
	If delays are uniformly bounded over time, i.e.\ $d_i^t\le M$ for any $i$ and $t$, then $T_k\le Mk$.
\end{proposition}
Combining Theorem~\ref{th:asynch_epoch} and Proposition~\ref{pr:epoch} gives a lower bound on $k$ and implies Theorem~\ref{th:asynch}.

\end{document}